\documentclass{jmlr} 

\usepackage{algorithm,algorithmic}
\usepackage{amssymb, multirow, paralist, color}

\newtheorem{thm}{Theorem}

\newtheorem{cor}[thm]{Corollary}
\newtheorem{ass}{Assumption}

\usepackage{enumitem}
\usepackage{textcase,booktabs}
\usepackage[T1]{fontenc}

\newcommand\yancomment[1]{}

\def \S {\mathbf{S}}
\def \A {\mathcal{A}}

\def \R {\mathbb{R}}

\def \w {\mathbf{w}}

\def \x {\mathbf{x}}

\def \E {\mathrm{E}}

\def \x {\mathbf{x}}

\def \z {\mathbf{z}}

\def \u {\mathbf{u}}

\def \P {\mathcal{P}}

\def \wh {\widehat{\w}}

\def \E {\mathrm{E}}
\def \x {\mathbf{x}}

\def \z {\mathbf{z}}
\def \u {\mathbf{u}}

\def \Z {\mathcal{Z}}

\def \w {\mathbf{w}}

\def \R {\mathbb{R}}
\def \S {\mathcal{S}}

\def \A {\mathcal{A}}

\def \G {\mathcal {G}}

\def \wh {\widehat{\w}}

\def \P {\mathbb{P}}

\def \start {\textsc{Start}}

\begin{document}

\title[Stagewise Training Accelerates Convergence of Testing Error Over SGD]{Stagewise Training Accelerates Convergence of Testing Error Over SGD}
\author{\Name{Zhuoning Yuan}$^\dagger$ \Email{zhuoning-yuan@uiowa.edu}\\
\Name{Yan Yan}$^\dagger$\Email{yan-yan-2@uiowa.edu}\\
\Name{Rong Jin}$^\ddagger$ \Email{jinrong.jr@alibaba-inc.com}\\
\Name{Tianbao Yang}$^\dagger$\Email{tianbao-yang@uiowa.edu}\\
\addr $^\dagger$Department of Computer Science, The University of Iowa, Iowa City, IA 52242 \\
      \addr $^\ddagger$Machine Intelligence Technology, Alibaba Group, Bellevue, WA 98004, USA \\
}

\maketitle
\vspace*{-0.4in}
 \begin{center}First version: December 8, 2018\\
Revised version: February 2, 2019\end{center}

\begin{abstract}
Stagewise training strategy is  widely used for learning neural networks, which runs a stochastic algorithm (e.g., SGD) starting with a relatively large step size (aka learning rate) and geometrically decreasing the step size after a number of iterations.  It has been  observed  that the stagewise SGD has much faster convergence  than the vanilla SGD with a polynomially decaying step size in terms of both training error and testing error. {\it But how to explain this phenomenon has been largely ignored by existing studies.} This paper provides some theoretical evidence for explaining this faster convergence. In particular, we consider a stagewise training strategy for minimizing empirical risk that satisfies the  Polyak-\L ojasiewicz (PL) condition, which has been observed/proved for neural networks and also holds for a broad family of convex functions. For convex loss functions and two classes of ``nice-behaviored" non-convex objectives that are close to a convex function, we establish faster convergence of stagewise training than the vanilla SGD under the PL condition on both training error and testing error. Experiments on stagewise learning of deep residual networks exhibits that  it satisfies one type of non-convexity assumption and therefore can be explained by our theory. Of independent interest, the testing error bounds for the considered non-convex loss functions are dimensionality and norm independent. 
\end{abstract}

\section{Introduction}
In this paper, we consider learning a predictive model by using a stochastic algorithm to minimize the expected risk via solving the following empirical risk problem: 
\begin{align}\label{eqn:ERM}
\min_{\w\in\Omega}F_\S(\w):= \frac{1}{n}\sum_{i=1}^nf(\w, \z_i),
\end{align}
where $f(\w, \z)$ is a smooth loss function of the model $\w$ on the data $\z$, $\Omega$ is a closed convex set, and $\S=\{\z_1, \ldots, \z_n\}$ denotes a set of $n$ observed data points that are sampled from an underlying distribution $\P_z$ with support on $\Z$. 

There are tremendous studies devoted to solving this empirical risk minimization (ERM) problem in machine learning and related fields.
Among all existing algorithms, stochastic gradient descent (SGD) is probably the simplest and attracts most attention, which takes the following update: 
\begin{align}\label{eqn:sgd}
\w_{t+1} = \Pi_{\Omega}[\w_t  - \eta_t\nabla f(\w_t, \z_{i_t})],
\end{align}
where $i_t\in\{1,\ldots, n\}$ is randomly sampled, $\Pi_{\Omega}$ is the projection operator, and $\eta_t$ is the step size that is usually decreasing to $0$. Convergence theories have been extensively studied for SGD with a polynomially decaying step size (e.g., $1/t$, $1/\sqrt{t}$) for an objective that satisfies various assumptions, e.g., convexity~\citep{Nemirovski:2009:RSA:1654243.1654247},  non-convexity~\citep{DBLP:journals/siamjo/GhadimiL13a}, strong convexity~\citep{DBLP:journals/ml/HazanAK07}, local strong convexity~\citep{pmlr-v48-qua16}, Polyak-\L ojasiewicz inequality~\citep{DBLP:conf/pkdd/KarimiNS16}, Kurdyka-\L ojasiewicz inequality~\citep{ICMLASSG}, etc. The list of papers about SGD is so long that can not be exhausted here.

The success of deep learning is mostly driven by stochastic algorithms as simple  as SGD running on big data sets~\citep{DBLP:conf/nips/KrizhevskySH12,DBLP:conf/cvpr/HeZRS16}. However, an interesting phenomenon that can  be observed in practice for deep learning  is that no one is actually using the vanilla SGD with a polynomially decaying step size that is well studied in theory for non-convex optimization~\citep{DBLP:conf/pkdd/KarimiNS16,DBLP:journals/siamjo/GhadimiL13a,sgdweakly18}. Instead, a common trick  used to speed up the convergence of SGD is by using a stagewise step size strategy, i.e., starting from a relatively large step size and decreasing it geometrically after a number of iterations~\citep{DBLP:conf/nips/KrizhevskySH12,DBLP:conf/cvpr/HeZRS16}. Not only the convergence of training error is accelerated but also is the convergence of testing error.  However, there is still a lack of theory for explaining this phenomenon. 
Although a stagewise step size strategy has been considered in some studies~\citep{hazan-20110-beyond,ICMLASSG,DBLP:conf/pkdd/KarimiNS16,pmlr-v80-kleinberg18a,chen18stagewise}, none of them  explains the benefit of stagewise training used in practice compared with standard SGD with a decreasing step size, especially on the convergence of testing error for non-convex problems.

\subsection{Our Contributions}This paper aims to provide some theoretical evidence to show that an appropriate stagewise training algorithm can have faster convergence than   SGD with a polynomially  deccaying step size under some condition. In particular, we  analyze a stagewise  training algorithm under the  Polyak-\L ojasiewicz condition~\citep{polyak63}: 
\begin{align*}
2\mu(F_\S(\w) - \min_{\w\in\Omega}F_\S(\w))\leq \|\nabla F_\S(\w)\|^2.
\end{align*}
This property has been recently observed/proved for learning deep and shallow neural networks~\citep{DBLP:journals/corr/HardtM16,DBLP:journals/corr/0002LS16,DBLP:conf/nips/LiY17,DBLP:journals/corr/abs-1710-06910,pmlr-v80-charles18a}, and it also holds for a broad family of convex functions~\citep{ICMLASSG}.  We will focus on the scenario that $\mu$ is a {\bf small positive value} and $n$ is large, which corresponds to ill-conditioned big data problems and is indeed the case for many problems~\citep{DBLP:journals/corr/HardtM16,pmlr-v80-charles18a}. We compare with two popular vanilla SGD variants with $\Theta(1/t)$ or $\Theta(1/\sqrt{t})$ step size scheme for both the convex loss and two classes of non-convex objectives that are close to a convex function. We show that the considered stagewise training algorithm has a better dependence on $\mu$ than the vanilla SGD with $\Theta(1/t)$ step size scheme for both the training error (under the same number of iterations)  and the testing error (under the same number of data and a less number of iterations), while keeping the same dependence on the number of data for the testing error bound. Additionally, it has faster convergence and a smaller testing error bound than the vanilla SGD with $\Theta(1/\sqrt{t})$ step size scheme for big data. 

To be fair for comparison between two algorithms, we adopt a unified approach that considers both the optimization error and the generalization error, which  together with algorithm-independent optimal empirical risk constitute the testing error. In addition,  we use the same tool for analysis of the generalization error - a key component in the testing error. We would like to point out that the techniques for us to prove the convergence of optimization error and testing error are simple and standard. In particular, the optimization error analysis is built on existing convergence analysis for solving convex problems, and the testing error analysis is built on the uniform stability analysis of a stochastic algorithm introduced by~\cite{DBLP:conf/icml/HardtRS16}.   It is of great interest to us that simple analysis of the widely used  learning strategy can possibly explain its greater success in practice than using the standard SGD method with a  decaying step size.  

Besides theoretical contributions, the considered algorithm also has additional   features that come with theoretical guarantee for the considered non-convex problems and help improve the generalization performance, including allowing for explicit algorithmic regularization at each stage, using an averaged solution for restarting, and returning the last stagewise solution as the final solution. It is also notable that the widely used stagewise SGD is covered by the proposed framework.
We refer to the considered algorithm as {\bf stagewise regularized training algorithm or {\start}}.

\subsection{Other related works assuming PL conditions}
It is notable that many papers have proposed and analyzed deterministic/stochastic  optimization algorithms under the PL condition, e.g.,~\citep{DBLP:conf/pkdd/KarimiNS16,DBLP:conf/nips/LeiJCJ17,pmlr-v48-reddi16,DBLP:journals/abs/1811.02564}. This list could be long if we consider its equivalent condition in the convex case. However, none of them exhibits  the benefit of stagewise learning strategy used in practice. One may also notice that linear convergence for the optimization error was proved  for a stochastic  variance reduction gradient method~\citep{pmlr-v48-reddi16}.  Nevertheless, its uniform stability  bound remains unclear for making a fair comparison with the considered algorithms in this paper, and variance reduction method is not widely used for deep learning.

We also notice that some recent studies~\citep{DBLP:conf/icml/KuzborskijL18,DBLP:journals/corr/abs-1802-06903,pmlr-v80-charles18a} have used other techniques  (e.g., data-dependent bound, average stability, point-wise stability) to analyze the generalization error of a stochastic algorithm. Nevertheless, we believe similar techniques can be also used for analyzing stagewise learning algorithm, which is beyond the scope of this paper.  Among these studies, \citet{DBLP:journals/corr/abs-1802-06903,pmlr-v80-charles18a} have analyzed the generalization error (or stability) of stochastic algorithms (e.g., the vanilla SGD with a decreasing step size or small constant step size) under the PL condition and other conditions. We emphasize that their results are not directly comparable to the results presented in this work.  In particular, \citet{DBLP:journals/corr/abs-1802-06903}  considered the generalization error of SGD with a decreasing step size in the form $\Theta(c/t\log t)$ with $2/\mu<c<1/L$ and $L$ being smoothness parameter, which corresponds to a good conditioned setting $L/\mu<1/2$~\footnote{this could never happen in unconstrained optimization where $\|\nabla F_\S(\x)\|^2\leq 2L(F_\S(\x) -\min F_\S(\x))$~\citep{opac-b1104789}[Theorem 2.1.5]. Together with the PL condition, it implies $L\geq \mu$.}. 
\citet{pmlr-v80-charles18a} made a strong technical assumption (e.g., the global minimizer is unique) for deriving their uniform stability results, which is unlikely to hold in the real-word and is avoided in this work for establishing a generalization error bound for the standard SGD. 

Finally, it was brought to our attention~\footnote{personal communication with Jason D. Lee at NeurIPS 2018.} when a preliminary version of this paper is almost done that an independent work~\citep{ge2019rethinking} observes a similar advantage of stagewise SGD over SGD with a polynomially decaying step size lying at the better dependence on the condition number. However, they only analyze the strongly convex quadratic case and the training error of ERM.

\section{Preliminaries and Notations}
Let $\A$ denote a randomized algorithm, which returns a randomized solution $\w_\S= \A(\S)$ based on the given data set $\S$. Denote by $\E_{\A}$ the expectation over the randomness in the algorithm and by $\E_{\S}$ expectation over the randomness in the data set. When it is clear from the context, we will omit the subscript $\S$ and $\A$ in the expectation notations.   Let $\w^*_\S \in \arg\min_{\w\in\Omega}F_\S(\w)$ denote an empirical risk minimizer, and $F(\w) = \E_{\Z}[f(\w, \z)]$ denote the true risk of $\w$ (also called {\bf testing error} in this paper). We use $\|\cdot\|$ to denote the Euclidean norm, and use $[n]=\{1,\ldots, n\}$. 

In order to analyze the testing error convergence of  a random solution, we use the following decomposition of testing error. 
\begin{align*}
&\E_{\A,\S}[F(\w_\S)]\leq \E_{\S}[F_\S(\w^*_\S)] + \E_{\S}\underbrace{\E_{\A}[F_\S(\w_\S) - F_\S(\w^*_\S)]}\limits_{\varepsilon_{opt}} + \underbrace{\E_{\A, \S}[F(\w_\S) - F_\S(\w_S)]}\limits_{\varepsilon_{gen}},
\end{align*}
where $\varepsilon_{opt}$ measures the optimization error, i.e., the difference between empirical risk (or called {\bf training error}) of the returned solution $\w_\S$ and the optimal value of the empirical risk,   and $\varepsilon_{gen}$ measures the generalization error, i.e., the difference between the true risk of the returned solution and the empirical risk of the returned solution. The difference $\E_{\A,\S}[F(\w_\S)] -  \E_{\S}[F_\S(\w^*_\S)]$ is an upper bound of the so-called {\bf excess risk bound} in the literature, which is defined as $\E_{\A,\S}[F(\w_\S)] - \min_{\w\in\Omega}F(\w)$.   It is notable that the first term $\E_\S[F_\S(\w^*_\S)]$ in the above bound is independent of the choice of randomized algorithms. Hence, in order to compare the performance of different randomized algorithms, we can  focus on analyzing $\varepsilon_{opt}$ and $\varepsilon_{gen}$. For analyzing the generalization error, we will leverage the uniform stability tool~\citep{Bousquet:2002:SG:944790.944801}. The definition of uniform stability is given below. 
\begin{definition}
A randomized algorithm $\A$ is called $\epsilon$-uniformly stable if for all data sets $\S, \S'\in\Z^n$ that differs at most one example the following holds: 
\begin{align*}
\sup_{\z}\E_{\A}[f(\A(\S), \z) - f(\A(\S'), \z)]\leq \epsilon.
\end{align*}
\end{definition}
A well-known result is that if $\A$ is $\epsilon$-uniformly stable, then its generalization error is bounded by $\epsilon$~\citep{Bousquet:2002:SG:944790.944801}, i.e., 
\begin{lemma}
If $\A$ is $\epsilon$-uniformly stable, we have $\varepsilon_{gen}\leq \epsilon$.
\end{lemma}
In light of the above results, in order to compare the convergence of testing error of different  randomized algorithms, it suffices to analyze their convergence in terms of optimization error and their uniform stability.

A function $f(\w)$ is $L$-smooth if it is differentiable and its gradient is $L$-Lipchitz continuous, i.e., $\|\nabla f(\w) - \nabla f(\u)\|\leq L\|\w - \u\|, \forall \w, \u\in\Omega$. 
 A function $f(\w)$ is $G$-Lipchitz continuous if $\|\nabla f(\w)\|\leq G, \forall \w\in\Omega$. Throughout the paper, we will make the following assumptions with some positive $L$, $\sigma$,  $G$, $\mu$ and $\epsilon_0$.
\begin{ass}\label{ass:1}Assume that 
\begin{itemize}
\item[(i)] $f(\w, \z)$ is  $L$-smooth in terms of $\w\in\Omega$ for every $\z\in\Z$.
\item[(ii)] $f(\w, \z)$ is  finite-valued and  $G$-Lipchitz continuous in terms of  $\w\in\Omega$ for every $\z\in\Z$.
\item[(iii)] there exists $\sigma$ such that $\E_i[\|\nabla f(\w, \z_i) - \nabla F_\S(\w)\|^2]\leq \sigma^2$ for $\w\in\Omega$.
\item[(iv)] $F_\S(\w)$ satisfies the PL condition, i.e., there exists $\mu$
\begin{align*}
2\mu(F_\S(\w) - F_\S(\w^*_\S))\leq \|\nabla F_\S(\w)\|^2, \forall \w\in\Omega.
\end{align*}
\item[(v)] For an initial solution $\w_0\in\Omega$, there exists  $\epsilon_0$ such that $F_\S(\w_0 ) - F_\S(\w^*_\S)\leq \epsilon_0$.  
\end{itemize}
\end{ass}

\noindent
{\bf Remark 1:} The second assumption is imposed for the analysis of uniform stability of a randomized algorithm. W.o.l.g we assume $|f(\w, \z)|\leq 1, \forall \w\in\Omega$.    The third assumption is for the purpose of analyzing optimization error. It is notable that $\sigma^2\leq 4G^2$. 
It is known that the PL condition is much weaker than strong convexity. If $F_\S$ is a strongly convex function, $\mu$ corresponds to the strong convexity parameter. In this paper, we are particularly interested in the case when $\mu$ is small, i.e. the condition number $L/\mu$ is large.  

\noindent
{\bf Remark 2:} It is worth mentioning that we do not assume the PL condition holds in the whose space $\R^d$. Hence, our analysis presented below can capture some cases that the PL condition only holds in a local space $\Omega$ that contains a global minimum.  For example, the recent paper by~\citep{du2018gradient} shows that the global minimum of learning a two-layer neural network resides in a ball centered around the initial solution and the PL condition holds in the ball~\citep{DBLP:journals/corr/0002LS16}. 


\section{Review: SGD under PL Condition}\label{sec:review}
In this section, we review the training error convergence and generalization error of SGD with a decreasing step size for functions satisfying the PL condition in order to derive its testing error bound. We will focus on SGD using the step size $\Theta(1/t)$ and briefly mention the results corresponding to $\Theta(1/\sqrt{t})$ at the end of this section.   We would like to emphasize the results presented in this section are mainly from existing works~\citep{DBLP:conf/pkdd/KarimiNS16,DBLP:conf/icml/HardtRS16}. The optimization error and the uniform stability of SGD have been studied in these two papers separately. Since we are not aware of any studies that piece them together, it is of our interest to summarize these results here for comparing with our new results established later in this paper.  

Let us first consider the optimization error convergence, which has been analyzed in \citep{DBLP:conf/pkdd/KarimiNS16} and is summarized below. 
\begin{thm}~\citep{DBLP:conf/pkdd/KarimiNS16}\label{thm:sgd}
Suppose $\Omega = \R^d$. Under Assumption~\ref{ass:1} (i), (iv) and $\E_i[\|\nabla f(\w, \z_i)\|^2]\leq G^2$, by setting $\eta_t =\frac{2t+1}{2\mu(t+1)^2} $ in the update of SGD~(\ref{eqn:sgd}), we have \begin{align}\label{eqn:sgdr}
\E[F_\S(\w_T) - F_\S(\w^*_\S)]\leq\frac{LG^2}{2T\mu^2},
\end{align}
and by setting $\eta_t = \eta$, we have
\begin{align}
\E[F_\S(\w_T) - F_\S(\w^*_\S)]&\leq (1-2\eta \mu)^T(F_\S(\w_0) - F_\S(\w^*_\S))\nonumber + \frac{\eta LG^2}{4\mu}.
\end{align}
\end{thm}

\noindent
{\bf Remark 3:} 
In order to have an $\epsilon$ optimization error, one can set $T = \frac{LG^2}{2\mu^2\epsilon}$ in the decreasing step size setting. In the constant step size setting, one can set $\eta = \frac{2\mu\epsilon}{LG^2}$ and $T  =\frac{LG^2}{4\mu^2\epsilon}\log(2\epsilon_0/\epsilon)$, where $\epsilon_0\geq F_\S(\w_0) - F_\S(\w^*_S)$ is the initial optimization error bound. \citet{DBLP:conf/pkdd/KarimiNS16} also mentioned a stagewise step size strategy based on the second result above. By starting with $\eta_1 = \frac{\epsilon_0\mu}{LG^2}$ and running for $t_1 = \frac{LG^2\log4}{2\mu^2\epsilon_0}$ iterations, and restarting the second stage with $\eta_2 = \eta_1/2$ and $t_2 = 2t_1$, then after  $K= \log(\epsilon_0/\epsilon)$ stages, we have optimization error less than $\epsilon$, and the total iteration complexity is $O(\frac{LG^2\log4}{\mu^2\epsilon})$. We can see that the analysis of \cite{DBLP:conf/pkdd/KarimiNS16} cannot explain why stagewise optimization strategy brings any improvement compared with SGD with a decreasing step size of $O(1/t)$. No matter which step size strategy is used among the ones discussed above, the total iteration complexity is $O(\frac{L}{\mu^2\epsilon})$. It is also interesting to know that the above convergence result does not require the convexity of $f(\w, \z)$. On the other hand, it is unclear how to directly analyze SGD with a polynomially decaying step size for a convex loss to obtain a better convergence rate than~(\ref{eqn:sgdr}).

The generalization error bound  by uniform stability for both convex and non-convex losses have been analyzed in~\citep{DBLP:conf/icml/HardtRS16}. We just need to plug the step size of SGD in Theorem~\ref{thm:sgd} into their results (Theorem 3.7 and Theorem 3.8) to prove the uniform stability, which is presented as follow.
\begin{thm}\label{thm:stab-sgd}
Suppose Assumption~\ref{ass:1} holds and $n>L/\mu$ is sufficiently large. If $f(\cdot, \z)$ is convex for any $\z\in\Z$, then SGD with step size $\eta_t = \frac{2t+1}{2\mu(t+1)^2}$ satisfies uniform stability with 
\begin{align*}
\varepsilon_{stab}\leq \frac{L}{n\mu} + \frac{2G^2}{n\mu}\sum_{t=1}^T\frac{1}{t+1}\leq \frac{L + 2G^2\log (T+1)}{n\mu}.
\end{align*}
If $f(\cdot, \z)$ is non-convex for any $\z\in\Z$, then SGD with step size $\eta_t = \frac{2t+1}{2\mu(t+1)^2}$ satisfies uniform stability with 
\begin{align*}
\varepsilon_{stab}\leq \frac{1+\mu/L}{n-1}(2G^2/\mu)^{1/(L/\mu + 1)}T^{\frac{L/\mu}{L/\mu + 1}}.
\end{align*}
\end{thm}
\begin{proof}
We combine the proof of Theorem 3.8 and the result of Lemma 3.11 in the long version of~\cite{DBLP:conf/icml/HardtRS16}.
For applying  Theorem 3.8, we need to have $\eta_t\leq 1/L$, i.e., $\frac{2t+1}{2\mu(t+1)^2}\leq 1/L$. Let us define $t_0 = \frac{L}{\mu}$. Then $\eta_t\leq 1/L, \forall t\geq t_0$. Then conditioned on $\delta_{t_0}=0$, we apply  Lemma 3.11 in~\cite{DBLP:conf/icml/HardtRS16} and have

\begin{align}
\label{eq:proof:stab_sgd_convex}
       \varepsilon_{stab}   
\leq & 
       \frac{t_{0}  }{n} + G \E[ \delta_{T} | \delta_{t_{0}} = 0 ]   \leq 
       \frac{t_{0}  }{n} + \frac{2G^{2}}{n} \sum_{t=t_{0}}^{T} \eta_{t}   \nonumber\\
       \leq& \frac{t_{0}  }{n} + \frac{2G^{2}}{n} \sum_{t=t_{0}}^{T} \frac{2t+1}{2\mu(t+1)^2} \leq  
       \frac{ L }{n \mu} + \frac{2G^{2}}{n \mu} \log(T + 1). 
\end{align}

Next we consider the case when $f(\cdot, \z)$ is non-convex. By noting $\eta_t\leq \frac{1/\mu}{t}$, we can directly applying their Theorem 3.12 of the long version of~\cite{DBLP:conf/icml/HardtRS16} and get
\begin{align*}
     \varepsilon_{stab} \leq 
       \frac{ 1 + \frac{\mu}{L} }{ n - 1 } \Big( \frac{2G^{2}}{\mu} \Big)^{\frac{1}{1 + L / \mu}} \cdot T^{\frac{L / \mu}{1 + L / \mu}}    .
\end{align*}
\end{proof}


Combining the optimization error and uniform stability, we obtain the convergence of testing error of SGD for smooth loss functions under the PL condition. 

\begin{thm}\label{thm:testcvx}
Suppose $\Omega =\R^d$, Assumption~\ref{ass:1} holds and let $\hat G = G^2/L$.  If $f(\cdot, \z)$ is convex for any $\z\in\Z$,  with step size $\eta_t = \frac{2t+1}{2\mu(t+1)^2}$ and $T$ iterations SGD returns a solution $\w_T$ satisfying
\begin{align*}
\E_{\A,\S}&[F(\w_T)] \leq \E_{\S}[F_\S(\w^*_\S)] + \frac{LG^2}{2T\mu^2} + \frac{(L+2G^2)\log (T+1)}{n\mu}.
\end{align*}
 If $f(\cdot, \z)$ is non-convex for any $\z\in\Z$, with the same setting SGD returns a solution $\w_T$ satisfying 
\begin{align*}
\E_{\A,\S}&[F(\w_T)] \leq \E_{\S}[F_\S(\w^*_\S)] + \frac{LG^2}{2T\mu^2} + \frac{2 T\min(2G/\sqrt{\mu}, e^{2\hat G})}{n-1}.
\end{align*}
\end{thm}

\begin{proof}
Based on the decomposition of testing error, the result of Theorem~\ref{thm:sgd} and Theorem~\ref{thm:stab-sgd}, we could upper bound the testing error by combining optimization error and generalization error together. For convex problems, we have
\begin{align*}
\E_{\A, \S} [ F( \w_{T} ) ] &\leq \E_{\S} [ F_{\S}( \w_{\S}^* ) ] + \frac{L G^{2} }{2T \mu} + \frac{ ( L + 2G^{2} ) \log(T + 1) }{ n \mu }   .
\end{align*}
For non-convex problems, we have
\begin{align*}
       \E_{\A, \S} [ F( \w_{T} ) ] 
\leq &
       \E_{\S} [ F_{\S}( \w_{\S}^* ) ] + \frac{LG^2}{2T\mu^2} + \frac{1 + \frac{\mu}{L} }{n - 1} \bigg( \frac{2G^{2} }{\mu} \bigg)^{ \frac{1}{\frac{L}{\mu} + 1} } T^{ \frac{ \frac{L}{\mu} }{ \frac{L}{\mu} + 1 } }  \\
\leq &
       \E_{\S} [ F_{\S}( \w_{\S}^* ) ] +  \frac{LG^2}{2T\mu^2} + \frac{ 2 }{n - 1} \bigg( \frac{2G^{2} }{\mu}\bigg)^{ \frac{1}{\frac{L}{\mu} + 1} } T
\end{align*}
Let $X = \frac{2G^{2}}{\mu} - 1$, which is positive when $\mu$ is very small. 
Given $(1 + X)^{1 / X} \leq e$, we have
\begin{align*}
       \bigg( \frac{2G^{2}}{\mu} \bigg)^{ \frac{1}{\frac{L}{\mu} + 1 } } 
&= 
       \bigg( \frac{2G^{2}}{\mu} \bigg)^{ \frac{\mu}{2G^{2} - \mu} \cdot \frac{2G^{2} - \mu }{\mu} \cdot \frac{\mu}{L + \mu} }
\leq 
       e^{\frac{2G^{2} - \mu}{L + \mu}} \leq 
       e^{\frac{2G^{2}}{L}}     .
\end{align*}
We also have $\bigg( \frac{2G^{2}}{\mu} \bigg)^{ \frac{1}{\frac{L}{\mu} + 1 } }\leq \frac{\sqrt{2}G}{\sqrt{\mu}}$ given that $\frac{2G^{2}}{\mu}\geq 1$  and $L/\mu\geq 1$ for small $\mu$. Thus, we complete the proof. 

\end{proof} 

By optimizing the value of $T$ in the above bounds, we obtain the testing error bound dependent on $n$ only. The results are summarized in the following corollary. 

\begin{cor}\label{thm:gen}
Suppose Assumption~\ref{ass:1} holds. If $f(\cdot, \z)$ is convex for any $\z\in\Z$,  with step size $\eta_t = \frac{2t+1}{2\mu(t+1)^2}$ and $T = \frac{n LG^2}{4(L + 2G^2)\mu}$ iterations SGD returns a solution $\w_T$ satisfying
\begin{align*}
&\E_{\A,\S}[F(\w_T)] \leq \E_{\S}[F_\S(\w^*_\S)] + \frac{2(L+2G^2)}{n\mu} + \frac{(L+2G^2)\log (T+1)}{n\mu}.
\end{align*}
 If $f(\cdot, \z)$ is non-convex for any $\z\in\Z$, with step size $\eta_t = \frac{2t+1}{2\mu(t+1)^2}$ and $T =\max\{\frac{\sqrt{(n-1)LG}}{\sqrt{8}\mu^{3/4}} \frac{\sqrt{(n-1)L}G}{2\mu e^{\hat G}}\}$ iterations SGD returns a solution $\w_T$ satisfying 
\begin{align*}
&\E_{\A,\S}[F(\w_T)] \leq \E_{\S}[F_\S(\w^*_\S)] +2\min\bigg\{ \frac{\sqrt{2}L^{1/2}G^{3/2} }{\sqrt{(n-1)}\mu^{5/4}}, \frac{\sqrt{Le^{2\hat G}}G}{\sqrt{n-1}\mu}\bigg\}.
\end{align*}
\end{cor}

\noindent
{\bf Remark 4:} If the loss is convex, the excess risk bound is in the order of $O(\frac{L\log(nL/\mu)}{n\mu})$ by running SGD with $T=O(nL/\mu)$ iterations. It notable that an $O(1/n)$ excess risk bound is called the fast rate in the literature.  If the loss is non-convex and $2G/\sqrt{\mu}>e^{2\hat G}$ (an interesting case~\footnote{We can always scale up $L$ such that $e^{2\hat G}$ is a small constant, which only scales up the bound by a constant factor.}), the excess risk bound is in the order of $O(\frac{\sqrt{L}}{\sqrt{n}\mu})$ by running SGD with $T=O(\sqrt{nL}/\mu)$ iterations.  When $\mu$ is very small, the convergence of testing error is very slow. In addition, the number of iterations is also scaled by $1/\mu$ for achieving a minimal excess risk bound. 

\noindent
{\bf Remark 5:} Another possible choice of decreasing step size is $O(1/\sqrt{t})$, which yields an $O(1/\sqrt{T})$ convergence rate for $F_\S(\wh_T) -F_\S(\w^*_\S)$ in the convex case~\citep{Nemirovski:2009:RSA:1654243.1654247} or for $\|\nabla F_S(\w_t)\|^2$ in the non-convex case with a randomly sampled $t$~\citep{DBLP:journals/siamjo/GhadimiL13a}. In the latter case, it also implies a worse convergence rate of $O(1/(\mu\sqrt{T}))$ for the optimization error $F_\S(\w_t) - \min_{\w}F_\S(\w)$ under the PL condition. 
Regarding the uniform stability, the step size of  $O(1/\sqrt{t})$ will also yield a worse growth rate in terms of $T$~\cite{DBLP:conf/icml/HardtRS16}.  For example, if the loss function is convex, the generalization error by uniform stability scales as $O(\sqrt{T}/n)$ and hence the testing error bound is in the order of $O(1/\sqrt{n\mu})$, which is worse than the above testing error bound $\widetilde O(1/(n\mu))$ for the big data setting $\mu\geq \Omega(1/n)$.  
Hence,  below we will focus on the comparison  with the theoretical results in Corollary~\ref{thm:gen}. 

\section{{\start} for a Convex Function}\label{sec:4}
First, let us present the algorithm that we intend to analyze in Algorithm~\ref{alg:meta}.   At the $k$-th stage, a regularized funciton $F^{\gamma}_{\w_{k-1}}(\w)$ is constructed that consists of the original objective $F_\S(\w)$ and a quadratic regularizer $\frac{1}{2\gamma}\|\w - \w_{k-1}\|^2$. The reference point $\w_{k-1}$ is a returned solution from the previous stage, which is also used for an initial solution for the current stage. 
Adding the strongly convex regularizer at each stage is not essential but could be  helpful for reducing the generalization error and is also important for one class of non-convex loss considered in next section.  For each regularized problem, the SGD with a constant step size is employed for a number of iterations with an appropriate returned solution. We will reveal the value of step size, the number of iterations and the returned solution for each class of problems separately. Note that the widely used stagewise SGD falls into the framework of {\start} when $\gamma=\infty$ and $\mathcal O=(\w_1,\ldots, \w_{T+1})  = \w_{T+1}$. 
\setlength{\textfloatsep}{5pt}
\begin{algorithm}[t]
    \caption{{\start} Algorithm: {\start}($F_\S, \w_0, \gamma, K$)}\label{alg:meta}
    \begin{algorithmic}[1]
        \STATE \textbf{Input:}  $\w_0$, $\gamma$ and $K$
        \FOR {$k = 1,\ldots, K$}
        \STATE Let $F^{\gamma}_{\w_{k-1}}(\w) = F_\S(\w)  +\frac{1}{2\gamma}\| \w-\w_{k-1}\|^2$ 
        \STATE $\w_{k} = \textsc{SGD}(F^{\gamma}_{\w_{k-1}}, \w_{k-1},  \eta_k, T_k)$
        \ENDFOR
        \STATE \textbf{Return:} $\w_K$
    \end{algorithmic}
\end{algorithm}

\begin{algorithm}[t]
    \caption{SGD$(F^\gamma_{\w_1}, \w_1, \eta, T)$}\label{alg:sgd}
\begin{algorithmic}[1]
    \FOR{$t=1,\ldots, T$}
    \STATE Sample a random data $\z_{i_t}\in\S$ 
    \STATE $\w_{t+1} =\min_{\w\in\Omega}\nabla f(\w_t, \z_{i_t})^{\top} \w + \frac{1}{2\eta}\|\w - \w_t\|^2 + \frac{1}{2\gamma}\|\w - \w_1\|^2$
    \ENDFOR
    \STATE \textbf{Output}: $\wh_T = \mathcal O(\w_1, \ldots, \w_{T+1})$
\end{algorithmic}
\end{algorithm}

In this section, we will analyze {\start} algorithm  for a convex function under the PL condition.  We would like to point out that similar algorithms have been proposed and analyzed  in~\citep{hazan-20110-beyond,ICMLASSG} for convex problems. They focus on analyzing the convergence of optimization error for convex problems under a quadratic growth condition or more general local error bound condition.   In the following, we will show that the PL condition implies a quadratic growth condition. Hence, their algorithms can be used for optimizing $F_\S$ as well enjoying a similar convergence rate in terms of optimization error. However, there is still slight difference between the analyzed algorithm from their considered algorithms. In particular, the regularization term $\frac{1}{2\gamma}\| \w-\w_{k-1}\|^2$ is absent in \citep{hazan-20110-beyond}, which corresponds to $\gamma = \infty$ in our case. However, adding a small regularization (with not too large $\gamma$) can possibly help reduce the generalization error.  In addition, their initial  step size is scaled by $1/\mu$. The initial step size of our algorithm depends on the quality of initial solution that seems more natural and practical.  A similar regularization at each stage is also used in~\citep{ICMLASSG}. But their algorithm will suffer from a large generalization error, which is due to the key difference between {\start} and their algorithm (ASSG-r). In particular, they use a geometrically decreasing the parameter $\gamma_k$ starting from a relatively large value in the order of $O(1/(\mu\epsilon))$ with a total iteration number $T = O(1/(\mu\epsilon))$. According to our analysis of generalization error (see Theorem~\ref{thm:startstab}), their algorithm has a generalization error in the order of $O(T/n)$ in contrast to $\log T/n$ of our algorithm.

\subsection{Convergence of Optimization Error}
We need the following lemma for our analysis. 
\begin{lemma}\label{lem:eb}
If $F_\S(\w)$ satisfies the PL condition, then  for any $\w\in\Omega$ we have
\begin{align}\label{eqn:EB}
\|\w - \w_\S^*\|^2\leq \frac{1}{2\mu}(F_\S(\w) - F_\S(\w_\S^*)),
\end{align}
where $\w_\S^*$ is the closest optimal solution to $\w$. 
\end{lemma}

\noindent
{\bf Remark 6:} The above result does not require the convexity of $F_\S$. For a proof, please refer to~\citep{arxiv:1510.08234,DBLP:conf/pkdd/KarimiNS16}. Indeed, this error bound condition instead of the PL condition is enough to derive the results in Section~\ref{sec:4} and Section~\ref{sec:5}.

Below, we let $\w^k_t$ denote the solution computed during the $k$-th stage at the $t$-th iteration, and $\E_k$ denote the conditional expectation given the randomness before $k$-th stage. We first analyze  the convergence of SGD for one stage. 
\begin{lemma}\label{lem:2}Suppose Assumption~\ref{ass:1}(i) and (iii) hold, and  $f(\w, \z)$ is a convex function of $\w$. By applying SGD (Algorithm~\ref{alg:sgd}) to $F_k=F^\gamma_{\w_{k-1}}$ with $\mathcal O(\w^k_1,\ldots, \w^k_{T_k+1}) = \sum_{t=1}^{T_k}\w^k_{t+1}/T_k$ and $\eta\leq 1/L$, for any $\w\in\Omega$, we have
\begin{align*}
\E_k[F_k(\w_k) -  F_k(\w)] \leq   \sigma^2\eta_k + \frac{\|\w_{k-1} - \w\|^2}{2\eta_k T_k}.
\end{align*}
\end{lemma}


The above convergence result can be boosted for showing the faster convergence of {\start} under the PL condition.

\begin{thm}\label{thm:startcvx}
Suppose Assumption~\ref{ass:1}, and  $f(\w, \z)$ is a convex function of $\w$. Then by setting $\gamma\geq  1.5/\mu$ and $T_k = \frac{9\sigma^2}{2\mu\epsilon_k\alpha}, \eta_k = \frac{\epsilon_k\alpha}{3\sigma^2}$, where $\alpha \leq \min(1, \frac{3\sigma^2}{\epsilon_{0} L})$, after $K = \log(\epsilon_k/\epsilon)$ stages we have
\begin{align*}
\E[F_\S(\w_K) - F_\S(\w^*_\S)]\leq \epsilon.
\end{align*}
The total iteration complexity is $O(L/(\mu\epsilon))$.
\end{thm}

\noindent
{\bf Remark 7:} Compared to the result in Theorem~\ref{thm:sgd}, the convergence rate of {\start} is faster by a factor of $O(1/\mu)$. It is also notable that $\gamma$ can be as large as $\infty$ in the convex case.

\begin{proof}
We will prove by induction that $\E[F_\S(\w_k) - F_\S(\w^{*}_\S)]\leq \epsilon_k $, where $\epsilon_k = \epsilon_0/2^k$, which is true for $k=0$ by the assumption. 
By applying Lemma~\ref{lem:2} to the $k$-th stage, for any $\w$
\begin{align}
     \E_k[F_\S(\w_k) - F_\S(\w)]
\leq \frac{\|\w_{k-1} - \w\|^2}{2\gamma} + \eta_k \sigma^2 
     + \frac{\|\w_{k-1} - \w\|^2}{2\eta_k T_k}
\end{align}
By plugging $\w = \w^{*}_S$ (the closest optimal solution to $\w_{k-1}$) into the above inequality we have
\begin{align*}
&      \E[F_\S(\w_k) - F(\w^{*}_\S)]
  \leq \frac{\E[\|\w_{k-1} - \w^*_\S\|^2]}{2\gamma}   
       + \eta_k\sigma^2 
       + \frac{\E[\|\w_{k-1} - \w^*_\S\|^2]}{2\eta_k T_k}    \\
& \leq \frac{\E[(F_\S(\w_{k-1}) - F_\S(\w^*_\S))]}{4\mu \gamma}   + \eta_k\sigma^2
       + \frac{\E[(F_\S(\w_{k-1}) - F_\S(\w_\S^*))]}{4\mu\eta_k T_k}\\
&\leq  \frac{\epsilon_{k-1}}{4\mu \gamma}   +  \eta_k\sigma^2 + \frac{\epsilon_{k-1}}{4\mu\eta_k T_k}
\end{align*}
where we use the result in Lemma~\ref{lem:eb}.  
Since $\eta_k \leq \frac{\epsilon_k \alpha}{3\sigma^2}$ and $T_k\eta_k \geq 1.5/\mu$ and $\gamma_k \geq  1.5/\mu$, we have 
\begin{align*}
&\E[F_\S(\w_k) - F_\S(\w_\S^{*})]\leq  \epsilon_k
\end{align*}
By induction, after $K=\lceil\log(\epsilon_0/\epsilon)\rceil$ stages, we have 
\begin{align*}
&\E[F_\S(\w_K) - F_\S(\w_\S^{*})]\leq   \epsilon
\end{align*}
The total iteration complexity is $\sum_{k=1}^KT_k= O(1/(\mu\epsilon))$. 
\end{proof}

\subsection{Analysis of Generalization Error}
In this subsection, we analyze the uniform stability of {\start}. By showing $\sup_{\z}\E_{\A}[f(\w_K, \z) - f(\w'_K, \z)]\leq\epsilon$, we can show the generalization error is bounded by $\epsilon$, where $\w_K$ is learned on a data set $\S$ and $\w'_K$ is learned a different data set $\S'$ that only differs from $\S$ at most one example. Our analysis is closely following the route  in~\citep{DBLP:conf/icml/HardtRS16}. The difference is that we have to consider the difference on the reference points $\w_{k-1}$ of two copies of our algorithm on two data sets $\S, \S'$. We first give the following lemma regarding the growth of stability within one stage of {\start}. To this end, we let $f_t, f'_t$ denote the loss functions used at the $t$-th iteration of the two copies of algorithm. 
\begin{lemma}\label{lem:kstab}
Assume $f$ is smooth and convex.  
Let $\w_t$ denote the sequence learned on $\S$ and $\w'_t$ be the sequence learned on $\S'$ by {\start} at one stage,  $\delta_t = \|\w_t- \w'_t\|$. If $\eta\leq 2/L$, then 
\begin{align*}
\delta_{t+1}\leq \left\{\begin{array}{lc}\frac{\eta}{\eta+ \gamma}\delta_1 +  \frac{\gamma}{\eta + \gamma}\delta_t& f_t= f'_t\\
 \frac{\eta}{\eta+ \gamma}\delta_1 +  \frac{\gamma}{\eta + \gamma}\delta_t  + \frac{2\eta\gamma G}{\eta + \gamma}& \text{otherwise}\end{array}\right..
\end{align*}
\end{lemma}


Based on the above result, we can establish the uniform stability  of {\start}. 
\begin{thm}\label{thm:startstab}
After $K$ stages, {\start} satisfies uniform stability with 
\begin{align*}
\varepsilon_{stab}\leq \left\{\begin{array}{ll}\frac{2\gamma G^2\sum_{k=1}^K(1 -(\frac{\gamma}{\eta + \gamma}) ^{T_k})}{n}& \text{if }\gamma<\infty\\ \frac{2G^2\sum_{k=1}^K\eta_k T_k}{n}&\text{else}\end{array}\right..
\end{align*}
\end{thm}
\begin{proof}
By applying the result  in Lemma~\ref{lem:kstab} to the $k$-th stage, omitting $k$ in the notation, for $t\geq 1$ we have
\begin{align*}
       \E[\delta_{t+1}]
& \leq (1- 1/n)\left( \frac{\eta}{\eta+ \gamma}\E[\delta_1] +  \frac{\gamma}{\eta + \gamma}\E[\delta_t]\right)  
       + \frac{1}{n}\left(  \frac{\eta}{\eta+ \gamma}\E[\delta_1] +  \frac{\gamma}{\eta + \gamma}\E[\delta_t]  
       + \frac{2\eta\gamma G}{\eta + \gamma}\right)\\
& =    \frac{\eta}{\eta+ \gamma}\E[\delta_1] +  \frac{\gamma}{\eta + \gamma}\E[\delta_t] + \frac{1}{n} \frac{2\eta\gamma G}{\eta + \gamma}     \\
& =    \frac{\eta}{\eta+ \gamma}\E[\delta_1] \sum_{\tau=0}^{t-1}\left(\frac{\gamma}{\eta + \gamma}\right)^\tau 
       +  \left(\frac{\gamma}{\eta + \gamma}\right)^t\E[\delta_1]
       + \frac{1}{n} \frac{2\eta\gamma G}{\eta + \gamma}\sum_{\tau=0}^{t-1}\left(\frac{\gamma}{\eta + \gamma}\right)^\tau \\
& =    \frac{\eta}{\eta+ \gamma}\E[\delta_1]\frac{1 -(\frac{\gamma}{\eta + \gamma}) ^t}{1 - \frac{\gamma}{\eta + \gamma}} +  \left(\frac{\gamma}{\eta + \gamma}\right)^t\E[\delta_1] 
       + \frac{1}{n} \frac{2\eta\gamma G}{\eta + \gamma}\frac{1 -(\frac{\gamma}{\eta + \gamma}) ^t}{1 - \frac{\gamma}{\eta + \gamma}}\\
& =    \E[\delta_1] + \frac{2\gamma G(1 -(\frac{\gamma}{\eta + \gamma}) ^{t})}{n}.
\end{align*}
Then,
\begin{align*}
\E\left[\sum_{t=1}^T\delta_{t+1}/T\right]\leq \E[\delta_1] + \frac{2\gamma G(1 -(\frac{\gamma}{\eta + \gamma}) ^{T})}{n}.
\end{align*}
For the $k$-stage, we have $\w_k = \sum_{t=1}^T\w^k_{t+1}/T$ and $\w_{k-1} = \w_1$. Then 
\begin{align*}
\E[\delta_{k}]\leq \E[\delta_{k-1}] + \frac{2\gamma G(1 -(\frac{\gamma}{\eta + \gamma}) ^{T_k})}{n},
\end{align*}
where $\delta_k = \|\w_k - \w_k'\|$. By summing the above inequality for $K$ stages and noting that $\sup_{\z}\E_{\A}[f(\w_K, \z) - f(\w'_K, \z)]\leq G\|\w_K - \w'_K\|$, we prove the theorem. 

\end{proof}

\subsection{Put them Together}
Finally, we have the following testing  error bound of $\w_K$ returned by {\start}. 
\begin{thm}
After $K=\log(\epsilon_0/\epsilon)$ stages with a total number of iterations $T = \frac{18\sigma^2}{\alpha \mu\epsilon}$. The testing error  of $\w_K$ is bounded by
\begin{align*}
\E_{\A,\S}[F(\w_K)]&\leq \E[F_\S(\w_{\S}^*)] + \epsilon + \frac{3 G^2\log (\epsilon_0/\epsilon)}{n\mu}.
\end{align*}
\end{thm}

\noindent
{\bf Remark 8:} Let $\epsilon = \frac{1}{n\mu}$,  the excess risk bound becomes $O(\log(n\mu)/(n\mu))$ and the total iteration complexity is $T=O(nL)$. This improves the convergence of testing error of SGD stated in Corollary~\ref{thm:gen} for the convex case when $\mu\ll 1$, which needs $T=O(nL/\mu)$ iterations and has a testing error bound of $O(L\log(nL/\mu)/(n\mu))$.

\section{{\start} for Non-Convex Functions}\label{sec:5}
Next, we will establish faster convergence of {\start} than SGD for ``nice-behaviored" non-convex functions. In particular, we will consider  two classes of non-convex functions that are close to a convex function, namely one-point weakly quasi-convex and weakly convex functions.  We first introduce the definitions of these functions and then present some discussions followed by their convergence results.

\begin{figure}[t]
\centering
\includegraphics[scale=0.5]{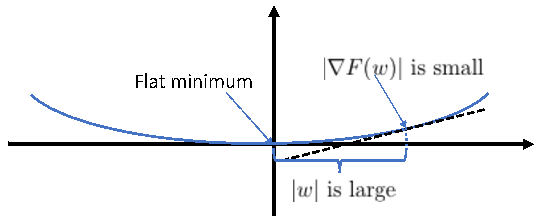}\hspace*{0.1in}
\vspace*{-0.15in}
\caption{Illustration of $\|w - w^*\|\geq \|\nabla F(w)\|$ around a flat minimum.}
\label{fig:0}
\end{figure}

\begin{definition}[One-point Weakly Quasi-Convex] A non-convex function $F$ is called one-point $\theta$-weakly quasi-convex for $\theta>0$ if there exists a global minimum $\w_*$ such that 
\begin{align}\label{eqn:qc}
\nabla F(\w)^{\top}(\w - \w^*)\geq \theta(F(\w) - F(\w^*)), \forall\w\in\Omega.
\end{align}
\end{definition}
\begin{definition}[Weakly Convex] A non-convex function $F$ is  $\rho$-weakly convex for $\rho>0$ if $F(\w) + \frac{\rho}{2}\|\w\|^2$ is convex. 
\end{definition}
\vspace*{-0.1in}
It is interesting to connect one-point weakly quasi-convexity to one-point strong convexity that has  been considered for non-convex optimization, especially optimizing neural networks~\citep{DBLP:conf/nips/LiY17,pmlr-v80-kleinberg18a}. 
\begin{definition}[One-point Strongly Convex] A non-convex function $F$ is one-point strongly  convex with respect to a global minimum $\w^*$  if there exists $\mu_1>0$ such that
\begin{align*}
\nabla F(\w)^{\top}(\w - \w^*)\geq \mu_1 \|\w - \w^*\|^2.
\end{align*} 
\end{definition}
\vspace*{-0.1in}The following lemma shows that one-point strong convexity implies both the PL condition and the one-point weakly quasi-convexity. 
\begin{lemma}\label{lem:onepoint}
Suppose $F$ is $L$-smooth and one-point strongly convex w.r.t $\w^*$ with $\mu_1>0$ and   $\nabla F(\w^*) = 0 $, then 
\begin{align*}
      \min \{\|\nabla F(\w)\|^2,  \nabla F(\w)^{\top}(\w - \w^*)\}
 \geq \frac{2\mu_1}{L}(F(\w) - F(\w^*))  .
\end{align*} 
\end{lemma}

\yancomment{  
\begin{proof}
The inequality regarding $\|\nabla F(\w)\|^2$ can be found in~\citep{DBLP:conf/pkdd/KarimiNS16}. The inequality regarding $\nabla F(\w)^{\top}(\w - \w^*)$ can be easily seen from the definition of one-point strong convexity and the $L$-smoothness condition of $F(\w)$ and the condition $\nabla F(\w^*)=0$, i.e., 
\begin{align*}
     F(\w) - F(\w^*) 
\leq \nabla F(\w^*)^{\top}(\w - \w^*) + \frac{L}{2}\|\w - \w^*\|^2 
\leq \frac{L}{2\mu_1}\nabla F(\w)^{\top}(\w - \w^*).
\end{align*}
\end{proof}
} 

For ``nice-behaviored" one-point weakly quasi-convex function $F_\S(\w)$ that satisfies the PL condition, we are interested in the case that $\theta$ is a constant close to or larger than $1$.  Note that a convex function has $\theta=1$ and a strongly convex function has $\theta>1$. For the case of $\mu\ll 1$ in the PL condition, this  indicates  that  $ \nabla F(\w)^{\top}(\w - \w^*)$ is larger than $\|\nabla F(\w)\|^2$, which further implies that $\|\w - \w^*\|\geq  \|\nabla F(\w)\|$. Intuitively, this inequality (illustrated in Figure~\ref{fig:0}) also connects itself to the flat minimum that has been observed in deep learning experiments~\cite{DBLP:journals/corr/ChaudhariCSL16}. 
For ``nice-behaviored" weakly convex function, we are interested in the case that $\rho\leq \mu/4$ is close to zero. Weakly convex functions with a small $\rho$ have been considered in the literature of non-convex optimization~\citep{DBLP:journals/corr/abs/1805.05411}. In both cases, we will establish faster convergence of optimization error and testing error of {\start}.

\subsection{Convergence of Optimization Error}
The approach of analysis for the considered non-convex functions is similar to that of convex functions. We also first analyze the convergence of SGD for each stage in Lemma~\ref{lem:sgd2-qc} and Lemma~\ref{lem:sgd2} for one-point weakly quasi-convex and weakly convex functions, respectively.
Then we extend these results to $K$ stages for {\start} in Theorem~\ref{thm:ncvx-1} and Theorem~\ref{thm:ncvx}.  

\begin{lemma}\label{lem:sgd2-qc}Assume $F_\S$ is {\bf one-point $\theta$-weakly quasi-convex} w.r.t $\w^*_\S$ . By applying SGD (Algorithm~\ref{alg:sgd}) to $F_k = F^{\gamma}_{\w_{k-1}}$ with  $\w_k = \w_\tau$ where $\tau\in\{1, \ldots, T_k\}$ is randomly sampled, we have
\begin{align*}
\E_k[ F_\S(\w_k) - F_\S( \w_\S^*)]\leq& \frac{\| \w_{k-1} - \w^*_\S\|^2}{2\theta\eta_k T_k} + \frac{\eta_k G^2}{2\theta }  
+ \frac{1}{2\gamma \theta}\|\w_{k-1} - \w_\S^*\|^2   .
\end{align*}
\end{lemma}

\yancomment{ 
\begin{proof}
Without loss of generality, we consider minimizing $F_1 = F_\S + \frac{1}{2\gamma}\|\w - \w_0\|^2$. 
Let $r(\w) = \frac{1}{2\gamma}\|\w - \w_{0}\|^2$. The initial solution of SGD $\w_1 = \w_0$.  
Following the standard analysis of stochastic proximal SGD, we have
\begin{align*}
     \nabla f(\w_t, \z_{i_t})^{\top}(\w_t - \w) + r(\w_{t+1}) - r(\w)
\leq \frac{\|\w - \w_t\|^2}{2\eta} + \frac{\|\w - \w_{t+1}\|^2}{2\eta} + \frac{\eta}{2}\|\nabla f(\w_t, \z_{i_t})\|^2   .
\end{align*}
Taking expectation on both sides, we have 
\begin{align*}
     \E[\nabla& F_\S(\w_t)^{\top}(\w_t - \w) + r(\w_{t+1}) - r(\w)]
\leq \E\bigg[\frac{\|\w - \w_t\|^2}{2\eta} -  \frac{\|\w - \w_{t+1}\|^2}{2\eta} + \frac{\eta G^2}{2}\bigg]   .
\end{align*}
Plugging $\w = \w_\S^*$, summing over $t=1,\ldots, T$ and using the one-point weakly quasi-convexity, we have
\begin{align*}
\E\bigg[ & \sum_{t=1}^T\theta(F_\S(\w_t) - F_\S( \w_\S^*)) + r(\w_{t}) - r(\w^*_\S)\bigg]   \\
         & \leq \frac{\|\w^*_\S - \w_1\|^2}{2\eta} + \frac{\eta G^2 T}{2} + \E[r(\w_1) - r(\w_{T+1})]  .
\end{align*}
As a result, 
\begin{align*}
     \E\bigg[ F_\S(\w_\tau) - F_\S( \w_\S^*) \bigg]
\leq \frac{\|\w^*_\S - \w_1\|^2}{2\theta\eta T} + \frac{\eta G^2}{2\theta }  + \frac{1}{2\gamma \theta}\|\w_{0} - \w_\S^*\|^2  ,
\end{align*}
where $\tau\in\{1, \ldots, T\}$ is randomly selected. Applying the above result to the $k$-th stage, we complete the proof.
\end{proof}
} 

\begin{thm}\label{thm:ncvx-1}
Suppose $F_\S(\w)$ is {\bf one-point  $\theta$-weakly quasi-convex} w.r.t $\w^*_\S$ and~(\ref{eqn:EB}) holds for the same $\w^*_\S$. Then by setting $\gamma\geq  1.5/(\theta\mu)$,  $\eta_k = \frac{2\epsilon_k\theta}{3G^2}$, $T_k = \frac{9G^2}{4\mu\epsilon_k\theta^2}$ and $\mathcal O(\w^k_1, \ldots, \w^k_{T_k+1}) = \w^k_\tau$ where $\tau\in\{1, \ldots, T_k\}$ is randomly sampled, after $K = \log(\epsilon_k/\epsilon)$ stages we have
\begin{align*}
\E[F_\S(\w_K) - F_\S(\w^*_\S)]\leq \epsilon.
\end{align*}
The total iteration complexity is $O(\frac{1}{\theta^2\mu\epsilon})$.
\end{thm}

\begin{proof} 
We will prove by induction that $\E[F(\w_k) - F(\w_{*})]\leq \epsilon_k $, where $\epsilon_k = \epsilon_0/2^k$, which is true for $k=0$ by the assumption. 
By applying Lemma~\ref{lem:sgd2-qc} to the $k$-th stage,
\begin{align*}
       \E_k[F_\S(\w_k) - F_\S(\w^*_\S)]
& \leq \frac{\| \w_{k-1} - \w^*_\S\|^2}{2\theta\eta_k T_k} + \frac{\eta_k G^2}{2\theta }  
       + \frac{1}{2\gamma \theta}\|\w_{k-1} - \w_\S^*\|^2\\
& \leq  \frac{\E[(F_\S(\w_{k-1}) - F_\S(\w_\S^*))]}{4\mu\gamma\theta}   +  \frac{\eta_k G^2}{2\theta } 
       + \frac{\E[(F_\S(\w_{k-1}) - F_\S(\w_\S^*))]}{4\theta\mu \eta_kT_k}\\
& \leq  \frac{\epsilon_{k-1}}{4\mu\gamma\theta}   +    \frac{\eta_k G^2}{2\theta } +  \frac{\epsilon_{k-1}}{4\mu\theta \eta_kT_k}.
\end{align*}
By the setting $\eta_k =\frac{2\epsilon_k \theta}{3 G^2}$ and $T_k\eta_k = 1.5/(\theta\mu)$ and $\gamma \geq  1.5/(\theta\mu)$, we have 
\begin{align*}
&\E[F_\S(\w_k) - F_\S(\w_\S^{*})]\leq  \epsilon_k.
\end{align*}
By induction, after $K=\lceil\log(\epsilon_0/\epsilon)\rceil$ stages, we have 
\begin{align*}
&\E[F_\S(\w_K) - F_\S(\w_\S^{*})]\leq   \epsilon.
\end{align*}
The total iteration complexity is $\sum_{k=1}^KT_k= O(1/(\theta^2\mu\epsilon))$. 
\end{proof}

The following is the results for $\rho$-weakly convex.

\begin{lemma}\label{lem:sgd2}
Assume $F_\S$ is $\rho$-weakly convex. By applying SGD (Algorithm~\ref{alg:sgd}) to $F_k = F^{\gamma}_{\w_{k-1}}$ with $\gamma\leq 1/\rho$, $\eta\leq 1/L$ and $\w_k =\mathcal O(\w^k_1, \ldots, \w^k_{T_k+1}) = \sum_{t=1}^{T_k}\w^k_{t+1}/T_k$, for any $\w\in\Omega$, we have
\begin{align*}
\E_k[F_k(\w_k) - F_k(\w) ] \leq  \sigma^2\eta_k +& \frac{\|\w_{k-1} - \w\|^2}{2T_k} \left(\frac{1}{\eta_k} + \frac{1}{\gamma}\right).
\end{align*}
\end{lemma}

\noindent
{\bf Remark 9:} The above result indicates that $\gamma$ can not be as large as infinity. However, for a small value $\rho$, the  added regularization term $1/\gamma\|\w - \w_{k}\|^2$ is not large.


\begin{thm}\label{thm:ncvx}
Suppose Assumption~\ref{ass:1} holds, and  $F_\S(\w)$ is {\bf $\rho$-weakly convex} with $\rho \leq \mu/4$. Then by setting $\gamma = 4/\mu\leq 1/\rho$, $\eta_k =\frac{\epsilon_k\alpha}{4\sigma^2}\leq 1/L$, $T_k= \frac{4\sigma^2}{\mu\epsilon_k\alpha}$  and $\mathcal O(\w^k_1,\ldots, \w^k_{T_k+1}) = \sum_{t=1}^{T_k}\w^k_{t+1}/T_k$,  where $\alpha \leq \min(1, \frac{2\sigma^2}{\epsilon_0 L})$, and after $K = \log(\epsilon_0/\epsilon)$ stages we have
\begin{align*}
\E[F_\S(\w_K) - F_\S(\w_\S^*)]\leq \epsilon.
\end{align*}
The total iteration complexity is $O(\frac{1}{\alpha\mu\epsilon})$.
\end{thm}

\noindent
{\bf Remark 10: } Several differences are noticeable between the two classes of non-convex functions: (i)  $\gamma$ in the weakly quasi-convex case can be as large as $\infty$, in contrast it is required to be smaller than $1/\rho$ in the weakly convex case; (ii) the returned solution by SGD at the end of each stage is a randomly selected solution in the weakly quasi-convex case and is an averaged solution in the weakly convex case. Finally, we note that the total iteration complexity for both cases is $O(1/\mu\epsilon)$ under $\theta\approx 1$ and $\rho\leq O(\mu)$, which is better than $O(1/\mu^2\epsilon)$ of SGD as in Theorem~\ref{thm:sgd}. 

\begin{proof} 
We will prove by induction that $\E[F(\w_k) - F(\w_{*})]\leq \epsilon_k $, where $\epsilon_k = \epsilon_0/2^k$, which is true for $k=0$ by the assumption. 
By applying Lemma~\ref{lem:2} to the $k$-th stage, for any $\w\in\Omega$
\begin{align*}
       \E_k[F_\S(\w_k) - F_\S(\w)]
& \leq \frac{\|\w_{k-1} - \w\|^2}{2\gamma}   + \eta_k \sigma^2
       + \frac{\|\w_{k-1} - \w\|^2}{2 \eta_k T_k} + \frac{\|\w_{k-1} - \w\|^2}{2 \gamma T_k} .
\end{align*}
By plugging $\w = \w_\S^{*}$ into the above inequality we have
\begin{align*}
&      \E[F_\S(\w_k) - F_\S(\w_\S^{*})]   \\
& \leq \frac{\E[\|\w_{k-1} - \w_\S^*\|^2]}{2\gamma_k}   + \eta_k \sigma^2 
       + \frac{\E[\|\w_{k-1} - \w_\S^{*} \|^2]}{2 \eta_k T_k} + \frac{\E[\|\w_{k-1} - \w_\S^{*} \|^2]}{2 \gamma T_k}     \\
& \leq \frac{\E[(F_\S(\w_{k-1}) - F_\S(\w_\S^*))]}{4\mu\gamma}   + \eta_k \sigma^2
       + ( 1 / \eta_{k} + 1 / \gamma ) \frac{\E[(F_\S(\w_{k-1}) - F_\S(\w_\S^*))]}{4\mu T_k}\\
&\leq  \frac{\epsilon_{k-1}}{4\mu\gamma}   +  \eta_k \sigma^2 + ( 1 / \eta_{k} + 1 / \gamma )  \frac{\epsilon_{k-1}}{4\mu T_k}   ,
\end{align*}
where we use Lemma~\ref{lem:eb}.  By the setting $\eta_k =\frac{\epsilon_k \alpha}{4 \sigma^2}$ and $T_k\eta_k = 1/\mu$ and $\gamma = 4/\mu$, we have 
\begin{align*}
&\E[F_\S(\w_k) - F_\S(\w_\S^{*})]\leq  \epsilon_k.
\end{align*}
By induction, after $K=\lceil\log(\epsilon_0/\epsilon)\rceil$ stages, we have 
\begin{align*}
&\E[F_\S(\w_K) - F_\S(\w_\S^{*})]\leq   \epsilon.
\end{align*}
The total iteration complexity is $\sum_{k=1}^KT_k= O(1/(\mu\epsilon))$. 
\end{proof}

\subsection{Generalization Error}
The analysis of generalization error for the two cases are similar with a unified result presented below. 
We will first establish the recurrence of stability within one stage in Lemma~\ref{lem:ncvxstab-1} and then use it to analyze the convergence of \start.

\begin{lemma}\label{lem:ncvxstab-1}
Assume $f$ is $L$-smooth.  Let $\w_t$ denote the sequence learned on $\S$ and $\w'_t$ be the sequence learned on $\S'$ by {\start} at one stage,  $\delta_t = \|\w_t- \w'_t\|$.
Then 
\begin{align*}
\delta_{t+1}\leq \left\{\begin{array}{lc}\frac{\eta}{\eta+ \gamma}\delta_1 + \frac{\gamma(1+\eta L)}{\eta + \gamma}\delta_t
& f_t= f'_t\\
  \frac{\eta}{\eta+ \gamma}\delta_1 +  \frac{\gamma}{\eta + \gamma}\delta_t  + \frac{2\eta\gamma G}{\eta + \gamma}& \text{otherwise}\end{array}\right.
\end{align*}
\end{lemma}

\yancomment{ 
\begin{proof}
Let us consider two scenarios. The first scenario is $f = f'$. Then 
\begin{align*}
&      \|\G(\w_t; f, \w_1) -\G(\w'_t; f', \w'_1)  \| \\
& =    { \bigg\|\frac{\gamma \w_t + \eta \w_1 - \eta \gamma \nabla f(\w_t)}{\eta + \gamma}
       - \frac{\gamma \w'_t + \eta \w'_1 - \eta \gamma \nabla f(\w'_t)}{\eta + \gamma}\bigg\|}\\
& \leq \frac{\eta}{\eta+ \gamma}\|\w_1 - \w'_1\| + \frac{\gamma(1+\eta L)}{\eta + \gamma}\|\w_t - \w_t'\| \\
& =    \frac{\eta}{\eta+ \gamma}\delta_1 + \frac{\gamma(1+\eta L)}{\eta + \gamma}\delta_t.
\end{align*}
 Next, let us consider the second scenario $f \neq f'$. Then 
\begin{align*}
&      \|\G(\w_t; f, \w_1) -\G(\w'_t; f', \w'_1)  \|\\
& =    { \bigg\|\frac{\gamma \w_t + \eta \w_1 - \eta \gamma \nabla f(\w_t)}{\eta + \gamma}
 - \frac{\gamma \w'_t + \eta \w'_1 - \eta \gamma \nabla f'(\w'_t)}{\eta + \gamma}\bigg\| }\\
& \leq \frac{\eta}{\eta+ \gamma}\|\w_1 - \w'_1\| 
       + \frac{\gamma}{\eta + \gamma}\|\w_t - \eta\nabla f(\w_t) - \w'_t + \eta\nabla f'(\w'_t)\|\\
& \leq \frac{\eta}{\eta+ \gamma}\|\w_1 - \w'_1\| +  \frac{\gamma}{\eta + \gamma}\|\w_t - \w_t'\|  + \frac{2\eta\gamma G}{\eta + \gamma}\\
& =    \frac{\eta}{\eta+ \gamma}\delta_1 +  \frac{\gamma}{\eta + \gamma}\delta_t  + \frac{2\eta\gamma G}{\eta + \gamma}  .
\end{align*}
\end{proof}
}  

Next, we can establish the uniform stability of \start.

\begin{thm}\label{thm:stabncvx}
Let $S_{K-1}=\sum_{k=1}^{K-1}T_k =O(\frac{1}{\mu\epsilon\bar\alpha})$ and $\eta_K\leq c/(\mu T_K)$, where $\bar\alpha=\theta^2, c=1.5/\theta$ in the one-point weakly quasi-convex case and $\bar\alpha=\alpha, c =1$ in the weakly convex case.   Then we have
\begin{align*}
\varepsilon_{stab}\leq \frac{S_{K-1}}{n} + \frac{1+\mu/(Lc)}{n-1}(2G^2c/\mu )^{1/(1+Lc/\mu)}T_k^{\frac{Lc/\mu}{Lc/\mu + 1}}
\end{align*}
\end{thm}

\begin{proof} 
After establishing the recurrence of stability within one stage in Lemma~\ref{lem:ncvxstab-1}, we will apply the similar conditional analysis for the non-convex loss as in~\citep{DBLP:conf/icml/HardtRS16}. 
In particular, we will condition on $\w_{k-1} = \w_{k-1}'$, i.e., the different example will be used within the last stage, and prove the bound for $\|\w_K - \w_K'\|$. The result of Theorem \ref{thm:stabncvx} follows directly from Theorem 3.8 in~\citep{DBLP:conf/icml/HardtRS16} and Lemma 3.11 in the long version of~\citep{DBLP:conf/icml/HardtRS16}. 
\end{proof}

By putting the optimization error and generalization error together, we have the following testing error bound. 
\begin{thm}\label{thm:ncvxgen}
Under the same assumptions as in Theorem~\ref{thm:ncvx-1} or~\ref{thm:ncvx} and $\mu\ll 1$. 
After $K=\log(\epsilon_0/\epsilon)$ stages with a total number of iterations $T =O(\frac{1}{\bar\alpha \mu\epsilon})$, the testing error bound  of $\w_K$ is
\begin{align*}
\E_{\A,\S}&[F(\w_K)]\leq \E[F_\S(\w_{\S}^*)] + \epsilon + O(\frac{1}{n\bar\alpha\mu\epsilon}).
\end{align*}
\end{thm}

\noindent
{\bf Remark 11:} We are mostly interested in the case when $\theta$ is constant close to or larger than $1$. 
By setting $\epsilon = \Theta(1/\sqrt{n\bar\alpha\mu})$, we have the excess risk bounded by $O(\frac{1}{\sqrt{n\bar\alpha\mu}})$ under the total iteration complexity $T=O(\sqrt{\frac{n}{\bar\alpha\mu}})$. This improves the testing error bound of SGD stated in Corollary~\ref{thm:gen} for the non-convex case when $\mu\leq \bar\alpha$, which needs $T=O(\sqrt{n}/\mu)$ iterations and suffers a testing error bound of   $O(\frac{1}{\sqrt{n}\mu})$. 

\begin{proof}
The proof is done simply by combining the convergence of optimization error and generalization error bound with the following simplification as done in the proof of Theorem~\ref{thm:testcvx}: 
$$
       \bigg( \frac{2G^{2}c}{\mu} \bigg)^{ \frac{1}{\frac{Lc}{\mu} + 1 } } 
= 
       \bigg( \frac{2G^{2}c}{\mu} \bigg)^{ \frac{\mu}{2G^{2}c - \mu} \cdot \frac{2G^{2}c - \mu }{\mu} \cdot \frac{\mu}{Lc + \mu} }
\leq 
       e^{\frac{2G^{2}c - \mu}{Lc + \mu}} 
       \leq 
       e^{\frac{2G^{2}}{L}}     .
$$
\end{proof}
Finally, it is worth noting that our analysis is applicable to an approximate optimal solution $\w^*_\S$  as long as the inequality~(\ref{eqn:EB}) and~(\ref{eqn:qc}) hold for that particular $\w^*_\S$. This fact is helpful for us to verify the assumptions in numerical experiments.

\section{Numerical Experiments}
In this section, we perform experiments grouped into two subsections, i.e., non-convex deep learning and convex problems.
In both subsections, we study the covergence of \start~and compare it with commonly used SGD with polynomially decaying step size.
Additionally, we verify two assumptions (i.e., one-point weakly quasi-convex and PL condition) made in non-convex problems.

\begin{figure*}[t]
\centering
\includegraphics[width=.225\textwidth]{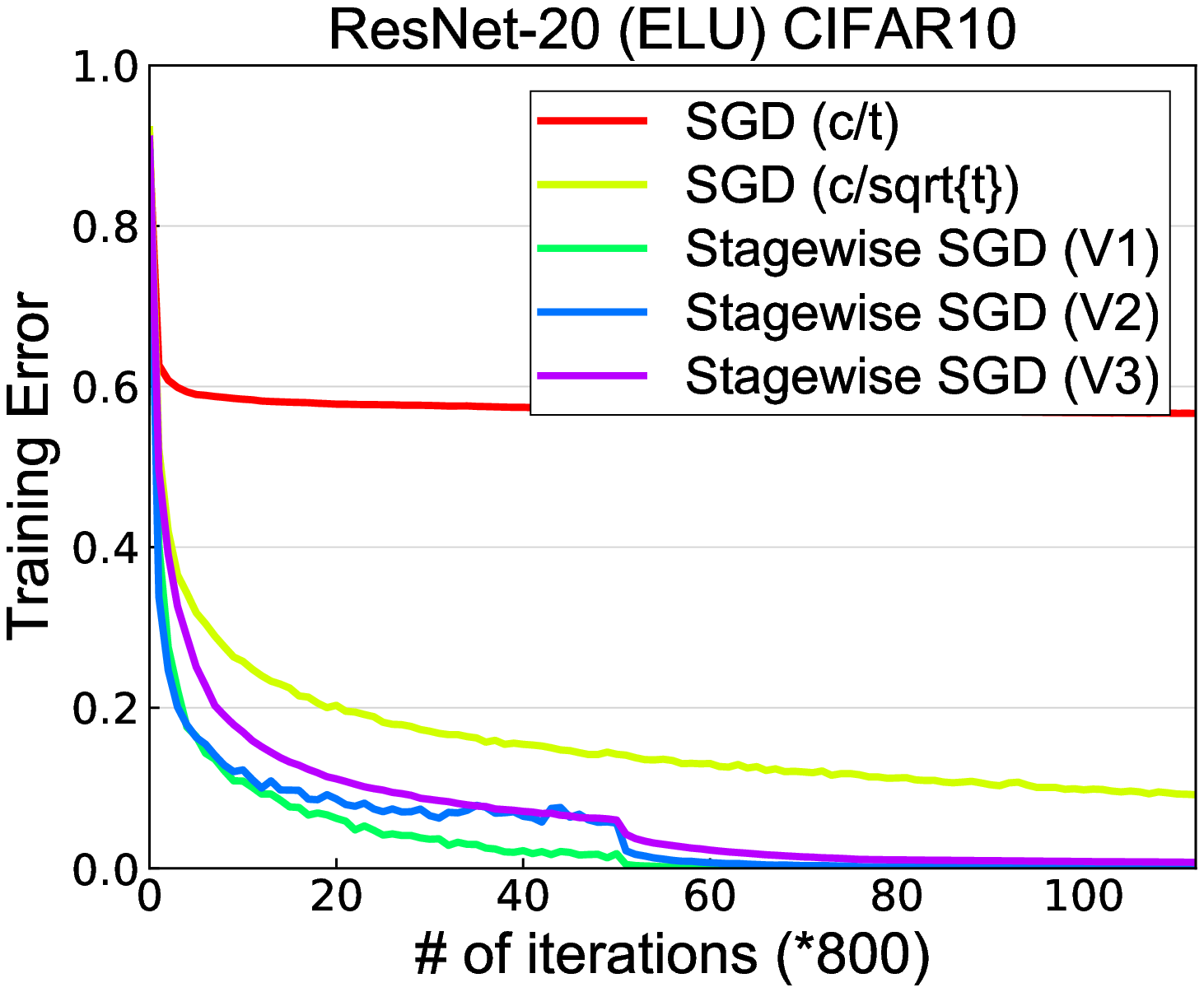}\hspace*{0.15in}
\includegraphics[width=.225\textwidth]{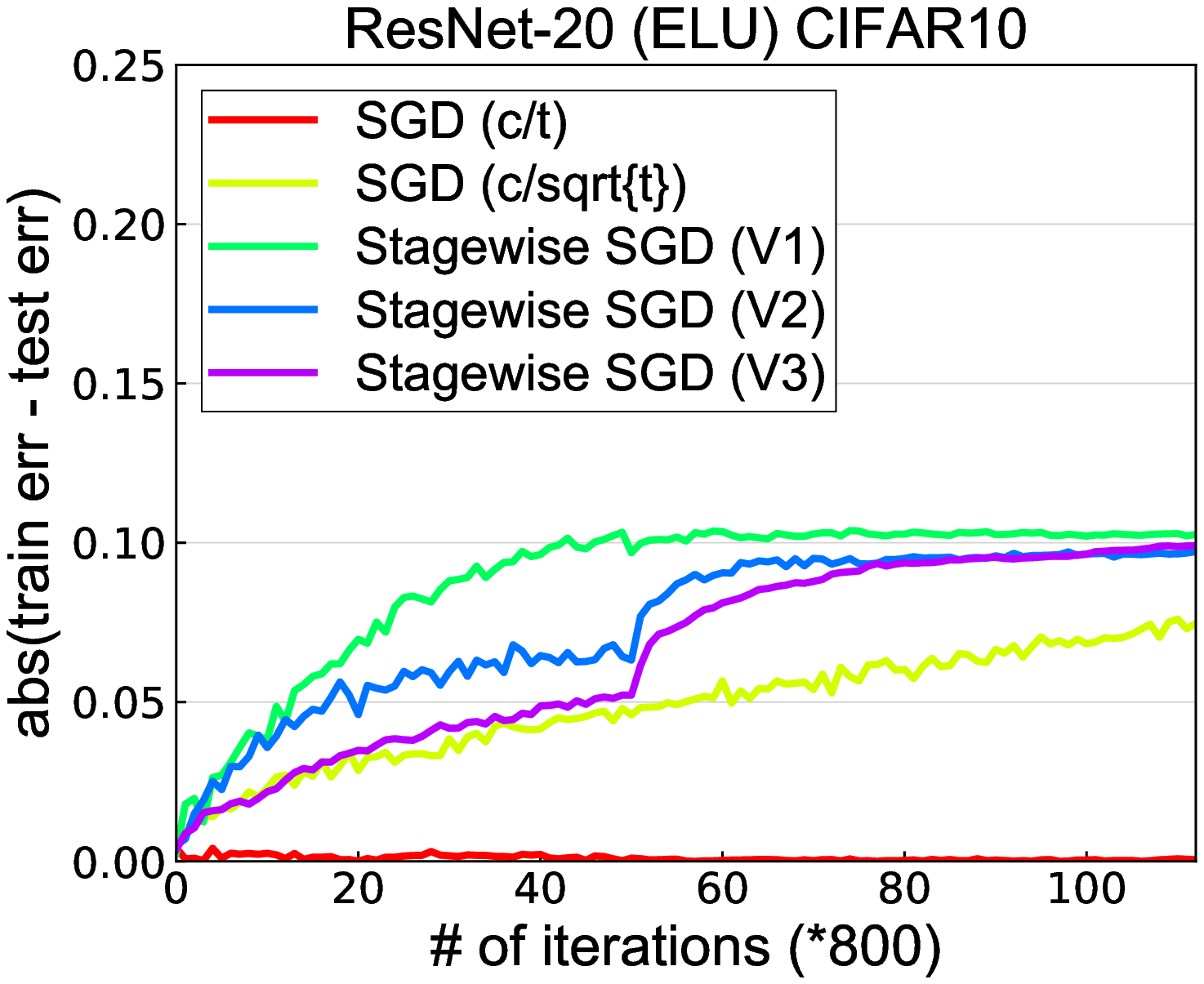}\hspace*{0.15in}
\includegraphics[width=.225\textwidth]{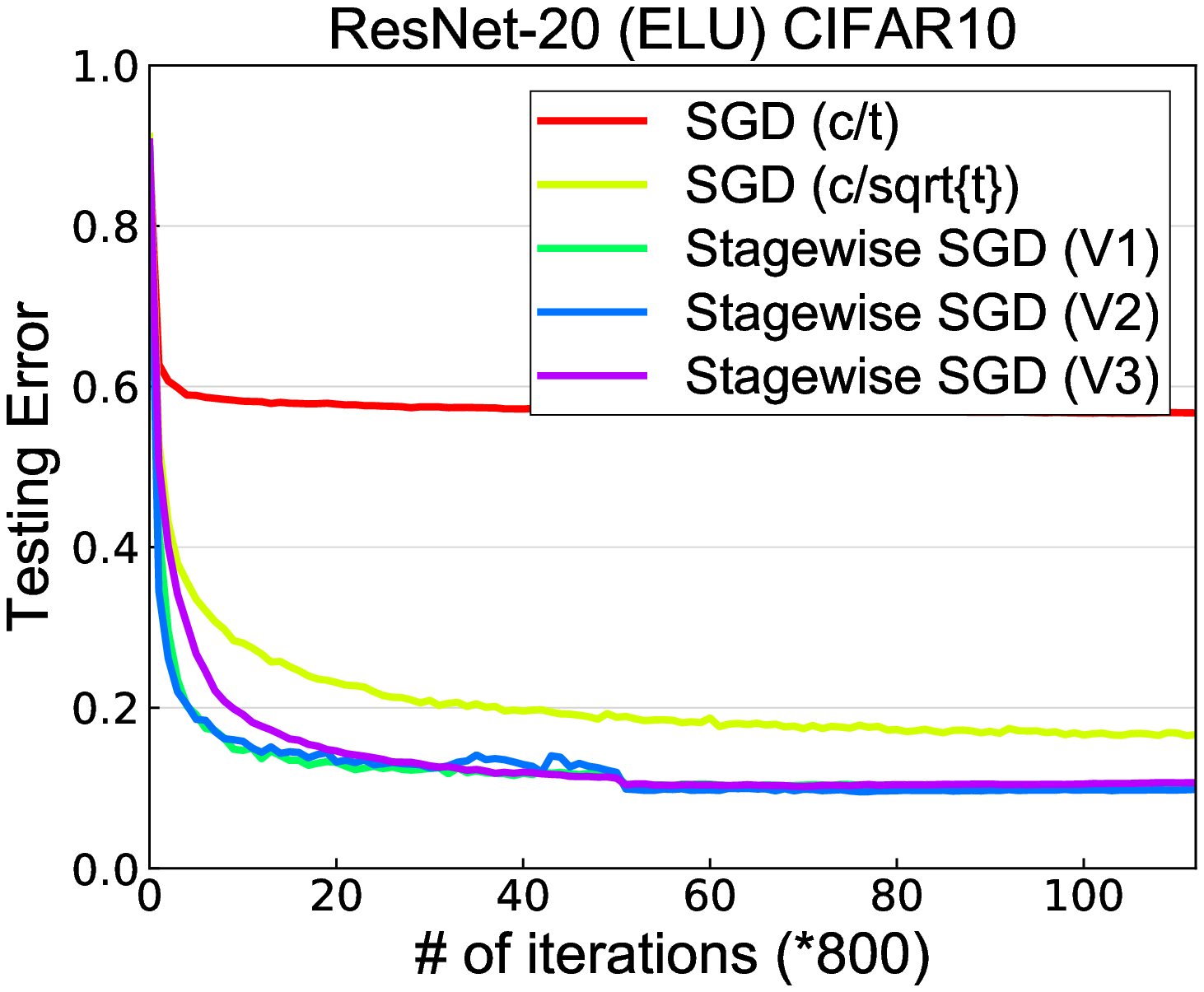}\hspace*{0.15in}
\includegraphics[width=.225\textwidth]{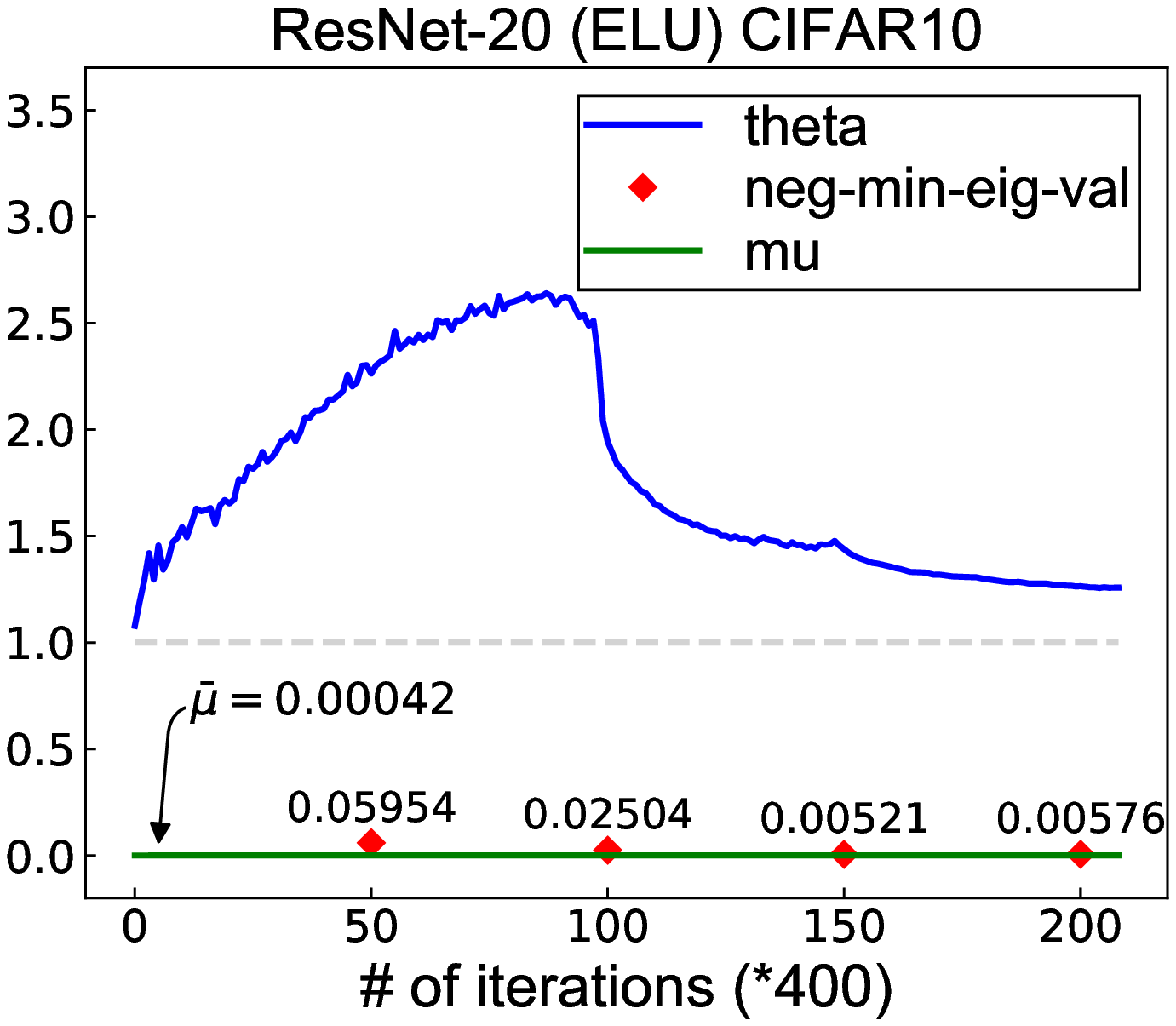}

\includegraphics[width=.225\textwidth]{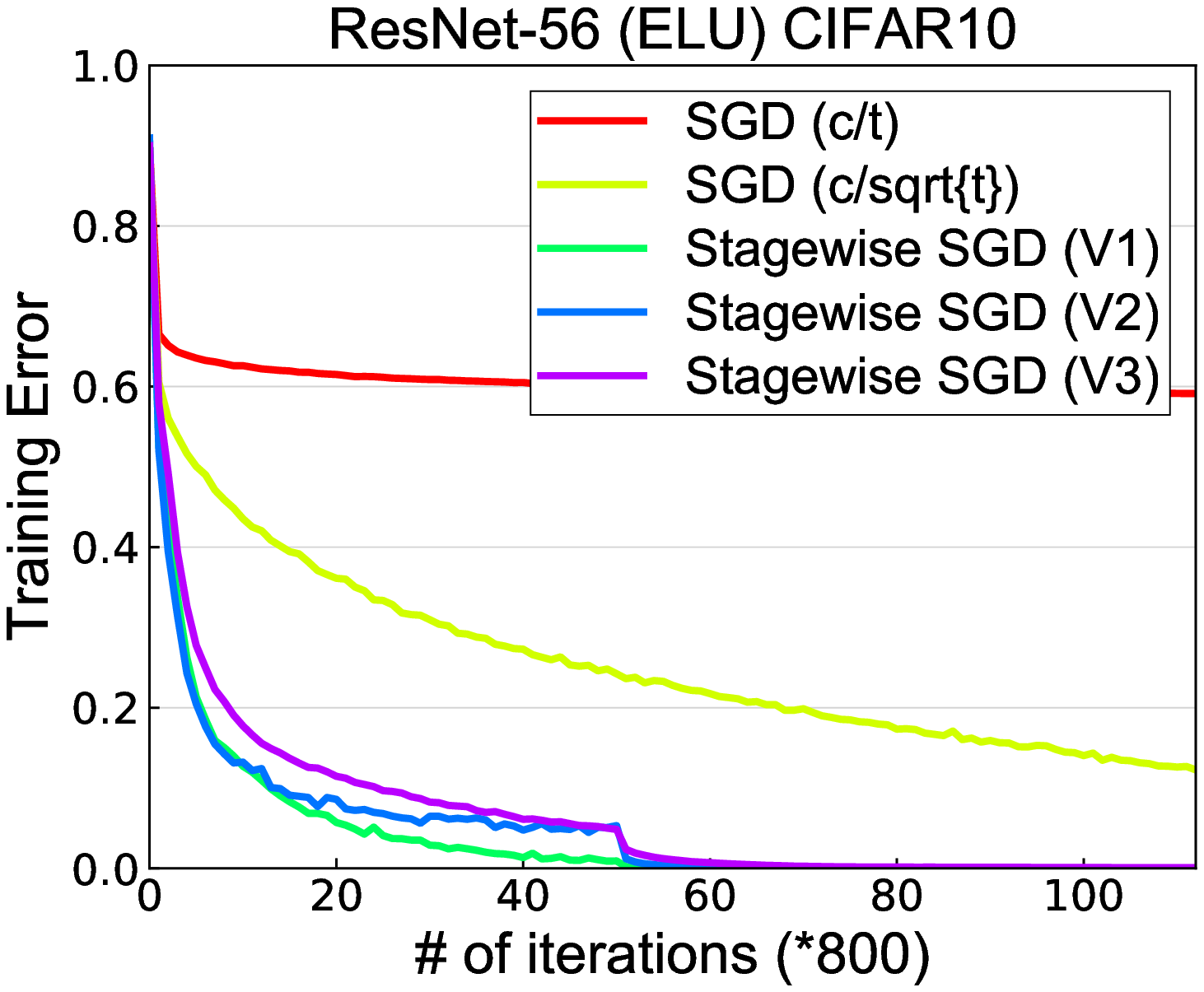}\hspace*{0.15in}
\includegraphics[width=.225\textwidth]{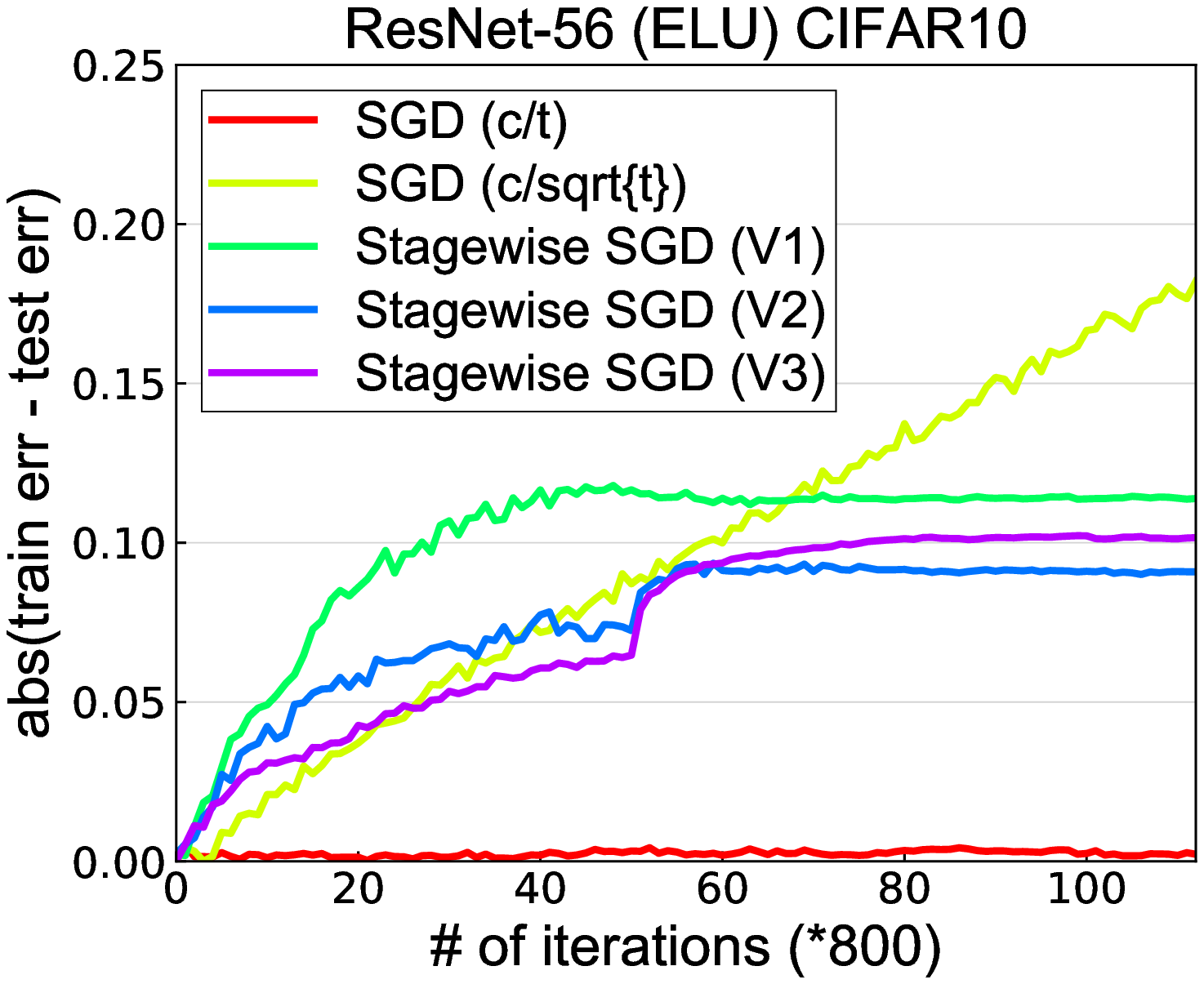}\hspace*{0.15in}
\includegraphics[width=.225\textwidth]{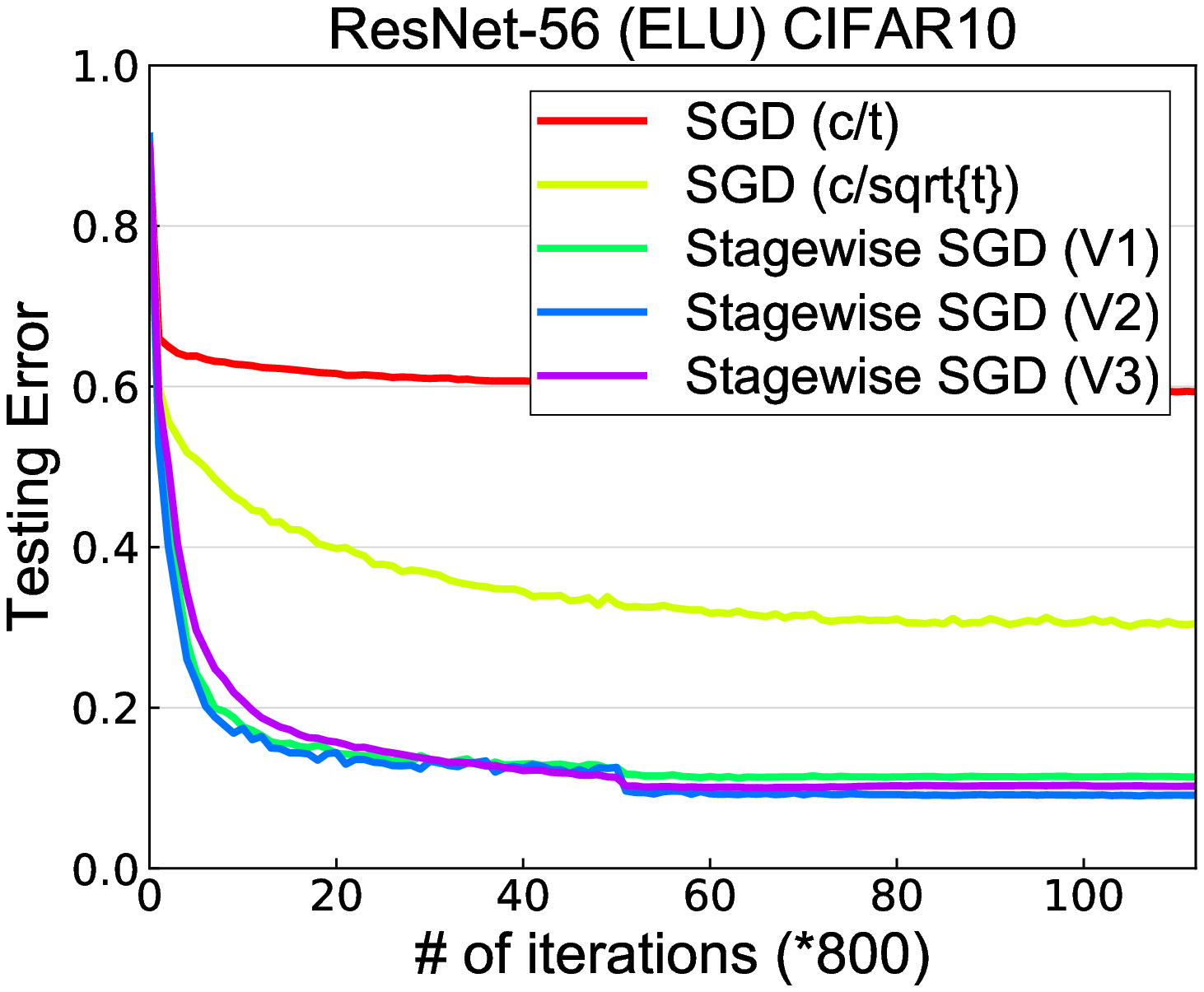}\hspace*{0.15in}
\includegraphics[width=.225\textwidth]{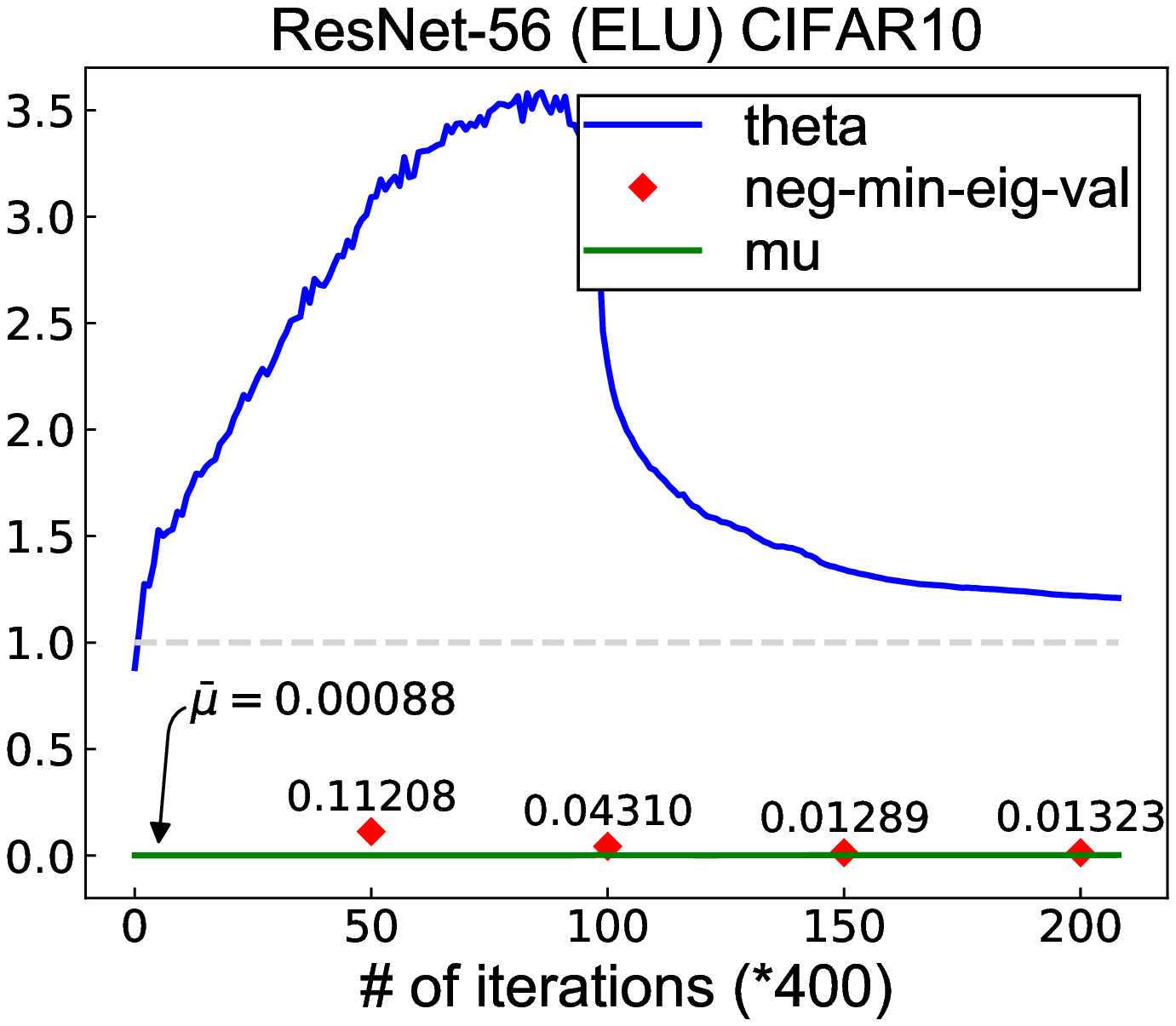}

\includegraphics[width=.225\textwidth]{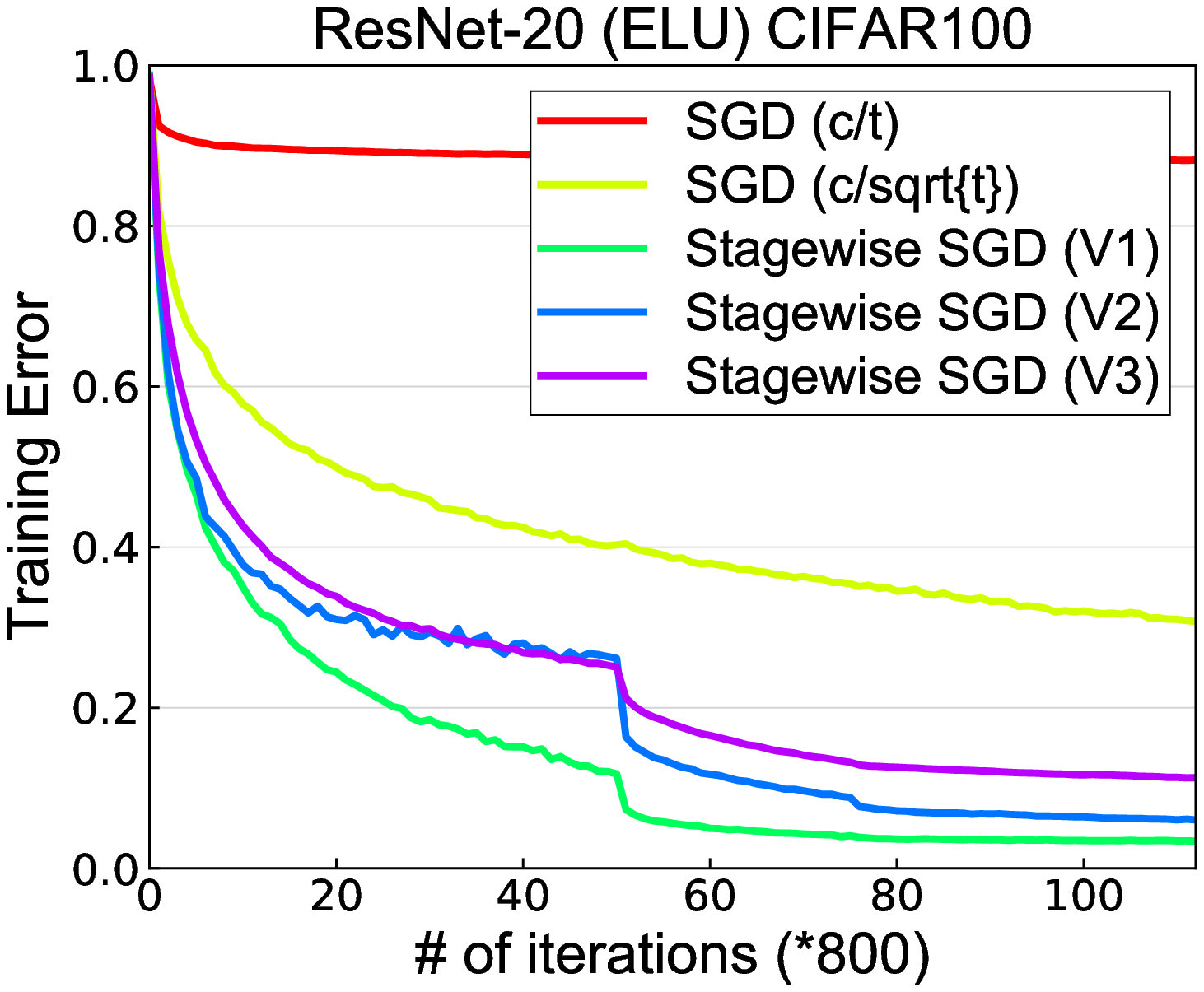}\hspace*{0.15in}
\includegraphics[width=.225\textwidth]{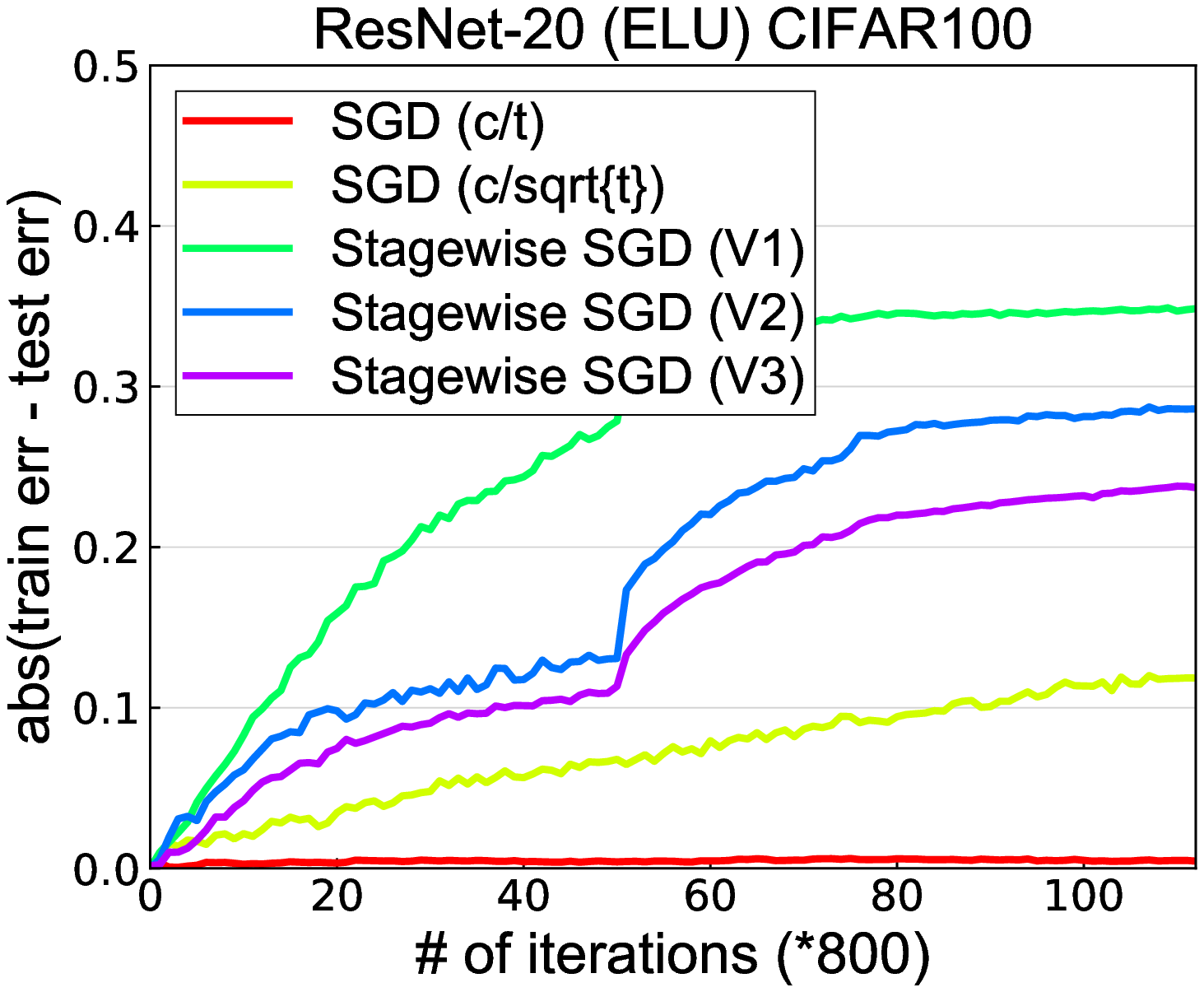}\hspace*{0.15in}
\includegraphics[width=.225\textwidth]{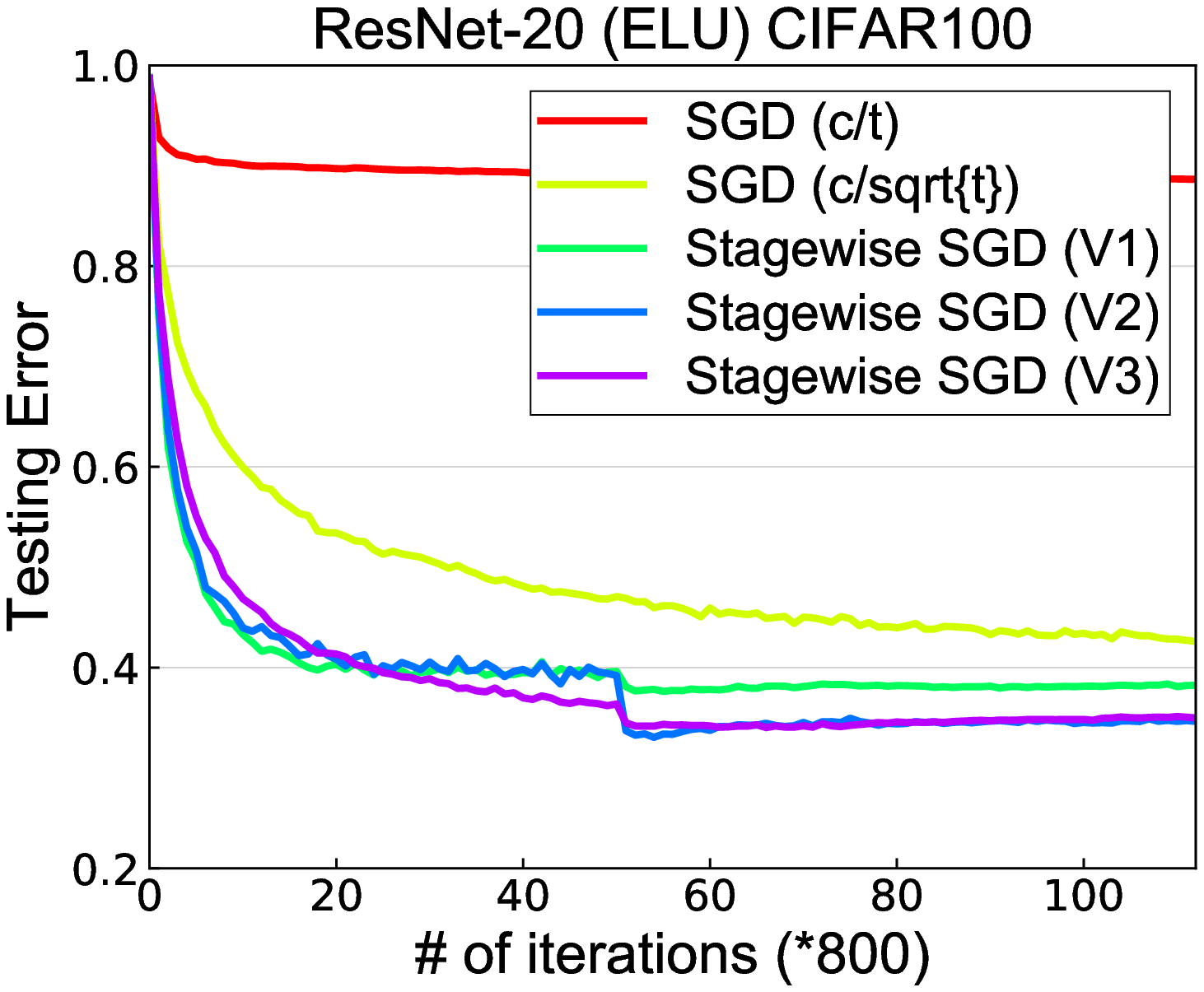}\hspace*{0.15in}
\includegraphics[width=.225\textwidth]{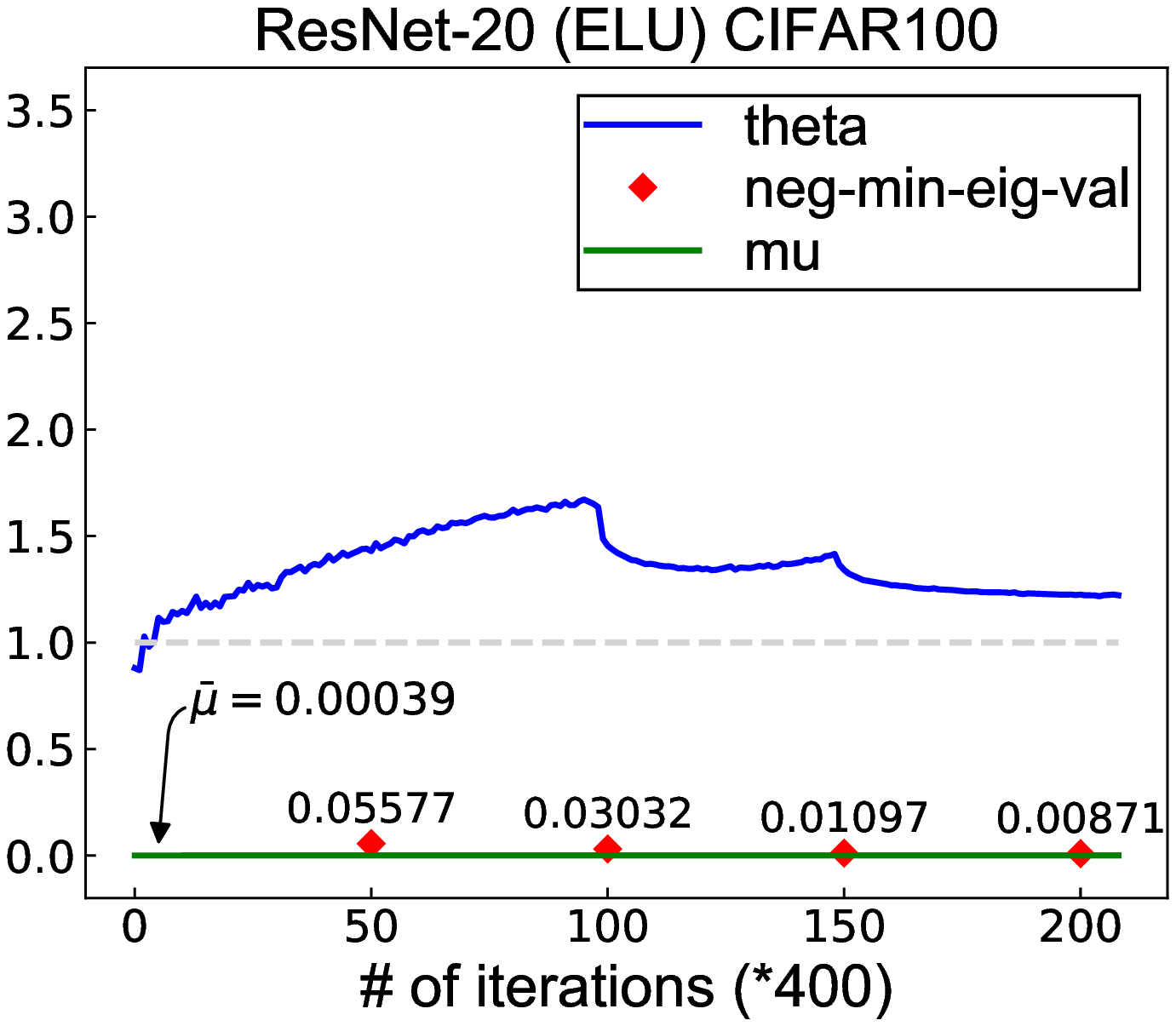}

\includegraphics[width=.225\textwidth]{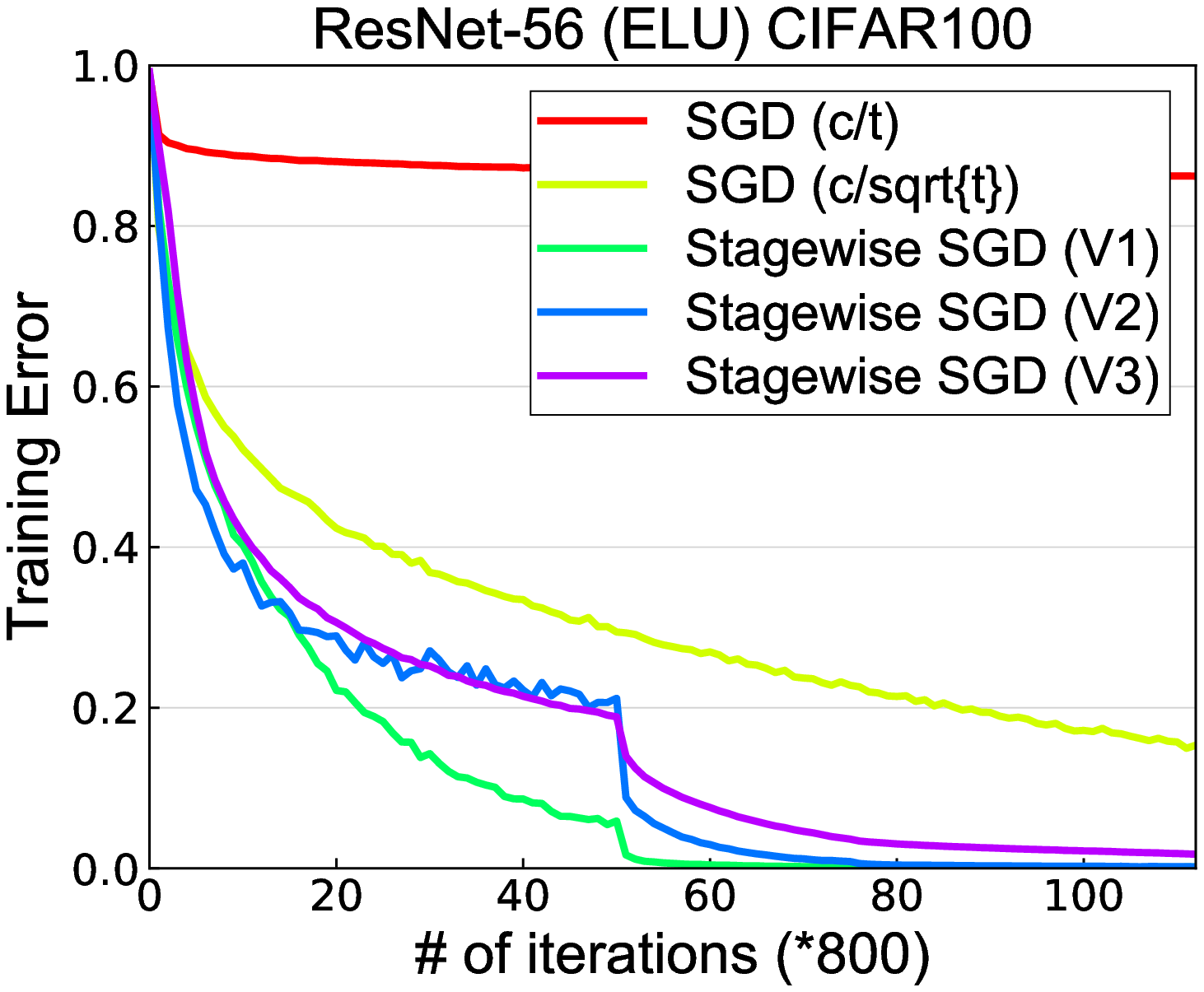}\hspace*{0.15in}
\includegraphics[width=.225\textwidth]{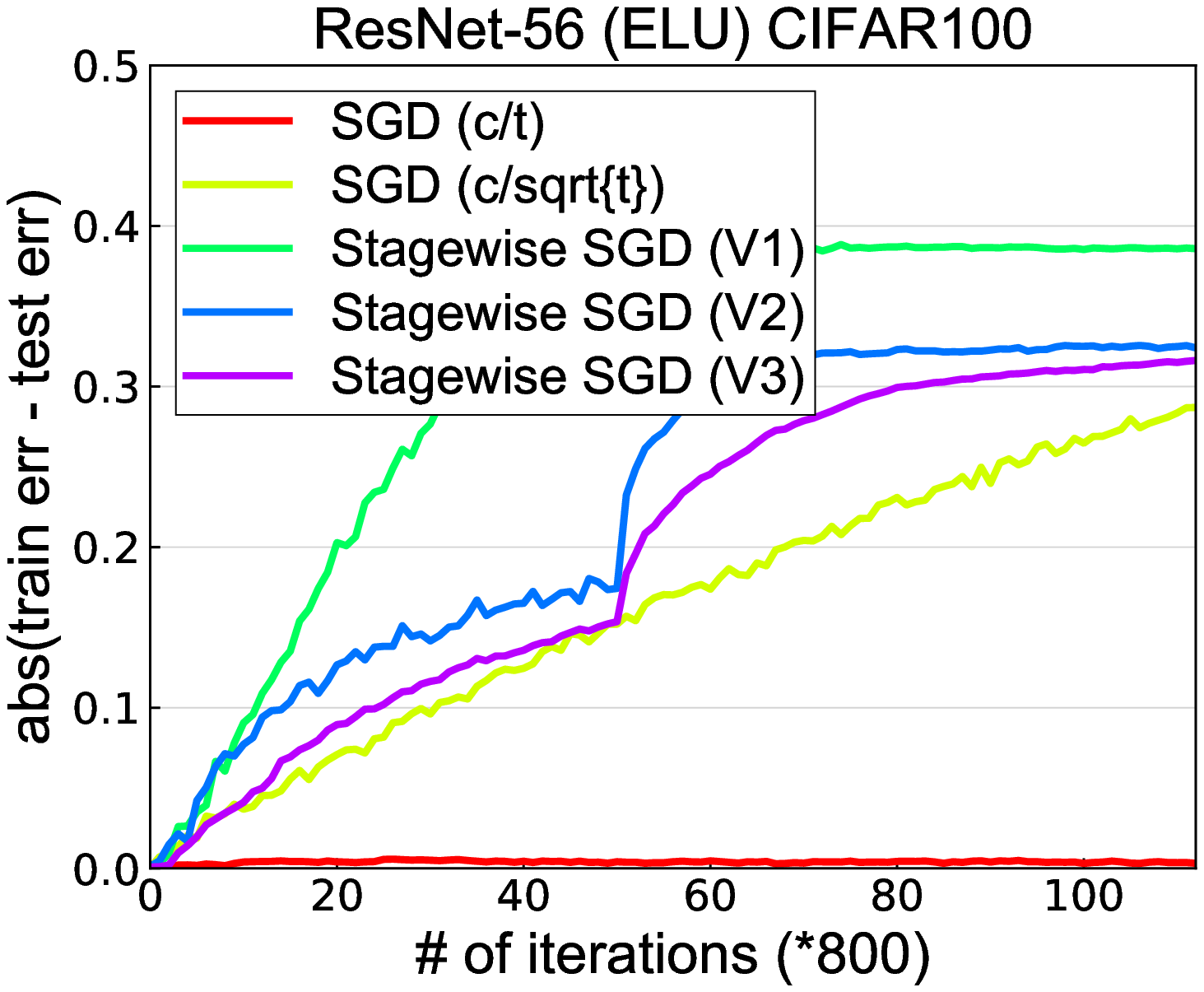}\hspace*{0.15in}
\includegraphics[width=.225\textwidth]{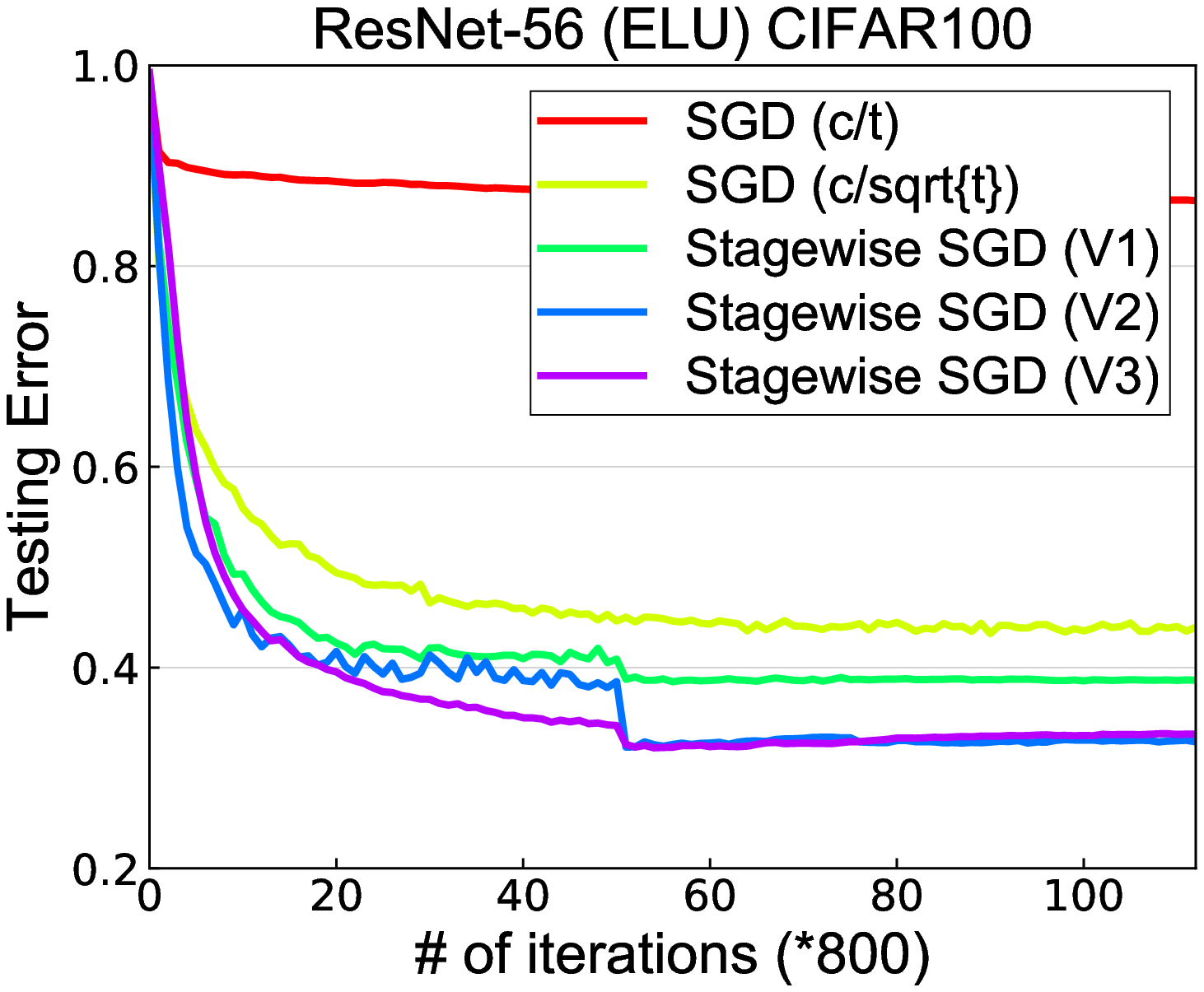}\hspace*{0.15in}
\includegraphics[width=.225\textwidth]{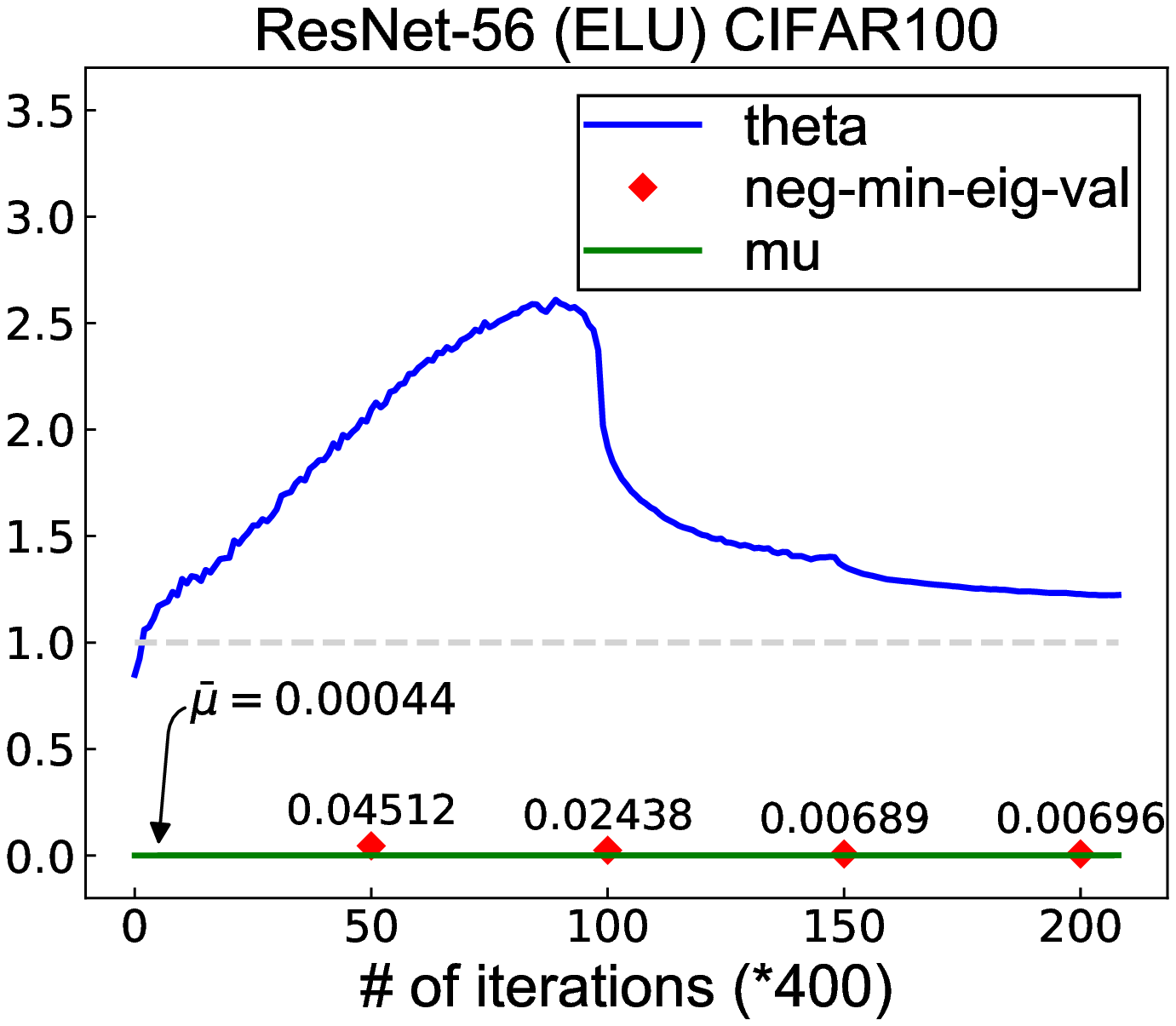}

\includegraphics[width=.225\textwidth]{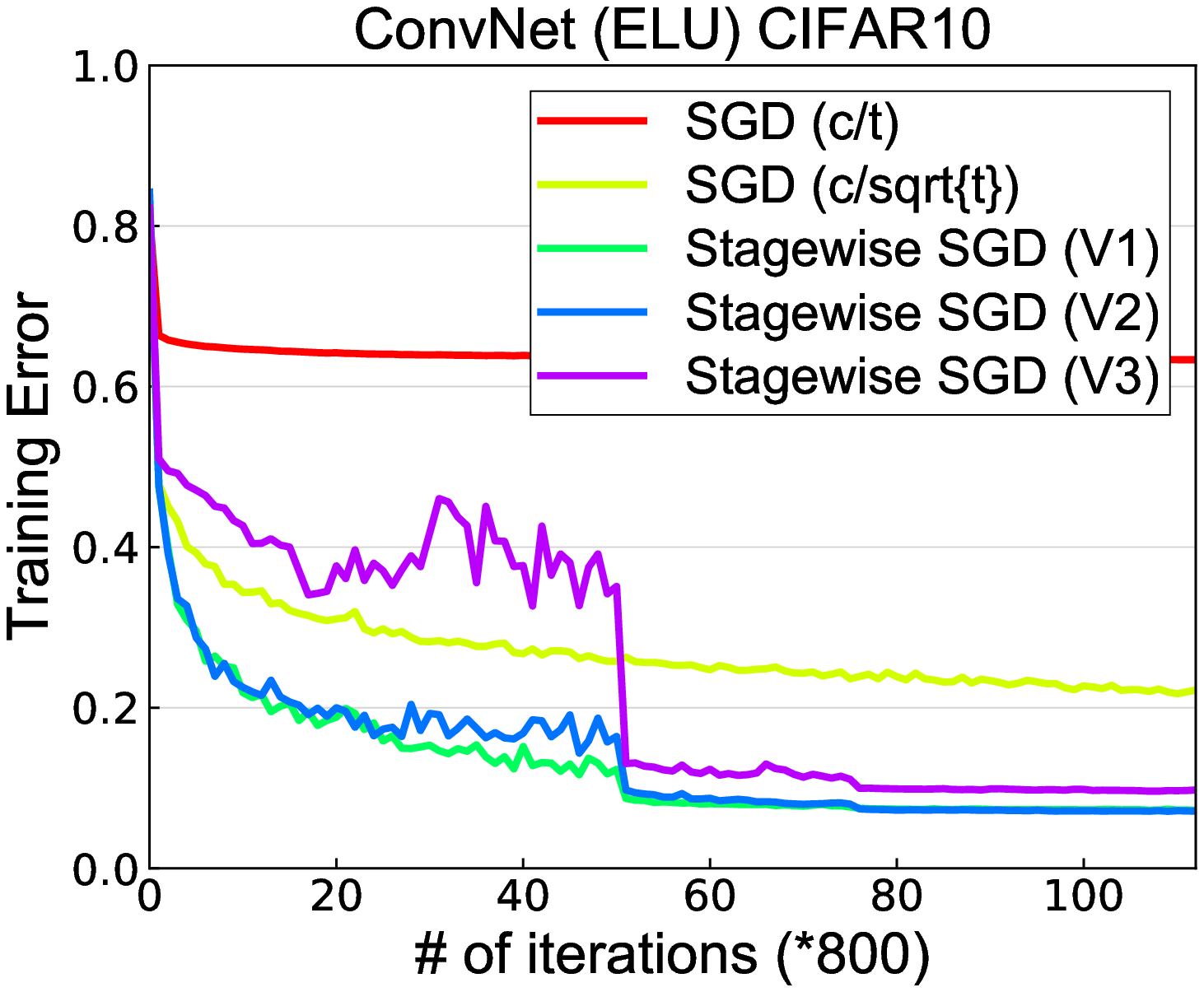}\hspace*{0.15in}
\includegraphics[width=.225\textwidth]{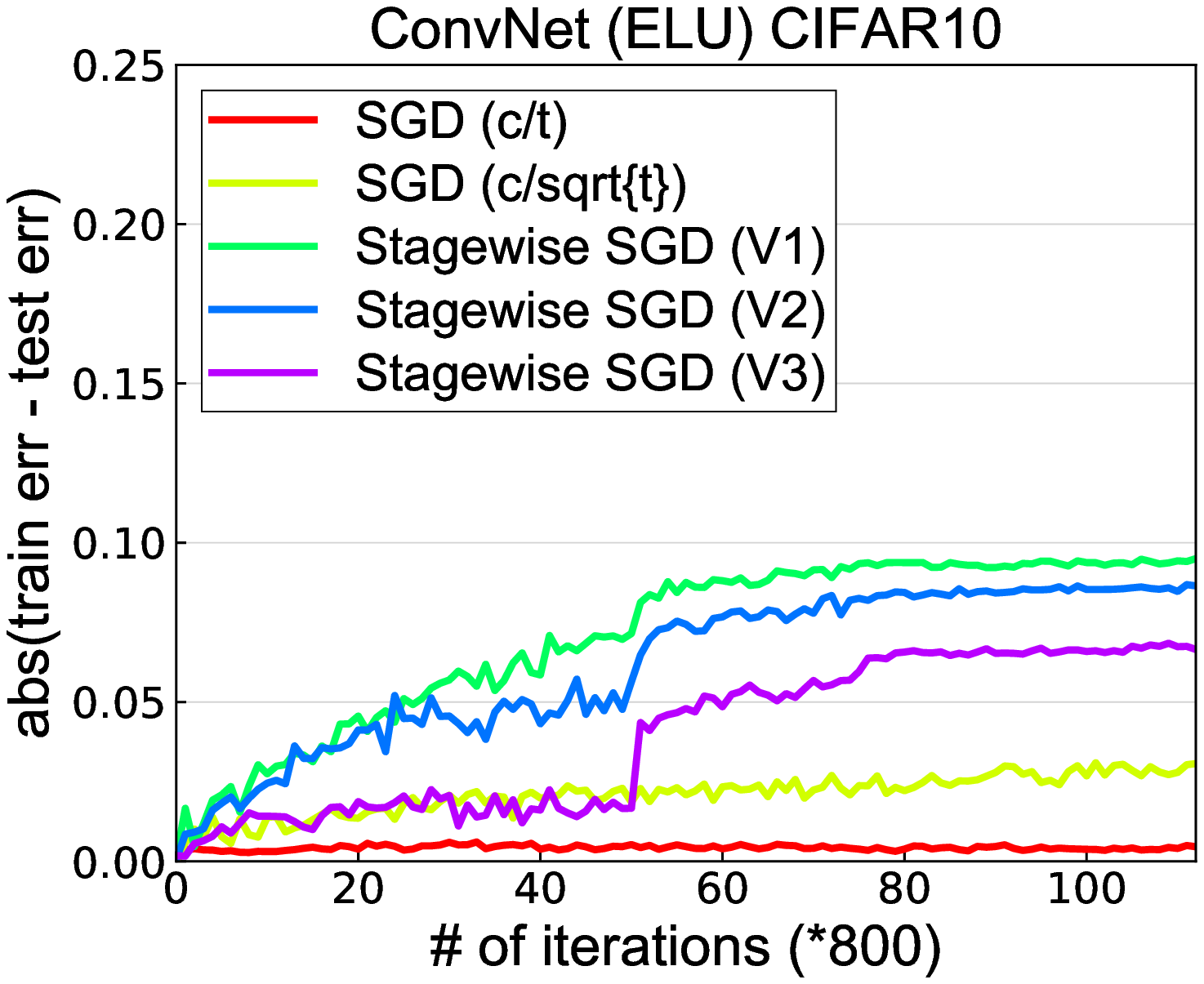}\hspace*{0.15in}
\includegraphics[width=.225\textwidth]{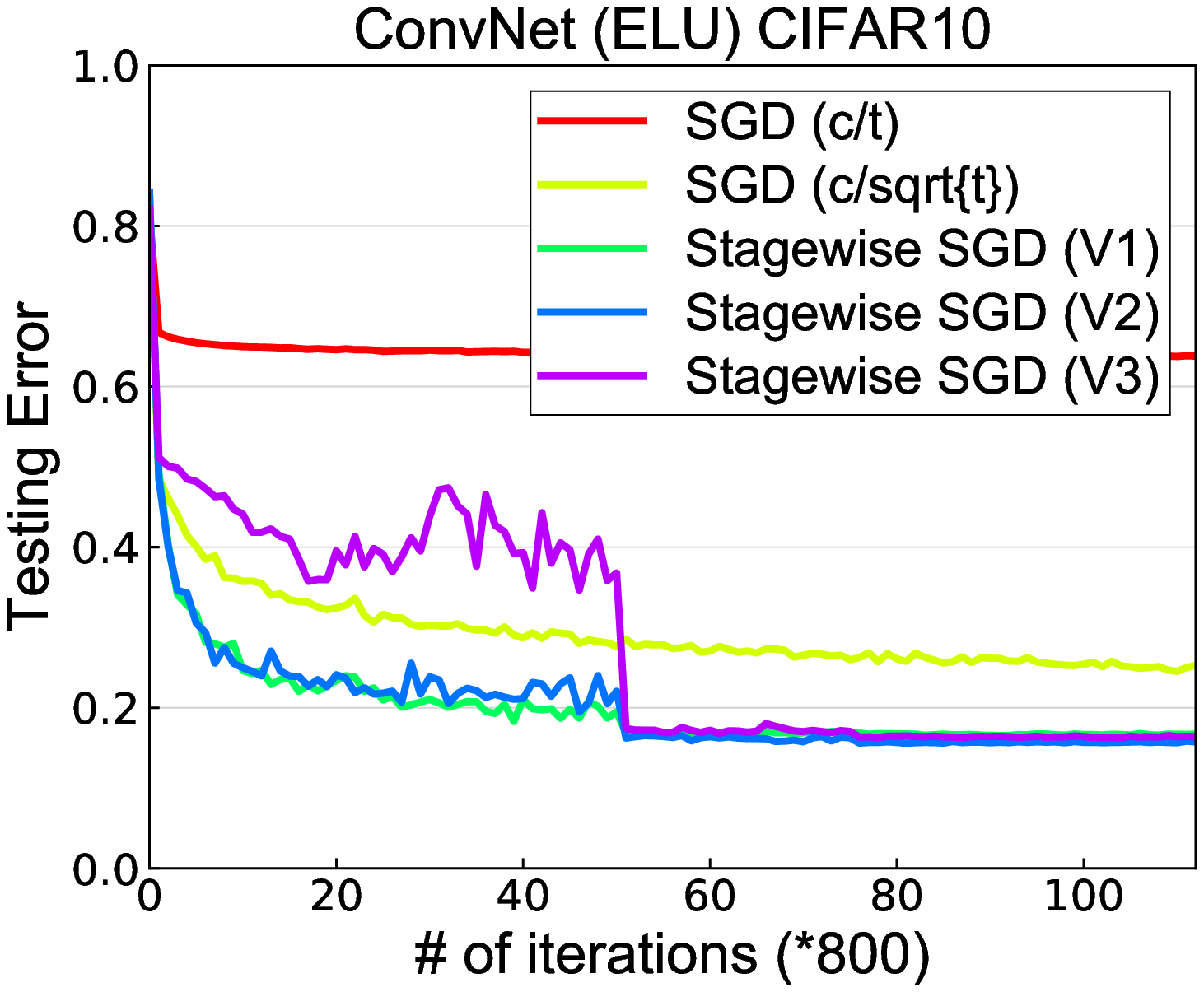}\hspace*{0.15in}
\includegraphics[width=.225\textwidth]{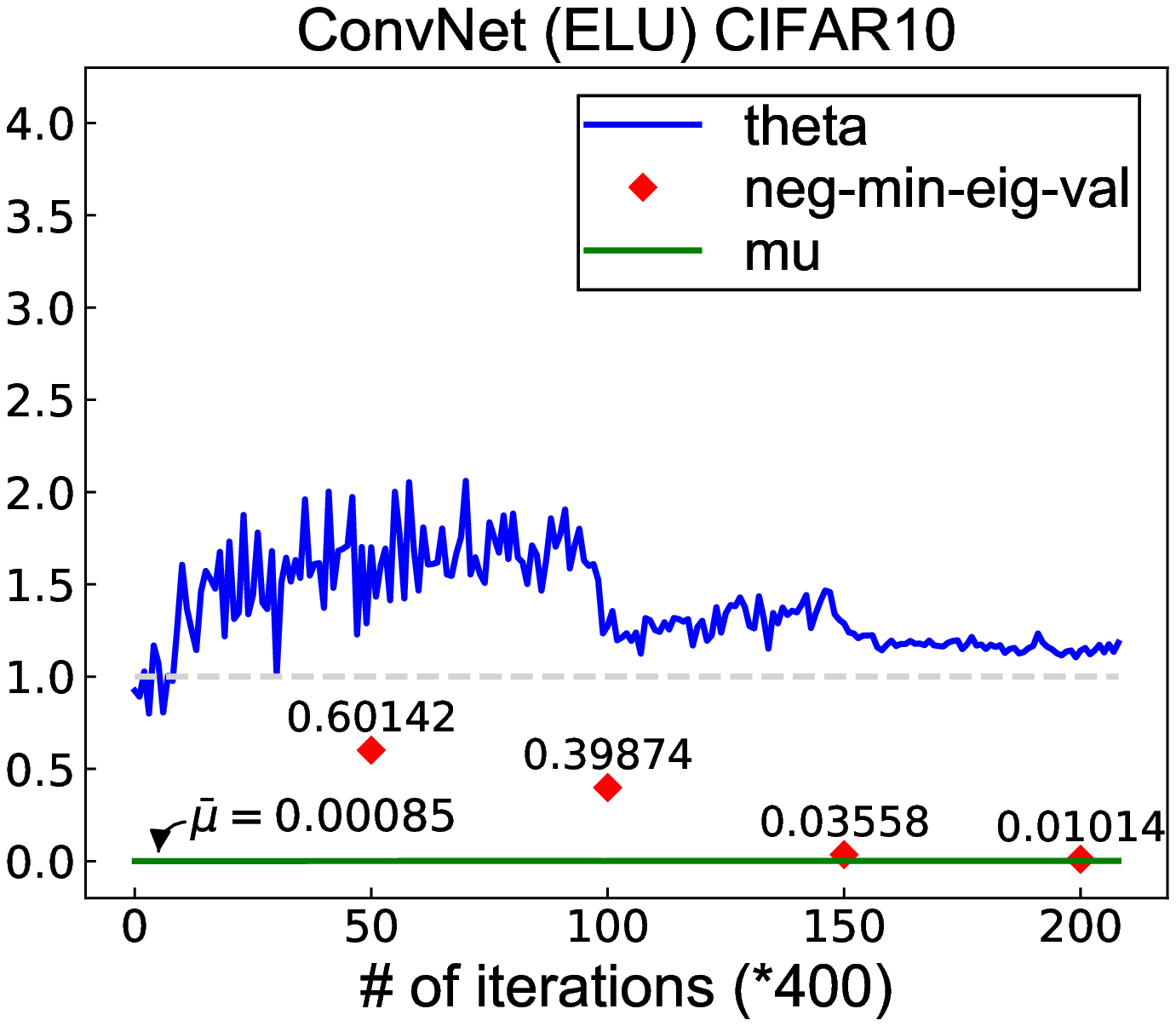}
\vspace*{-0.1in}
\caption{From left to right: training error, generalization error, testing error and verifying assumptions for stagewise learning of ResNets and ConvNet using ELU without weight decay. 
         }
\label{fig:resnet_convnet_ELU_wo_L2}
\vspace*{-0.2in}
\end{figure*}

\begin{figure*}[t]
\centering
\includegraphics[width=.225\textwidth]{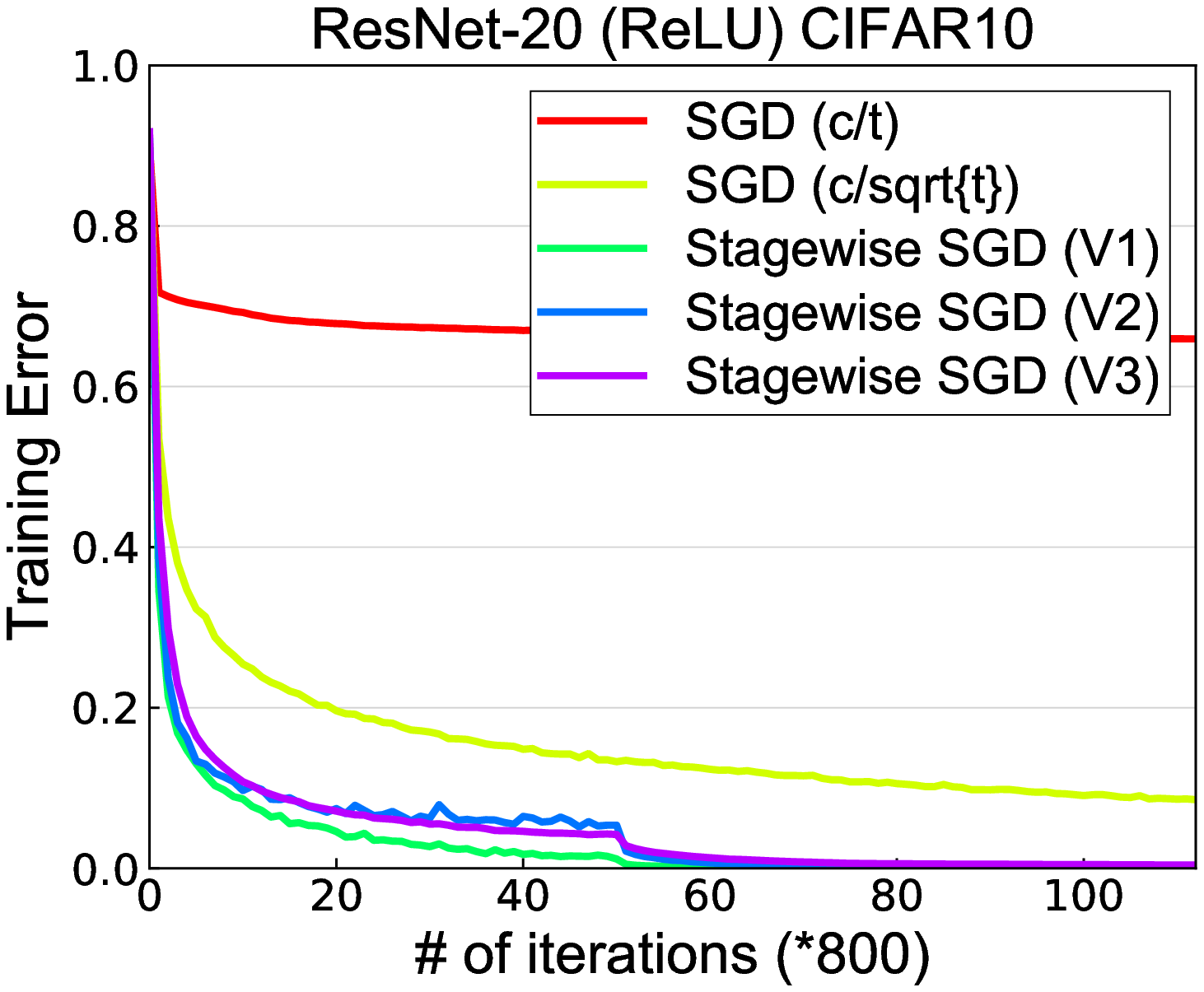}\hspace*{0.15in}
\includegraphics[width=.225\textwidth]{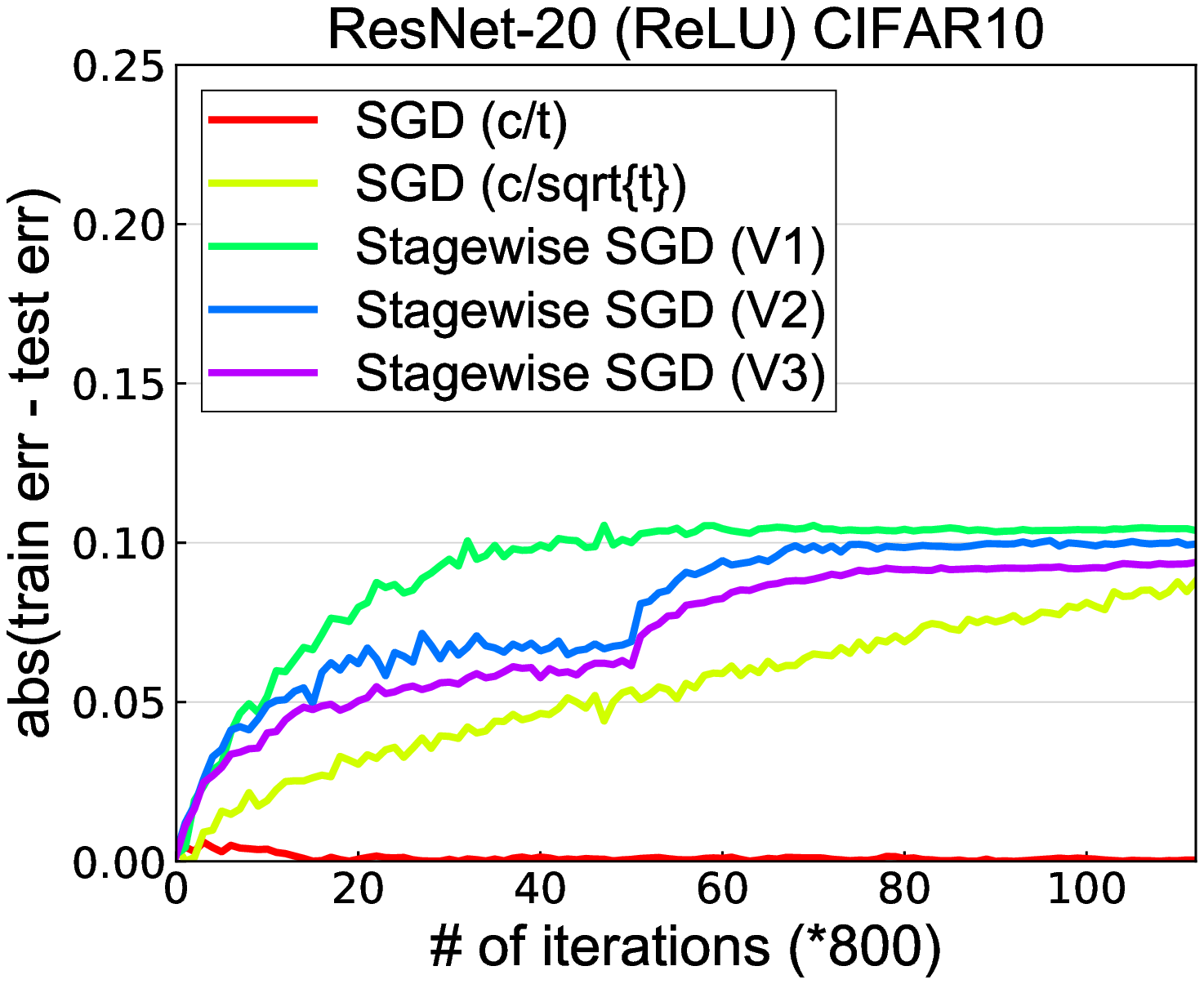}\hspace*{0.15in}
\includegraphics[width=.225\textwidth]{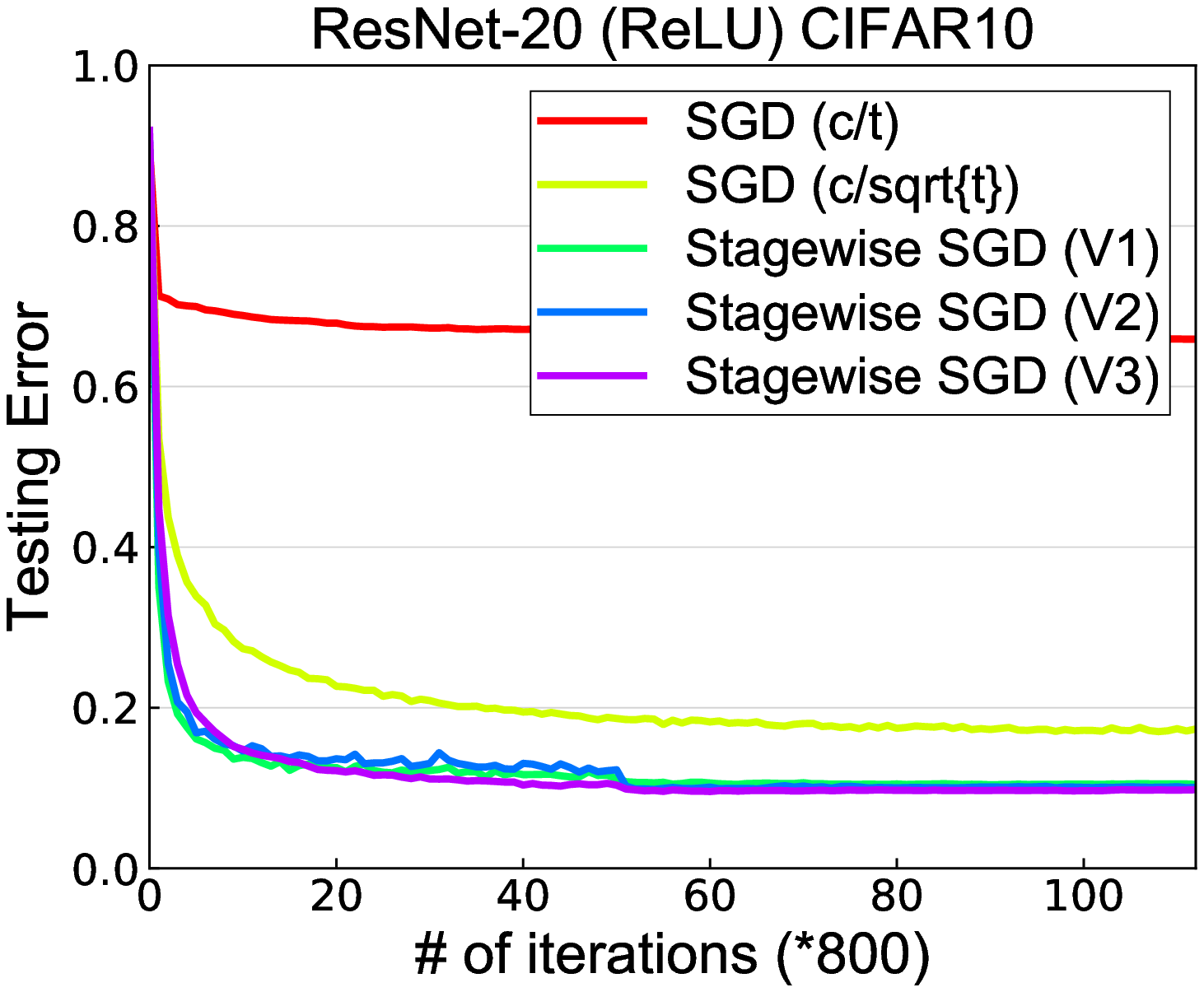}\hspace*{0.15in}
\includegraphics[width=.225\textwidth]{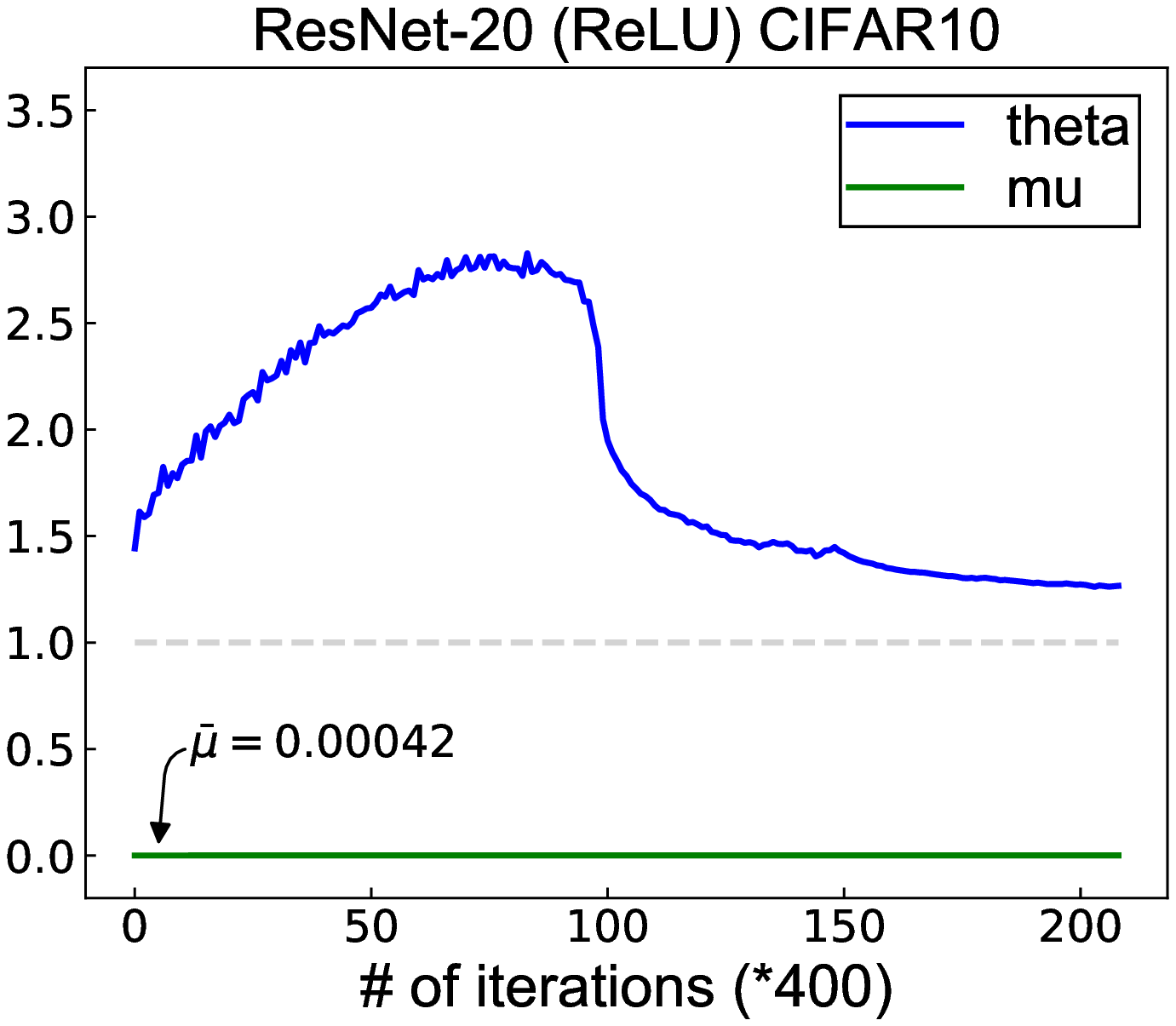} 

\includegraphics[width=.225\textwidth]{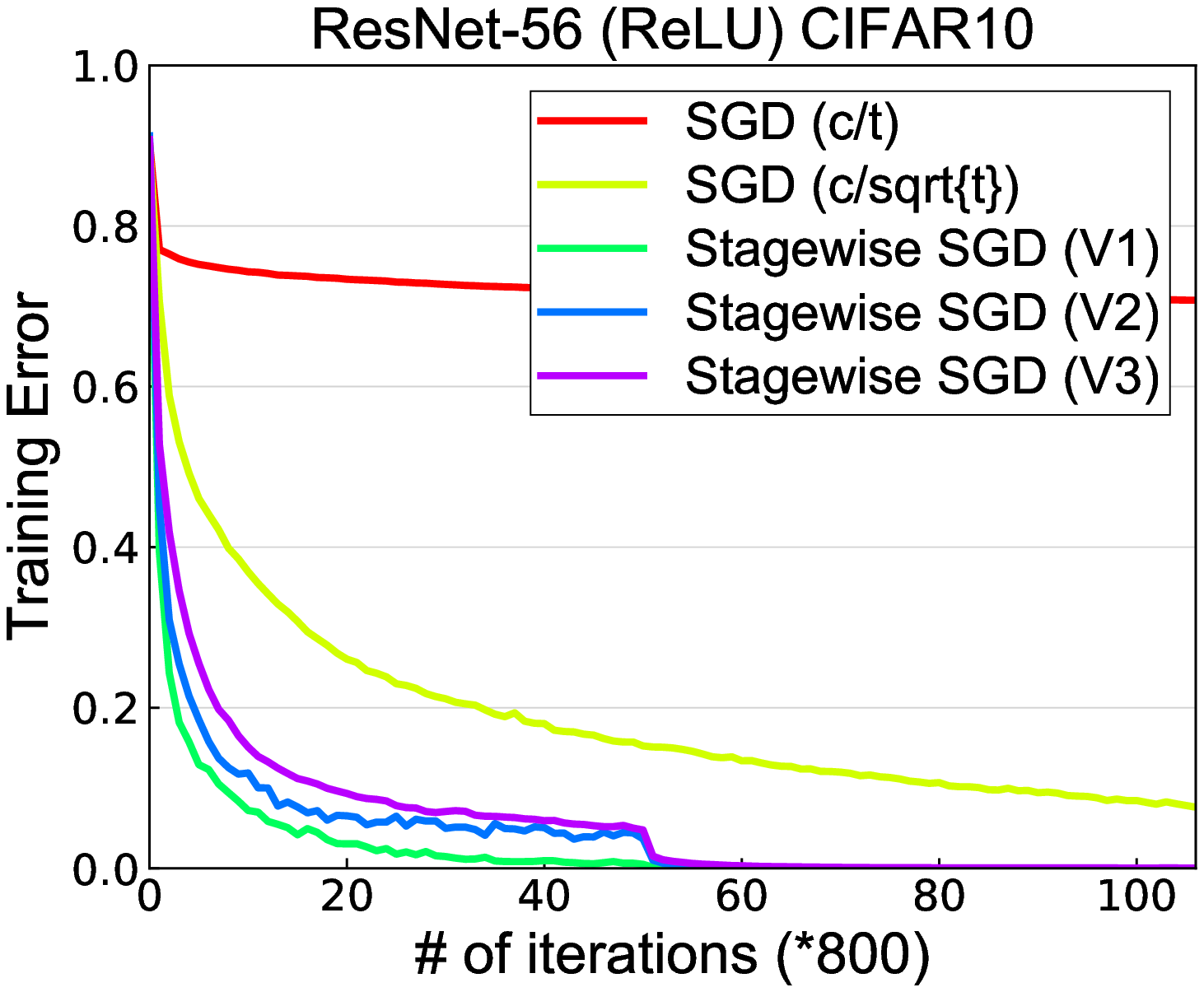}\hspace*{0.15in}
\includegraphics[width=.225\textwidth]{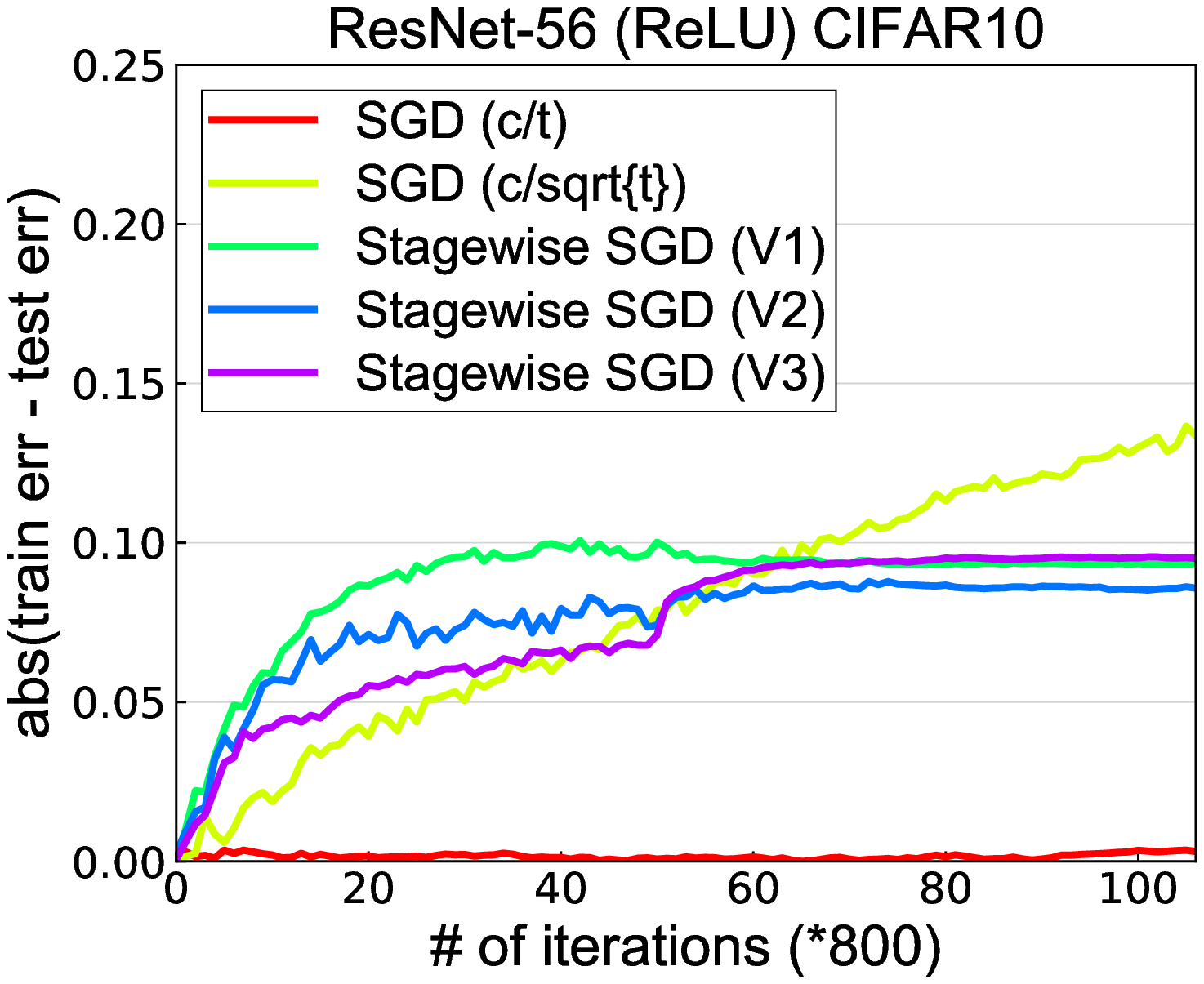}\hspace*{0.15in}
\includegraphics[width=.225\textwidth]{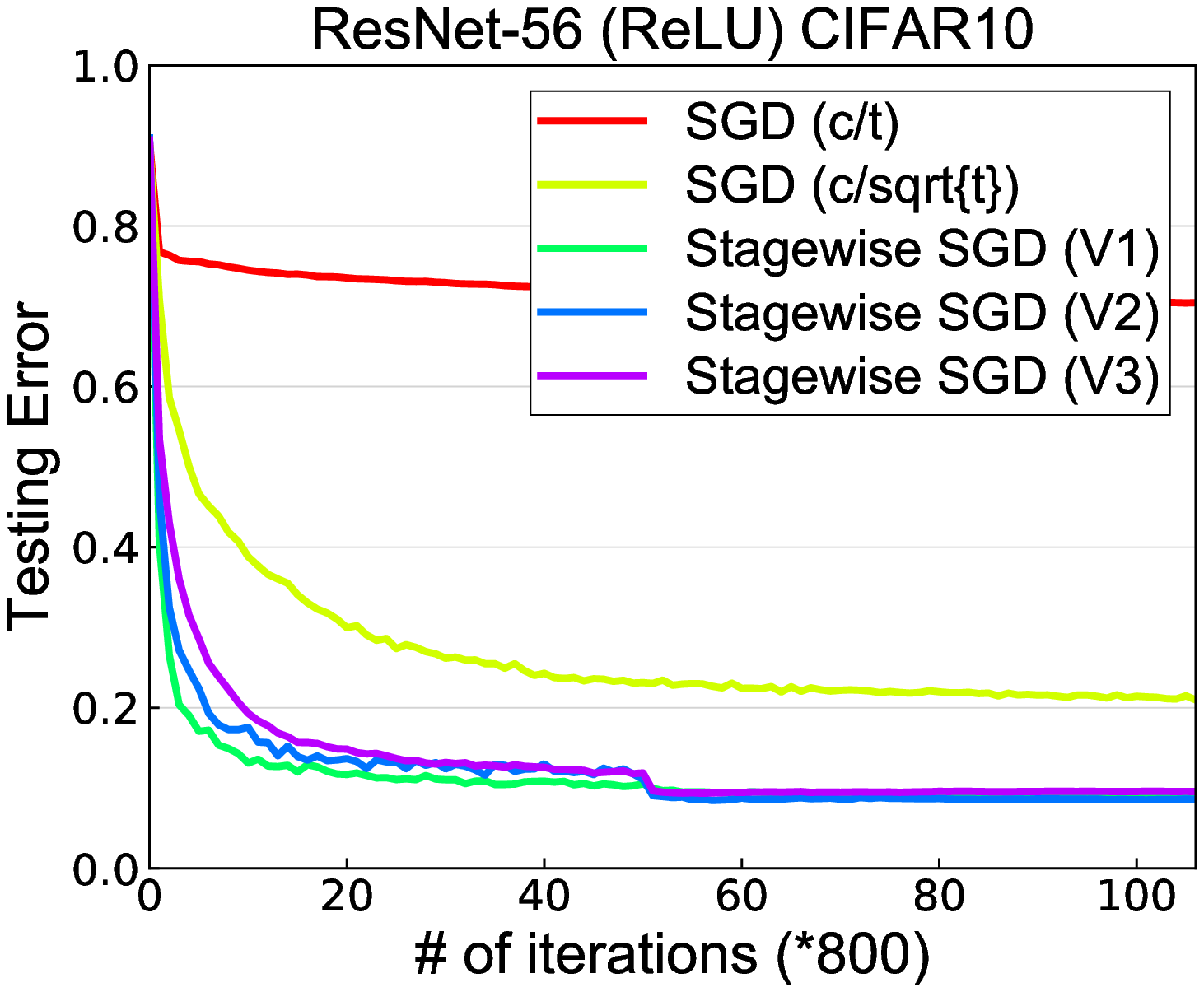}\hspace*{0.15in}
\includegraphics[width=.225\textwidth]{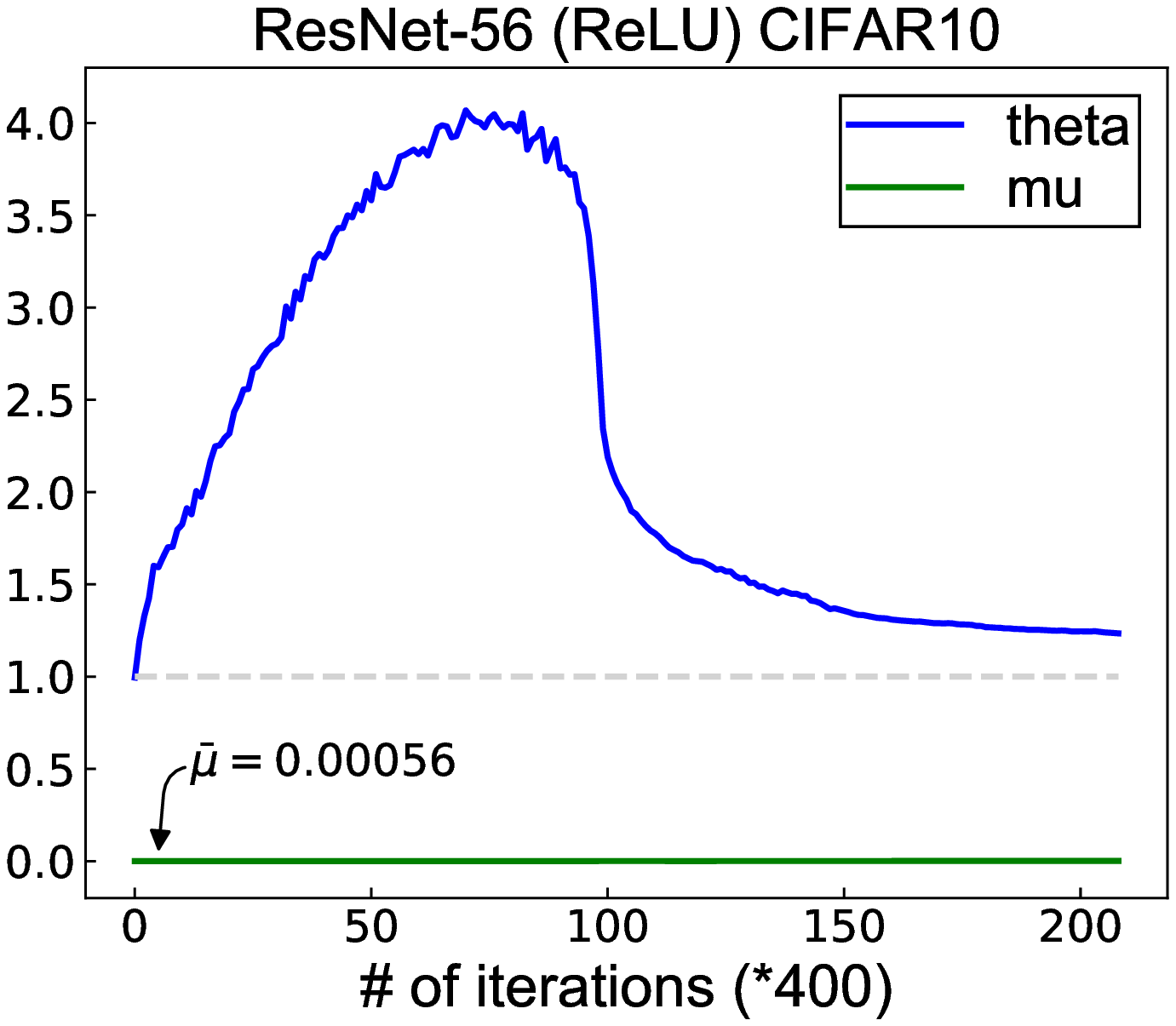}

\includegraphics[width=.225\textwidth]{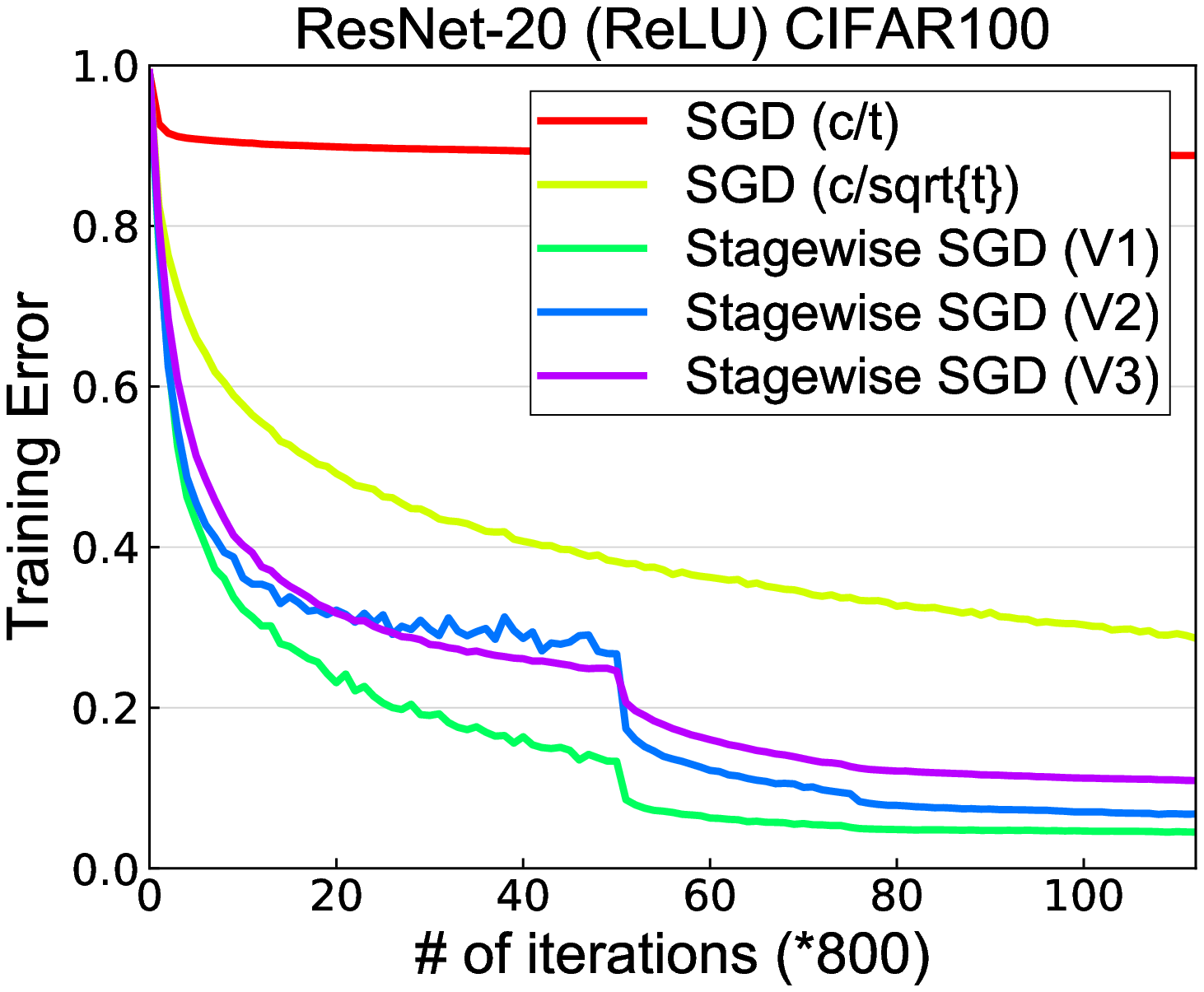}\hspace*{0.15in}
\includegraphics[width=.225\textwidth]{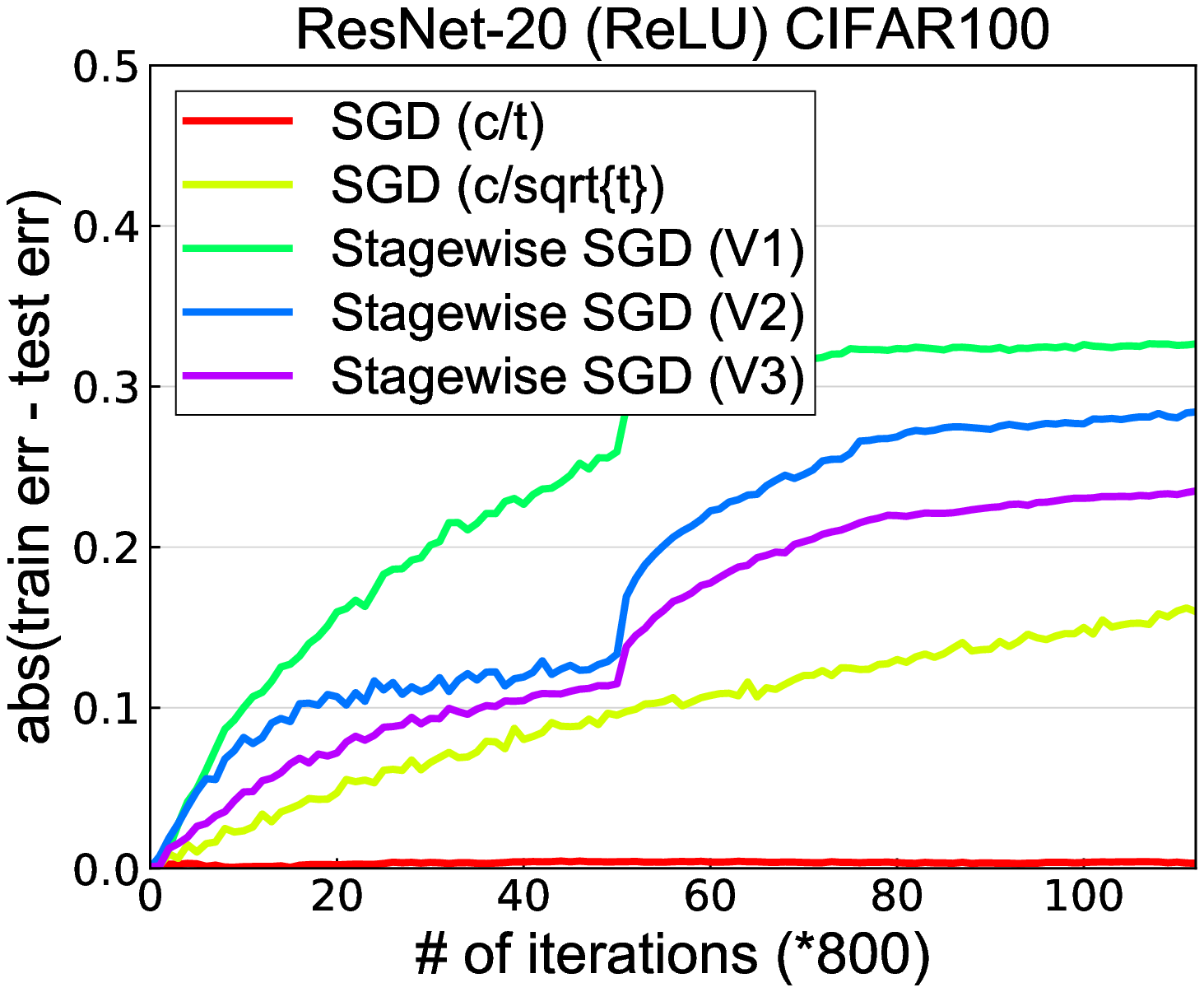}\hspace*{0.15in}
\includegraphics[width=.225\textwidth]{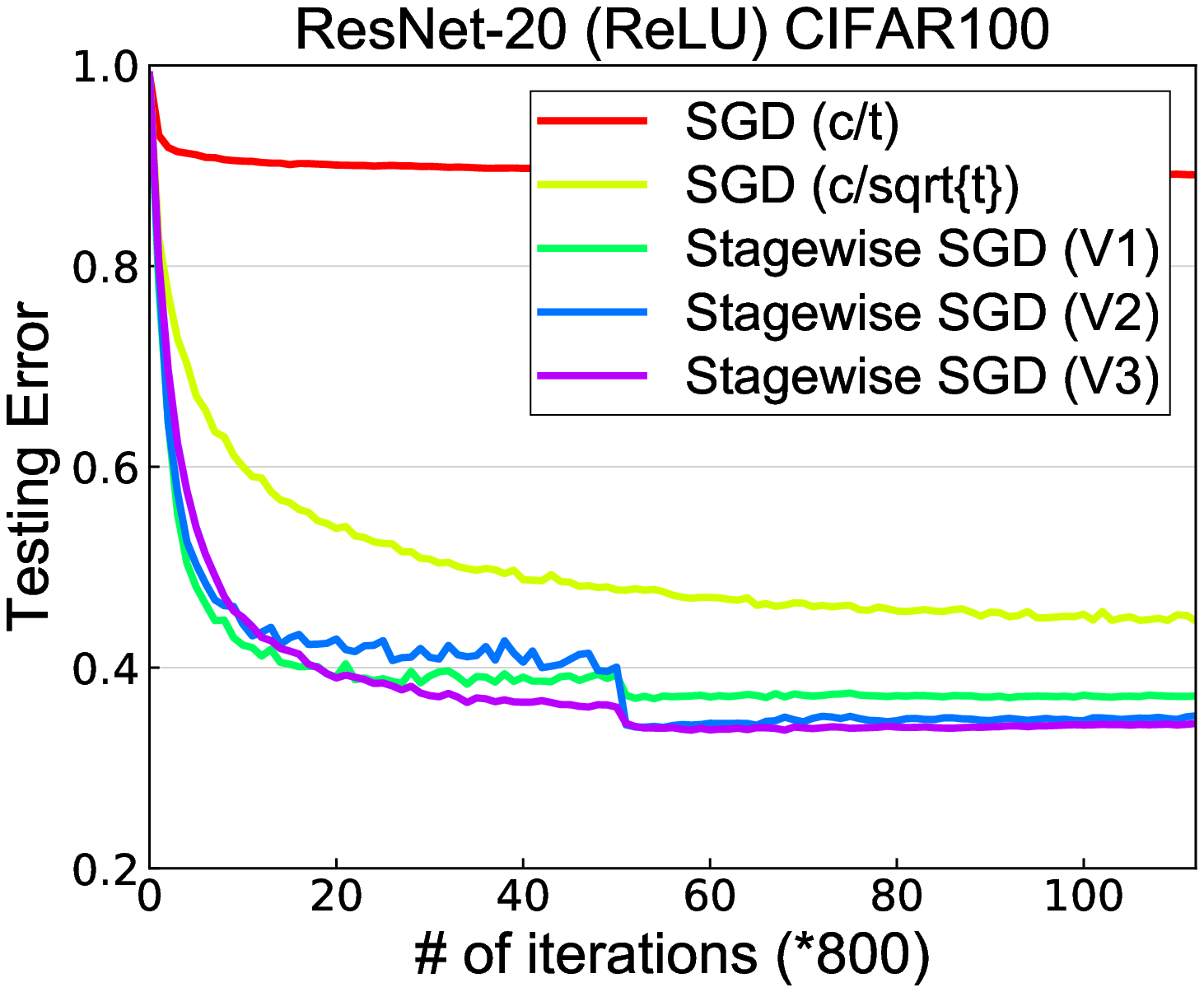}\hspace*{0.15in}
\includegraphics[width=.225\textwidth]{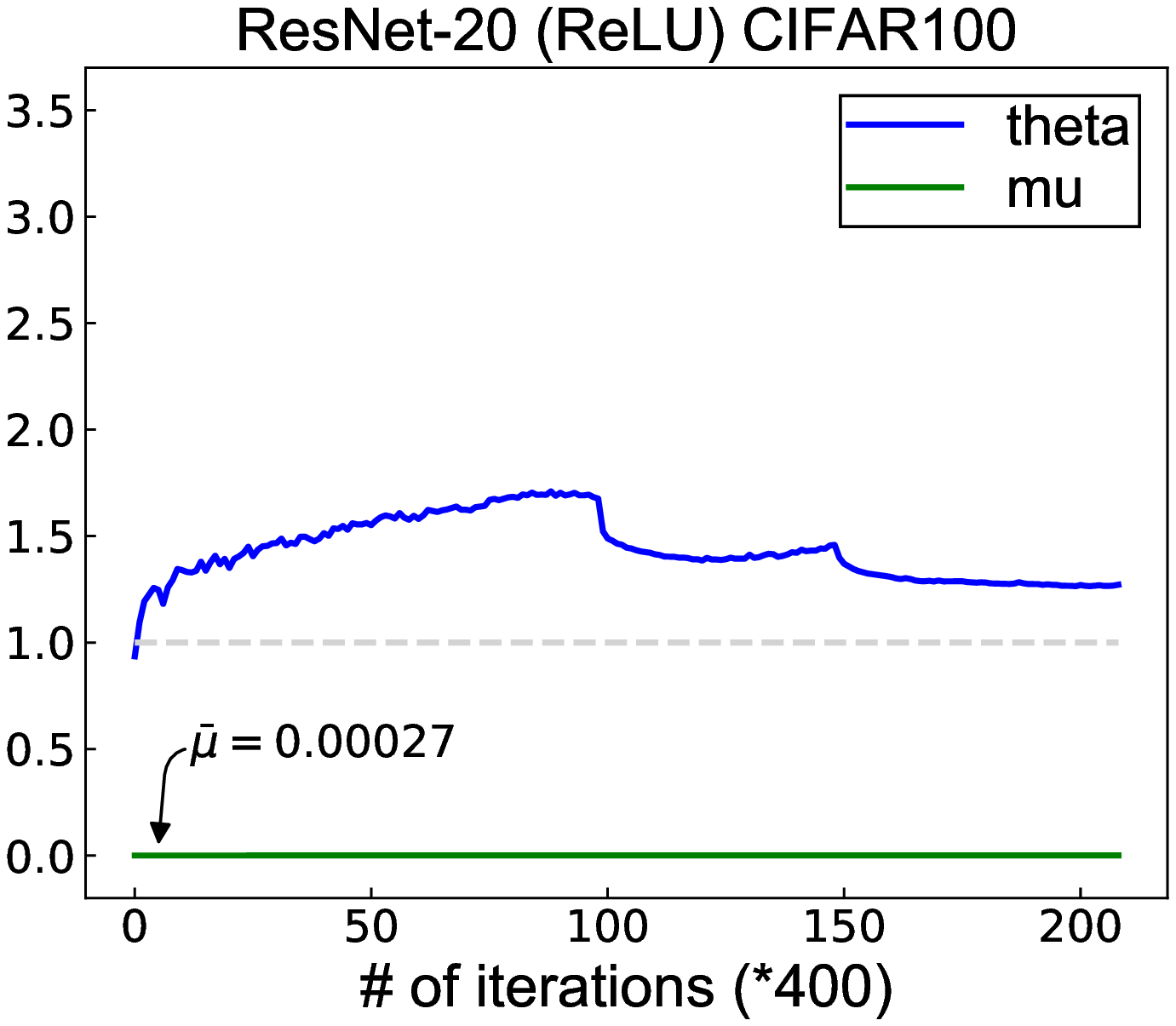}

\includegraphics[width=.225\textwidth]{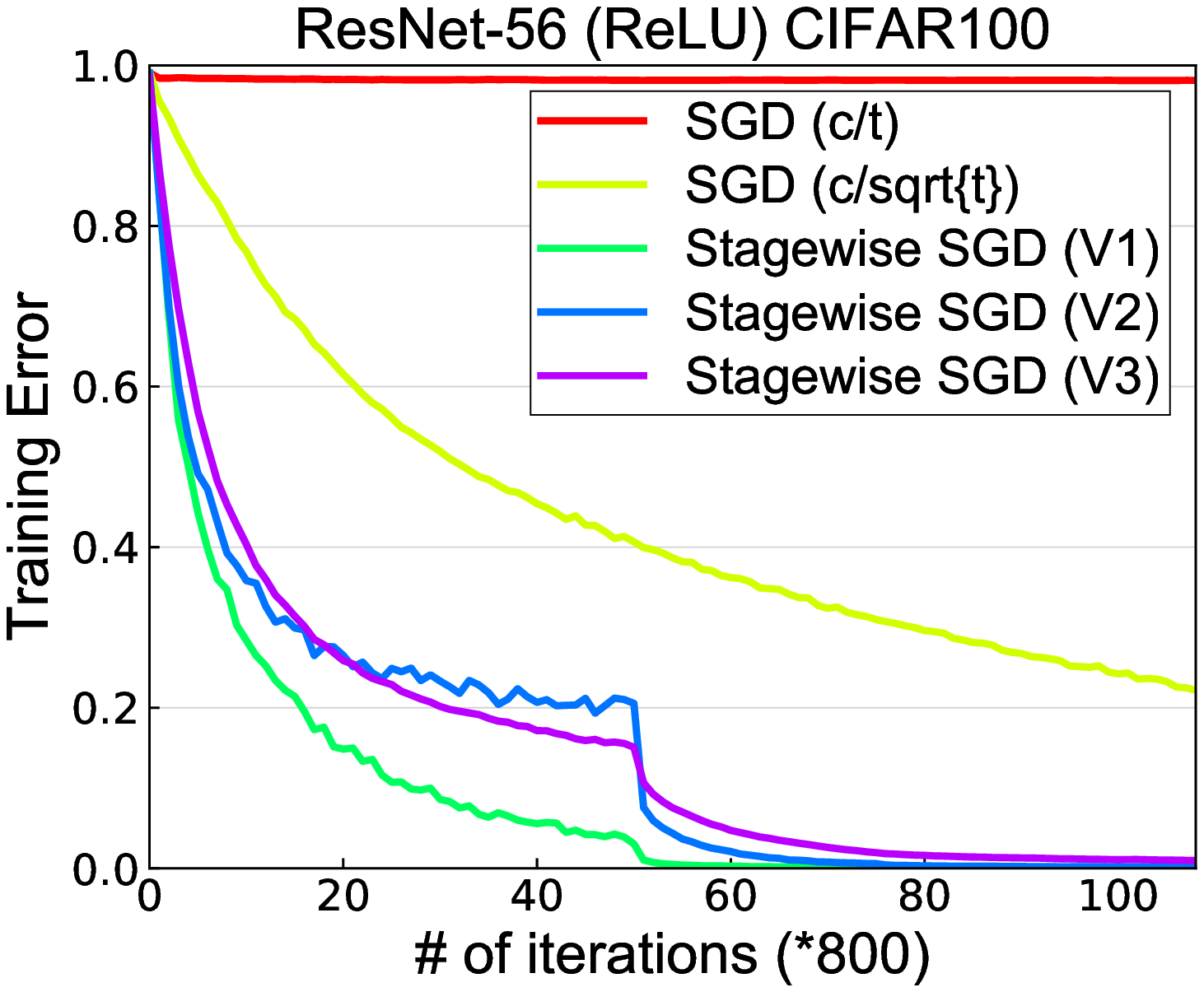}\hspace*{0.15in}
\includegraphics[width=.225\textwidth]{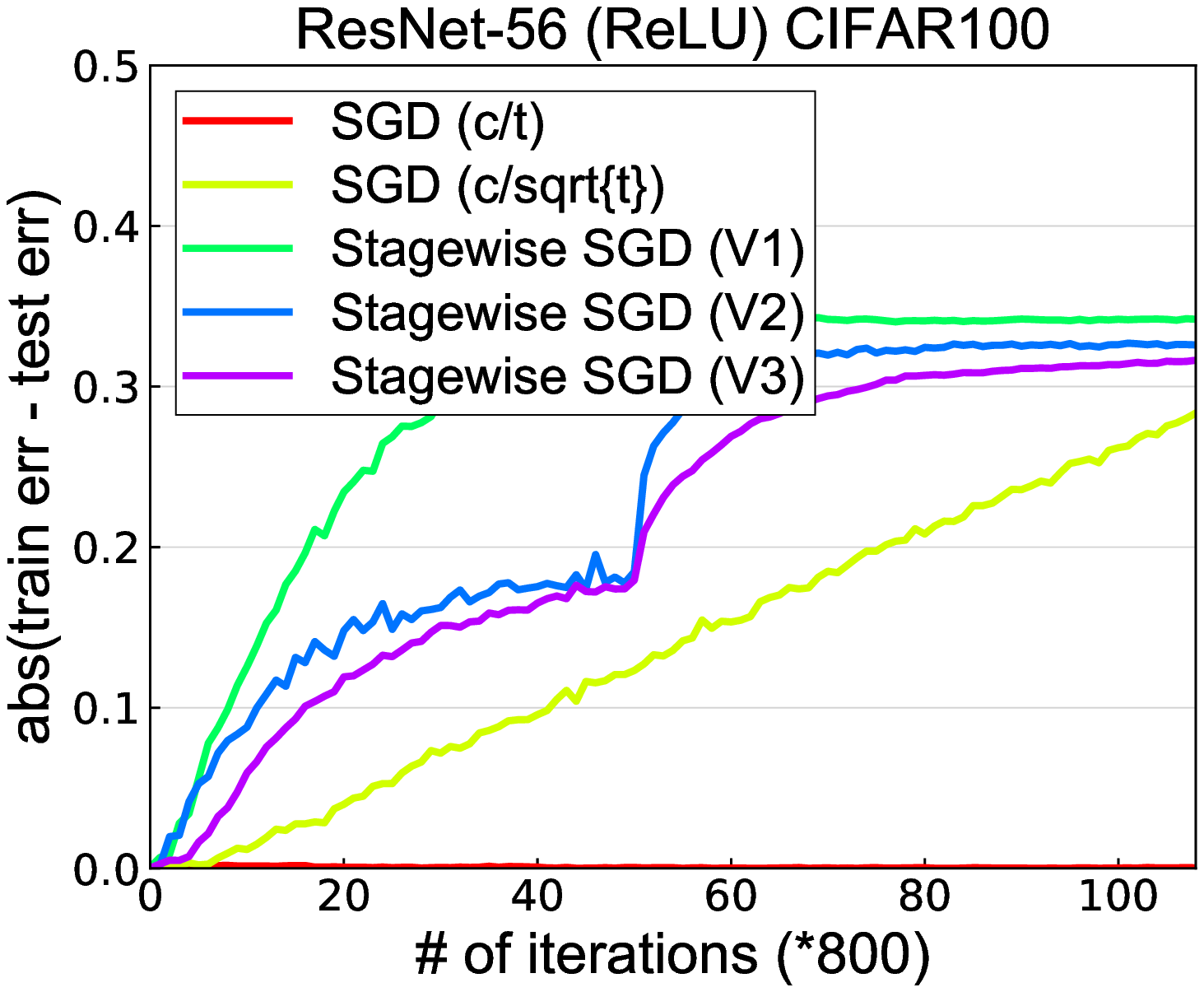}\hspace*{0.15in}
\includegraphics[width=.225\textwidth]{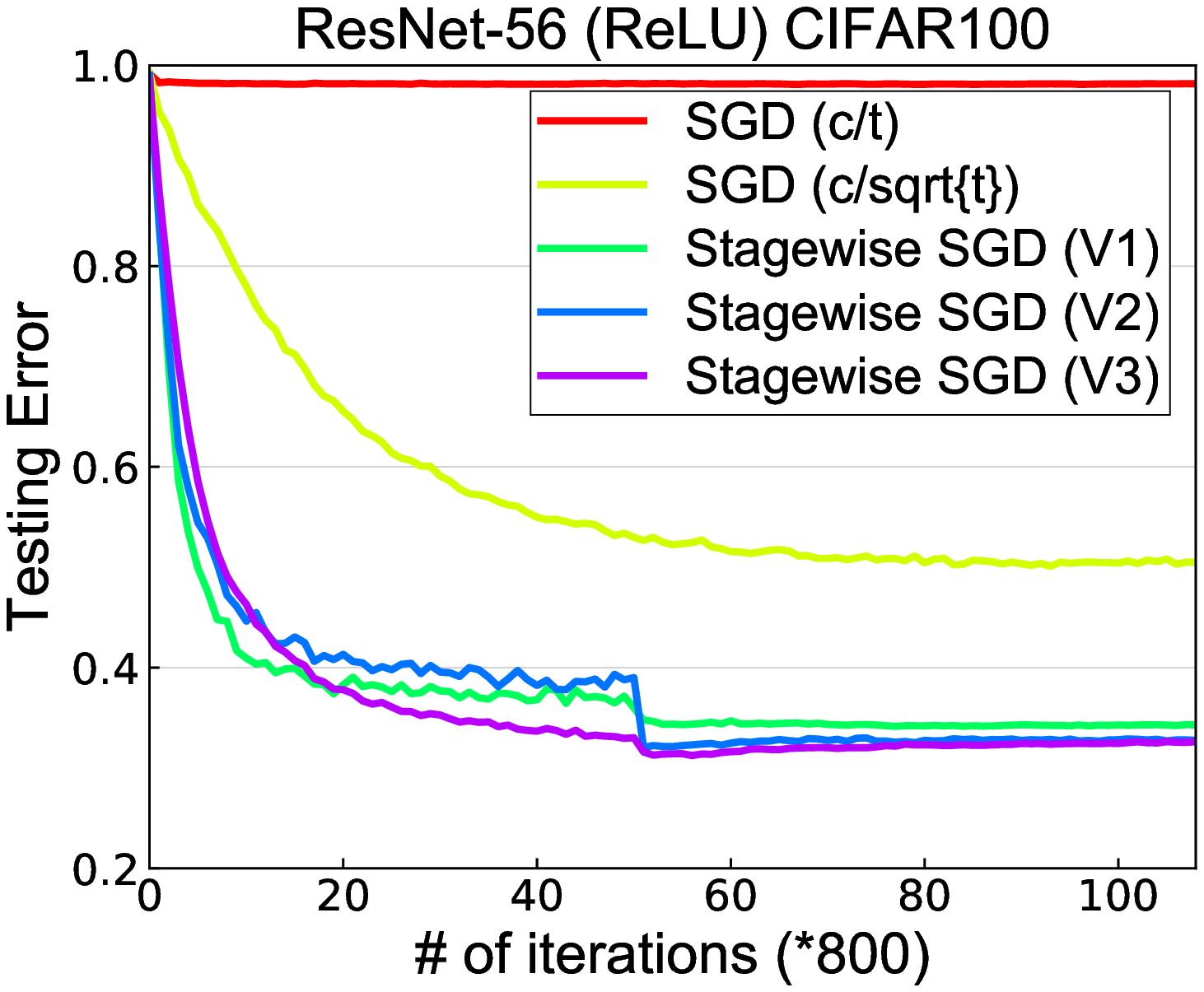}\hspace*{0.15in}
\includegraphics[width=.225\textwidth]{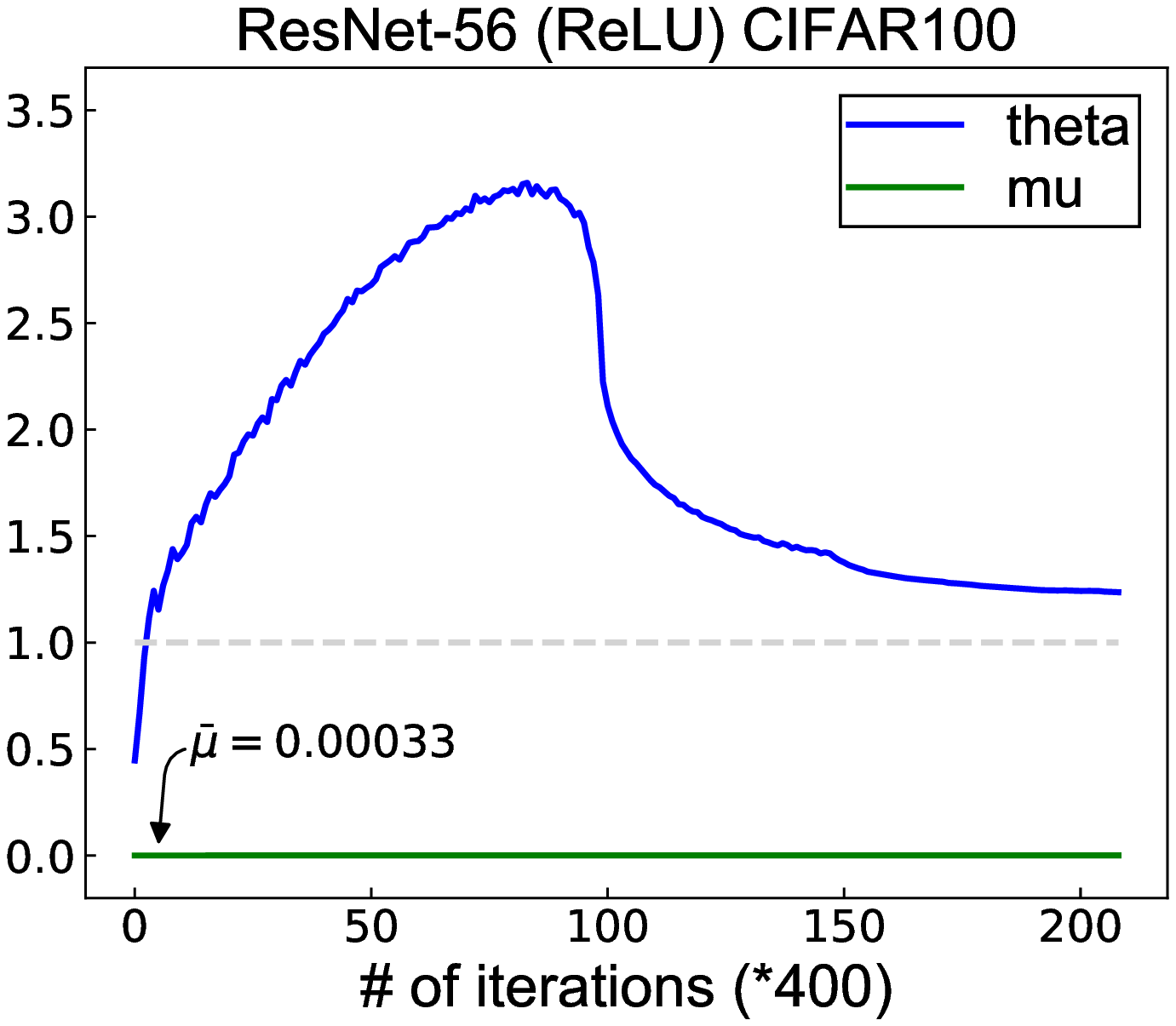}

\includegraphics[width=.225\textwidth]{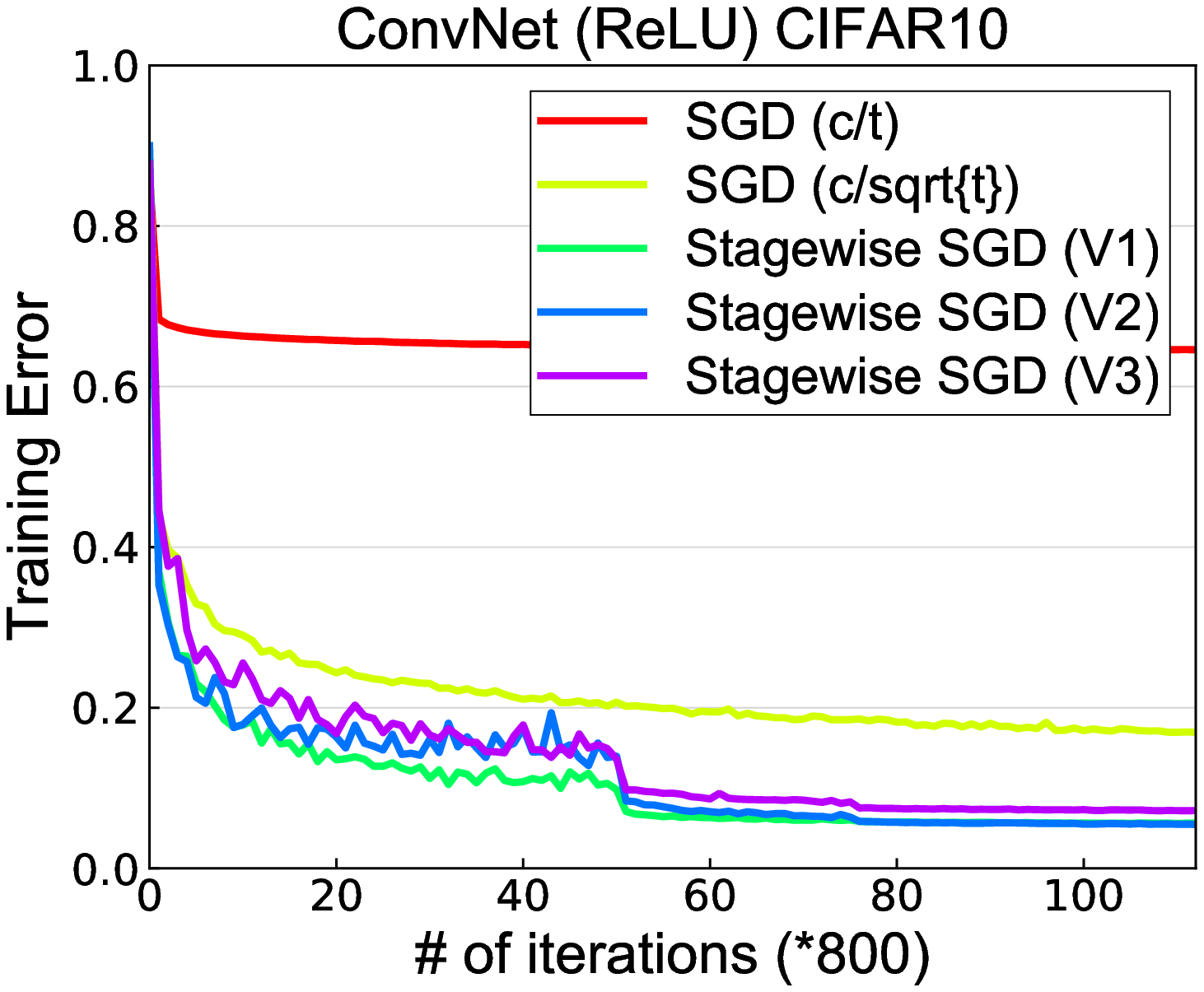}\hspace*{0.15in}
\includegraphics[width=.225\textwidth]{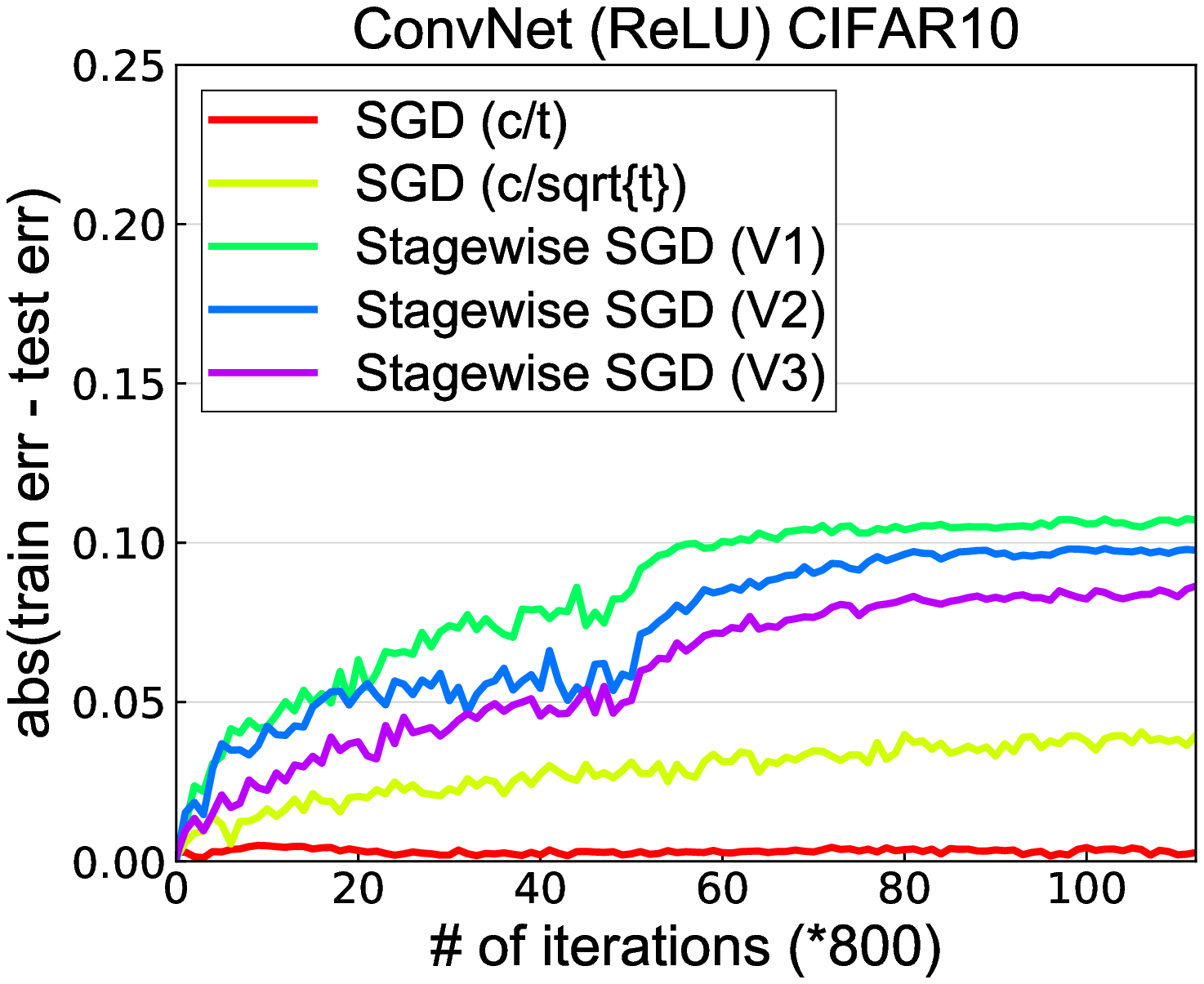}\hspace*{0.15in}
\includegraphics[width=.225\textwidth]{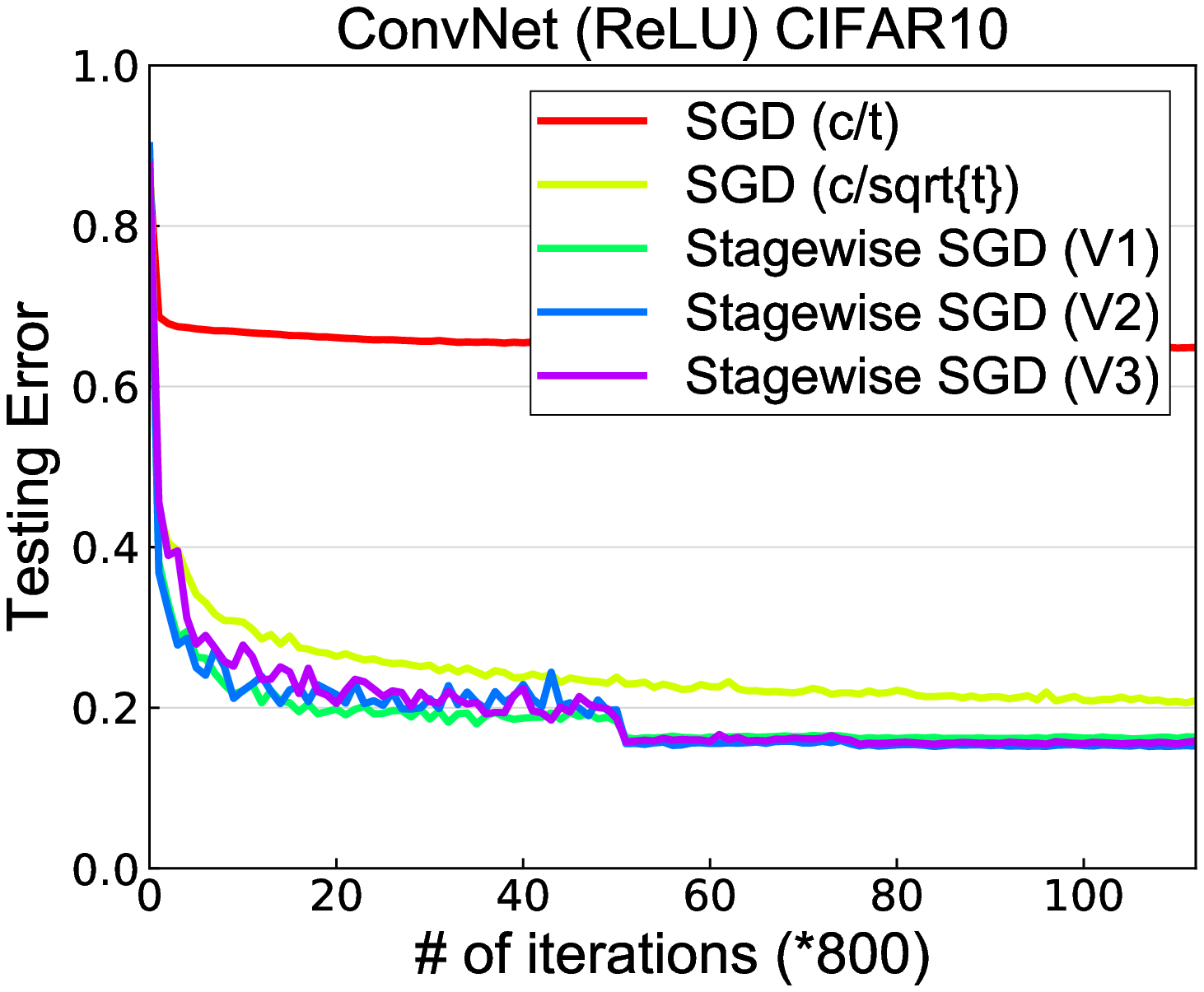}\hspace*{0.15in}
\includegraphics[width=.225\textwidth]{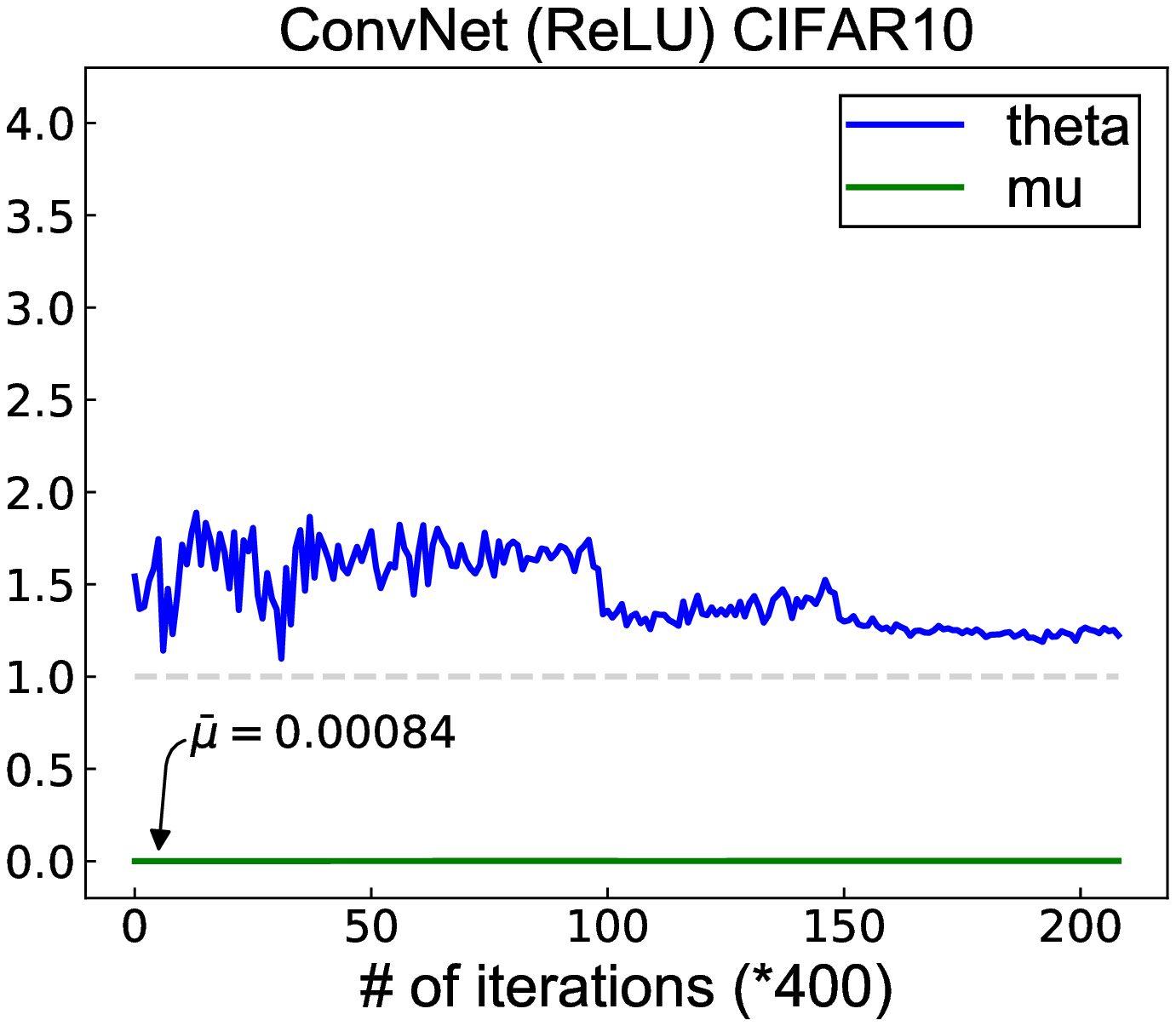}

\vspace*{-0.1in}
\caption{From left to right: training, generalization and testing error, and verifying assumptions for stagewise learning of ResNets and ConvNet using RELU without weight decay.  We do not compute the minimum eigen-value of the Hessian matrix because the function is non-smooth, which renders the Lanczos method based on finite-difference of gradients not converging.}
\label{fig:resnet_convnet_RELU_wo_L2}
\vspace*{-0.2in}
\end{figure*}

\subsection{Learning with Non-convex Deep Learning}\label{sec:experiment_deep_learning}

In this subsection, we focus experiments on non-convex deep learning, and also include in the supplement some experimental results of {\start} for convex functions that satisfy the PL condition. 
The numerical experiments in this subsection mainly serve two purposes: 
(i) verifying that using different algorithmic choices in practice (e.g, regularization, averaged solution) is consistent with the provided theory; 
(ii) verifying the assumptions made for non-convex objectives in our analysis in order to explain the great performance of stagewise learning. 

We compare stagewise learning with different algorithmic choices against  SGD using two polynomially decaying step sizes (i.e.,  $O(1/t)$ and $O(1/\sqrt{t})$). For stagewise learning, we consider the widely used version that corresponds to {\start} with $\gamma = \infty$ and the returned solution at each stage being the last solution, which is denoted as stagewise SGD (V1). We also implement other two variants of {\start} that solves a regularized function at each stage (corresponding to $\gamma<\infty$) and uses the last solution or the averaged solution for the returned solution at each stage. We refer to these variants as stagewise SGD (V2) and (V3), respectively.  

We conduct experiments on two datasets CIFAR-10, -100 using different neural network structures, including residual nets and convolutional neural nets without skip connection. Two residual nets  namely ResNet20 and ResNet56~\citep{DBLP:conf/cvpr/HeZRS16} are used for CIFAR-10 and CIFAR-100. For each network structure, we use two types of activation functions, namely RELU and  ELU ($\alpha=1$)~\citep{DBLP:journals/corr/ClevertUH15}. ELU is smooth that is consistent with our assumption. Although RELU is non-smooth, we would like to show that the provided theory can also explain the good performance of stagewise SGD.   For stagewise SGD on CIFAR datasets, we use the same stagewise step size strategy as in \citep{DBLP:conf/cvpr/HeZRS16}, i.e.,  the step size is decreased by $10$ at 40k, 60k iterations. For all algorithms, we select the best initial step size from $10^{-3}\sim 10^3$ and the best regularization  parameter $1/\gamma$ of  stagewise SGD (V2, V3) from  $0.0001\sim 0.1$ by cross-validation based on performance on a validation data.  
We report the results for using ResNets and ConvNet without weight decay in this section, and include the results with weight decay (i.e., including $5 * 10^{-4}||\w||_2^2$ regularization) in the Appendix. 

The training error, generalization error and testing error are shown in Figure~\ref{fig:resnet_convnet_ELU_wo_L2} and Figure~\ref{fig:resnet_convnet_RELU_wo_L2} where we employ ELU and RELU as activation functions, respectively. 
We can see that SGD with a decreasing step size converges slowly, especially SGD with a step size proportional to $1/t$. It is because that the initial step size of SGD ($c/t$) is selected as a small value less than 1. 
We observe that when using a large step size it cannot lead to convergence.  
In terms of different algorithmic choices of {\start}, we can see that using an explicit regularization  as in V2, V3 can help reduce the generalization error that is consistent with theory, but also slows down the training a little. Using an averaged solution as the returned solution in V3 can further reduce the generalization error but also further slow downs the training. Overall, stagwise SGD (V2) achieves the best tradeoff in training error convergence and generalization error, which leads to the best testing error.


Finally, we verify the assumptions about the non-convexity made in Section~\ref{sec:5}. To this end, on a selected $\w_t$ we compute the value of $\theta$, i.e., the ratio of $\nabla F_\S(\w_t)^{\top}(\w_t - \w^*_\S)$ to $F_\S(\w_t) - F_\S(\w^*_\S)$ as in~(\ref{eqn:qc}), and the value of $\mu$, i.e., the ratio of $F_\S(\w_t) - F_\S(\w^*_\S)$ to $2\|\w_t - \w^*_\S\|^2$ as in~(\ref{eqn:EB}).  For $\w^*_\S$, we use the solution found by stagewise SGD (V1) after a large number of iterations (200k), which gives a small objective value close to zero. 
We select 200 points during the process of training by stagewise SGD (V1) across all stages, and plot the curves for the values of $\theta$ and $\mu$ averaged over 5 trials in the most right panel of Figure~\ref{fig:resnet_convnet_ELU_wo_L2} and Figure~\ref{fig:resnet_convnet_RELU_wo_L2}. 
We can clearly see that our assumptions about $\mu\ll 1$ and one-point weakly quasi-convexity with $\theta>1$ are satisfied. Hence, the provided theory for stagewise learning is applicable.

We also compute the minimum eigen-value of the Hessian on several selected solutions by the Lanczos  method. The Hessian-vector product is approximated by the finite-difference using gradients. The negative value of minimum eigen-value is marked as $\diamond$ in the same figure of $\theta, \mu$. We can see that the assumption about $\rho\leq O(\mu)$ seems not hold for learning deep neural networks.

\subsection{Learning with Convex Loss Functions}
We consider minimizing an empirical  loss under an $\ell_1$ norm constraint, i.e., 
\begin{align*}
\min_{\|\w\|_1\leq B} F_\S(\w) = \frac{1}{n}\sum_{i=1}^n\ell(y_i, \w^{\top}\x_i),
\end{align*}
where $\x_i\in\R^d$ denotes the feature vector of the $i$-th example and $y_i\in\{1, -1\}$ is the corresponding label for classification or is continuous for regression. 
Two loss functions are used for classification, namely squared hinge loss $\ell(y, z) = \max(0, 1 -yz)^2$ and logistic loss $\ell(y, z) = \log(1+\exp(-yz))$.  Two loss functions are also used for  regression,  namely square loss $\ell(y, z) = (y_i - \w^{\top}\x_i)^2$ and Huber loss ($\delta=1$): 
\[
\ell_\delta(y, z) = \left\{\begin{array}{ll}\frac{1}{2}z^2& \text{ if } |y - z|\leq \delta,\\ \gamma(| y - z| - \frac{1}{2}\delta)& \text{otherwise},\end{array}\right.
\]
For all considered problems here, the PL condition (or the equivalent the quadratic growth condition) holds~\citep{ICMLASSG}. 

For classification, we use real-sim data from the libsvm website, which has $n=72,309$ total examples and  $d=20,958$ features. For regression, we use the E2006-tfidf data from the libsvm website, which has  $16,087$ (training examples), $3,308$ (testing examples), and $d=150,360$ features. For real-sim data, we randomly select $10,000$ examples for testing, and for E2006, we use the provided testing set. For parameter selection, we also  divide the training examples into two parts, i.e., the validation data and the training data. The size of the validation set is the same as the testing set. We run the analyzed {\start} algorithm with $\gamma =\infty$ and the averaged solution as a returned solution at each stage. The number of iterations per-stage is determined according to the performance on the validation data, i.e., when the error on the validation data does not change significantly after 1000 iterations we terminate one stage and restart the next stage. For classification, the insignificant change means the error rate does not improve by $0.01$ and for regression it means the relative improvement on root mean square error does not change by a factor of $0.1$. The initial step sizes are tuned to get the fast training convergence and the value of $B$ are is tuned based on the performance on the validation data. 

The results averaged over 5 random trials  are shown in Figure~\ref{fig:3}. They clearly show the superior performance of {\start} comparing with SGD with polynomially decaying step size (i.e., $O(1/t)$ and $O(1/\sqrt{t})$).

\begin{figure*}[t]
\centering
\includegraphics[width=.29\textwidth]{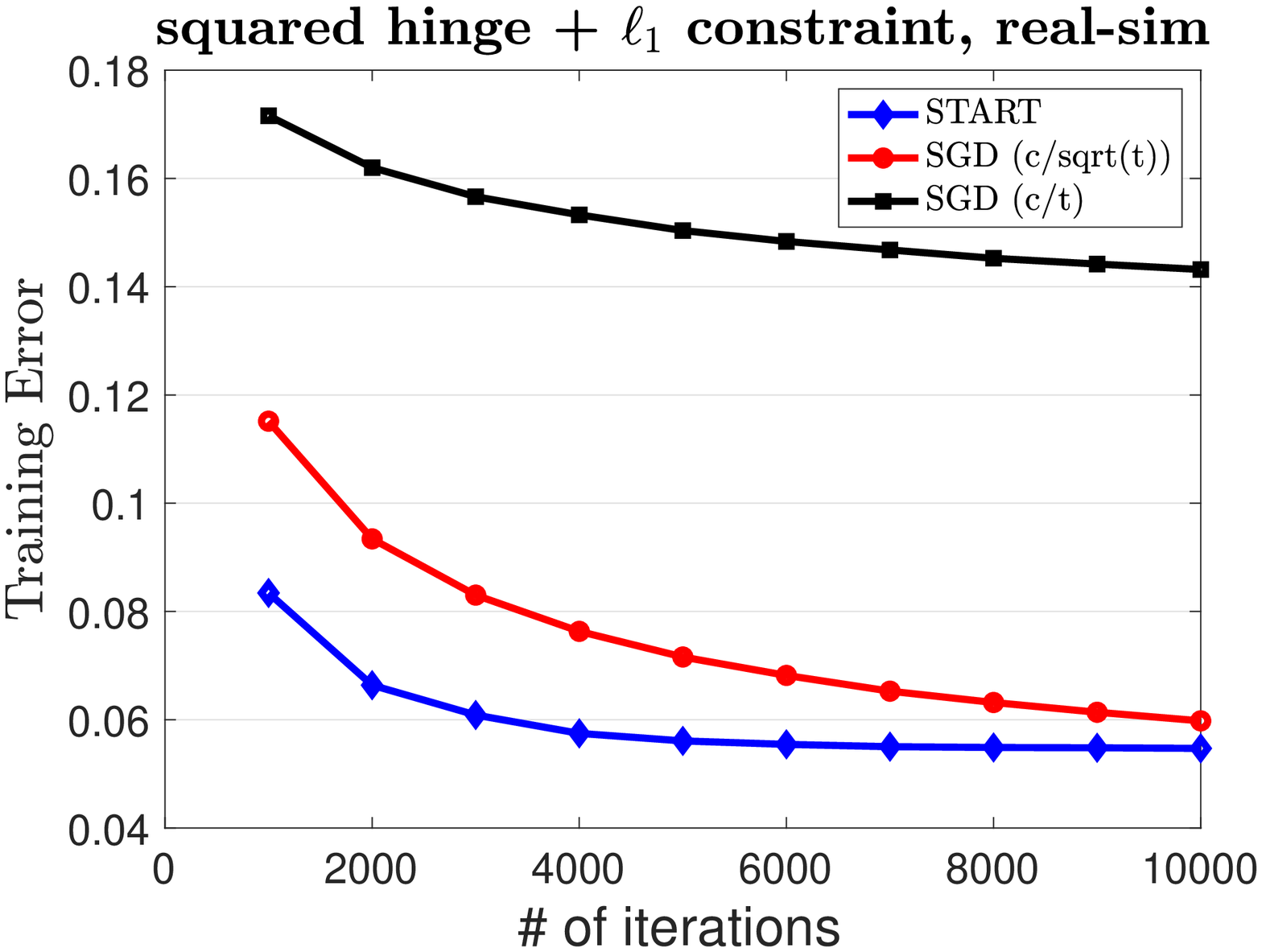}\hspace*{0.1in}
\includegraphics[width=.29\textwidth]{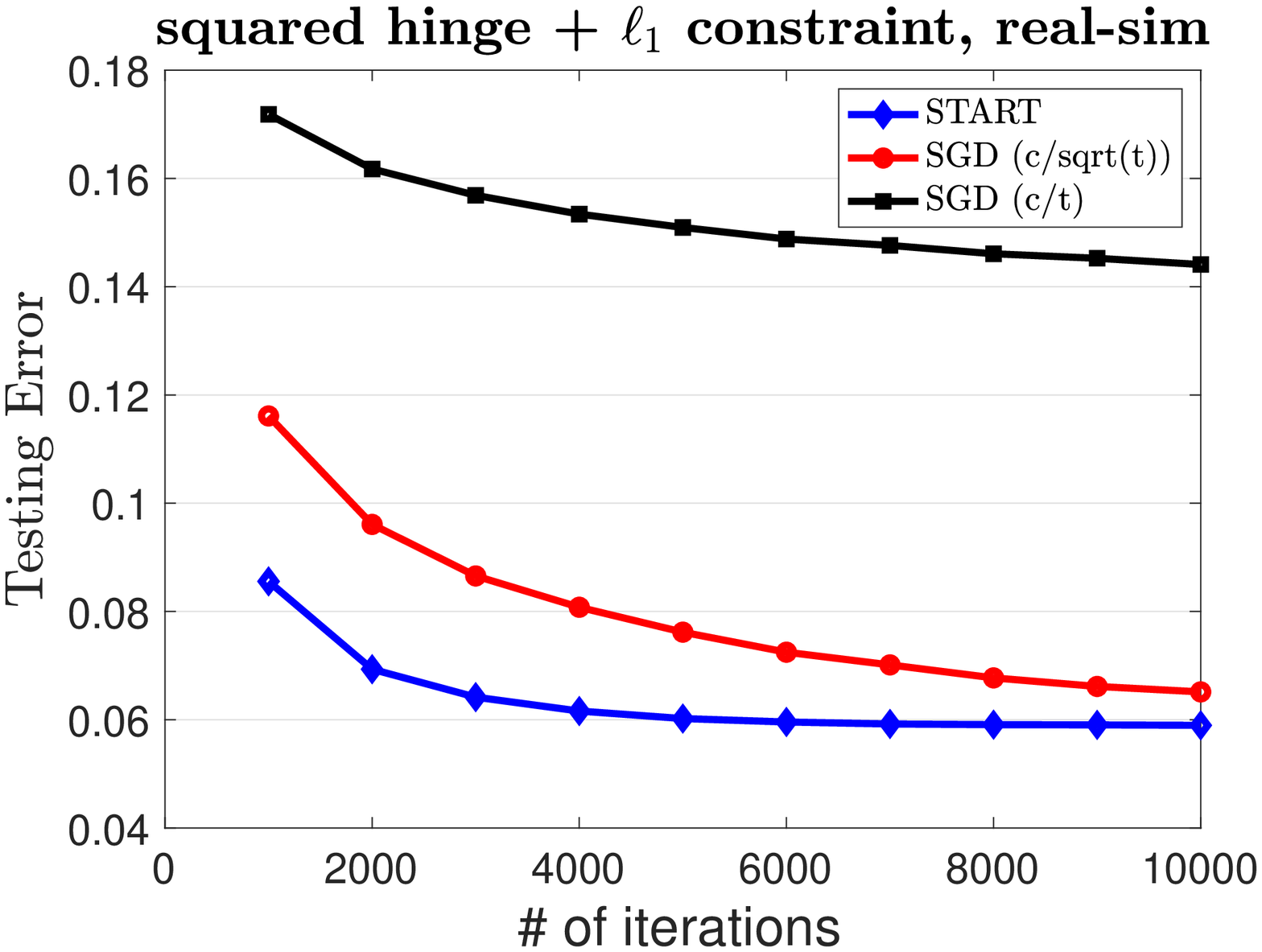}
\includegraphics[width=.29\textwidth]{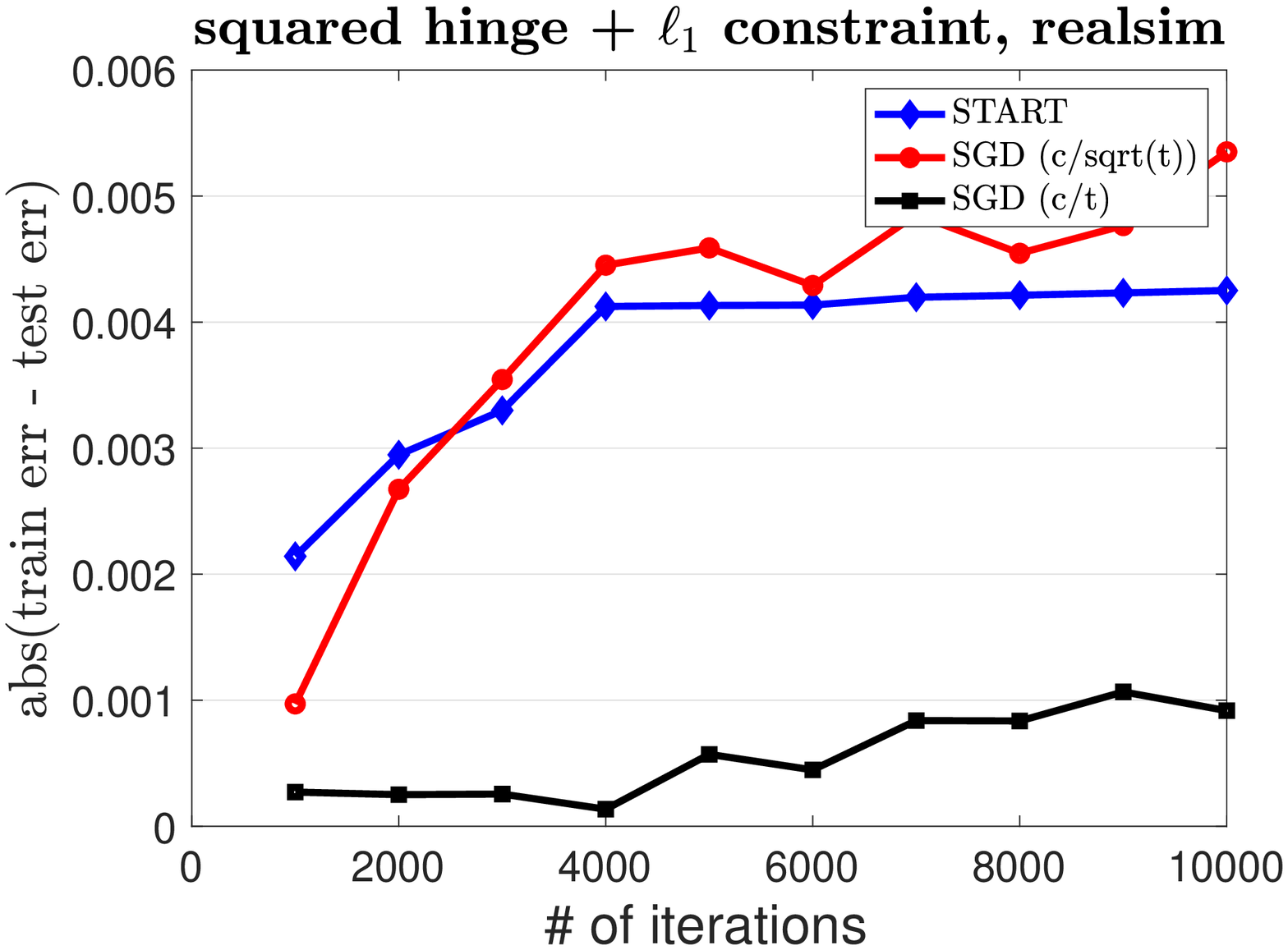}\hspace*{0.1in}

\includegraphics[width=.29\textwidth]{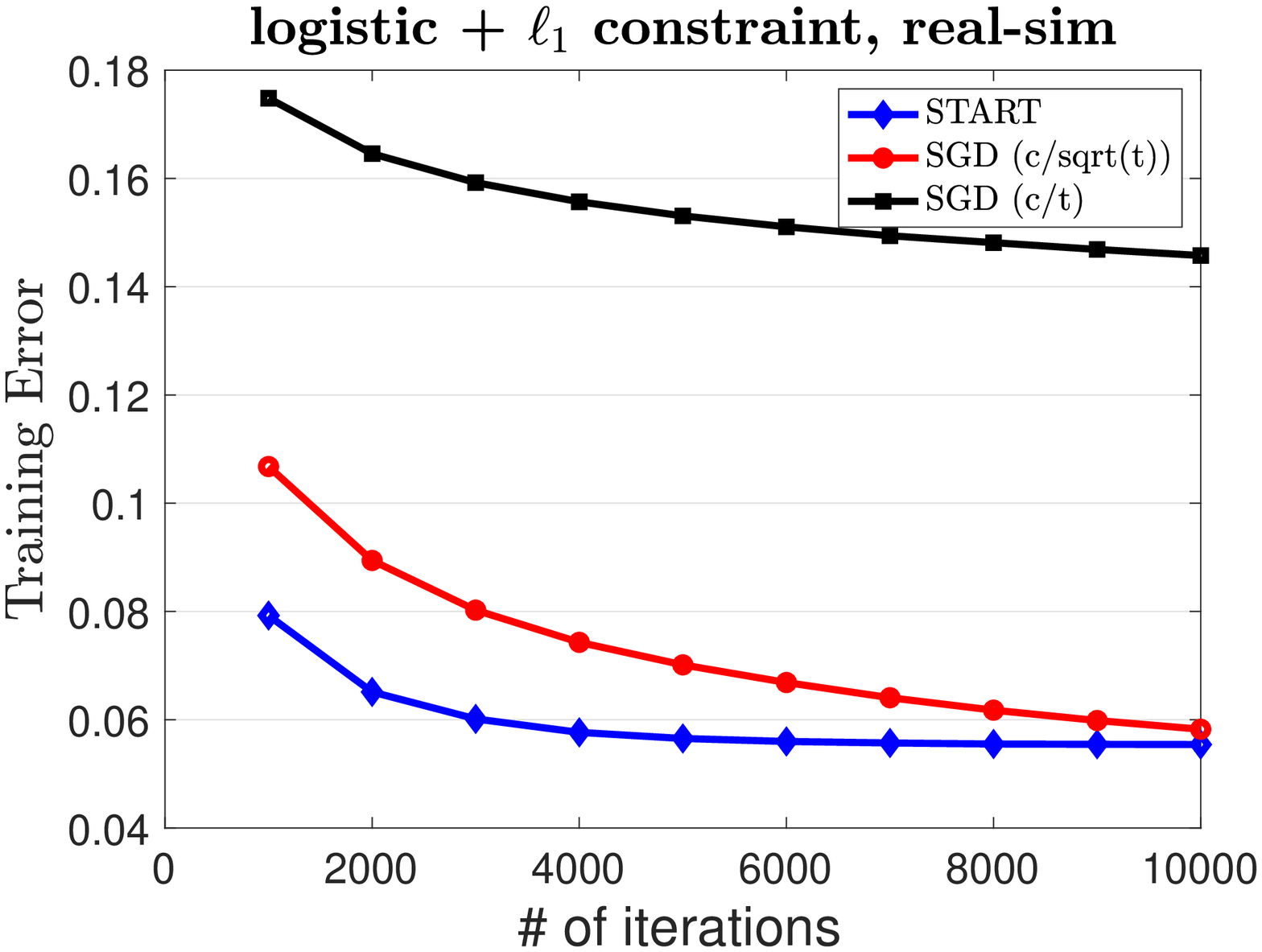}\hspace*{0.1in}
\includegraphics[width=.29\textwidth]{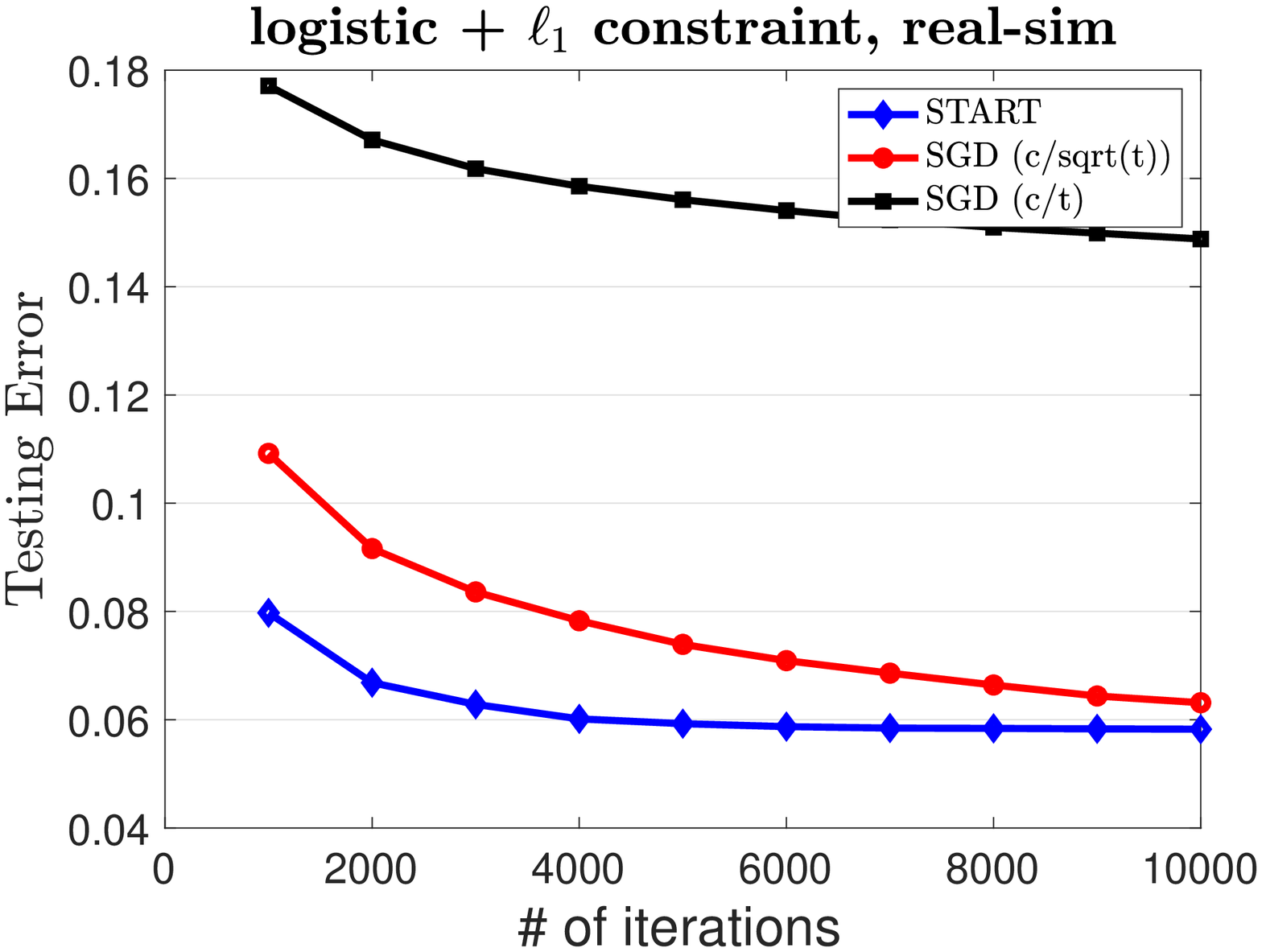}
\includegraphics[width=.29\textwidth]{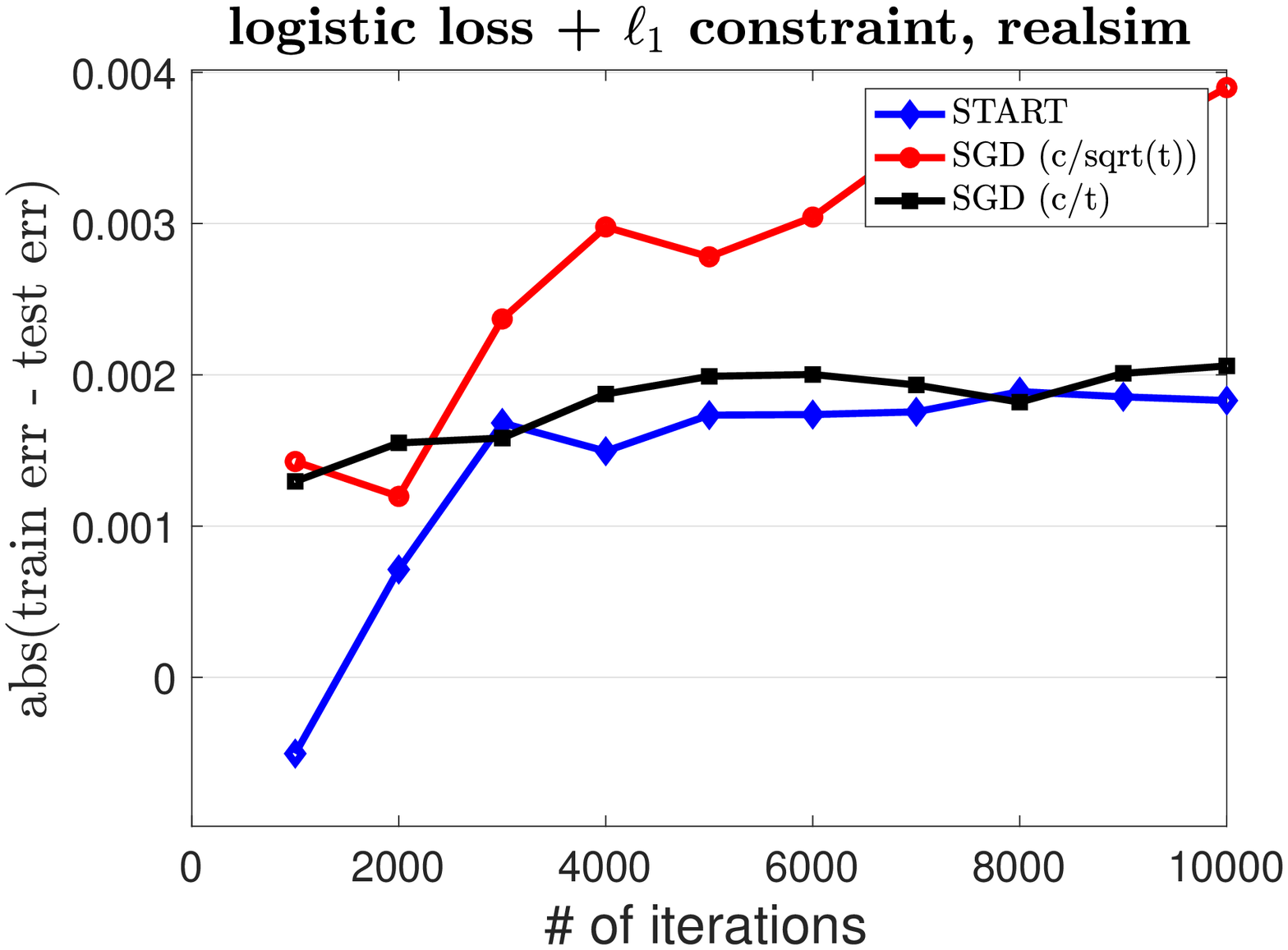}\hspace*{0.1in}

\includegraphics[width=.3\textwidth]{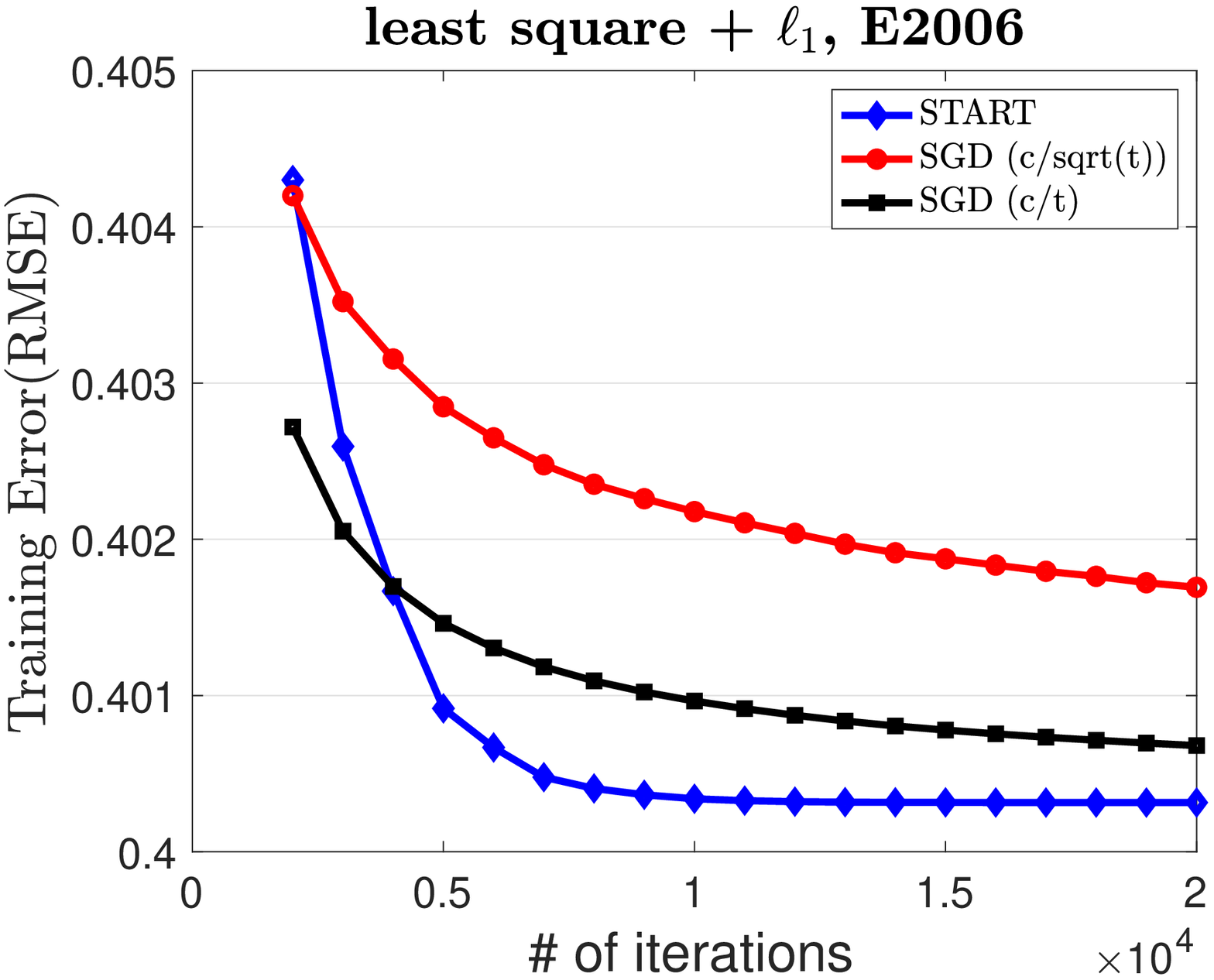}\hspace*{0.1in}
\includegraphics[width=.3\textwidth]{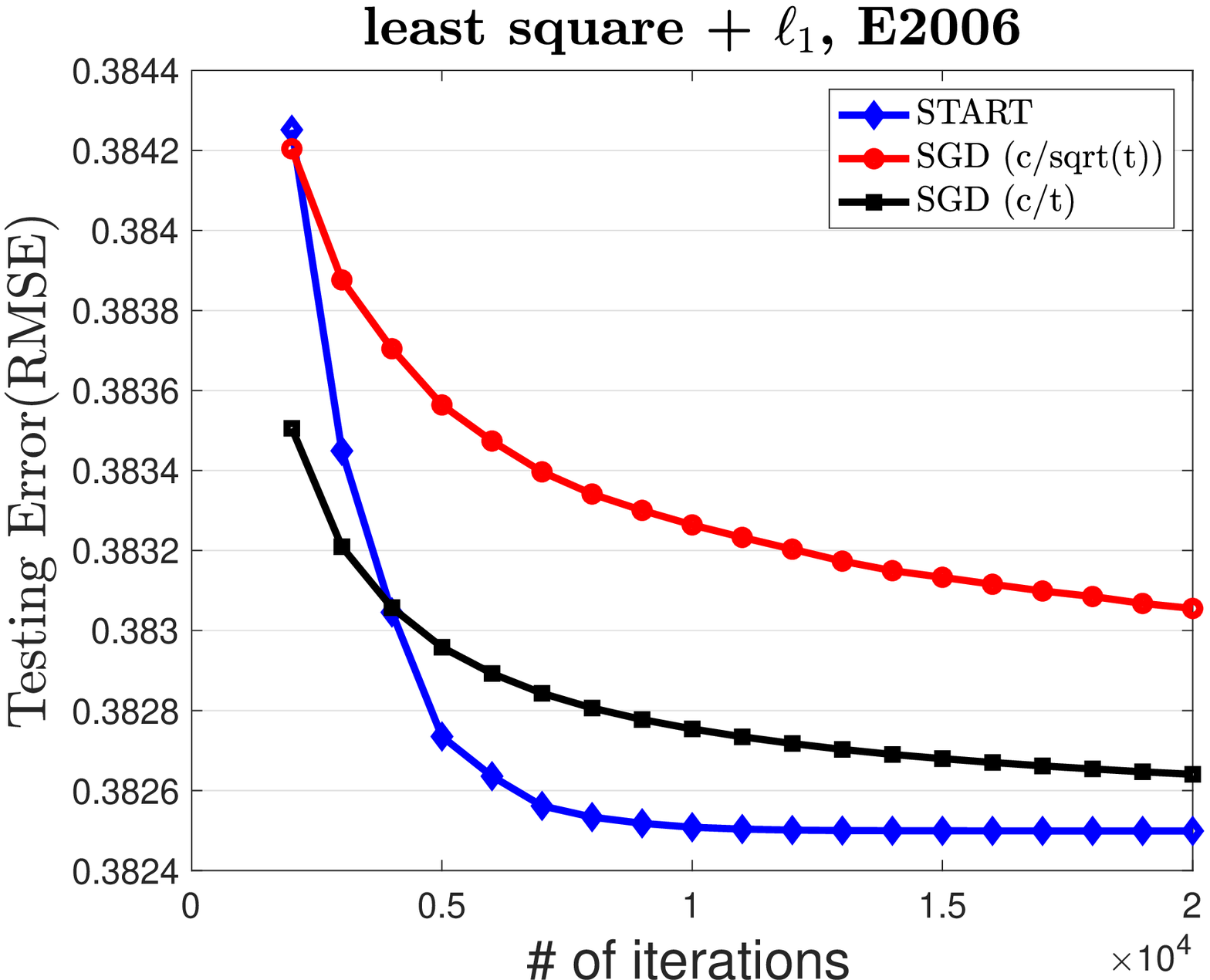}
\includegraphics[width=.3\textwidth]{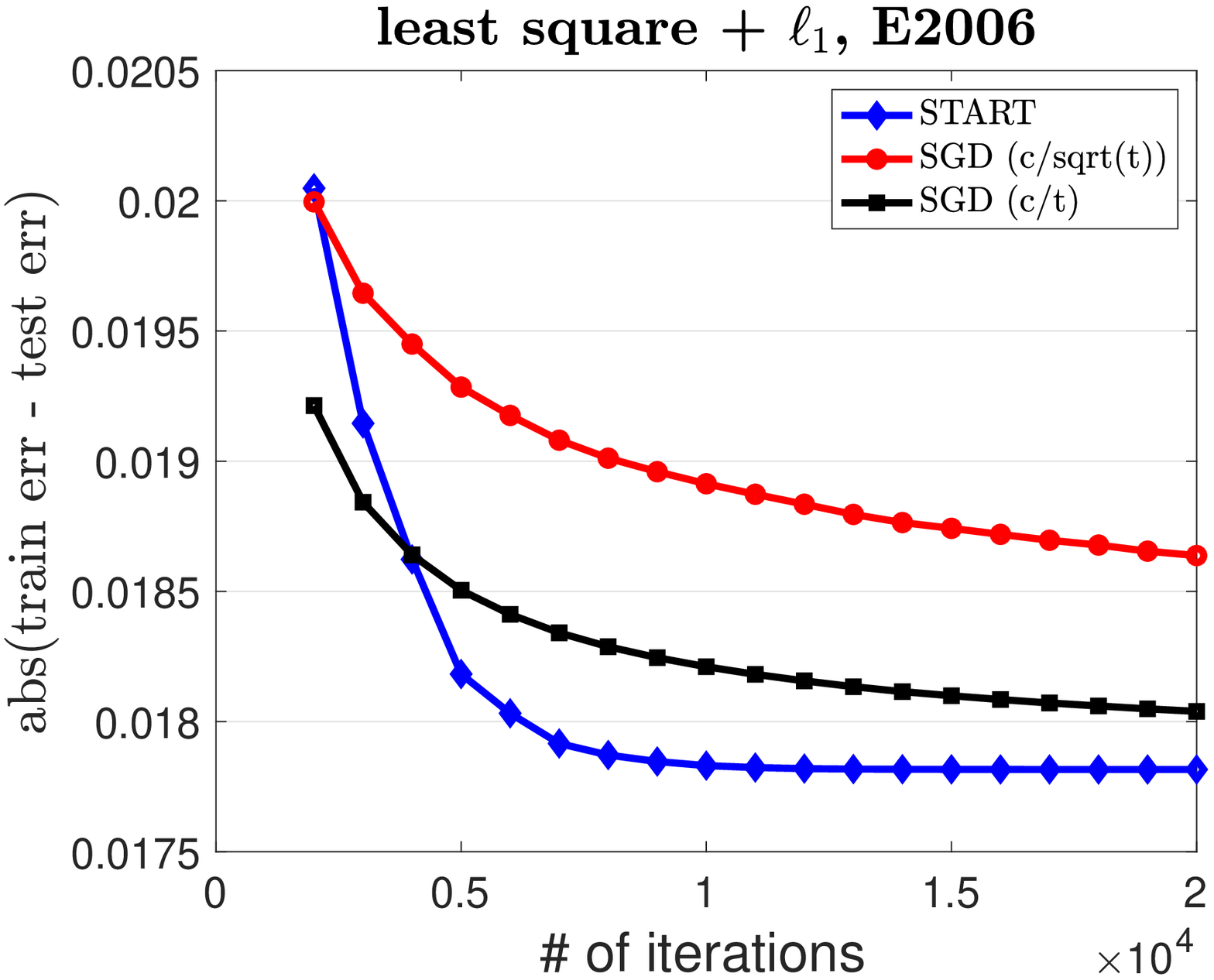}\hspace*{0.1in}

\includegraphics[width=.3\textwidth]{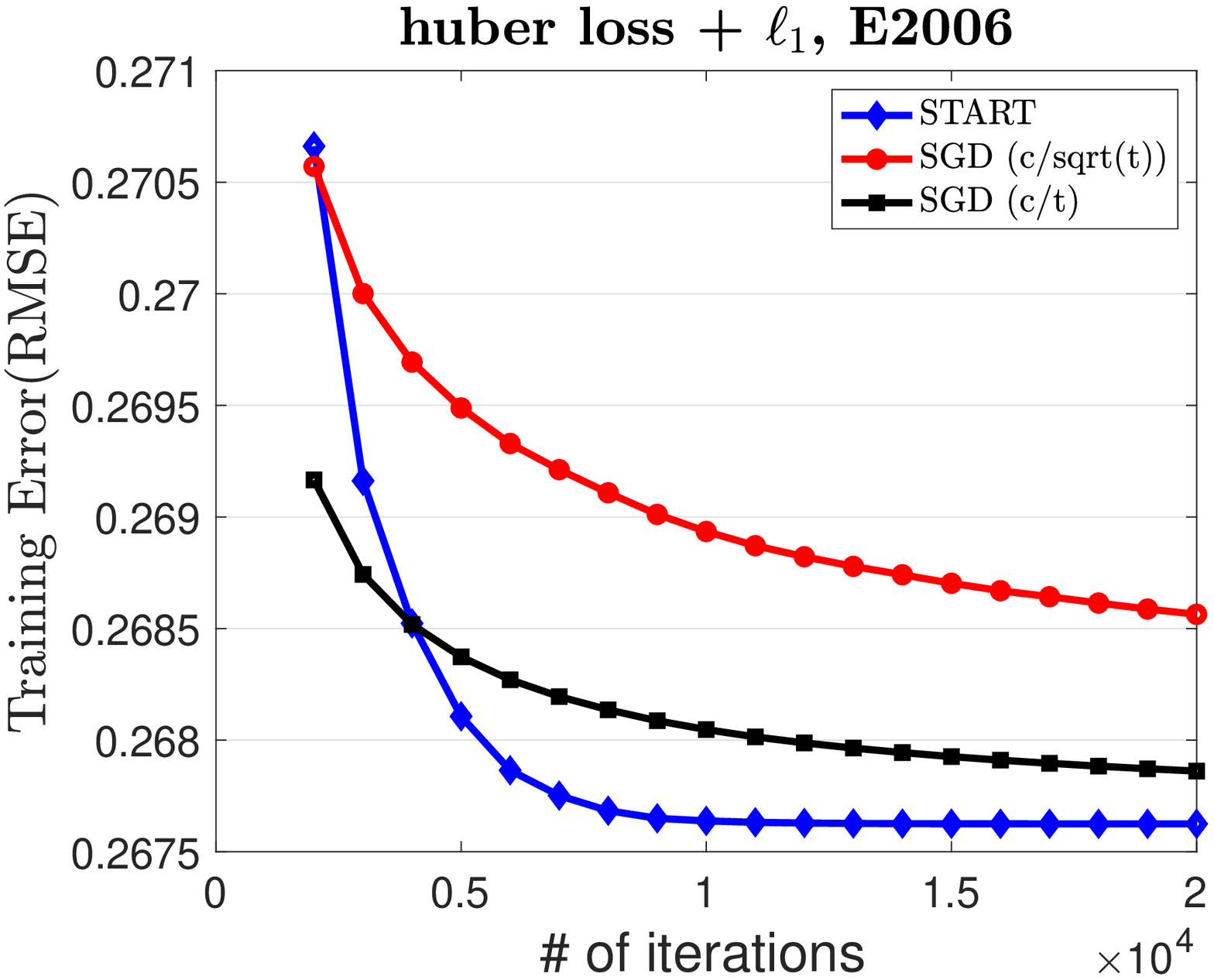}\hspace*{0.1in}
\includegraphics[width=.3\textwidth]{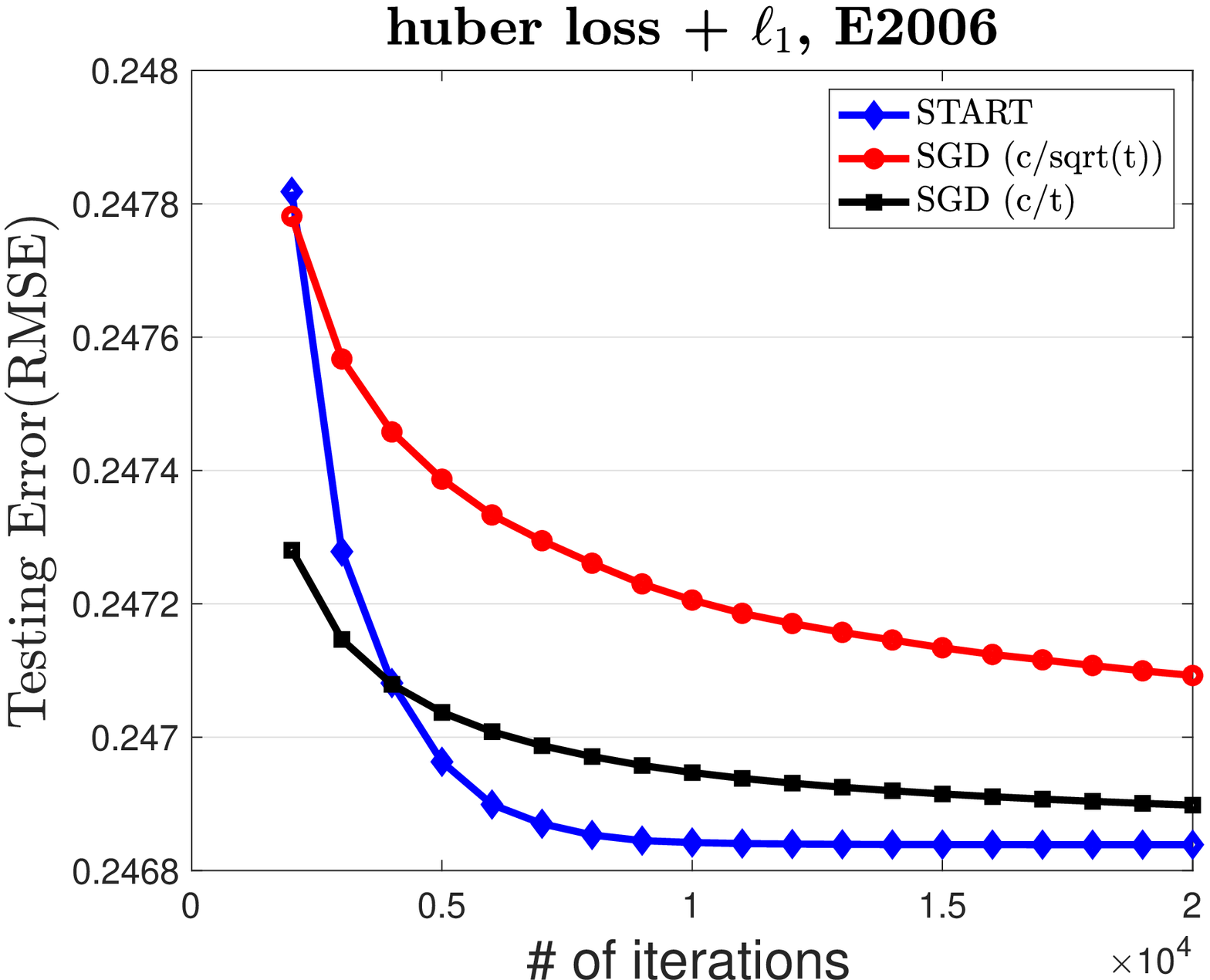}
\includegraphics[width=.3\textwidth]{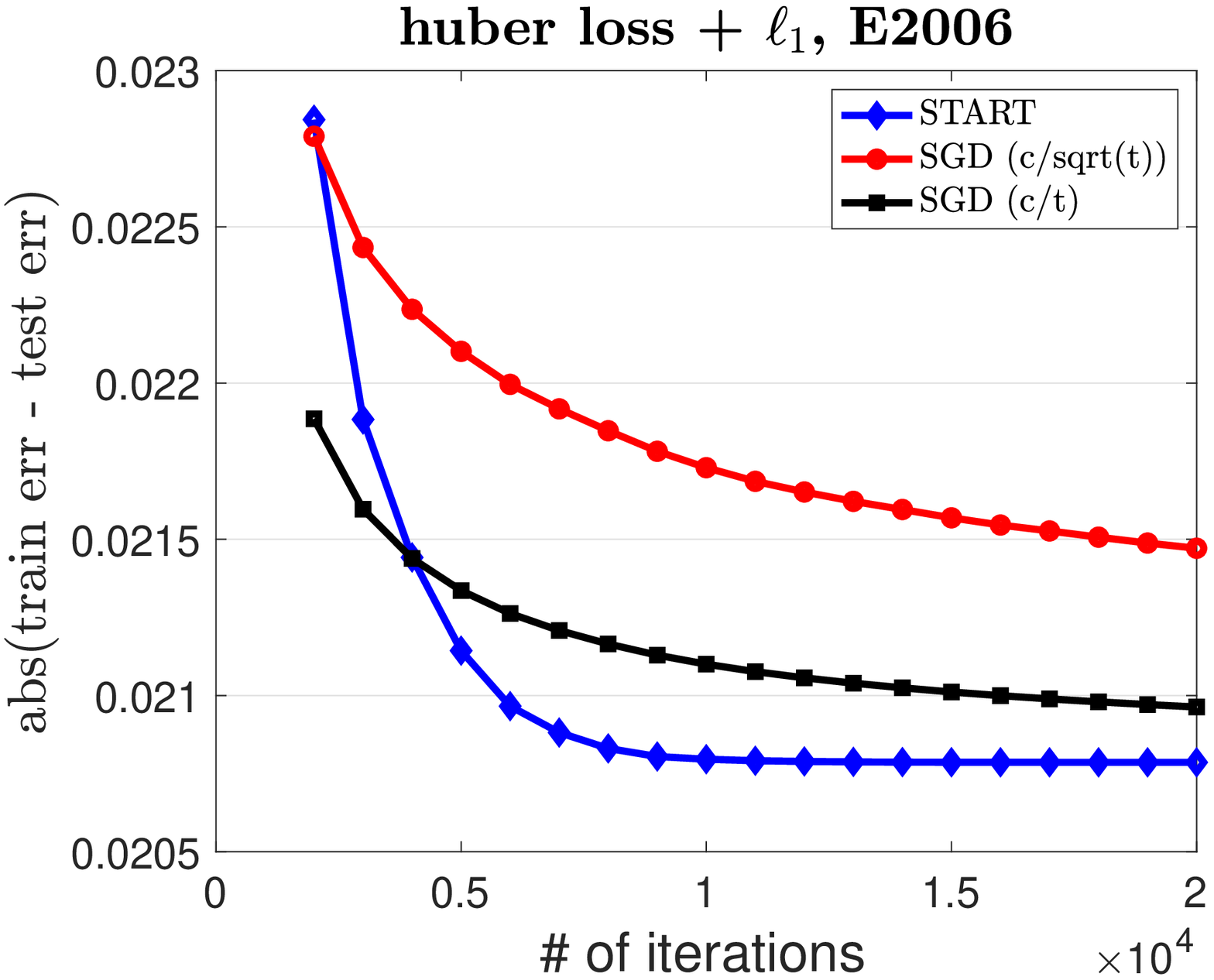}\hspace*{0.1in}

\vspace*{-0.1in}
\caption{From left to right: training error, testing error, generalization error for stagewise learning with convex loss functions. }
\label{fig:3}
\vspace*{-0.2in}
\end{figure*}

\section{Conclusion}
In this paper, we have analyzed the convergence of training error and testing error of a  stagewise regularized training algorithm for solving empirical risk minimization under the Polyak-\L ojasiewicz condition. For non-convex objectives, we consider two classes of functions that are close to a convex function for which stagewise learning is proved to yield faster convergence than vanilla SGD with a polynomially  decreasing step size on both training and testing error. Our numerical experiments on deep learning verify that one class of non-convexity assumption holds and hence the provided theory of faster convergence applies. In future, we consider extending the theory to the non-smooth RELU activation.

\bibliography{all}

\newpage
\appendix

\section*{A. Proofs of Section~\ref{sec:4}}

\subsection*{A1. Proof of Lemma~\ref{lem:2}}

\begin{proof}
The proof of Lemma~\ref{lem:2} follows similarly as the one of Lemma 1 in~\citep{DBLP:conf/icml/ZhaoZ15}.
For completeness, we prove our result.

Recall that $F_{k} = F_{\S}( \w ) + \frac{1}{2\gamma} || \w - \w_{k-1} ||^{2}$.
Let $r_{k}(\w) = \frac{1}{2\gamma} || \w - \w_{k-1} ||^{2} + \delta_{\Omega}(\w)$, so $F_{k}(\w) = F_{\S}(\w) + r_{k}(\w)$, where $\delta_{\Omega}(\cdot)$ is the indicator function of $\Omega$. 
Due to 
the convexity of $F_{\S}(\w)$, 
the $\frac{1}{\gamma}$-strong convexity of $r_{k}(\w)$ and 
the $L$-smoothness of $f(\w; \z)$, 
we have the following three inequalities
\begin{align}\label{eq:three_ineq1}
       F_{\S}(\w) 
\geq & 
       F_{\S}(\w_{t}) + \langle \nabla F_{\S}(\w_{t}) , (\w - \w_{t}) \rangle    
                 \\
       r_{k}(\w)  
\geq & 
       r_{k}(\w_{t+1}) + \langle \partial r_{k}(\w_{t+1}) , \w - \w_{t+1} \rangle + \frac{1}{2\gamma} || \w - \w_{t+1} ||^{2}   
                 \nonumber\\
    \label{eq:three_ineq3}
       F_{\S}(\w_{t}) 
\geq &
       F_{\S}(\w_{t+1}) - \langle \nabla F_{\S}(\w_{t}) , \w_{t+1} - \w_{t} \rangle - \frac{L}{2} || \w_{t} - \w_{t+1} ||^{2}   .   
\end{align} 
Combining them together, we have
\begin{align}\label{eq:combine_three_equation}
     & F_{\S}(\w_{t+1}) + r_{k}(\w_{t+1}) - ( F_{\S}(\w) + r_{k}(\w) )     \nonumber\\
&\leq 
       \langle \nabla F_{\S}(\w_{t}) + \partial r_{k}(\w_{t+1}) , \w_{t+1} - \w \rangle + \frac{L}{2} || \w_{t} - \w_{t+1} ||^{2} - \frac{1}{2\gamma} || \w - \w_{t+1} ||^{2}  .
\end{align}
Recall Line 3 of Algorithm~\ref{alg:sgd}, we update $\w_{t+1}$ as follows
$$
\w_{t+1} = \arg\min_{\w \in\R^d} \nabla f(\w_{t}, \z_{i_{t}})^{\top} \w + \frac{1}{2\eta} || \w - \w_{t} ||^{2} + r_k(\w).
$$
If we set the gradient of the above problem in $\w_{t+1}$ to $0$, there exists $\partial r_k(\w_{t+1})$ such that 
$$
\partial r_{k}(\w_{t+1}) = - \nabla f(\w_{t}, \z_{i_{t}}) + \frac{1}{\eta} (\w_{t} - \w_{t+1}) .
$$
Plugging the above equation to~(\ref{eq:combine_three_equation}), we have
\begin{align*}
     & F_{\S}(\w_{t+1}) + r_{k}(\w_{t+1}) - ( F_{\S}(\w) + r_{k}(\w) )     \\
&\leq
       \langle \nabla F_{\S}(\w_{t}) - \nabla f(\w_{t}, \z_{i_{t}}) , \w_{t+1} - \w \rangle     
       \\
       & + \langle \frac{1}{\eta} (\w_{t} - \w_{t+1})  , \w_{t+1} - \w \rangle  
         + \frac{L}{2} || \w_{t} - \w_{t+1} ||^{2} - \frac{1}{2\gamma} || \w - \w_{t+1} ||^{2}  \\
& =     
       \langle \nabla F_{\S}(\w_{t}) - \nabla f(\w_{t}, \z_{i_{t}}) , \w_{t+1} - \hat{\w}_{t+1} + \hat{\w}_{t+1} - \w \rangle     
       \\&
       +  \frac{1}{2\eta} || \w_{t} - \w ||^{2}  - \frac{1}{2\eta} || \w_{t} - \w_{t+1} ||^{2}     
       {- \frac{1}{2\eta} || \w_{t+1} - \w ||^{2}+ \frac{L}{2} || \w_{t} - \w_{t+1} ||^{2} - \frac{1}{2\gamma} || \w - \w_{t+1} ||^{2} } \\
&\leq 
       || \nabla F_{\S}(\w_{t}) - \nabla f(\w_{t}, \z_{i_{t}}) || \cdot || \w_{t+1} - \hat{\w}_{t+1} ||  
       + \langle \nabla F_{\S}(\w_{t}) - \nabla f(\w_{t}, \z_{i_{t}}) , \hat{\w}_{t+1} - \w \rangle   \\
&      + \frac{1}{2\eta} || \w_{t} - \w ||^{2} - \frac{1}{2\eta} || \w_{t+1} - \w ||^{2} 
       - \frac{1}{2\gamma} || \w - \w_{t+1} ||^{2}  \\
&\leq 
       \eta || \nabla F_{\S}(\w_{t}) - \nabla f(\w_{t}, \z_{i_{t}}) ||^{2}
       + \langle \nabla F_{\S}(\w_{t}) - \nabla f(\w_{t}, \z_{i_{t}}) , \hat{\w}_{t+1} - \w \rangle   \\
     & + \frac{1}{2\eta} || \w_{t} - \w ||^{2} - \frac{1}{2\eta} || \w_{t+1} - \w ||^{2} 
       - \frac{1}{2\gamma} || \w - \w_{t+1} ||^{2}    .
\end{align*}
The first equality is due to 
$$
2 \langle x - y, y - z \rangle = || x - z||^{2} - || x - y ||^{2} - || y - z ||^{2}
$$
and 
$\hat{\w}_{t + 1} = \arg\min_{x \in \Omega} \w^{\top} \nabla F_{\S}(w) + \frac{1}{2\eta} || \w - \w_{t} ||^{2} + \frac{1}{2\gamma} || \w - \w_{1} ||^{2}$.
The second inequality is due to Cauchy-Schwarz inequality and setting $\eta \leq \frac{1}{L}$.
The third inequality is due to Lemma 3 of~\cite{xu2018stochastic}.

Taking expectation on both sides, we have
\begin{align*}
      & \E [ F_{k}(\w_{t+1}) - F_{k}(\w) ]
\leq  
       \eta \sigma^{2} + \frac{1}{2 \eta} || \w_{t} - \w ||^{2} 
       - \frac{1}{2 \eta} \E [|| \w_{t+1} - \w ||^{2}] - \frac{1}{2\gamma} \E [ || \w - \w_{t+1} ||^{2} ]   ,
\end{align*}
where $\E_{i} [ || \nabla f(\w, \z_{i}) - \nabla F_{\S}(\w) ||^{2} ] \leq \sigma^{2}$ by assumption.

Taking summation of the above inequality from $t = 1$ to $T$, we have
\begin{align*}
     & \sum_{t=1}^{T} F_{k}(\w_{t+1}) - F_{k}(\w)   \\
&\leq 
       \eta \sigma^{2} T + \frac{1}{2\eta} || \w_{1} - \w ||^{2} - \frac{1}{2\eta} \E [ || \w_{T+1} - \w ||^{2} ]
       - \frac{1}{2\gamma} \sum_{t=1}^{T} \E [ || \w - \w_{t+1} ||^{2} ]  .
\end{align*}

By employing Jensens' inequality on LHS, denoting the output of the $k$-th stage by $\w_{k} = \hat{\w}_{T} = \frac{1}{T}\sum_{t=1}^{T} \w_{t}$ and taking expectation, we have
\begin{align*}
\E [ F_{k}(\hat{\w}_{T}) - F_{k}(\w) ] \leq \sigma^{2} \eta + \frac{|| \w_{1} - \w ||^{2}}{2 \eta T}  .
\end{align*}
\end{proof}

\subsection*{A2. Proof of Lemma~\ref{lem:kstab}}
\begin{proof}
Let us define 
\begin{align*}
\G(\u; f,  \w_1)  = \frac{\gamma \u + \eta \w_1 - \eta \gamma \nabla f(\u)}{\eta + \gamma}.
\end{align*}
It is not difficult to show that $\w_{t+1} = \text{Proj}_{\Omega}[\G(\w_t; f_t,  \w_1)]$, where $\text{Proj}_{\Omega}[\cdot]$ denotes the projection operator. Due to non-expansive of the projection operator, it suffices to bound $\|G(\w_t; f_t,  \w_1)  -G(\w'_t; f'_t,  \w'_1)\|$.  Let us consider two scenarios. The first scenario is $f_t = f'_t = f$ (using the same data). Then 
\begin{align*}
&\|\G(\w_t; f, \w_1) -\G(\w'_t; f', \w'_1)  \| \\
&= { \bigg\|\frac{\gamma \w_t + \eta \w_1 - \eta \gamma \nabla f(\w_t)}{\eta + \gamma}
 - \frac{\gamma \w'_t + \eta \w'_1 - \eta \gamma \nabla f(\w'_t)}{\eta + \gamma}\bigg\| }\\
 &\leq {\frac{\eta}{\eta+ \gamma}\|\w_1 - \w'_1\| + \frac{\gamma}{\eta + \gamma}\|\w_t - \eta\nabla f(\w_t) - \w'_t + \eta\nabla f(\w'_t)\|}\\
 &\leq {\frac{\eta}{\eta+ \gamma}\|\w_1 - \w'_1\| +  \frac{\gamma}{\eta + \gamma}\|\w_t - \w_t'\| =  \frac{\eta}{\eta+ \gamma}\delta_1 +  \frac{\gamma}{\eta + \gamma}\delta_t},
\end{align*}
where last inequality is due to $1$-expansive of GD update with $\eta\leq 2/L$ for a convex function~\citep{DBLP:conf/icml/HardtRS16}. Next, let us consider the second scenario $f_t \neq f'_t$. Then 
\begin{align*}
&      \|\G(\w_t; f, \w_1) -\G(\w'_t; f', \w'_1)  \|\\
& =    {\bigg\|\frac{\gamma \w_t + \eta \w_1 - \eta \gamma \nabla f(\w_t)}{\eta + \gamma}
       - \frac{\gamma \w'_t + \eta \w'_1 - \eta \gamma \nabla f'(\w'_t)}{\eta + \gamma}\bigg\| }\\
& \leq \frac{\eta}{\eta+ \gamma}\|\w_1 - \w'_1\|   
       + \frac{\gamma}{\eta + \gamma}\|\w_t - \eta\nabla f(\w_t) - \w'_t + \eta\nabla f'(\w'_t)\|\\
& \leq \frac{\eta}{\eta+ \gamma}\|\w_1 - \w'_1\| +  \frac{\gamma}{\eta + \gamma}\|\w_t - \w_t'\|  + \frac{2\eta\gamma G}{\eta + \gamma}\\
& =   \frac{\eta}{\eta+ \gamma}\delta_1 +  \frac{\gamma}{\eta + \gamma}\delta_t  + \frac{2\eta\gamma G}{\eta + \gamma}.
\end{align*}
\end{proof}

\section*{B. Proofs of Section~\ref{sec:5}}
\subsection*{B1. Proof of Lemma~\ref{lem:onepoint}}
\begin{proof}
The inequality regarding $\|\nabla F(\w)\|^2$ can be found in~\citep{DBLP:conf/pkdd/KarimiNS16}. The inequality regarding $\nabla F(\w)^{\top}(\w - \w^*)$ can be easily seen from the definition of one-point strong convexity and the $L$-smoothness condition of $F(\w)$ and the condition $\nabla F(\w^*)=0$, i.e., 
\begin{align*}
     F(\w) - F(\w^*) 
\leq \nabla F(\w^*)^{\top}(\w - \w^*) + \frac{L}{2}\|\w - \w^*\|^2 
\leq \frac{L}{2\mu_1}\nabla F(\w)^{\top}(\w - \w^*).
\end{align*}
\end{proof}

\subsection*{B2. Proof of Lemma~\ref{lem:sgd2-qc}}
\begin{proof}
Without loss of generality, we consider minimizing $F_1 = F_\S + \frac{1}{2\gamma}\|\w - \w_0\|^2$. 
Let $r(\w) = \frac{1}{2\gamma}\|\w - \w_{0}\|^2$. The initial solution of SGD $\w_1 = \w_0$.  
Following the standard analysis of stochastic proximal SGD, we have
\begin{align*}
     \nabla f(\w_t, \z_{i_t})^{\top}(\w_t - \w) + r(\w_{t+1}) - r(\w)
\leq \frac{\|\w - \w_t\|^2}{2\eta} + \frac{\|\w - \w_{t+1}\|^2}{2\eta} + \frac{\eta}{2}\|\nabla f(\w_t, \z_{i_t})\|^2   .
\end{align*}
Taking expectation on both sides, we have 
\begin{align*}
     \E[\nabla& F_\S(\w_t)^{\top}(\w_t - \w) + r(\w_{t+1}) - r(\w)]
\leq \E\bigg[\frac{\|\w - \w_t\|^2}{2\eta} -  \frac{\|\w - \w_{t+1}\|^2}{2\eta} + \frac{\eta G^2}{2}\bigg]   .
\end{align*}
Plugging $\w = \w_\S^*$, summing over $t=1,\ldots, T$ and using the one-point weakly quasi-convexity, we have
\begin{align*}
\E\bigg[ & \sum_{t=1}^T\theta(F_\S(\w_t) - F_\S( \w_\S^*)) + r(\w_{t}) - r(\w^*_\S)\bigg]   \\
         & \leq \frac{\|\w^*_\S - \w_1\|^2}{2\eta} + \frac{\eta G^2 T}{2} + \E[r(\w_1) - r(\w_{T+1})]  .
\end{align*}
As a result, 
\begin{align*}
     \E\bigg[ F_\S(\w_\tau) - F_\S( \w_\S^*) \bigg]
\leq \frac{\|\w^*_\S - \w_1\|^2}{2\theta\eta T} + \frac{\eta G^2}{2\theta }  + \frac{1}{2\gamma \theta}\|\w_{0} - \w_\S^*\|^2  ,
\end{align*}
where $\tau\in\{1, \ldots, T\}$ is randomly selected. Applying the above result to the $k$-th stage, we complete the proof.
\end{proof}

\subsection*{B3. Proof of Lemma~\ref{lem:sgd2}}
\begin{proof}
The proof of Lemma~\ref{lem:sgd2} follows the one of Lemma~\ref{lem:2}.
The only difference lies on the weak convexity of $F_{\S}(\w)$.

We could replace the first  inequality in~(\ref{eq:three_ineq1}) by the following $\rho$-weak convexity condition of $F_{\S}(\cdot)$:
\begin{align*}
F_{\S}(\w) &\geq F_{\S}(\w_{t}) + \langle \nabla F_{\S}(\w_{t}) , (\w - \w_{t}) \rangle - \frac{\rho}{2} || \w_{t} - \w ||^{2} .
\end{align*}
Then we combine it with other two inequalities as follows
\begin{align*}
     & F_{\S}(\w_{t+1}) + r_{k}(\w_{t+1}) - ( F_{\S}(\w) + r_{k}(\w) )    \\
\leq &
       \langle \nabla F_{\S}(\w_{t}) + \partial r_{k}(\w_{t+1}) , \w_{t+1} - \w \rangle 
       + \frac{L}{2} || \w_{t} - \w_{t+1} ||^{2}
       - \frac{1}{2\gamma} || \w - \w_{t+1} ||^{2} + \frac{\rho}{2} || \w - \w_{t} ||^{2} .
\end{align*}
Then following the proof of Lemma~\ref{lem:2} under the condition $\eta\leq 1/L$ we have
\begin{align*}
     & F_{k}(\w_{t+1}) - F_{k}(\w)    \nonumber \\
&\leq 
       \langle \nabla F_{\S}(\w_{t}) - \nabla f(\w_{t}, \z_{i_t}) , \w_{t+1} - \hat{\w}_{t+1} \rangle 
       + \eta || \nabla F_{\S}(\w_{t}) - \nabla f(\w_{t}, \z_{i_t}) ||^{2}    \nonumber\\
     & + \frac{1}{2\eta} || \w_{t} - \w ||^{2} - \frac{1}{2\eta} || \w_{t+1} - \w ||^{2} 
       - \frac{1}{2\gamma} || \w - \w_{t+1} ||^{2}
       + \frac{\rho}{2} || \w - \w_{t} ||^{2}
\end{align*}
Taking expectation on both sides, summing from $t = 1$ to $T$ and applying Jensen's inequality, we have
\begin{align*}
       \E [ F_{k}(\wh_{T}) - F_{k}(\w) ]
\leq &
       \eta \sigma^{2} 
       + \frac{1}{2T\eta} || \w - \w_{1} ||^{2} 
       + \frac{1}{2T\gamma} || \w - \w_{1} ||^{2} 
\end{align*}
\end{proof}

The proof of Lemma~\ref{lem:sgd2} follows the one of Lemma~\ref{lem:2}.
The only difference lies on the weak convexity of $F_{\S}(\w)$.

We could replace the first  inequality in~(\ref{eq:three_ineq1}) by the following $\rho$-weak convexity condition of $F_{\S}(\cdot)$:
\begin{align*}
F_{\S}(\w) &\geq F_{\S}(\w_{t}) + \langle \nabla F_{\S}(\w_{t}) , (\w - \w_{t}) \rangle - \frac{\rho}{2} || \w_{t} - \w ||^{2} .
\end{align*}
Then we combine it with other two inequalities as follows
\begin{align*}
     & F_{\S}(\w_{t+1}) + r_{k}(\w_{t+1}) - ( F_{\S}(\w) + r_{k}(\w) )    \nonumber \\
\leq &
       \langle \nabla F_{\S}(\w_{t}) + \partial r_{k}(\w_{t+1}) , \w_{t+1} - \w \rangle 
       + \frac{L}{2} || \w_{t} - \w_{t+1} ||^{2}\\
       & - \frac{1}{2\gamma} || \w - \w_{t+1} ||^{2} + \frac{\rho}{2} || \w - \w_{t} ||^{2}  \nonumber
\end{align*}
Then following the proof of Lemma~\ref{lem:2} under the condition $\eta\leq 1/L$ we have
\begin{align*}
     & F_{k}(\w_{t+1}) - F_{k}(\w)    \nonumber \\
&\leq 
       \langle \nabla F_{\S}(\w_{t}) - \nabla f(\w_{t}, \z_{i_t}) , \w_{t+1} - \hat{\w}_{t+1} \rangle 
       \\
       &+ \eta || \nabla F_{\S}(\w_{t}) - \nabla f(\w_{t}, \z_{i_t}) ||^{2}    \nonumber\\
     & + \frac{1}{2\eta} || \w_{t} - \w ||^{2} - \frac{1}{2\eta} || \w_{t+1} - \w ||^{2} 
       - \frac{1}{2\gamma} || \w - \w_{t+1} ||^{2}\\
       &   + \frac{\rho}{2} || \w - \w_{t} ||^{2}
\end{align*}
Taking expectation on both sides, summing from $t = 1$ to $T$ and applying Jensen's inequality, we have
\begin{align*}
       \E [ F_{k}(\wh_{T}) - F_{k}(\w) ]
\leq &
       \eta\sigma^{2} + \frac{1}{2T\eta} || \w - \w_{1} ||^{2} + \frac{1}{2T\gamma} || \w - \w_{1} ||^{2} 
\end{align*}

\subsection*{B4. Proof of Lemma~\ref{lem:ncvxstab-1}}
\begin{proof}
Let us consider two scenarios. The first scenario is $f = f'$. Then 
\begin{align*}
&      \|\G(\w_t; f, \w_1) -\G(\w'_t; f', \w'_1)  \| \\
& =    { \bigg\|\frac{\gamma \w_t + \eta \w_1 - \eta \gamma \nabla f(\w_t)}{\eta + \gamma}
       - \frac{\gamma \w'_t + \eta \w'_1 - \eta \gamma \nabla f(\w'_t)}{\eta + \gamma}\bigg\|}\\
& \leq \frac{\eta}{\eta+ \gamma}\|\w_1 - \w'_1\| + \frac{\gamma(1+\eta L)}{\eta + \gamma}\|\w_t - \w_t'\| \\
& =    \frac{\eta}{\eta+ \gamma}\delta_1 + \frac{\gamma(1+\eta L)}{\eta + \gamma}\delta_t.
\end{align*}
 Next, let us consider the second scenario $f \neq f'$. Then 
\begin{align*}
&      \|\G(\w_t; f, \w_1) -\G(\w'_t; f', \w'_1)  \|\\
& =    { \bigg\|\frac{\gamma \w_t + \eta \w_1 - \eta \gamma \nabla f(\w_t)}{\eta + \gamma}
 - \frac{\gamma \w'_t + \eta \w'_1 - \eta \gamma \nabla f'(\w'_t)}{\eta + \gamma}\bigg\| }\\
& \leq \frac{\eta}{\eta+ \gamma}\|\w_1 - \w'_1\| 
       + \frac{\gamma}{\eta + \gamma}\|\w_t - \eta\nabla f(\w_t) - \w'_t + \eta\nabla f'(\w'_t)\|\\
& \leq \frac{\eta}{\eta+ \gamma}\|\w_1 - \w'_1\| +  \frac{\gamma}{\eta + \gamma}\|\w_t - \w_t'\|  + \frac{2\eta\gamma G}{\eta + \gamma}\\
& =    \frac{\eta}{\eta+ \gamma}\delta_1 +  \frac{\gamma}{\eta + \gamma}\delta_t  + \frac{2\eta\gamma G}{\eta + \gamma}  .
\end{align*}
\end{proof}

\section*{C. More Experimental Results}
 
In this section, we include more experimental results of non-convex deep learning.
The settings in this section are almost the same with those in Section~\ref{sec:experiment_deep_learning} except that we train the networks with weight decay (i.e. including $5 * 10^{-4} ||\w||^2_2$ regularization).
The results are shown in Figure~\ref{fig:resnet_convnet_ELU_w_L2}, \ref{fig:resnet_convnet_RELU_w_L2} with captions self-explaining the corresponding setting. 

\begin{figure*}[t]
\centering
\includegraphics[width=.225\textwidth]{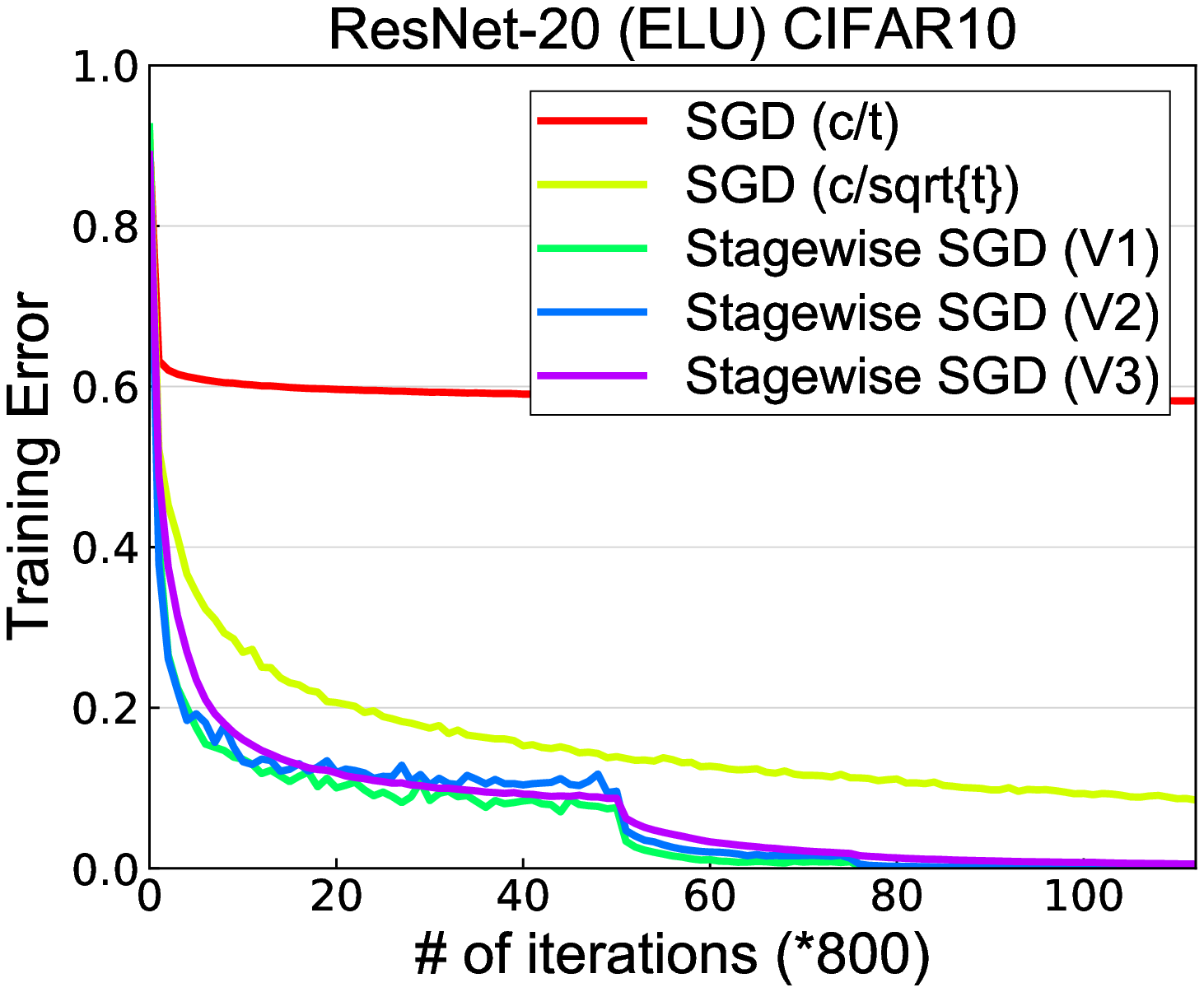}\hspace*{0.15in}
\includegraphics[width=.225\textwidth]{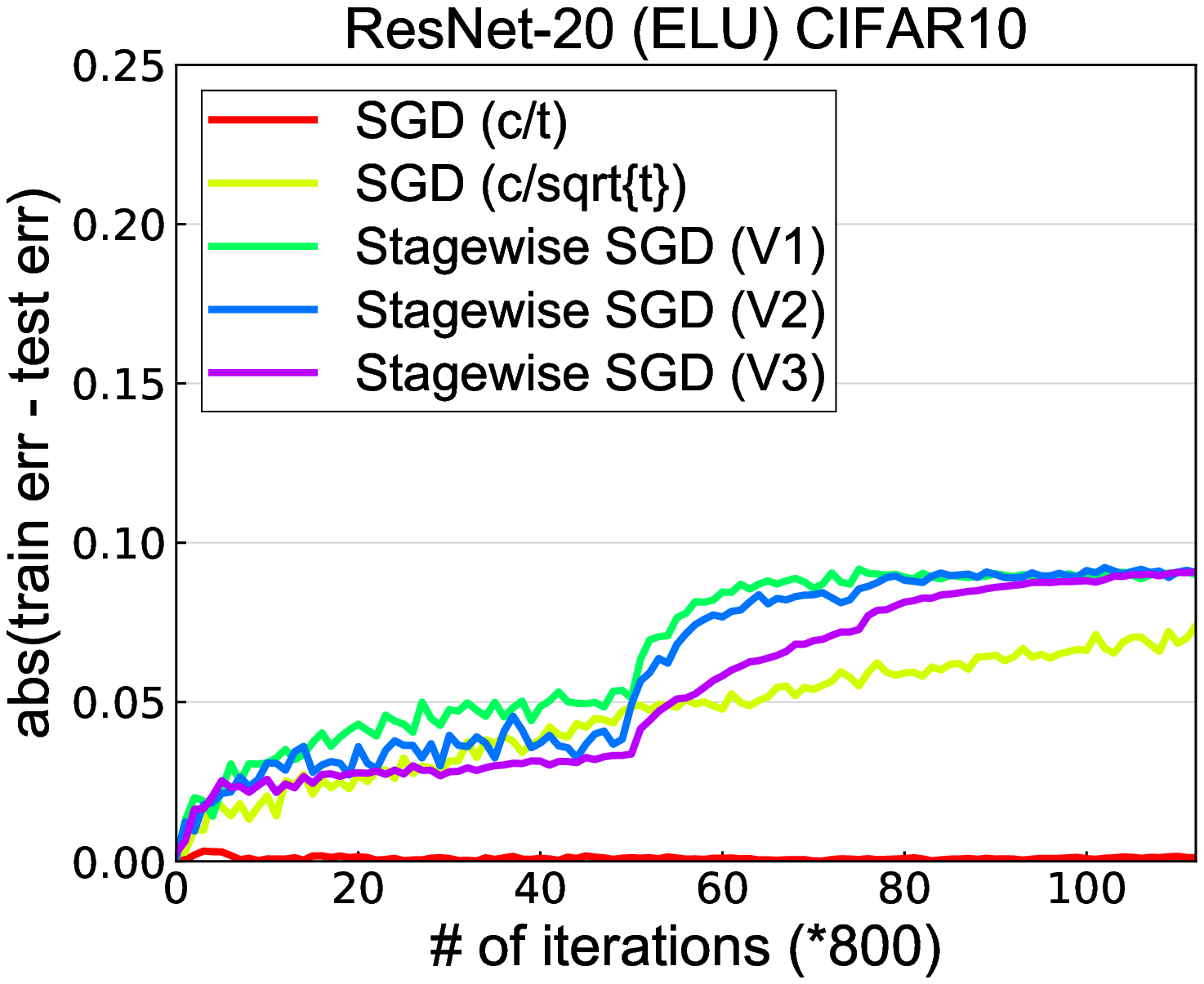}\hspace*{0.15in}
\includegraphics[width=.225\textwidth]{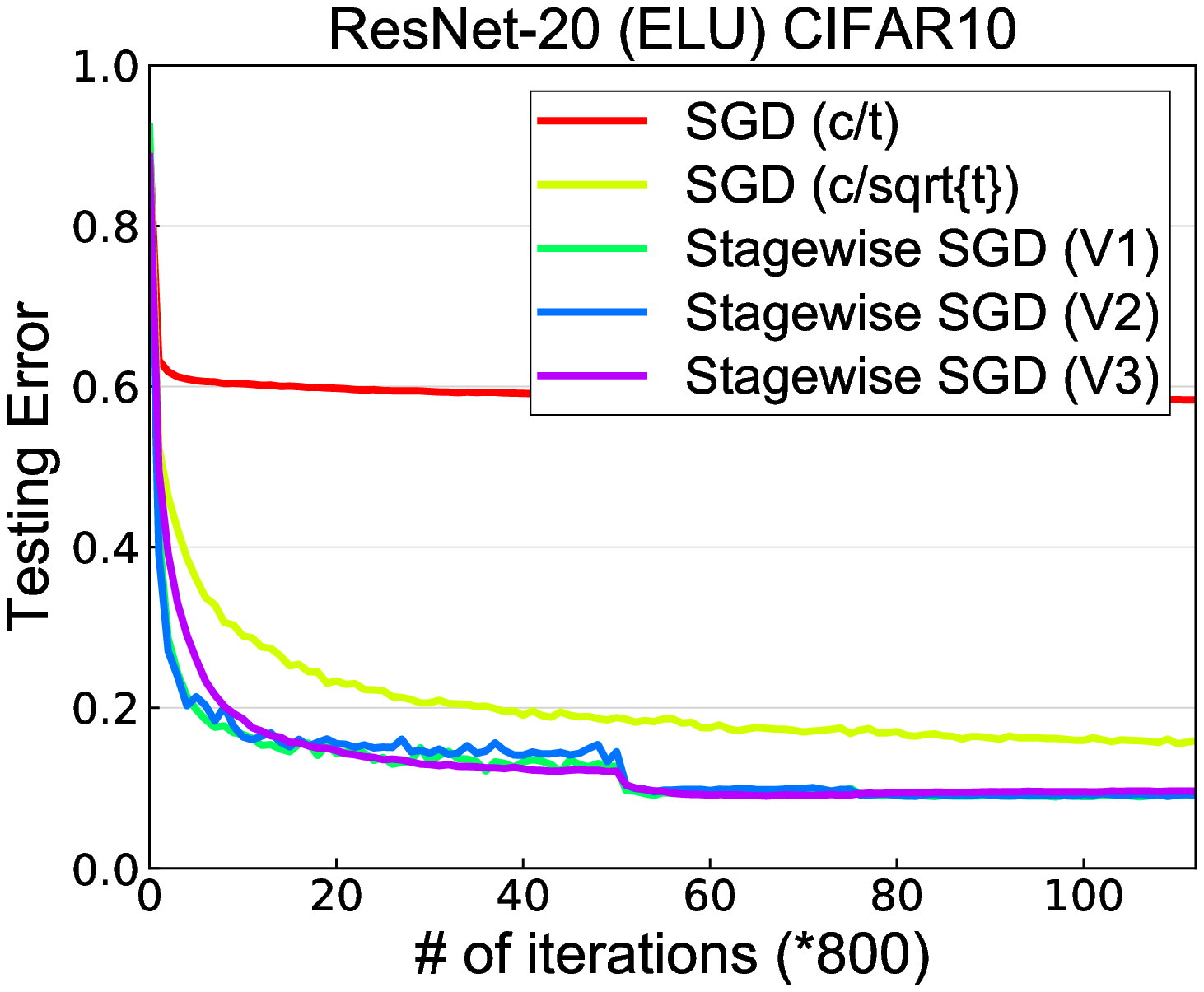}\hspace*{0.15in}
\includegraphics[width=.225\textwidth]{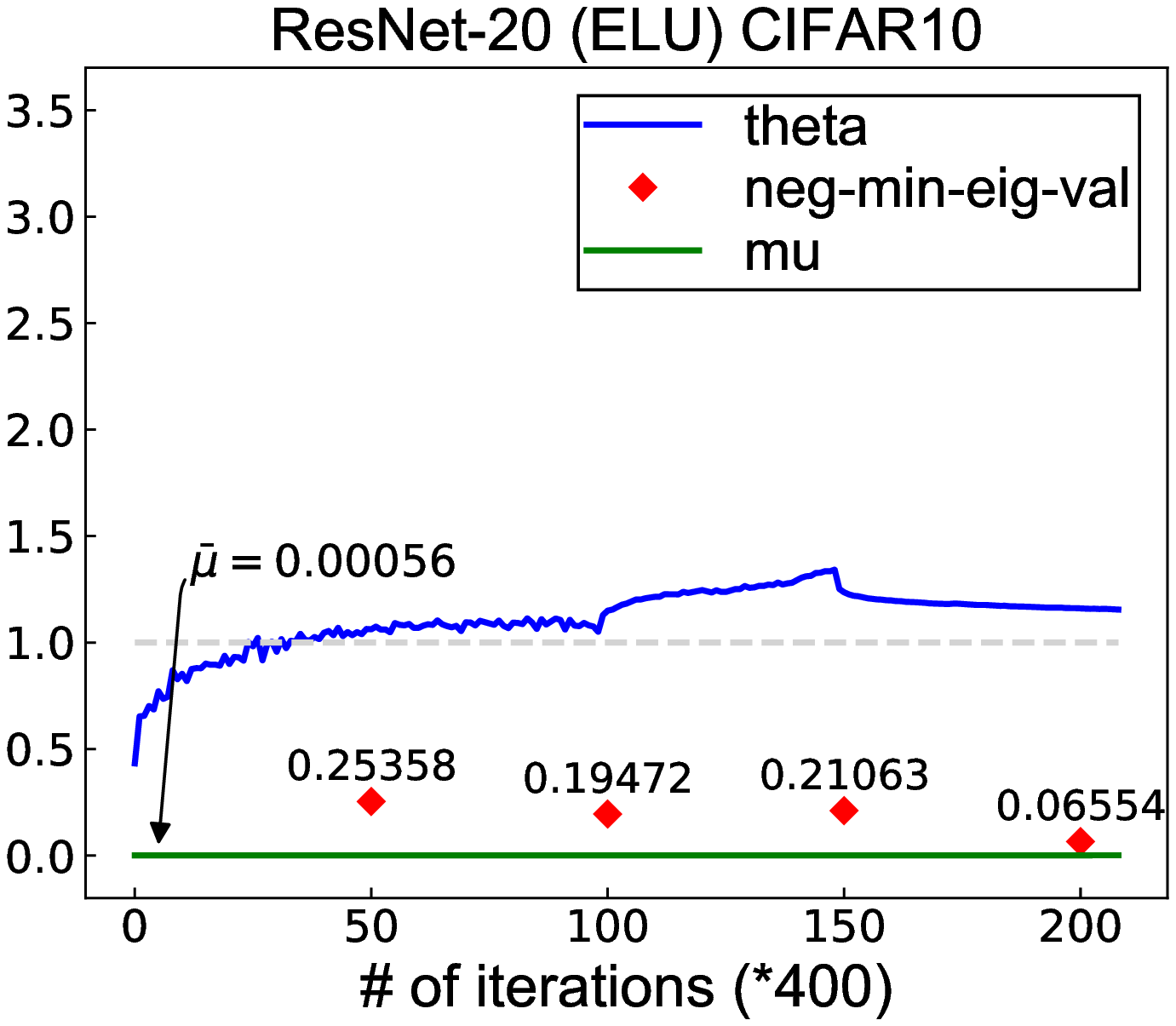}

\includegraphics[width=.225\textwidth]{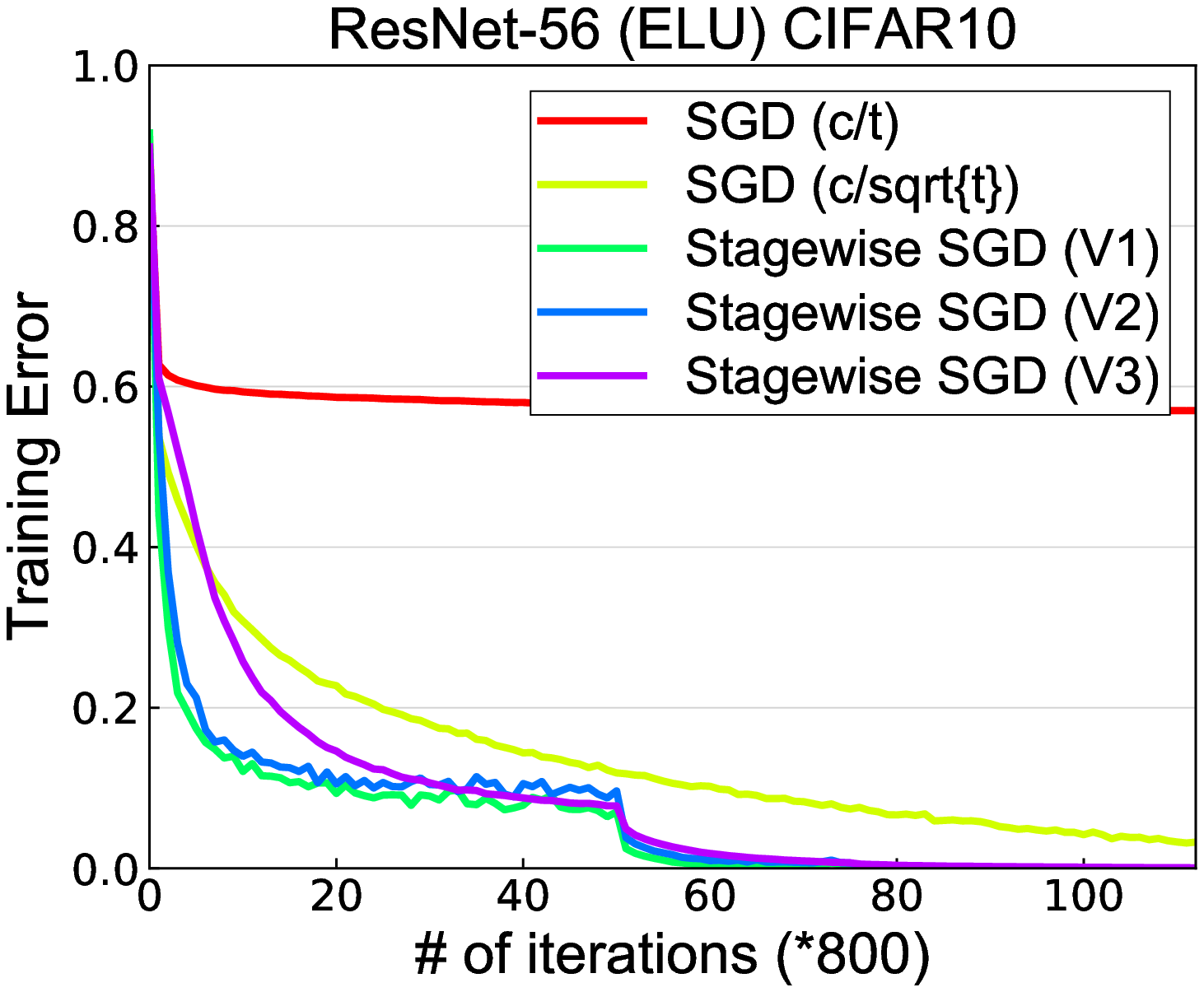}\hspace*{0.15in}
\includegraphics[width=.225\textwidth]{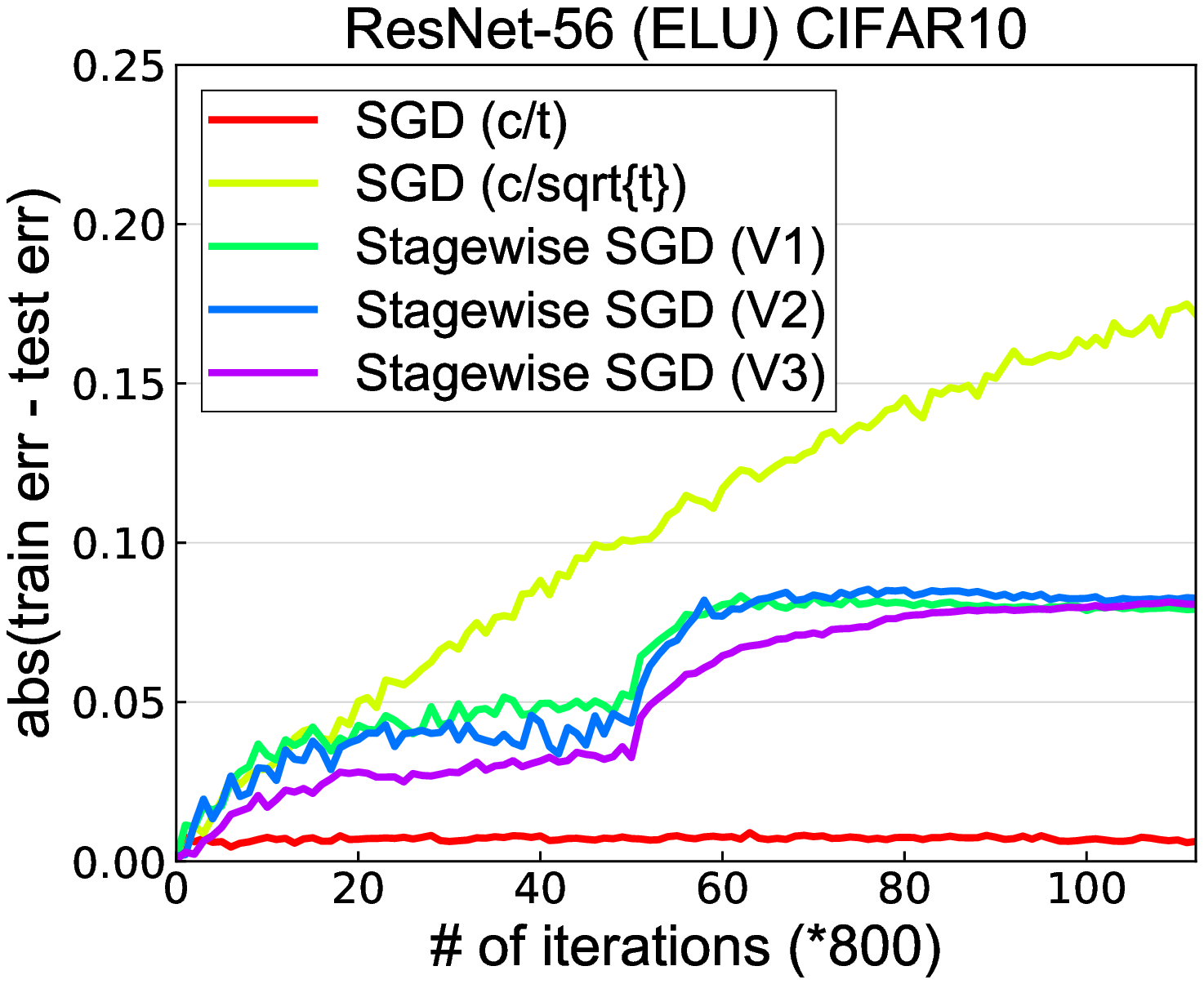}\hspace*{0.15in}
\includegraphics[width=.225\textwidth]{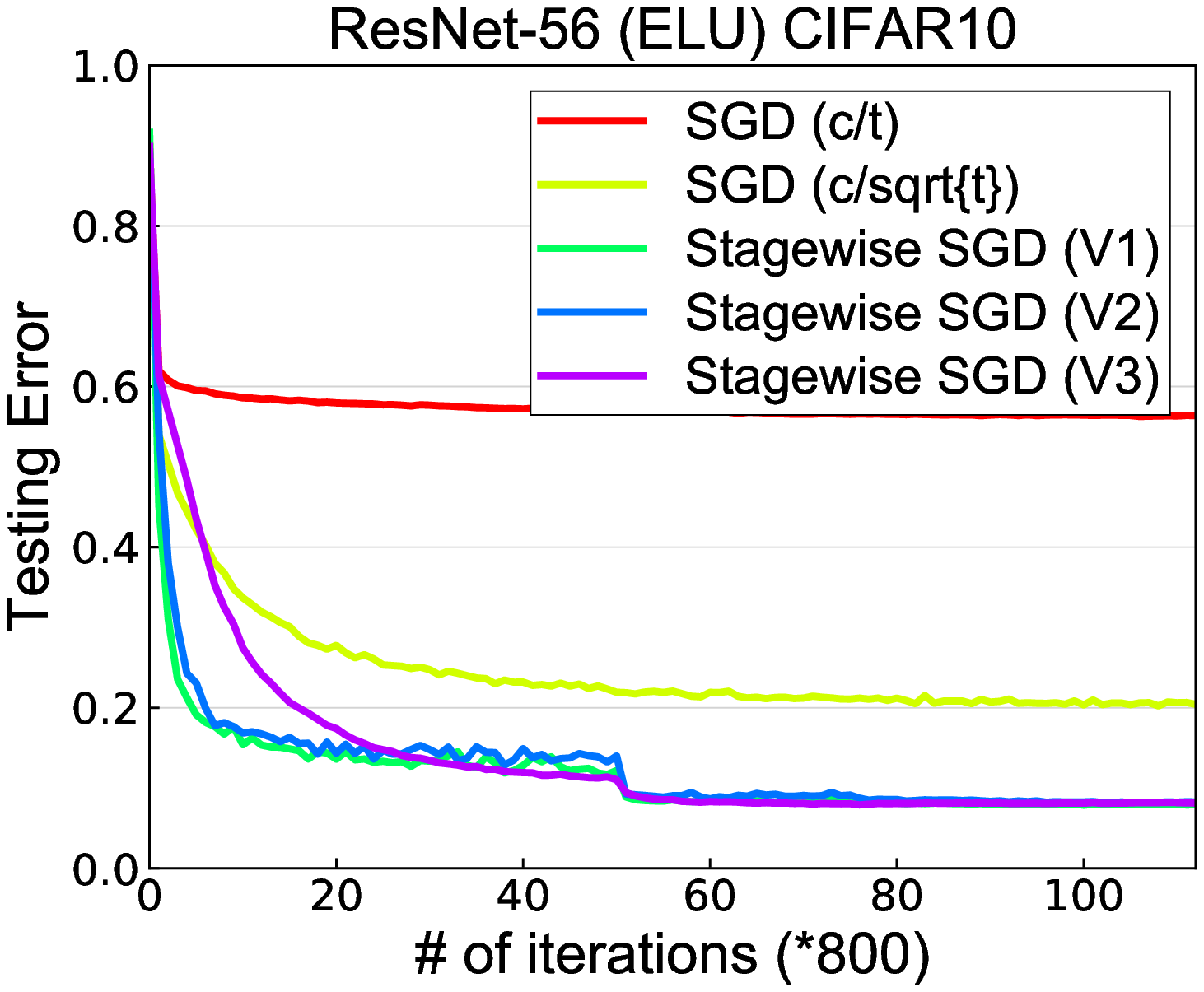}\hspace*{0.15in}
\includegraphics[width=.225\textwidth]{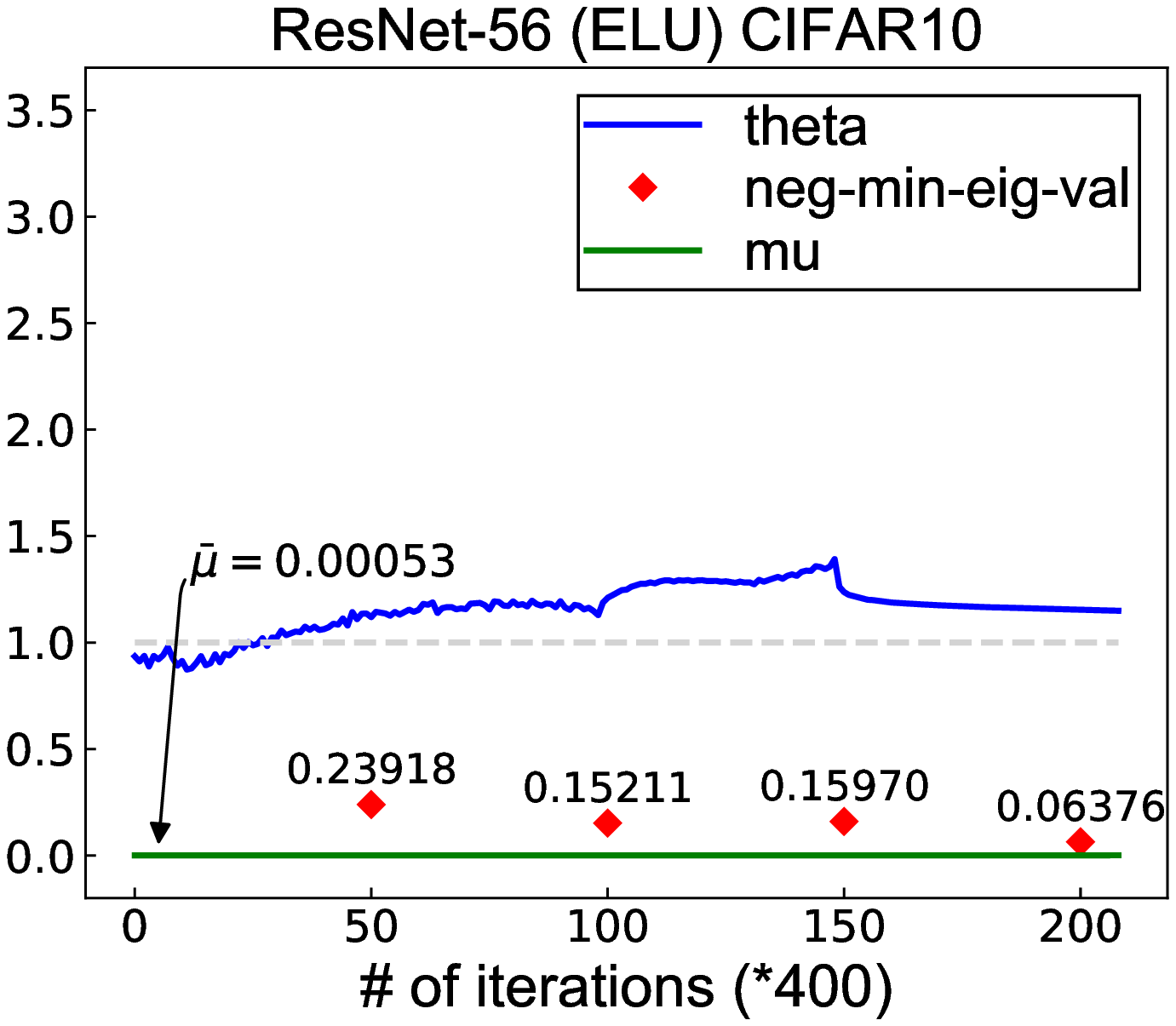}

\includegraphics[width=.225\textwidth]{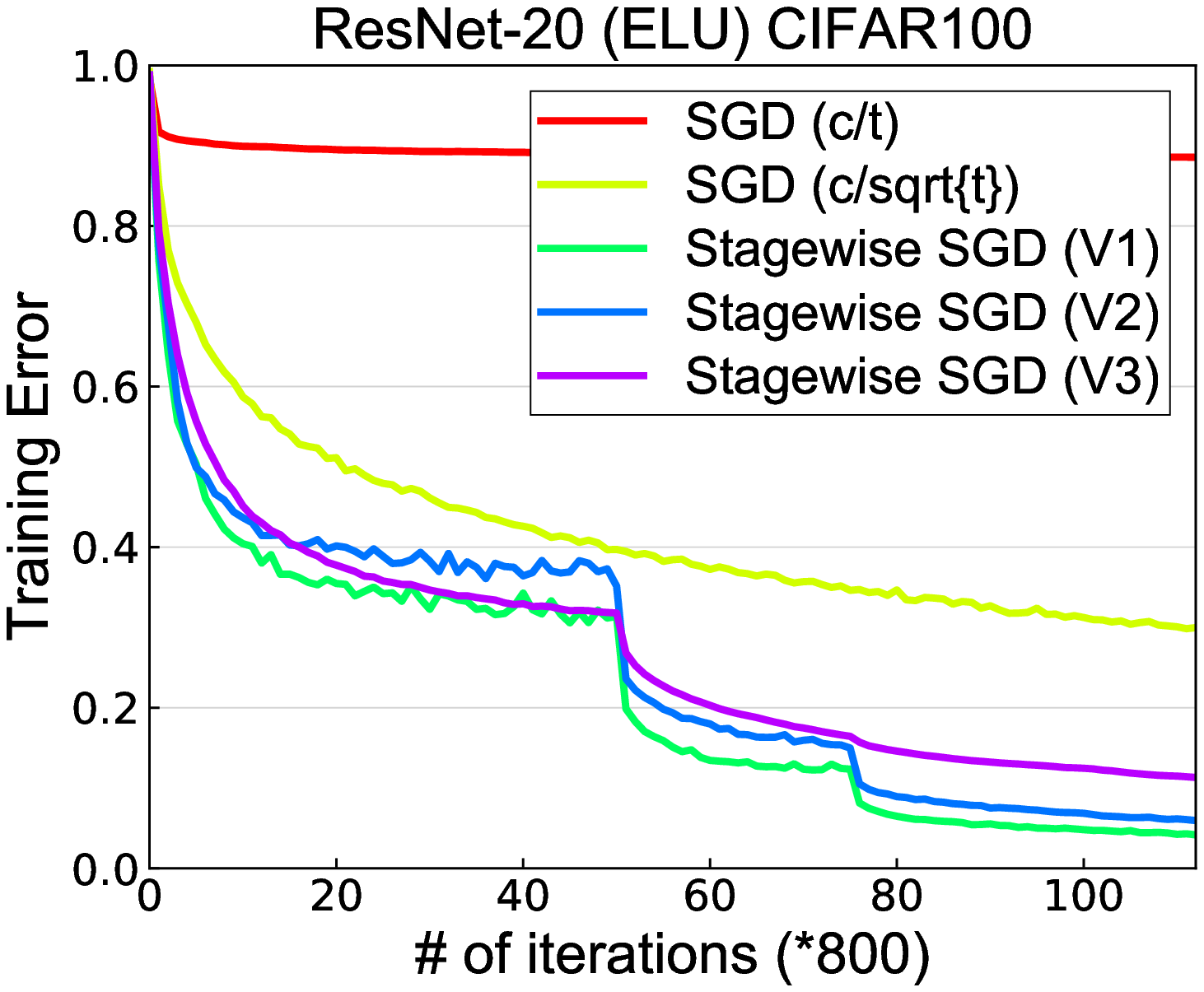}\hspace*{0.15in}
\includegraphics[width=.225\textwidth]{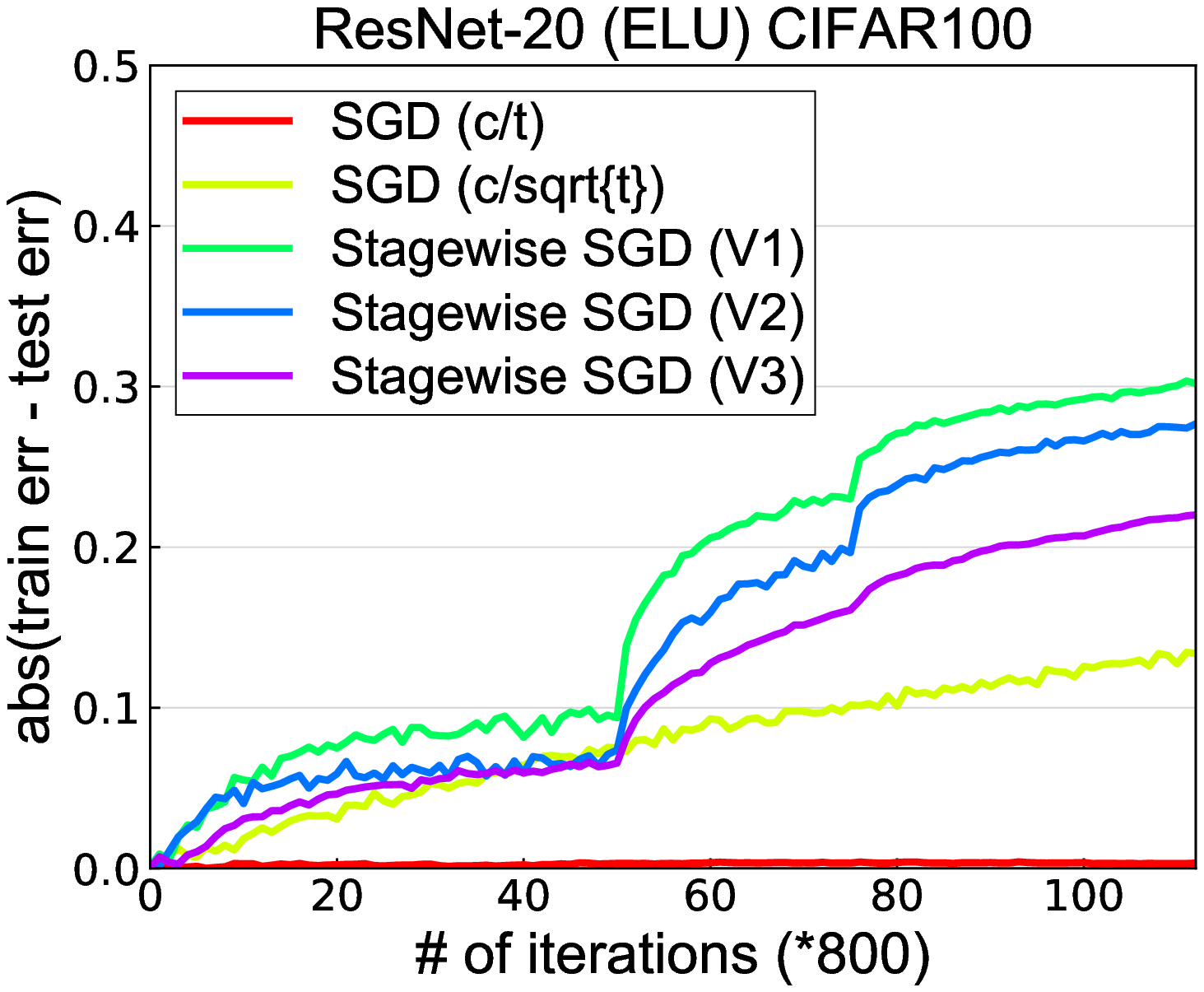}\hspace*{0.15in}
\includegraphics[width=.225\textwidth]{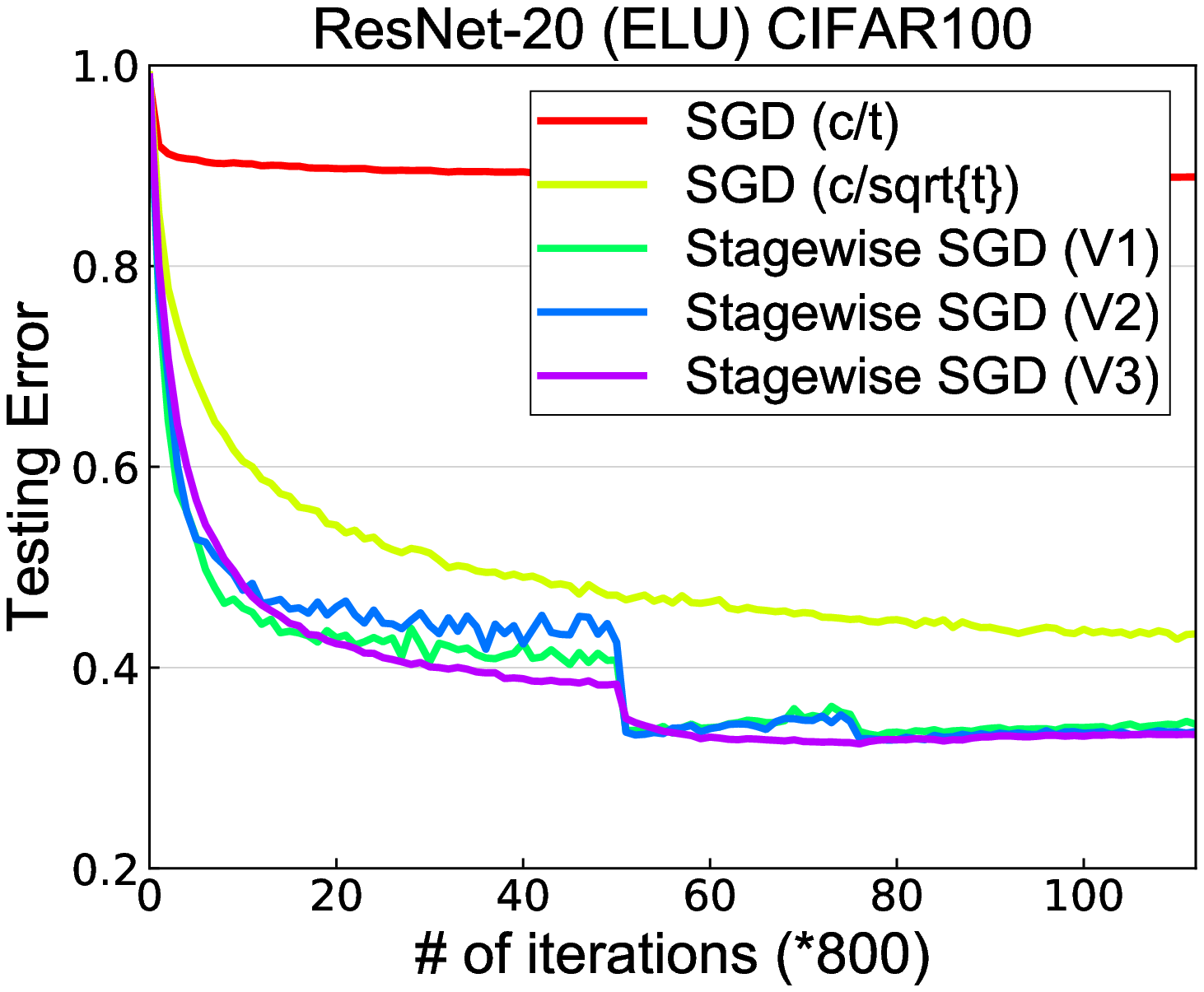}\hspace*{0.15in}
\includegraphics[width=.225\textwidth]{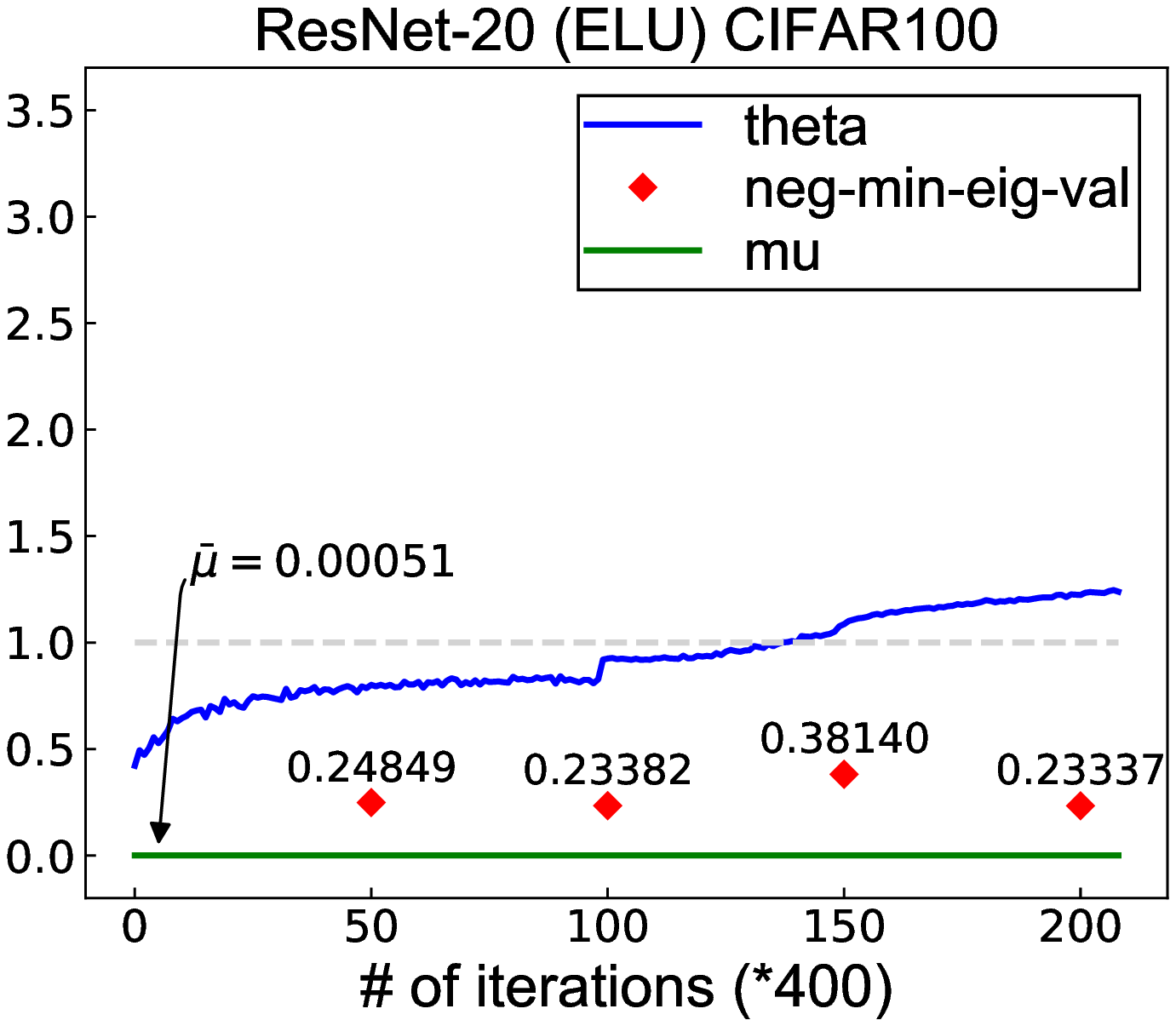}

\includegraphics[width=.225\textwidth]{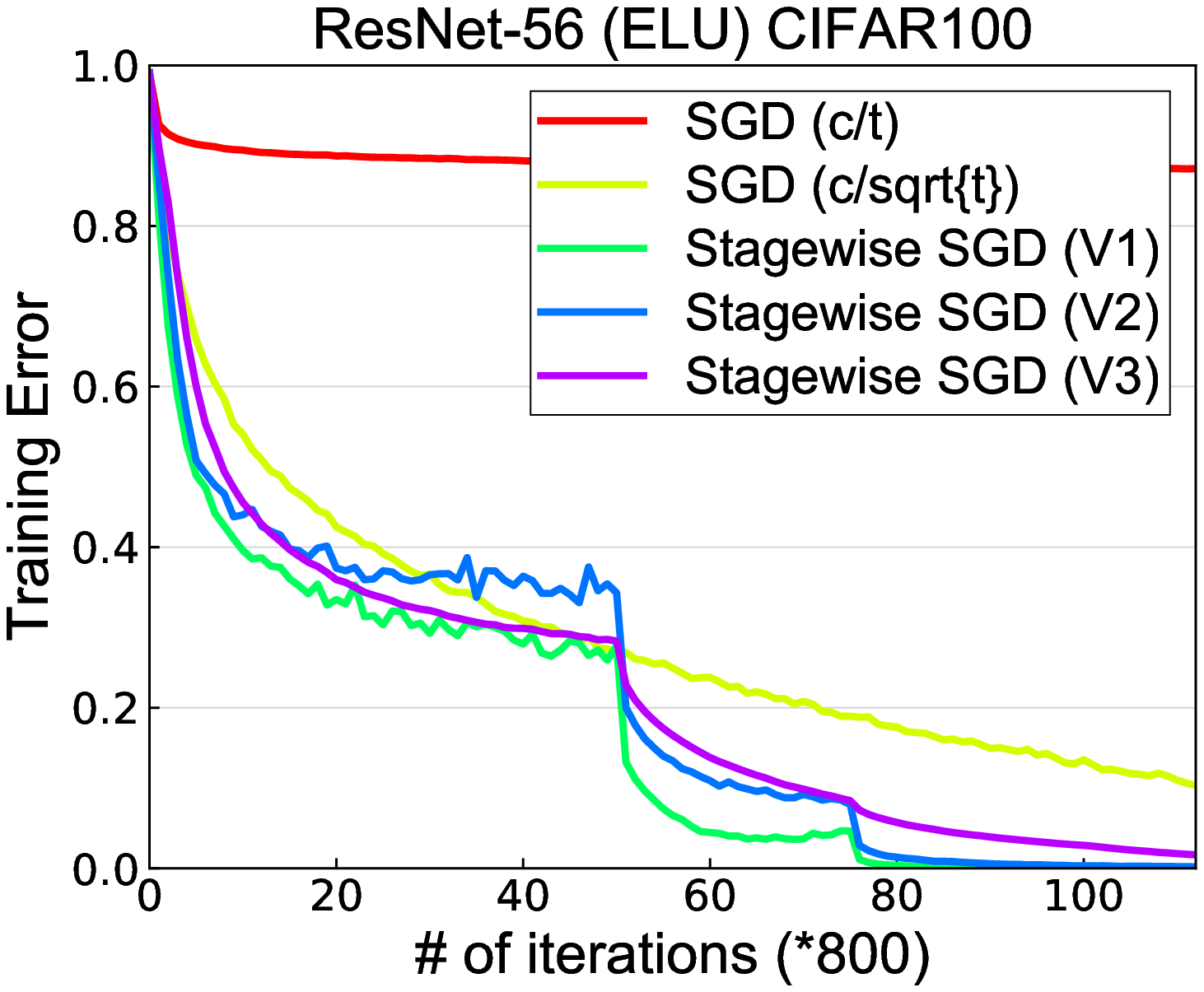}\hspace*{0.15in}
\includegraphics[width=.225\textwidth]{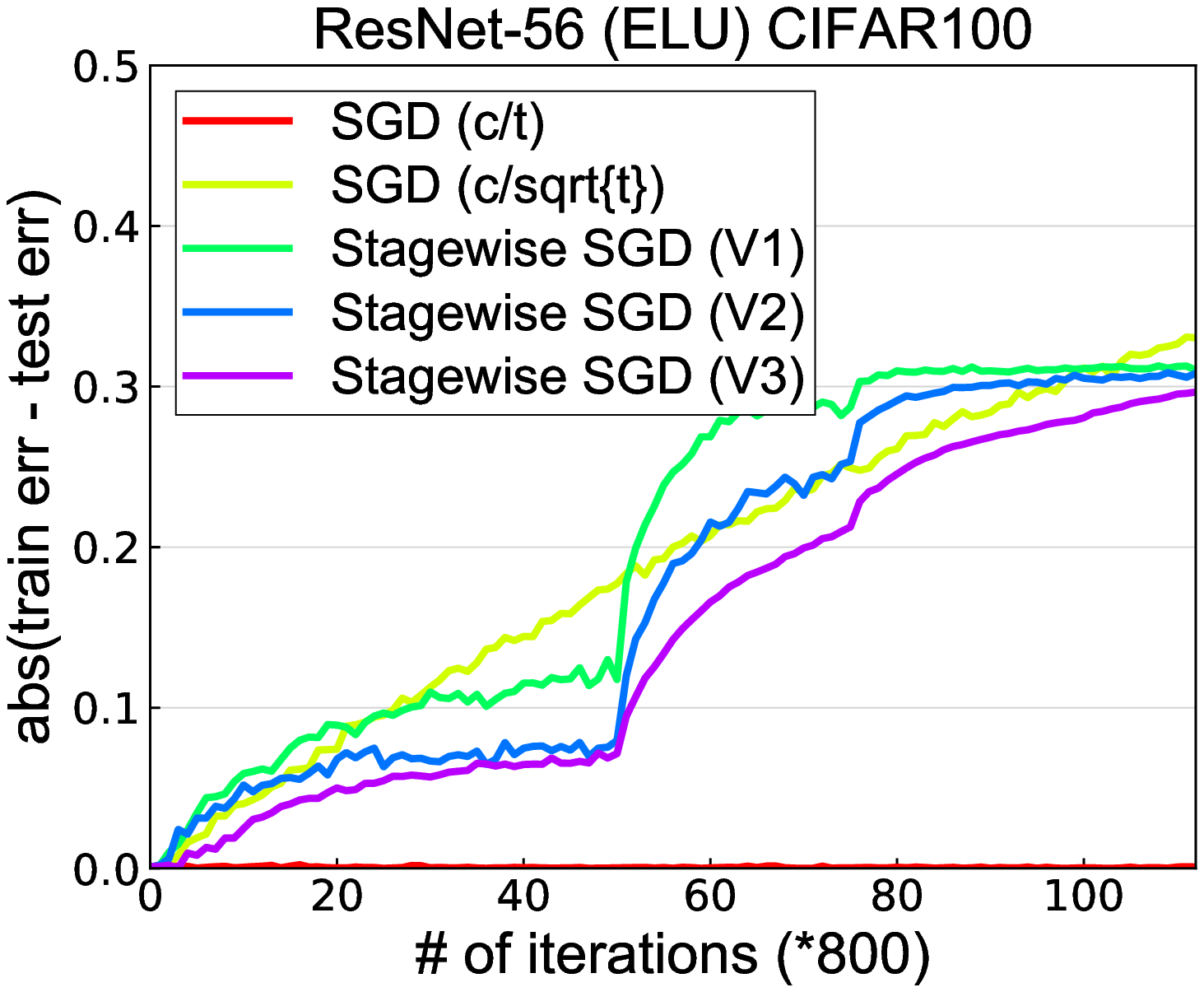}\hspace*{0.15in}
\includegraphics[width=.225\textwidth]{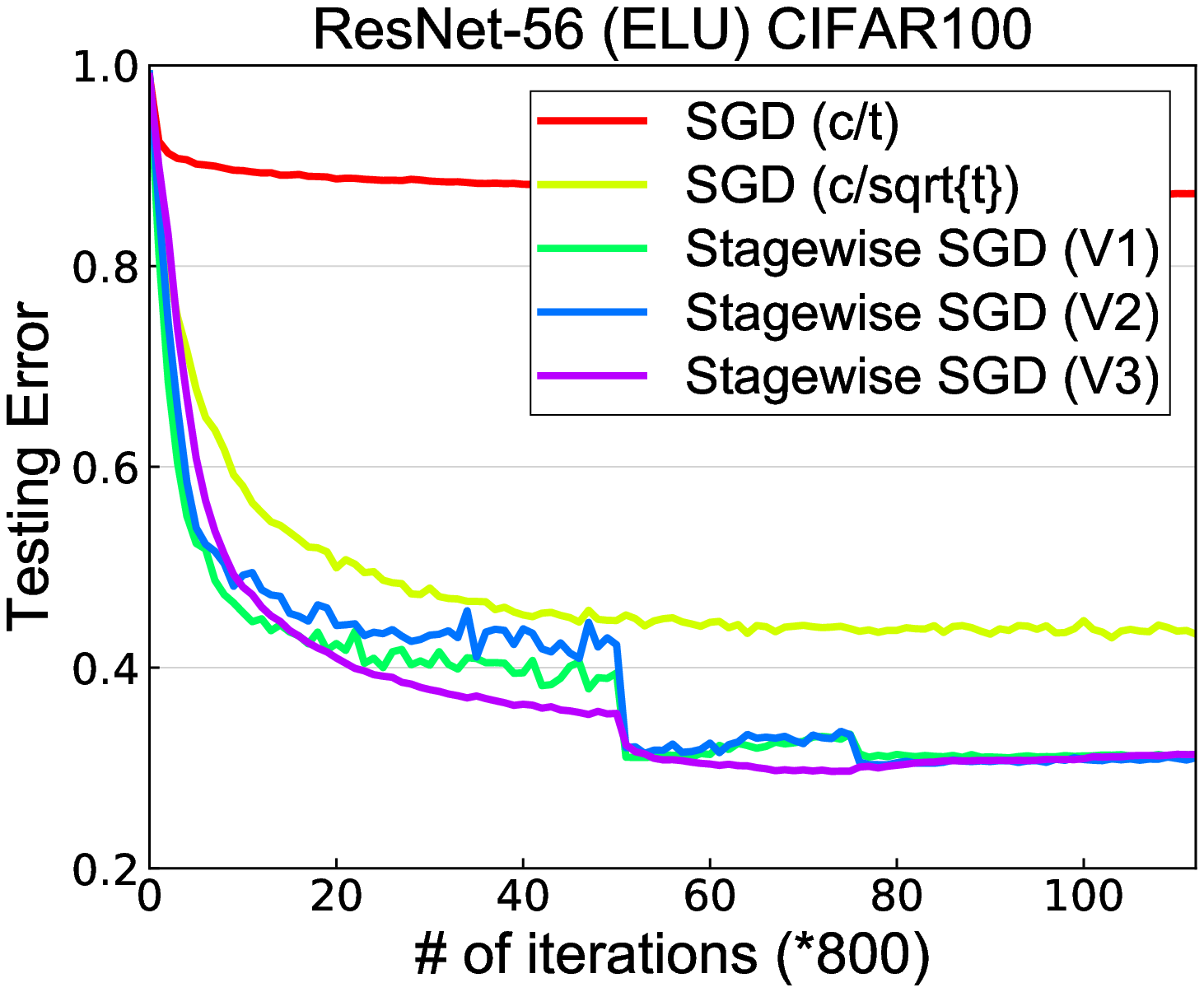}\hspace*{0.15in}
\includegraphics[width=.225\textwidth]{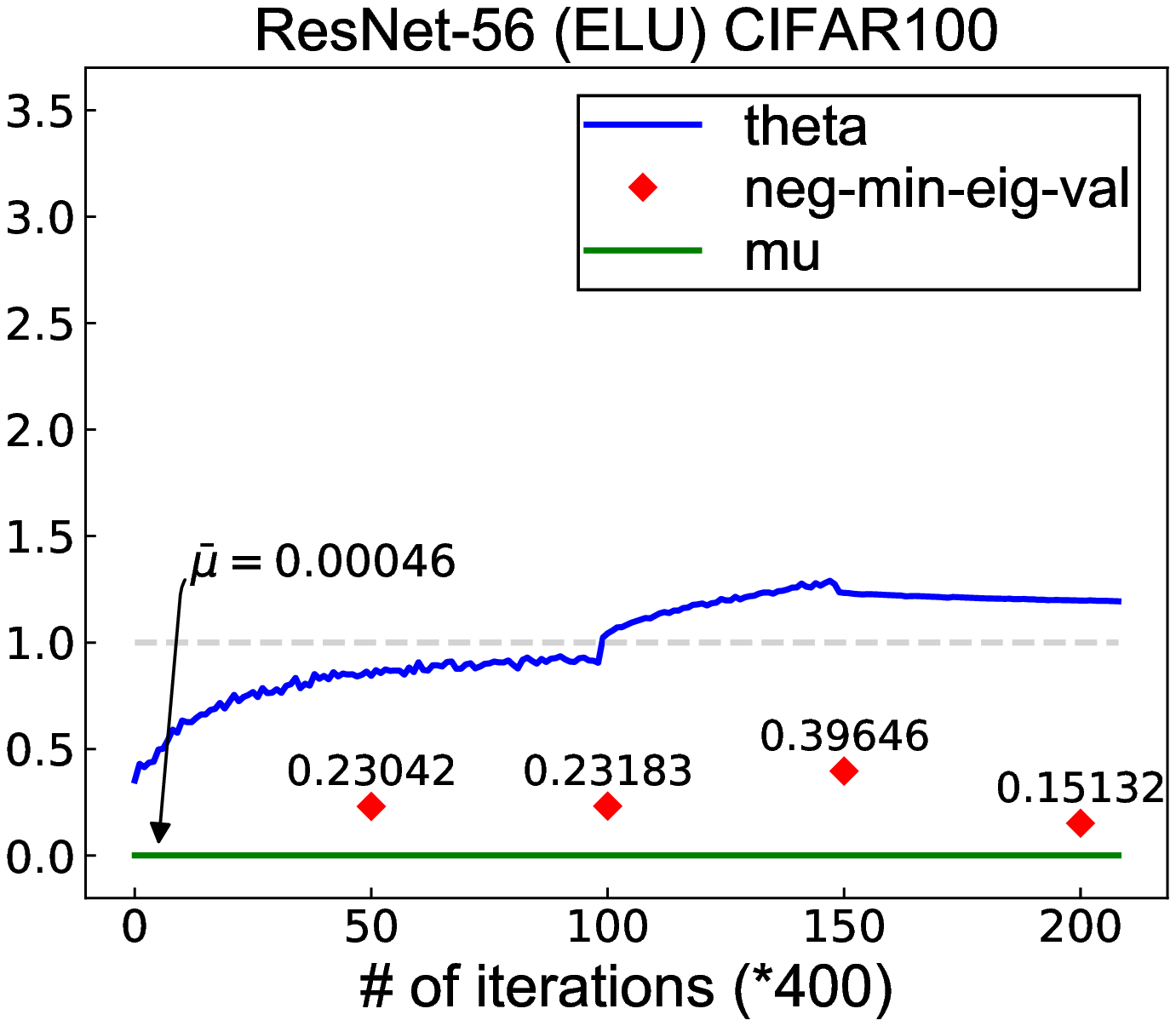}

\includegraphics[width=.225\textwidth]{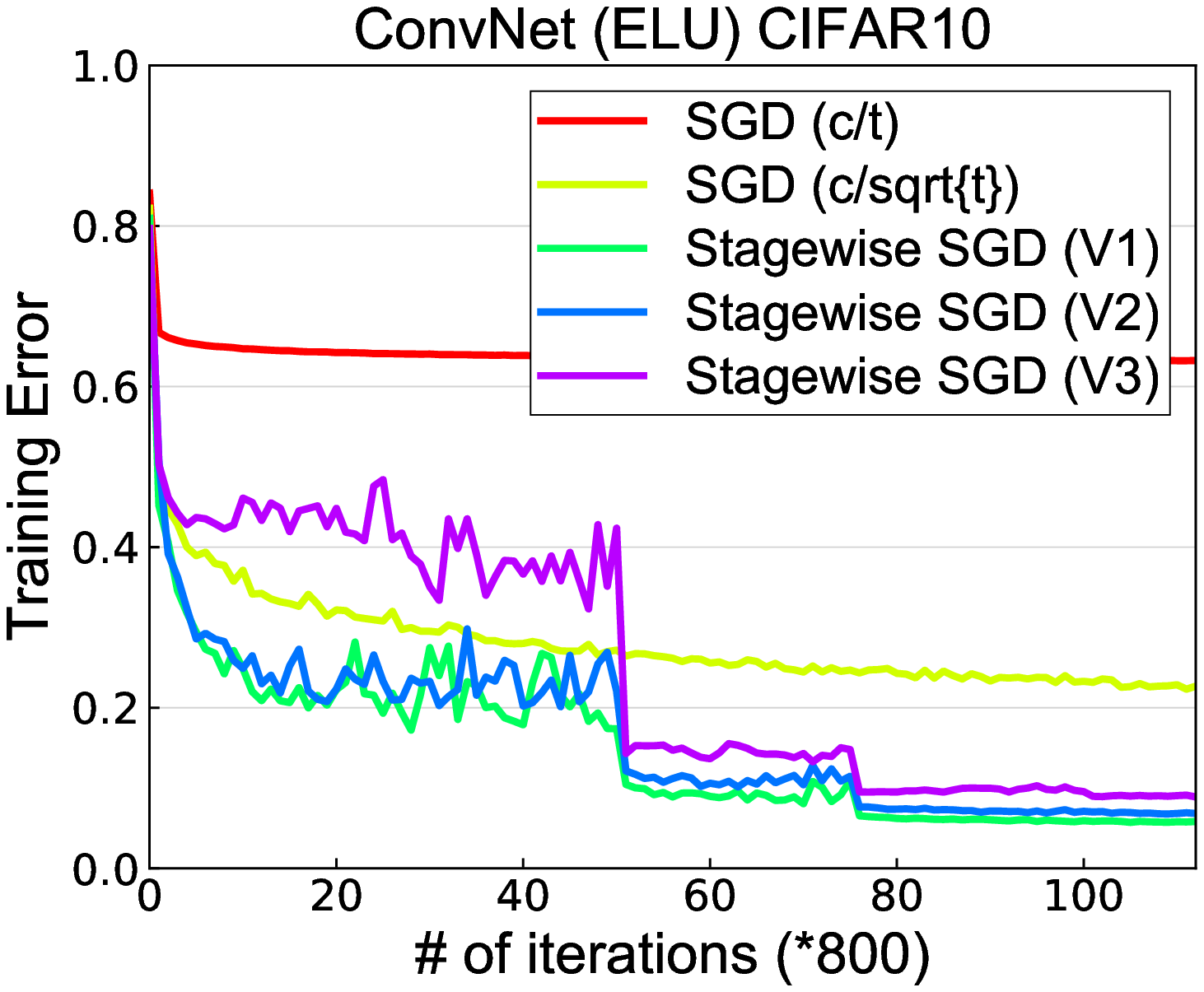}\hspace*{0.15in}
\includegraphics[width=.225\textwidth]{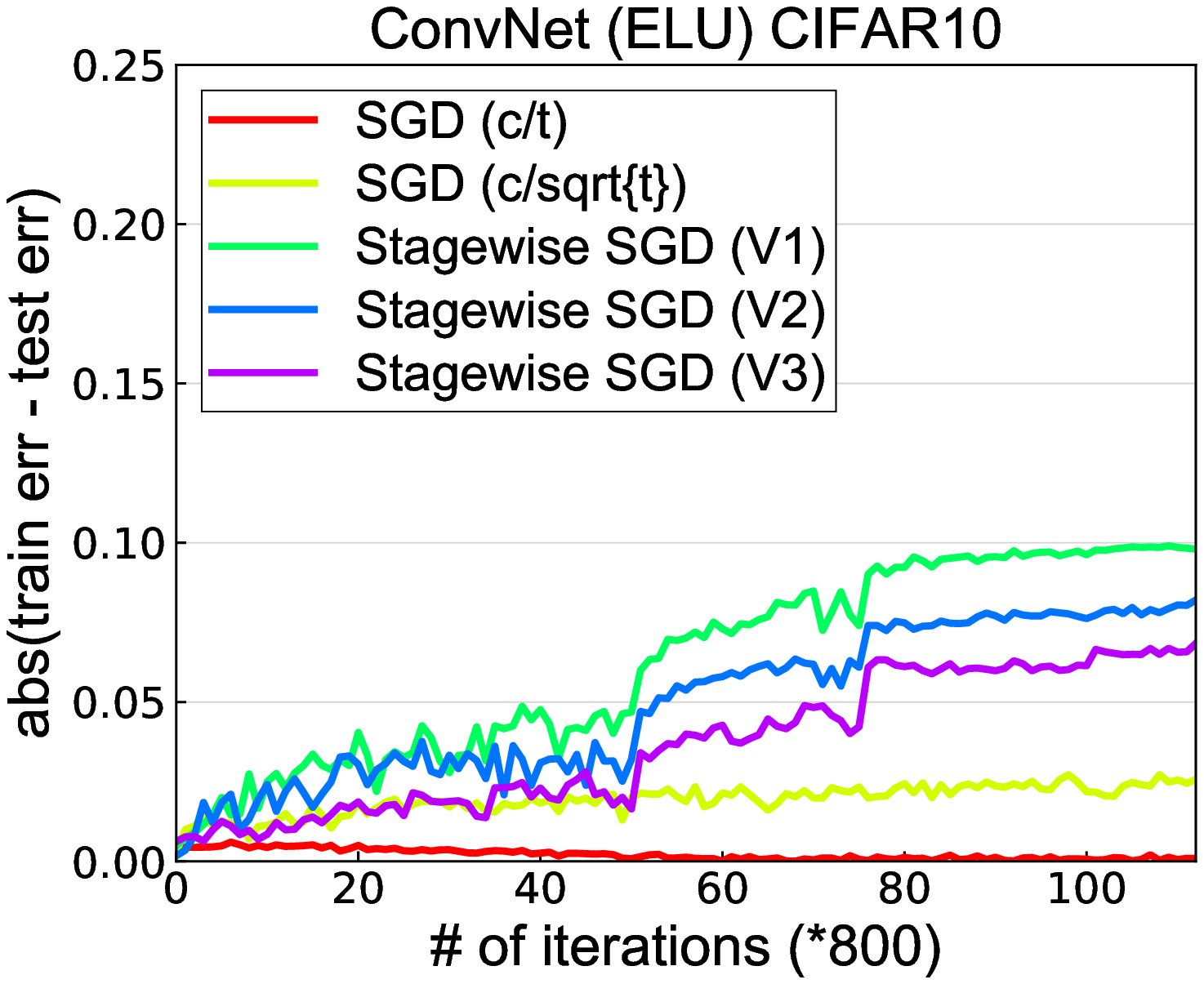}\hspace*{0.15in}
\includegraphics[width=.225\textwidth]{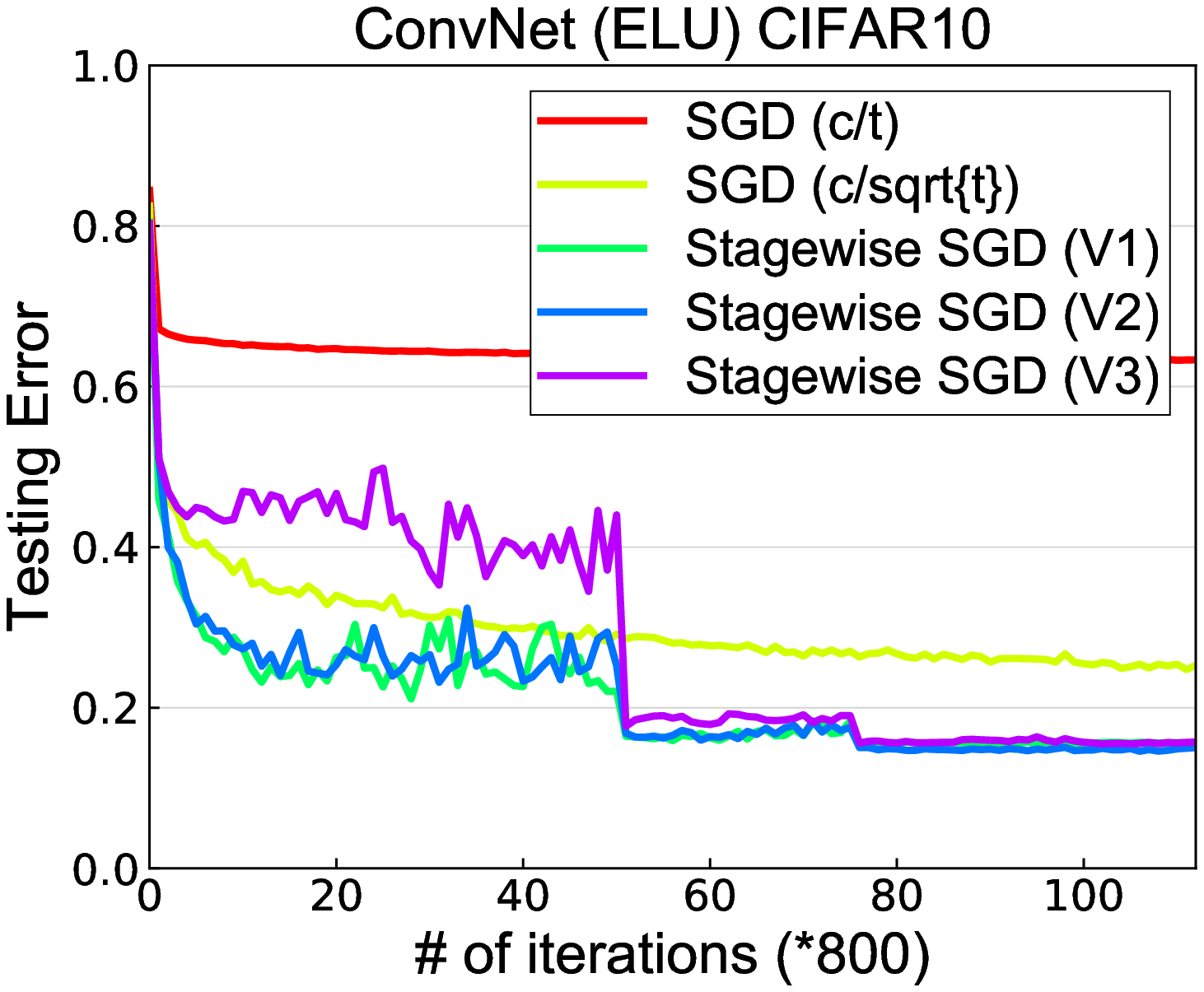}\hspace*{0.15in}
\includegraphics[width=.225\textwidth]{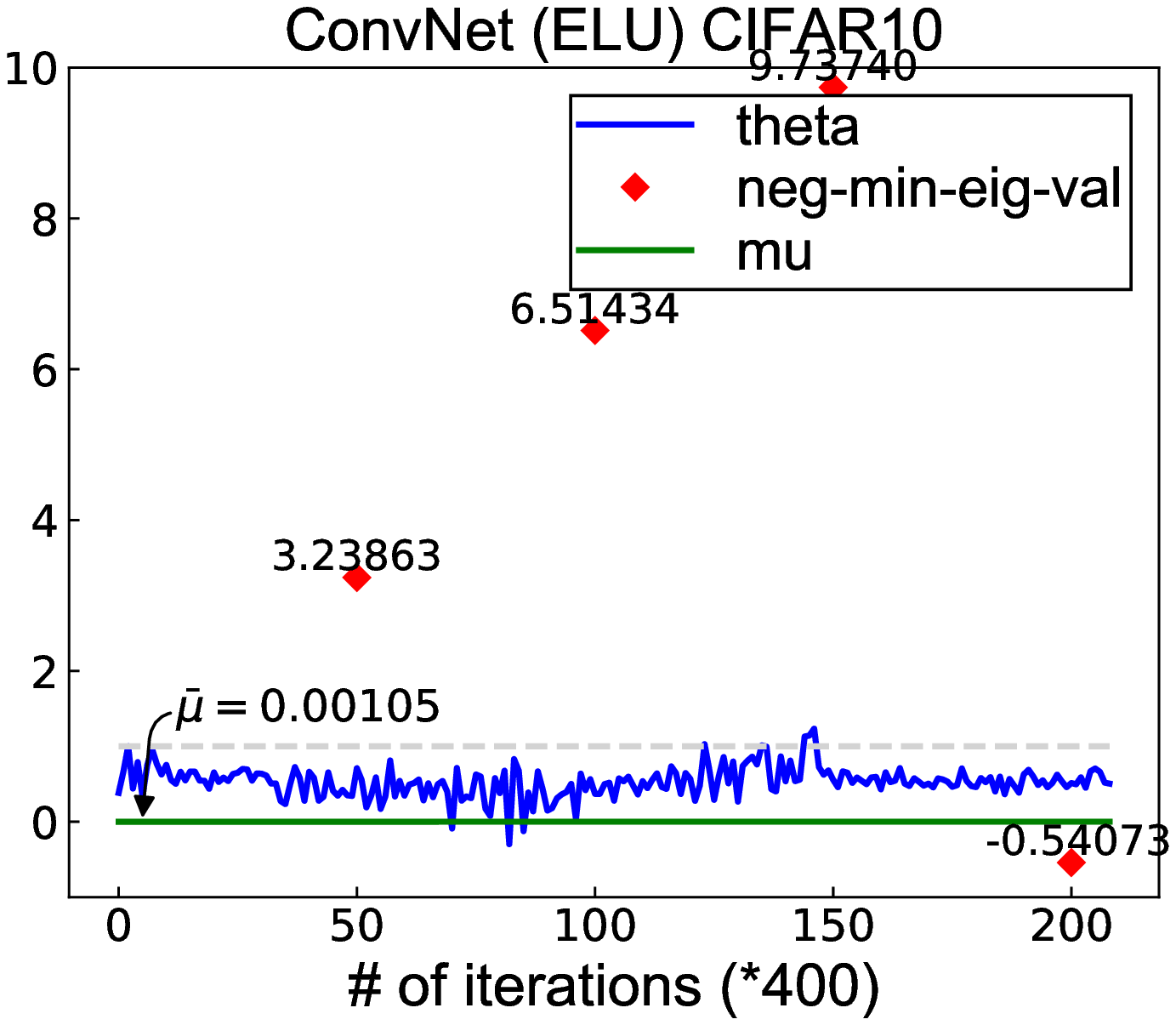}

\vspace*{-0.1in}
\caption{From left to right: training, generalization and testing error, and verifying assumptions for stagewise learning of ResNets and ConvNet using ELU with weight decay 
         ($5 * 10^{-4} ||\w||_2^2$ regularization). 
         The computation of $\theta$, $\mu$ and minimum eigen-value takes the regularization term into account.  }
\label{fig:resnet_convnet_ELU_w_L2}
\vspace*{-0.2in}
\end{figure*}

\begin{figure*}[t]
\centering
\includegraphics[width=.225\textwidth]{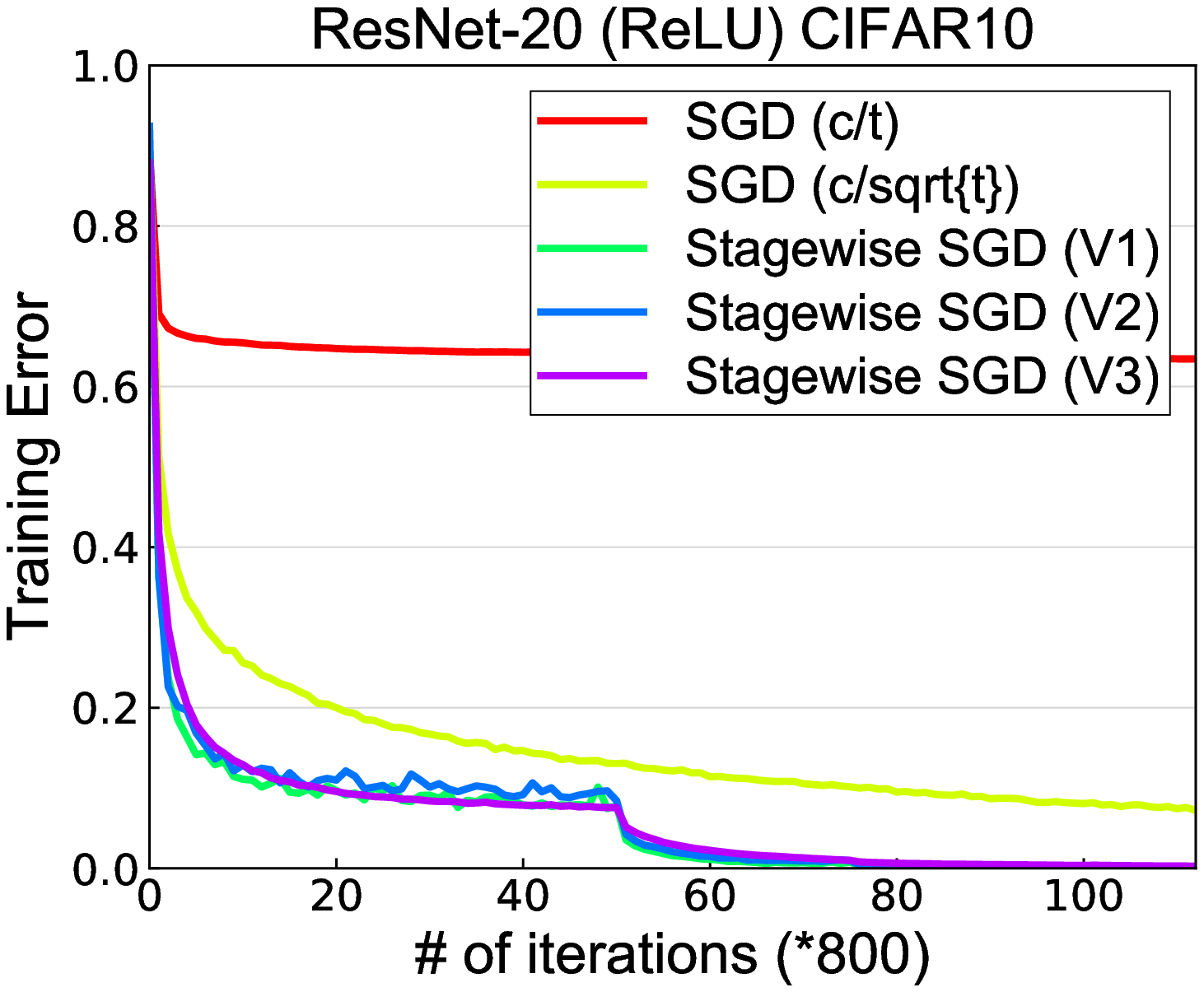}\hspace*{0.15in}
\includegraphics[width=.225\textwidth]{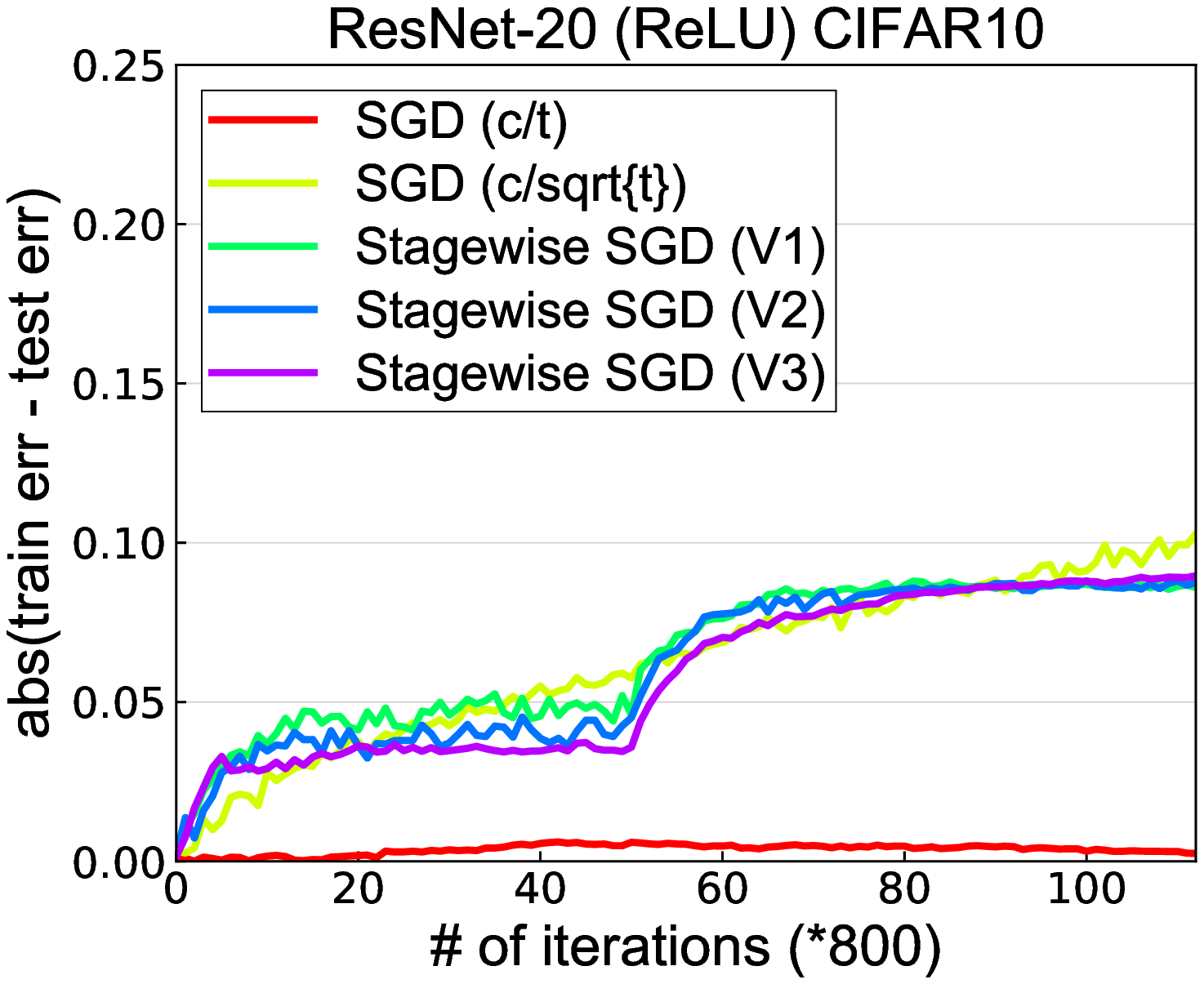}\hspace*{0.15in}
\includegraphics[width=.225\textwidth]{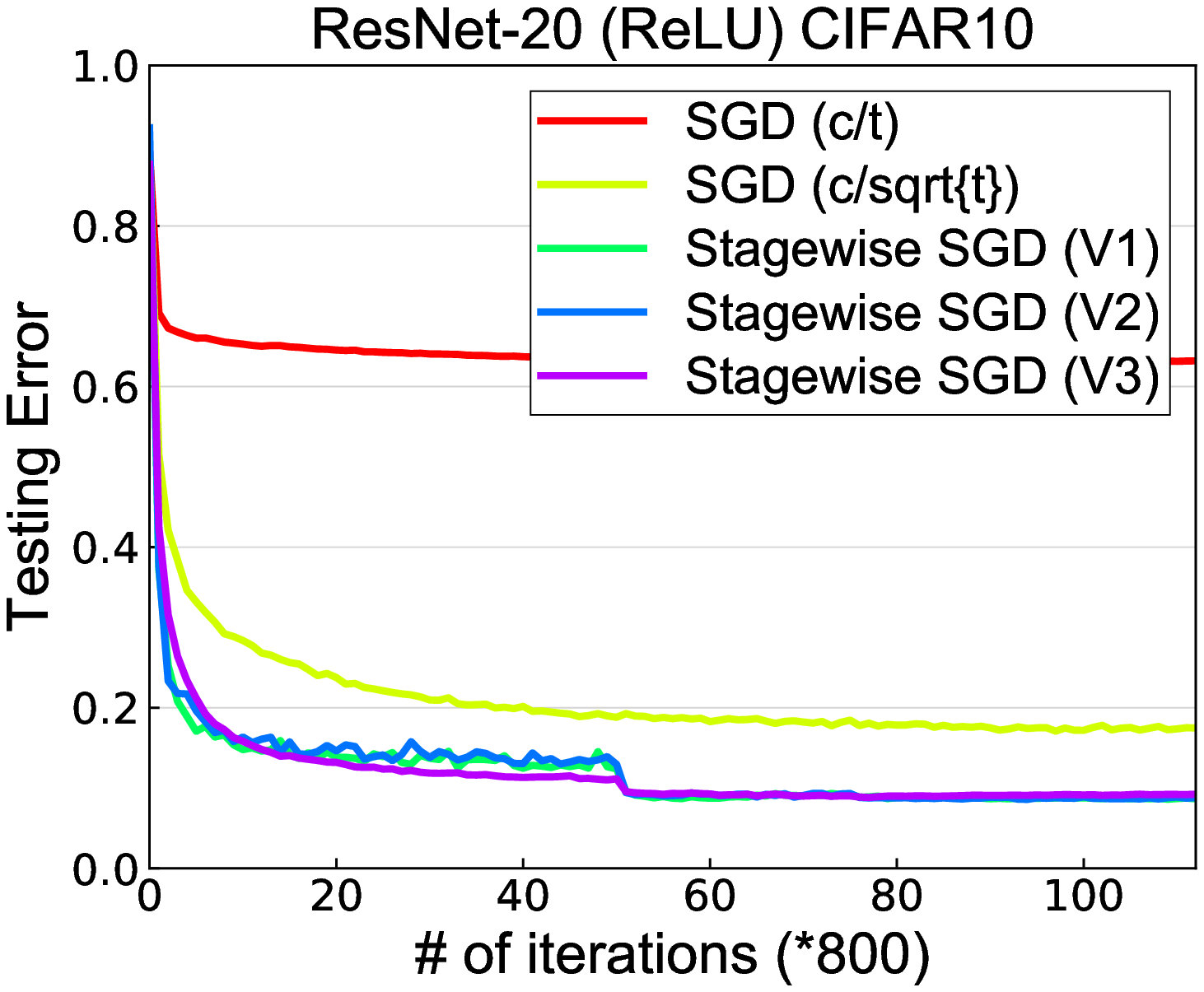}\hspace*{0.15in}
\includegraphics[width=.225\textwidth]{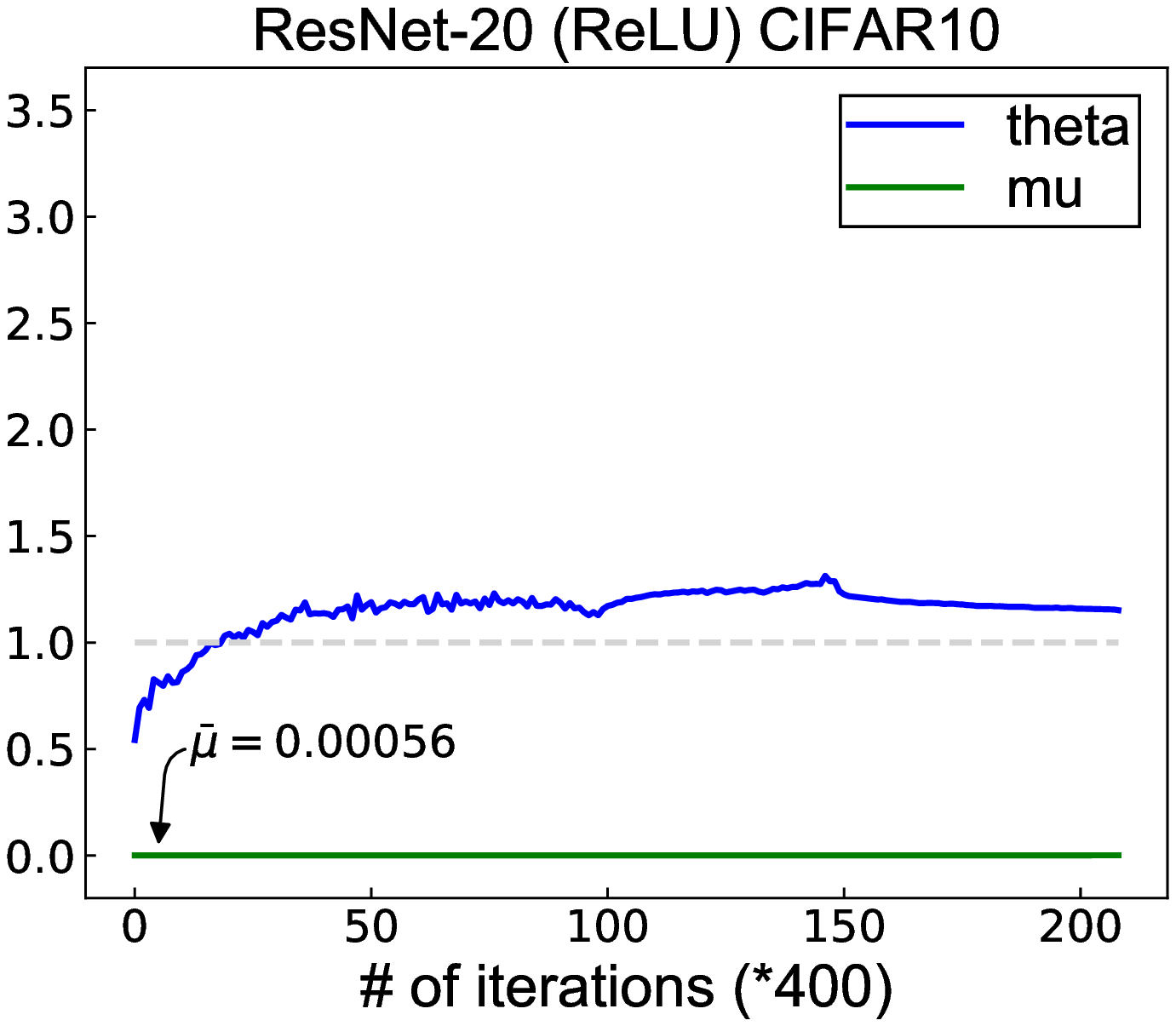}

\includegraphics[width=.225\textwidth]{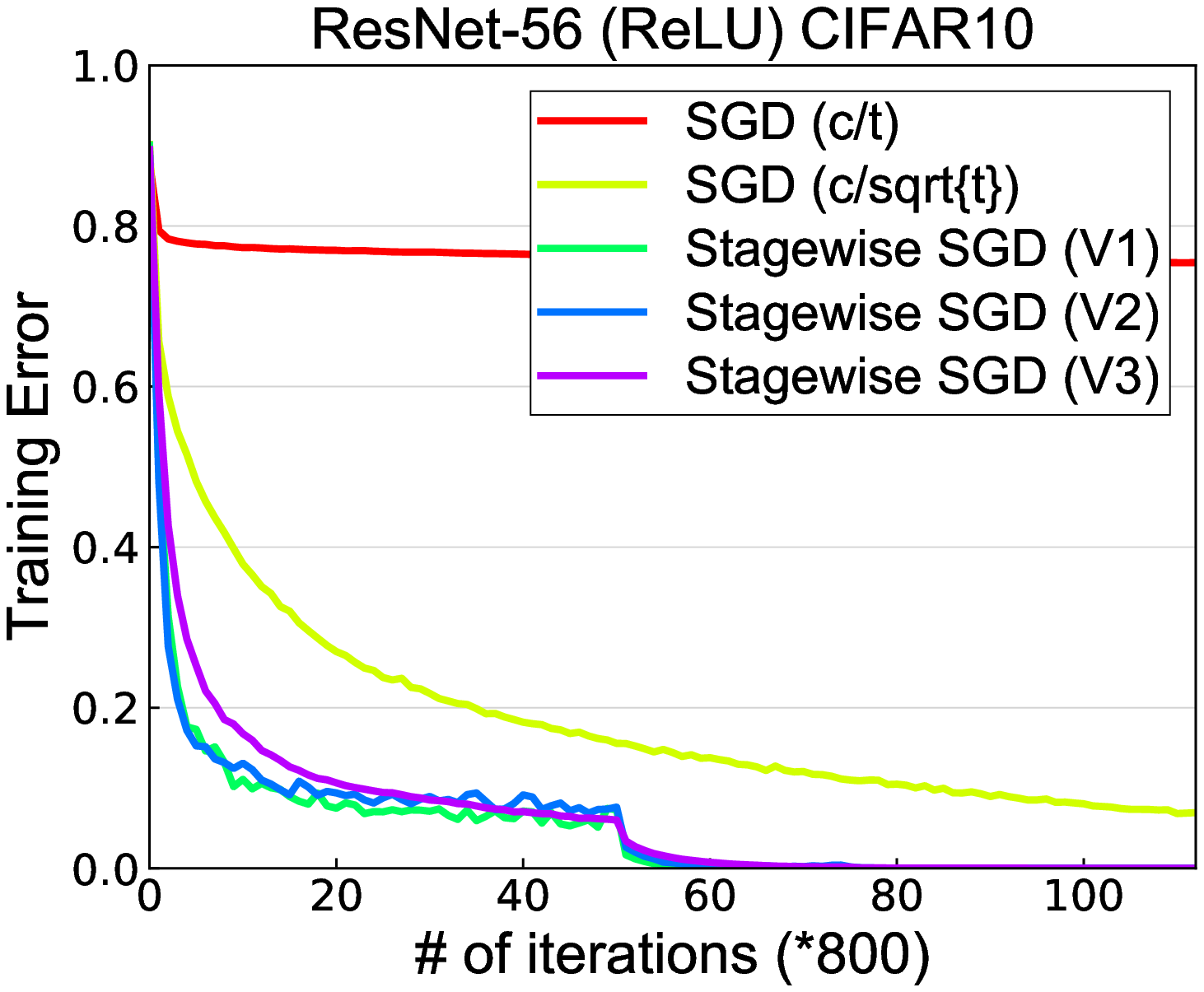}\hspace*{0.15in}
\includegraphics[width=.225\textwidth]{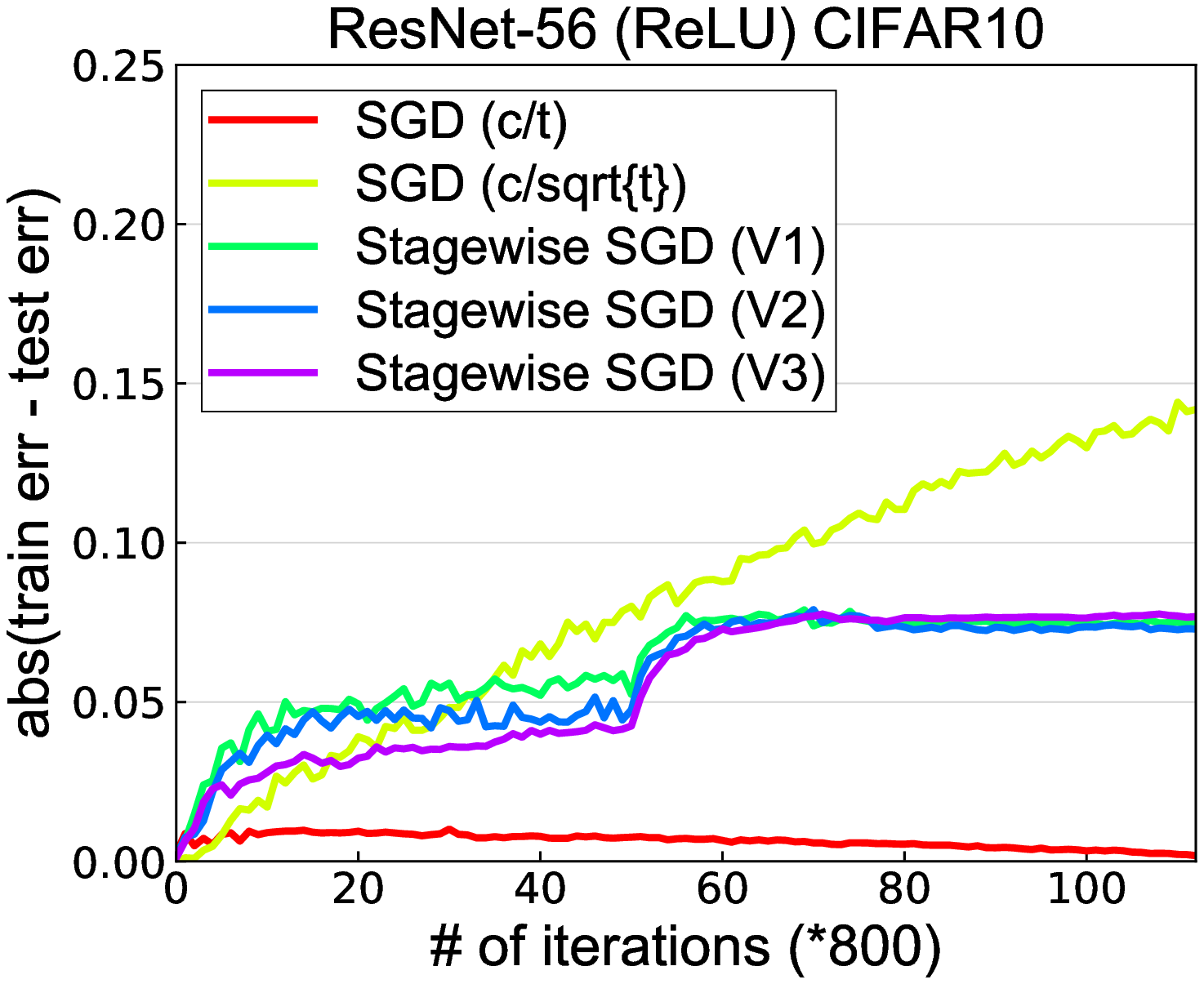}\hspace*{0.15in}
\includegraphics[width=.225\textwidth]{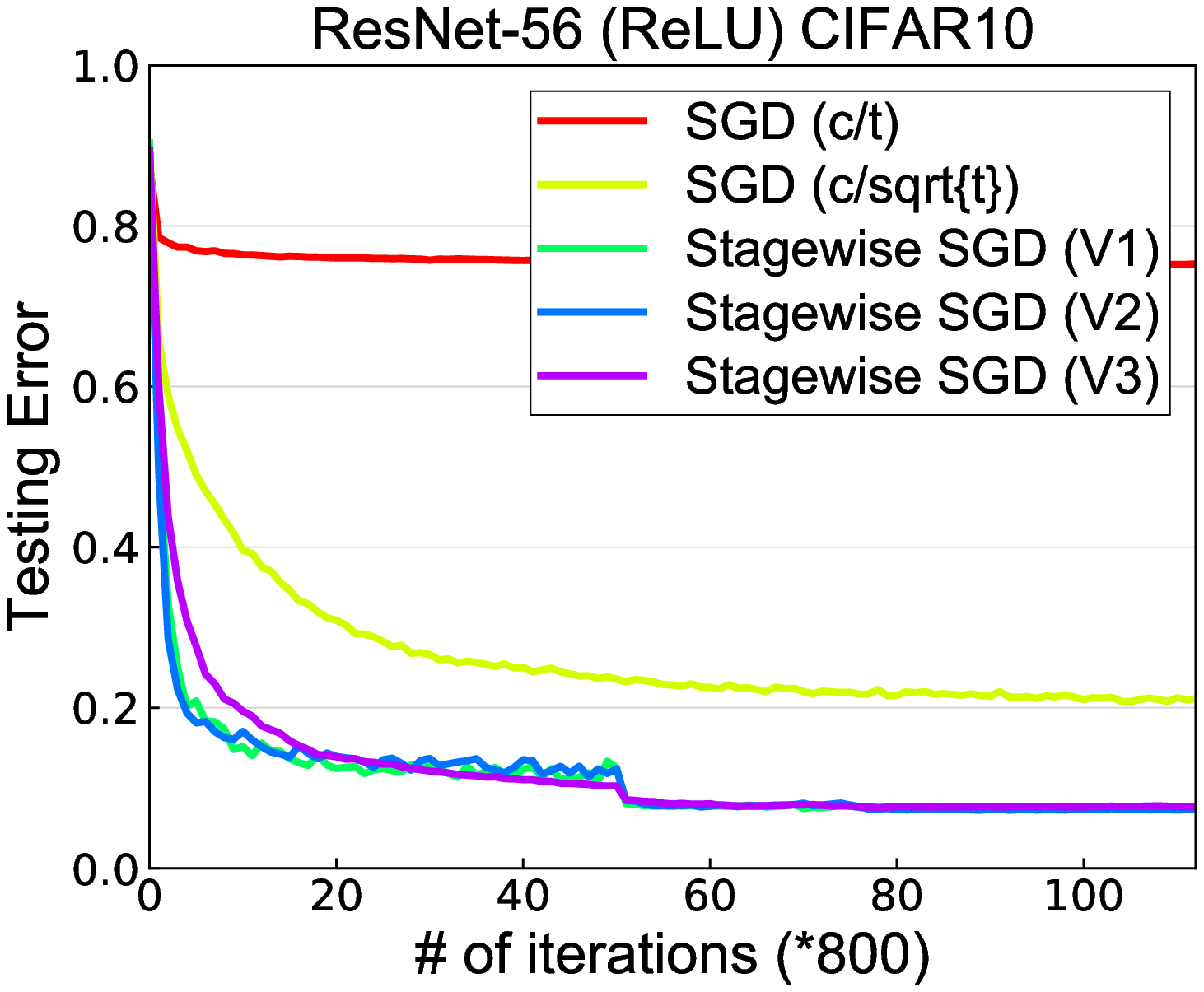}\hspace*{0.15in}
\includegraphics[width=.225\textwidth]{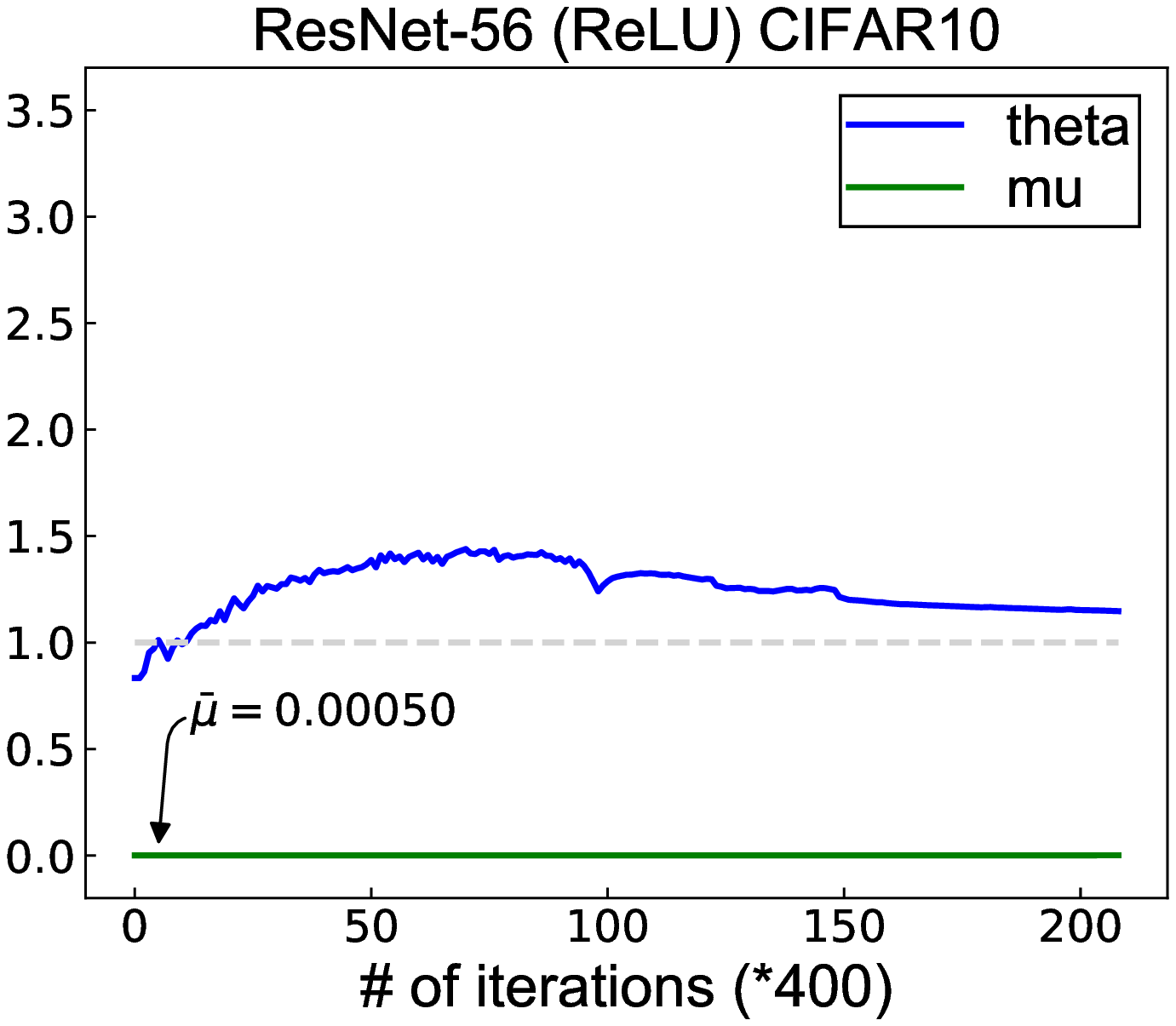}

\includegraphics[width=.225\textwidth]{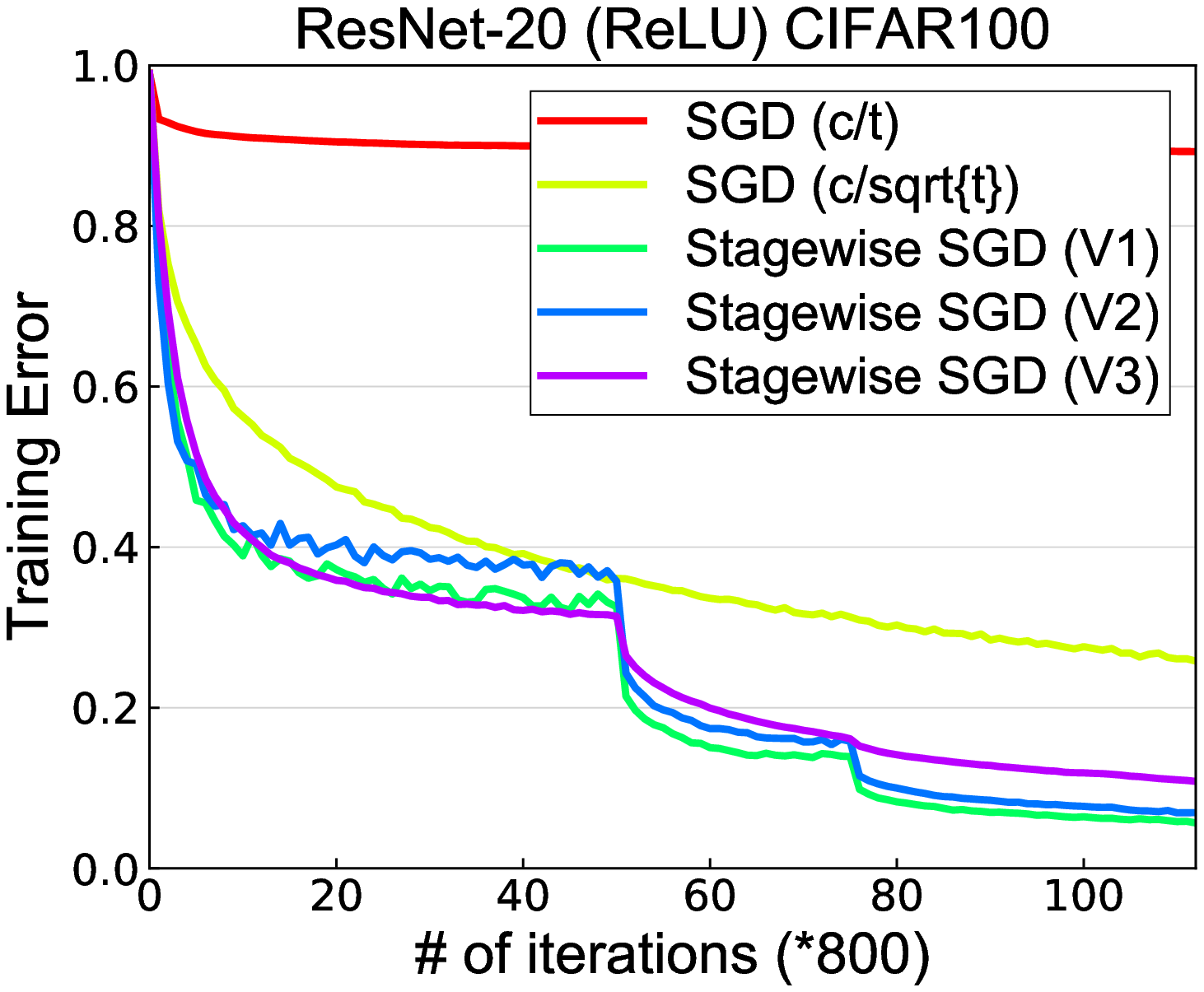}\hspace*{0.15in}
\includegraphics[width=.225\textwidth]{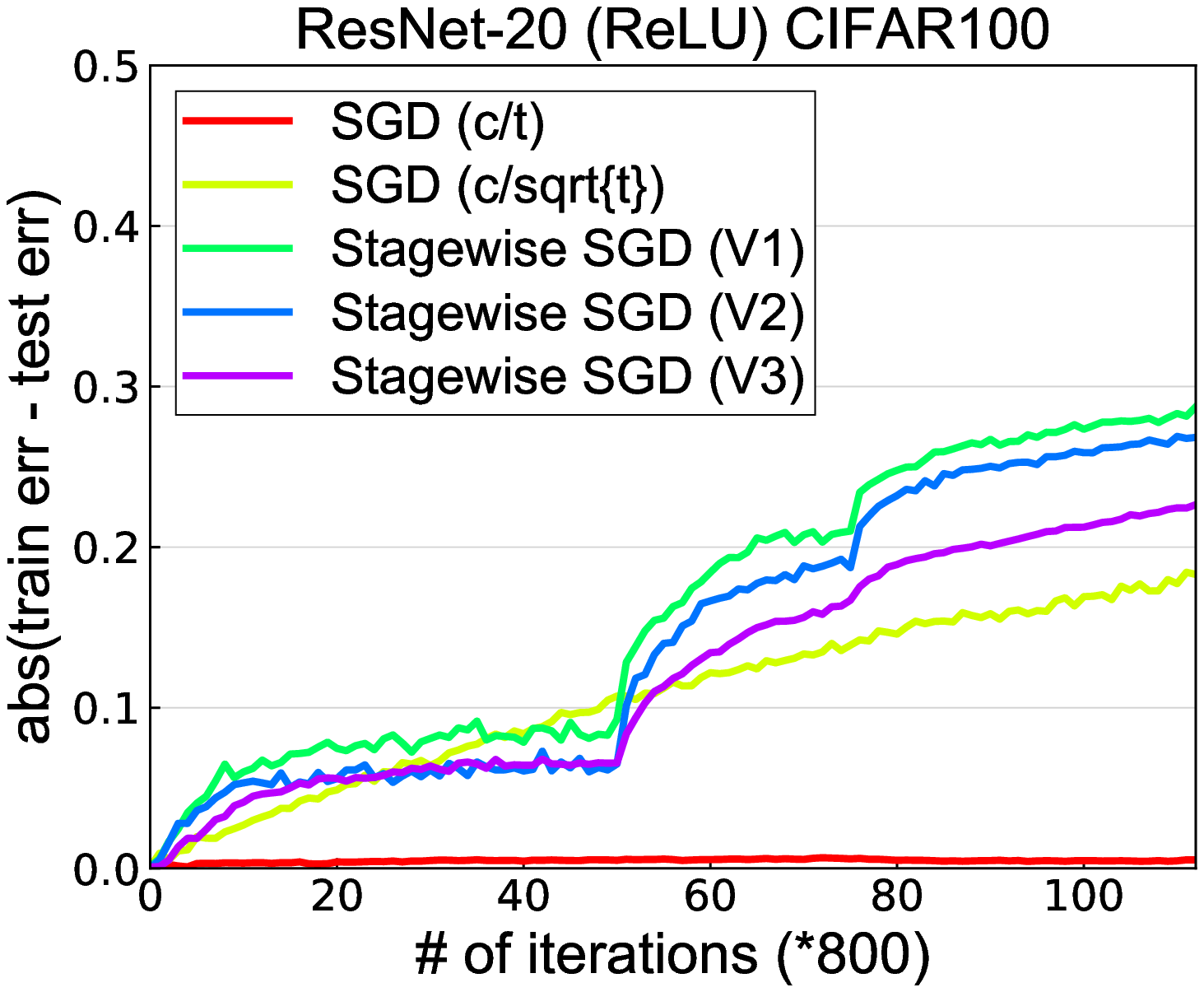}\hspace*{0.15in}
\includegraphics[width=.225\textwidth]{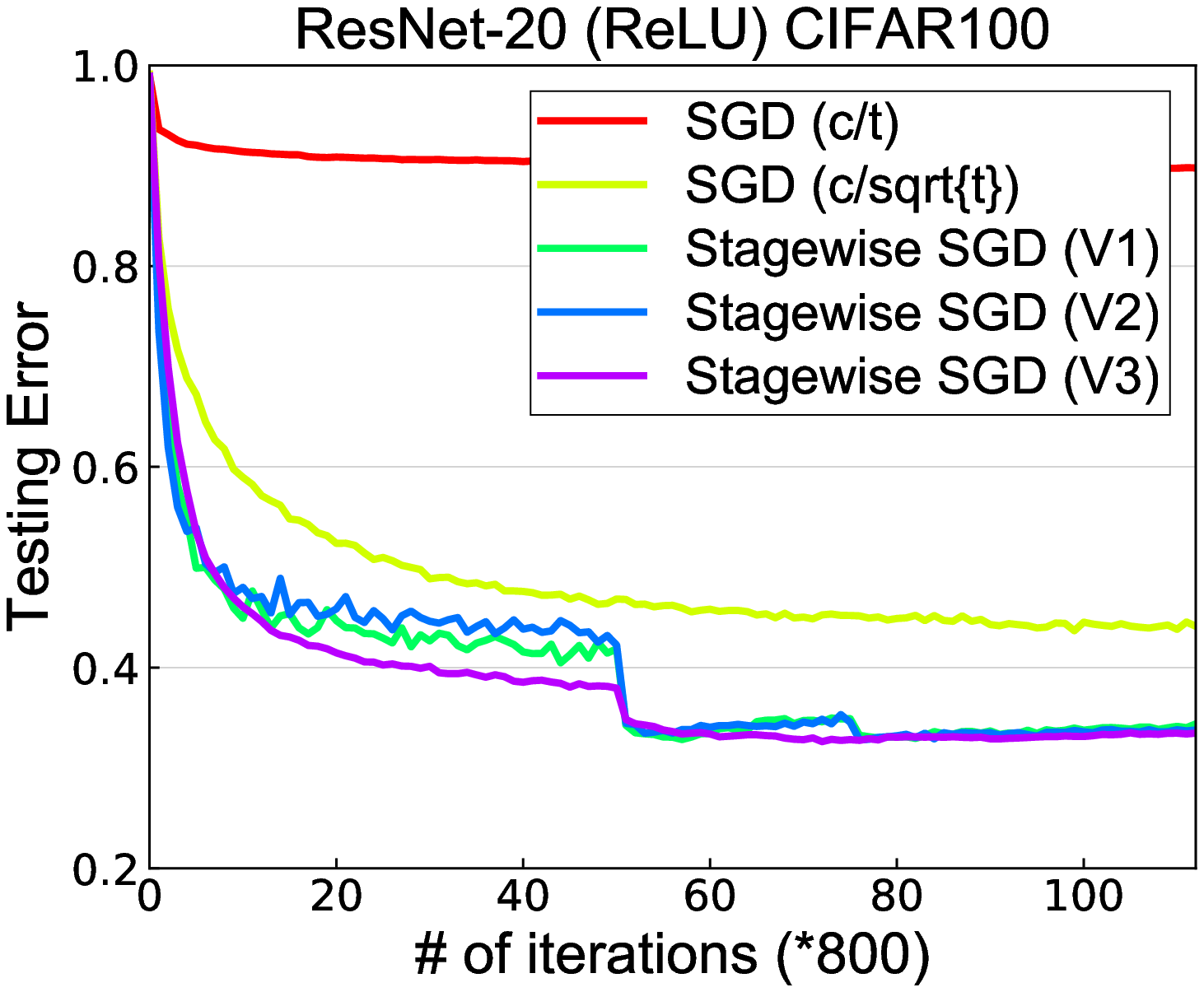}\hspace*{0.15in}
\includegraphics[width=.225\textwidth]{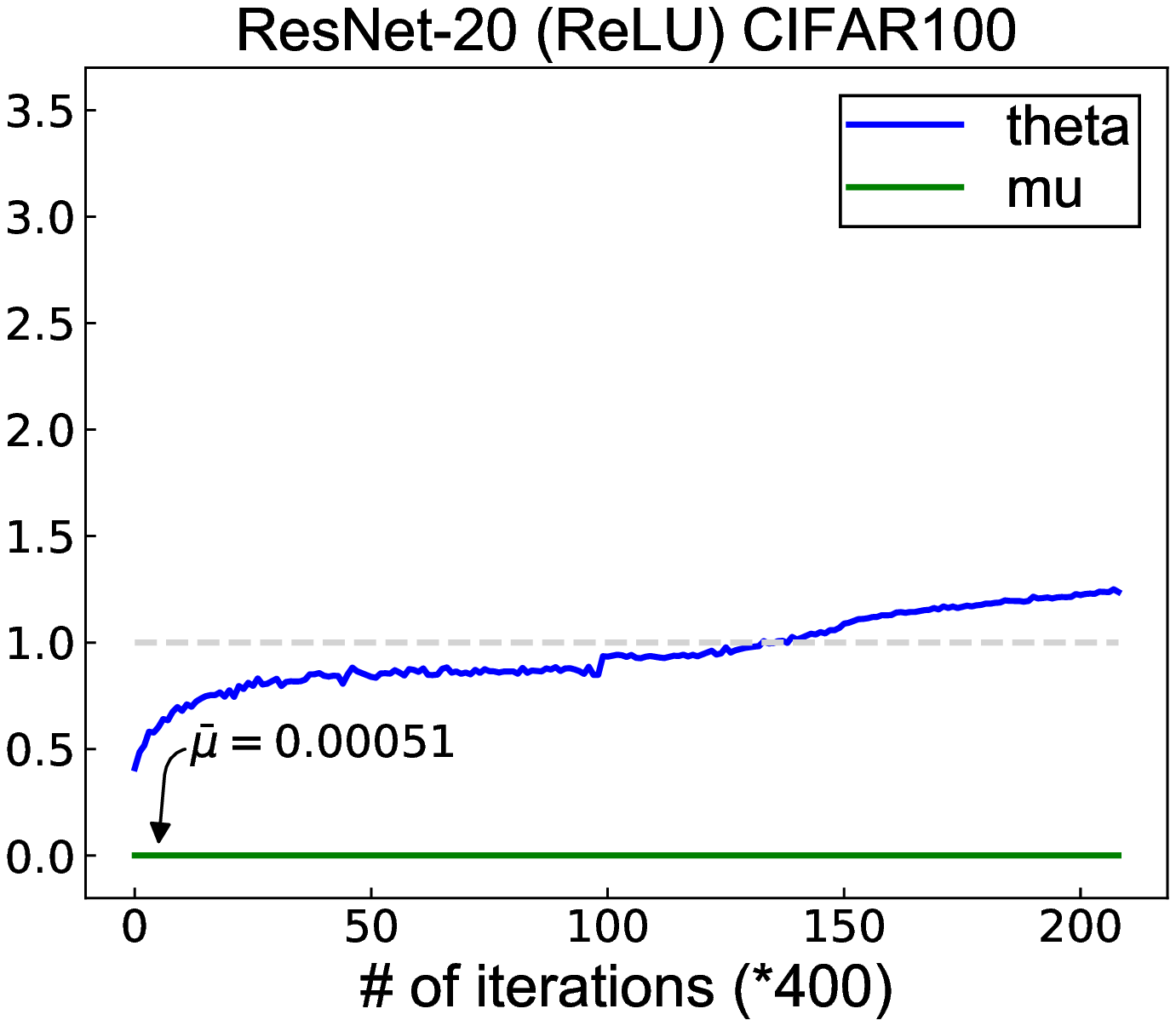}

\includegraphics[width=.225\textwidth]{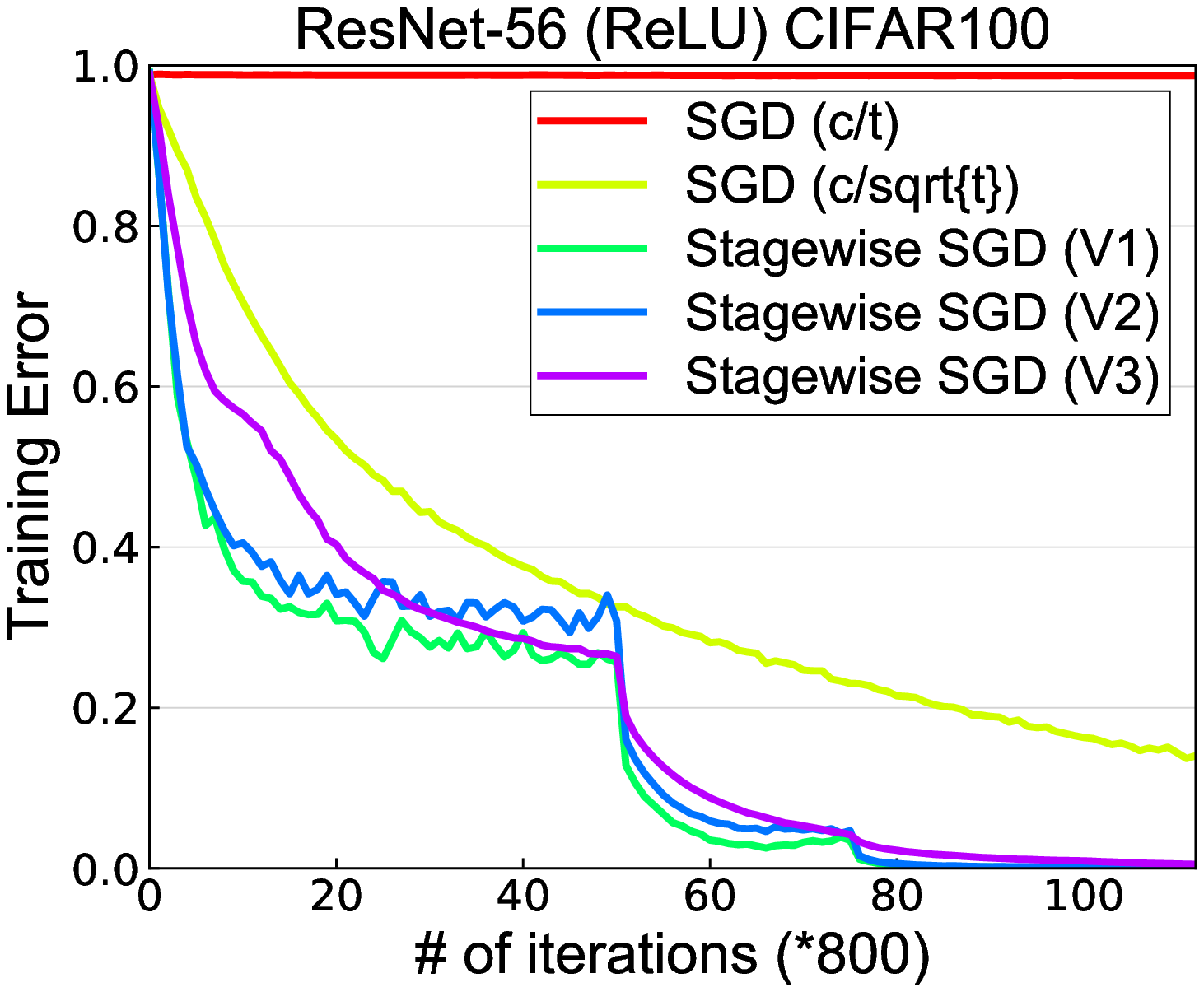}\hspace*{0.15in}
\includegraphics[width=.225\textwidth]{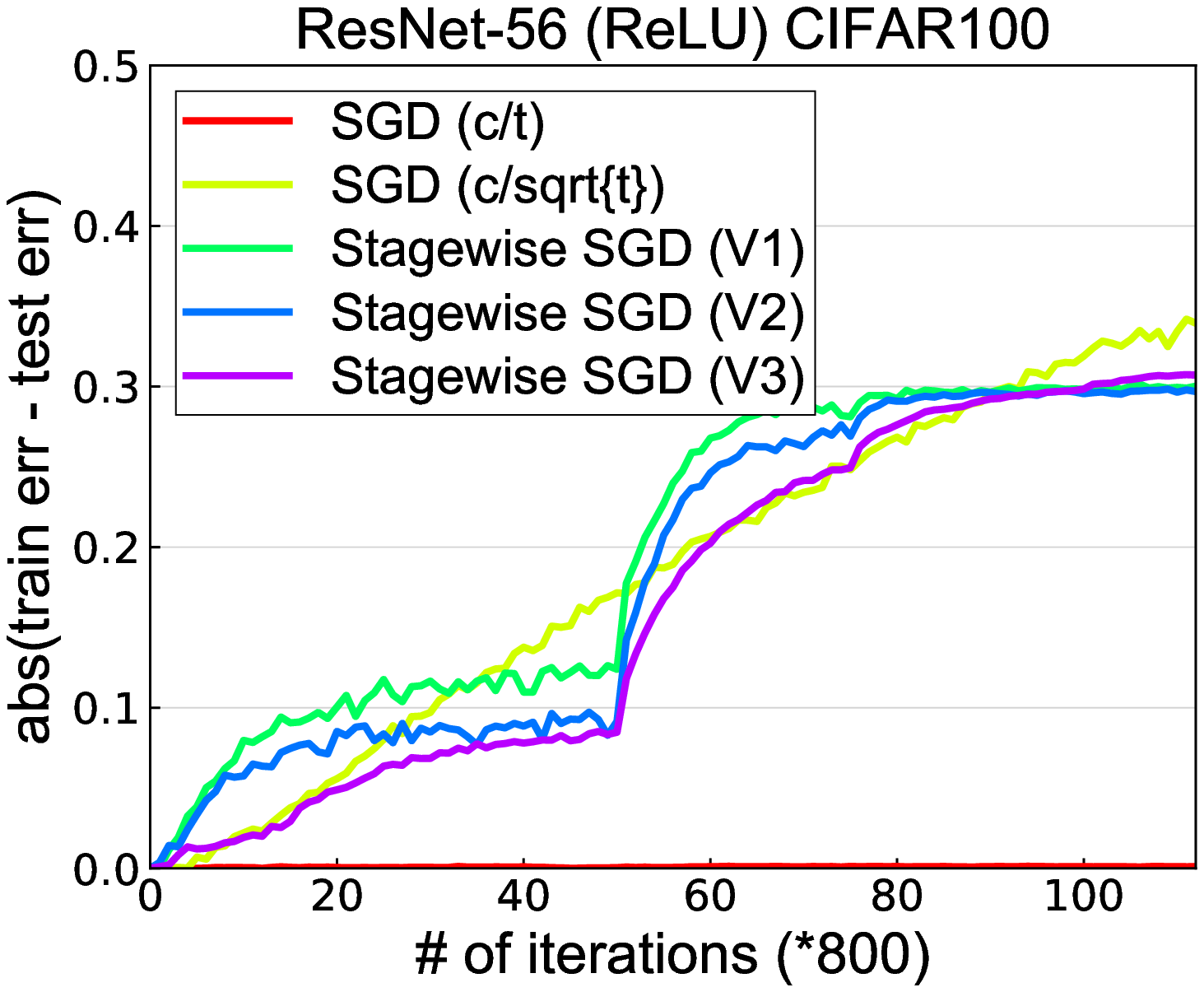}\hspace*{0.15in}
\includegraphics[width=.225\textwidth]{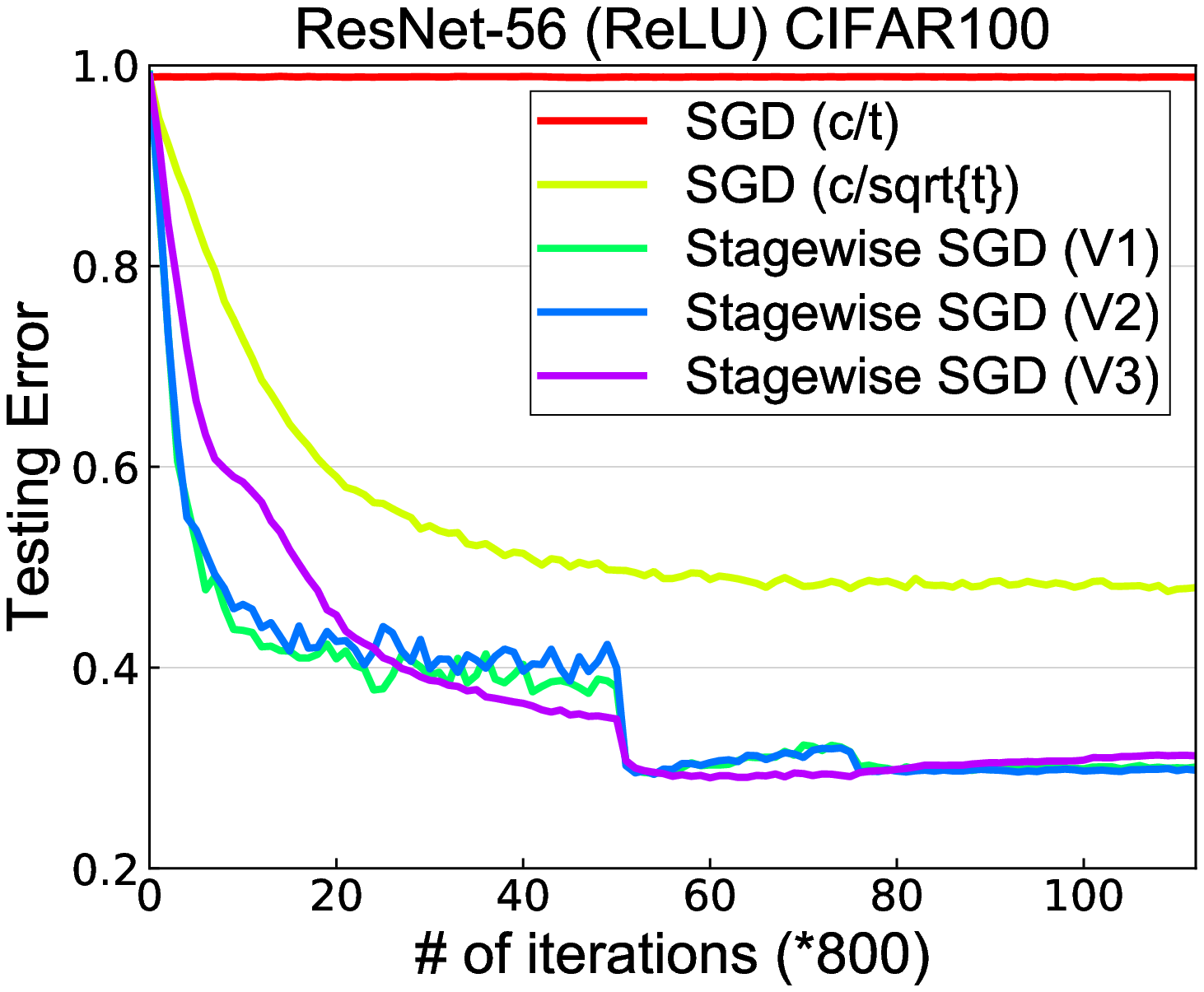}\hspace*{0.15in}
\includegraphics[width=.225\textwidth]{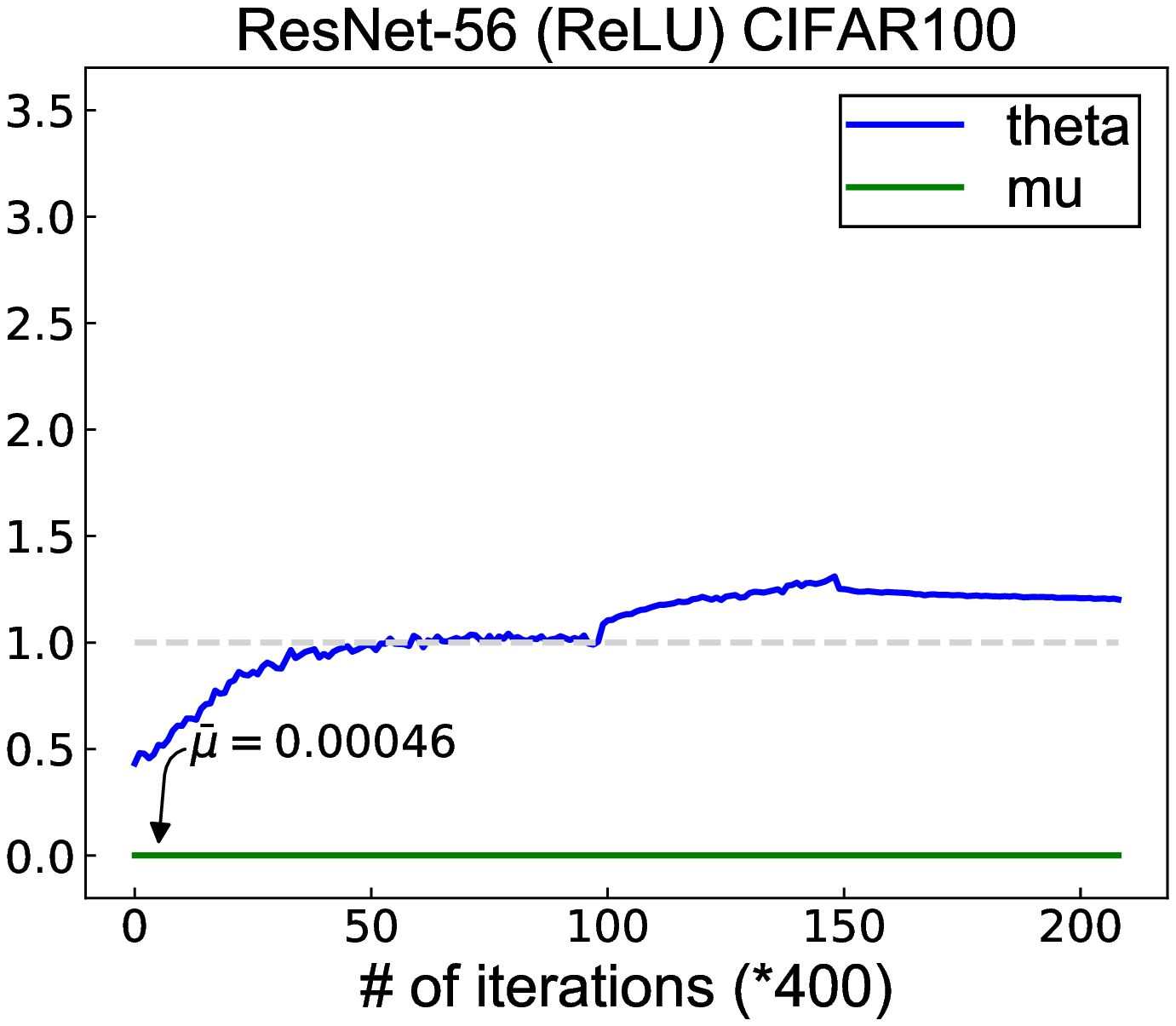}

\includegraphics[width=.225\textwidth]{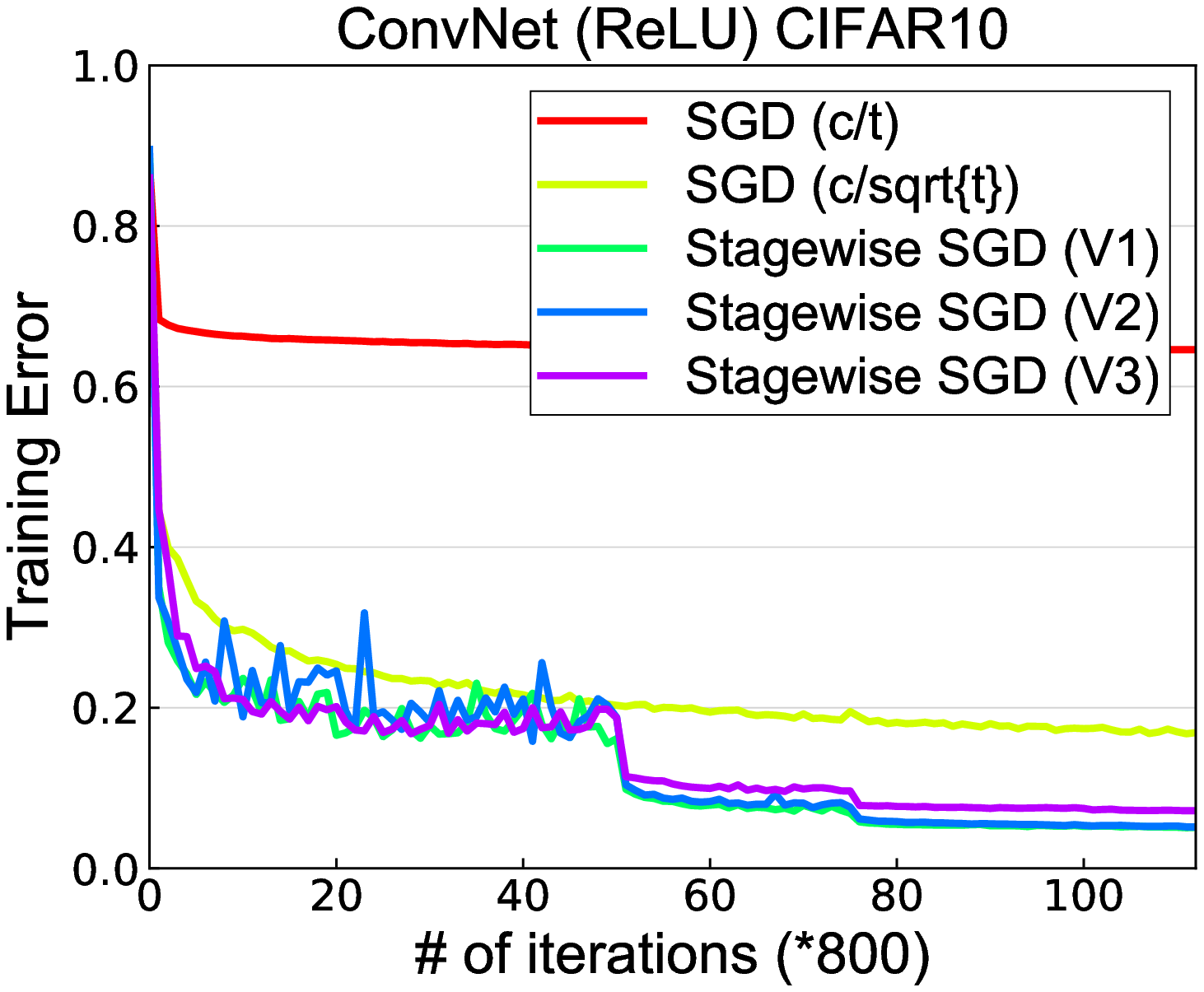}\hspace*{0.15in}
\includegraphics[width=.225\textwidth]{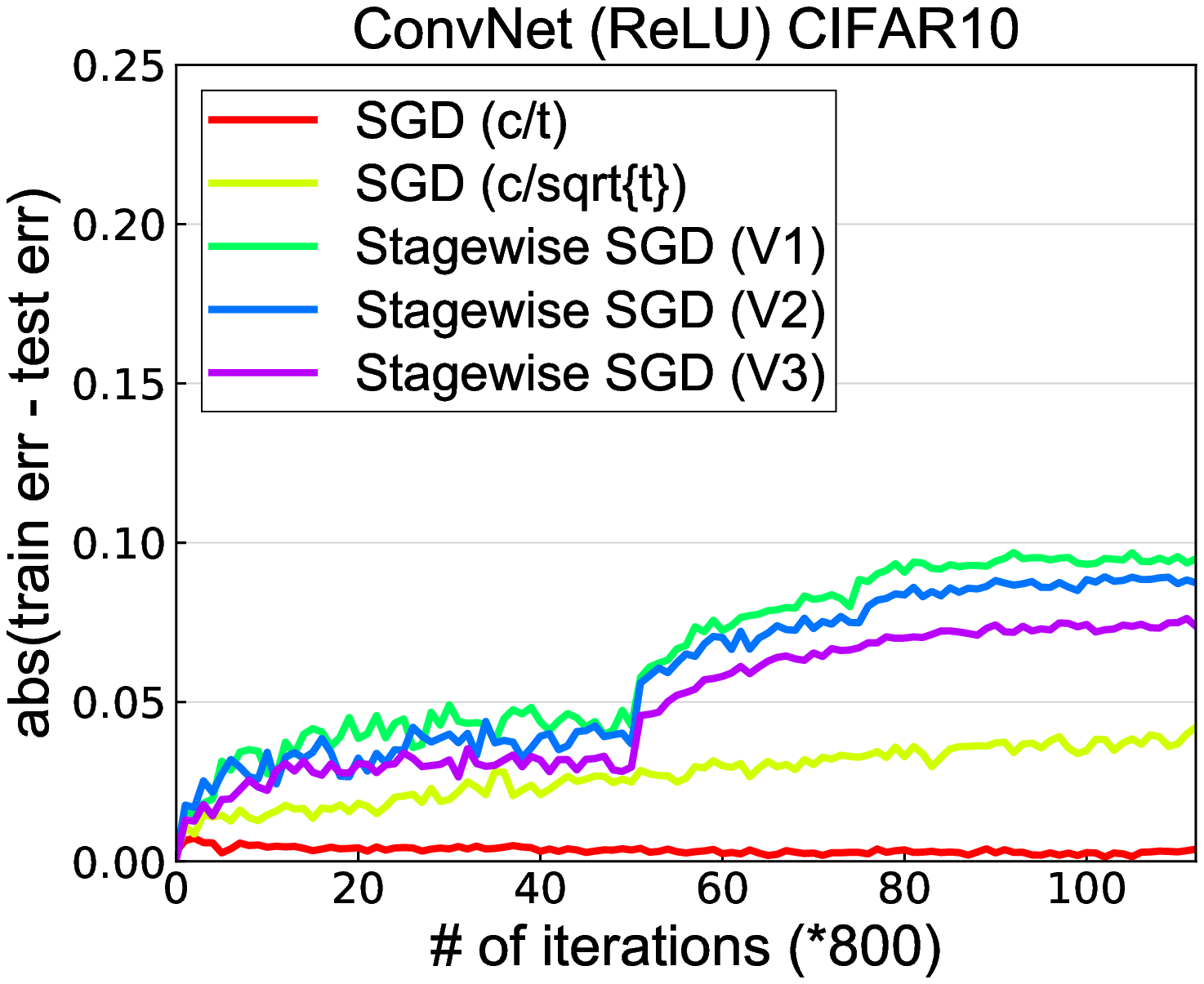}\hspace*{0.15in}
\includegraphics[width=.225\textwidth]{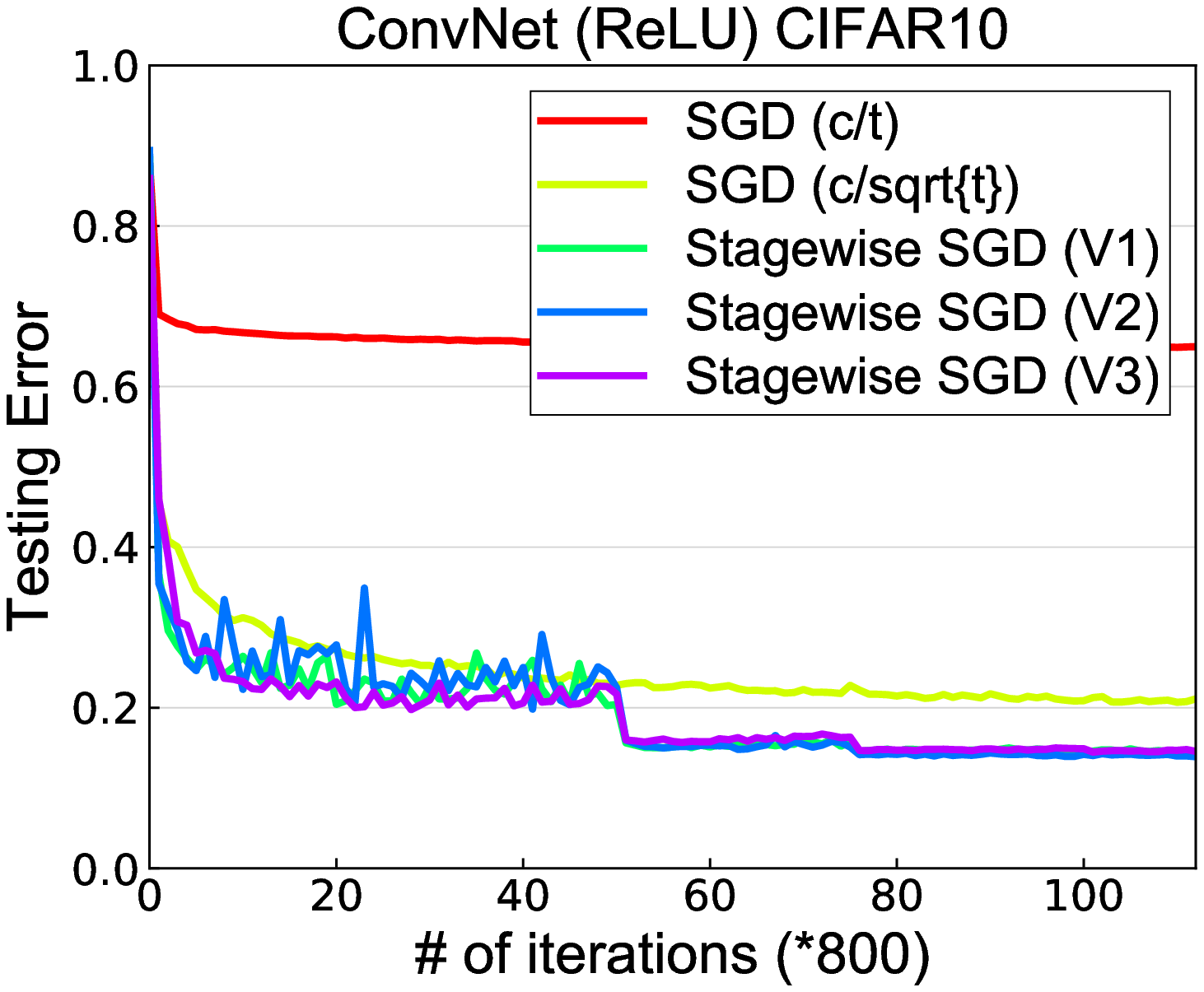}\hspace*{0.15in}
\includegraphics[width=.225\textwidth]{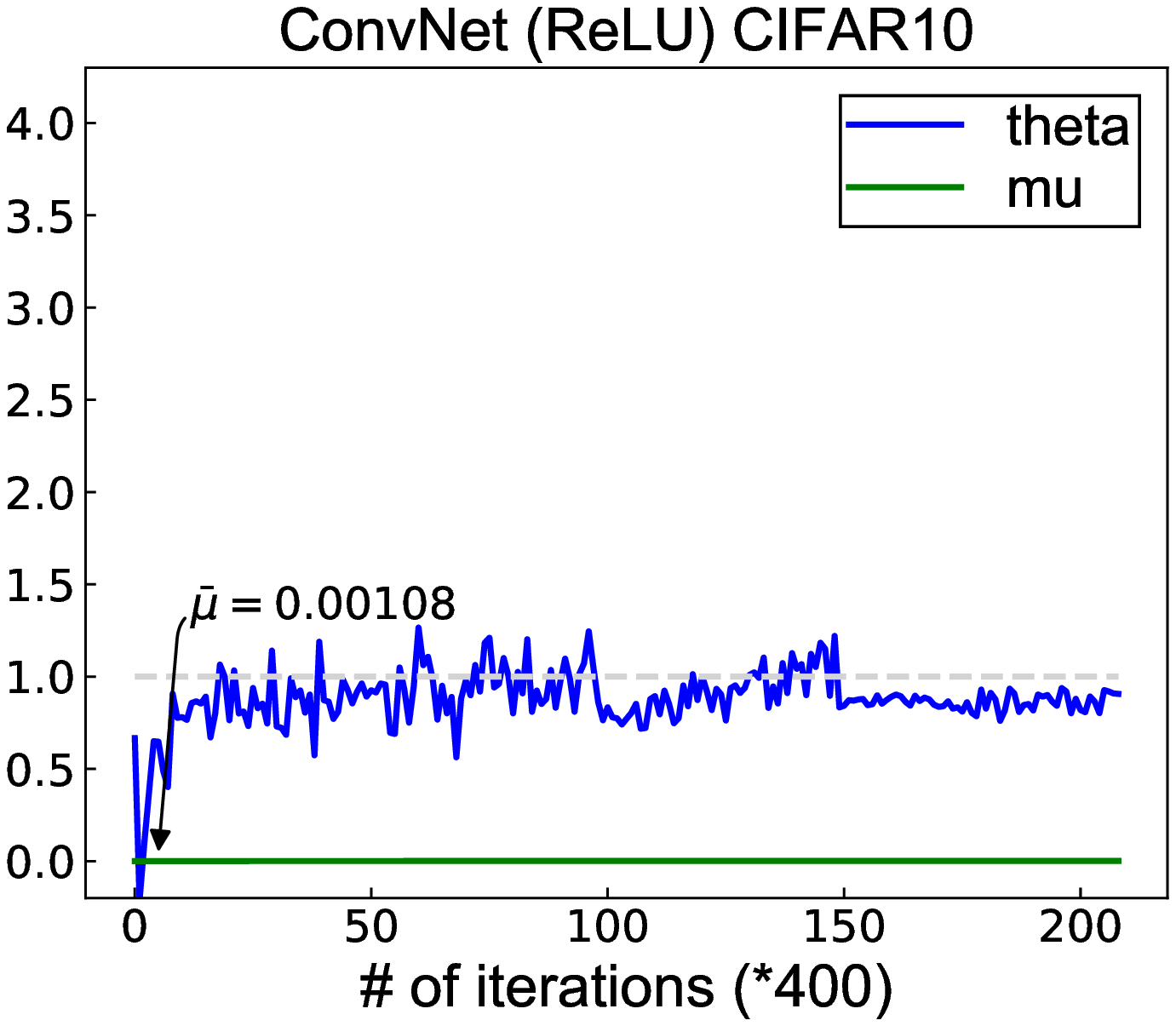}

\vspace*{-0.1in}
\caption{From left to right: training, generalization and testing error, and verifying assumptions for stagewise learning of ResNets and ConvNet using RELU with weight decay 
		 ($5 * 10^{-4} ||\w||_2^2$ regularization). 
         }
\label{fig:resnet_convnet_RELU_w_L2}
\vspace*{-0.2in}
\end{figure*}


\end{document}